%% file: main.tex
\documentclass[10pt]{article} 
\usepackage[accepted]{tmlr}

\usepackage{hyperref}
\usepackage{url}
\input{custom_commands}

\usepackage{etoolbox}
\newbool{tmlrtemp}
\setbool{tmlrtemp}{true}  
\newbool{iclrtemp}
\setbool{iclrtemp}{false}  
\newbool{spencertemp}
\setbool{spencertemp}{false}  

\title{Understanding In-Context Learning of Linear Models in Transformers Through an Adversarial Lens}

\author{
\name Usman Anwar 
\email ua237@cam.ac.uk \\
\addr University of Cambridge
\AND
Johannes von Oswald 
\email {jvoswald@google.com} \\
\addr Google, Paradigms of Intelligence Team 
\AND
Louis Kirsch 
\email {lkirsch@google.com} \\
\addr Google DeepMind 
\AND
David Krueger 
\email {kruegerd@mila.quebec} \\
\addr Mila \& Université de Montréal 
\AND
Spencer Frei
\email {sfrei@ucdavis.edu} \\
\addr UC Davis
}

\def\month{08}  
\def\year{2025} 
\def\openreview{\url{https://openreview.net/forum?id=CtMXJxO7SJ}} 

\begin{document}

\maketitle

\input{new_paper}

\bibliography{refs}
\bibliographystyle{tmlr}

\clearpage
\appendix
\section*{Appendix}

\input{appendix}


\end{document}

%% file: custom_commands.tex
\usepackage{nicefrac}
\usepackage{thmtools}
\usepackage{thm-restate}
\usepackage{amssymb,amsmath,amsthm,bbm}
\usepackage{mathtools,enumerate,enumitem,mathrsfs}
\usepackage{graphicx, subcaption}
\usepackage{wrapfig}
\usepackage{float}
\usepackage{booktabs}
\usepackage{xparse}
\usepackage{placeins}  
\usepackage{enumitem}

\def\testx{x}
\def\testy{y}
\newcommand{\testprompt}{P}

\newcommand{\normal}{\mathsf{N}}
\let\hat\widehat

\newcommand{\f}{\frac}

\newcommand{\iid}{\stackrel{\mathrm{ i.i.d.}}{\sim}}
\newcommand{\sip}[2]{\langle #1, #2 \rangle}

\newcommand{\lsa}{\mathsf{LSA}}
\newcommand{\WPV}{W^{PV}}
\newcommand{\WKQ}{W^{KQ}}

\newcommand{\params}{\theta}

\newcommand{\snorm}[1]{\| #1 \|}

\newcommand{\summm}[3]{\sum_{#1 = #2}^{#3}}
\let\hat\widehat

\newcommand{\query}{\mathsf{query}}

\newcommand\numberthis{\addtocounter{equation}{1}\tag{\theequation}}

\newcommand{\PV}{{PV}}
\newcommand{\KQ}{{KQ}}

\renewcommand{\r}{\right}
\renewcommand{\l}{\left} 
\newcommand{\adv}{\mathsf{adv}}
\newcommand{\bad}{\mathsf{bad}}

\def\R{\mathbb{R}}

\def\N{\mathbb{N}}
\def\E{\mathbb{E}}

\usepackage[colorinlistoftodos,prependcaption,textsize=tiny]{todonotes}
\usepackage{xargs}
\newcommandx{\must}[2][1=]{\todo[linecolor=red,backgroundcolor=red!25,bordercolor=red,#1]{#2}}
\newcommandx{\maybe}[2][1=]{\todo[linecolor=blue,backgroundcolor=blue!25,bordercolor=blue,#1]{#2}}

\usepackage{xcolor}
\definecolor{junglegreen}{RGB}{41, 171, 135}

\newcommand{\interleave}{\mathsf{Interleave}}
\newcommand{\concat}{\mathsf{Concat}}

%% file: new_paper.tex
\newcommand{\changed}[1]{\textcolor{black}{#1}}
\newcommand{\xattack}{\texttt{feature-attack}}
\newcommand{\yattack}{\texttt{label-attack}}
\newcommand{\zattack}{\texttt{joint-attack}}
\newcommand{\xattackkk}{{feature-attack}}
\newcommand{\yattackkk}{{label-attack}}
\newcommand{\zattackkk}{{joint-attack}}
\input{figures}

\begin{abstract}
In this work, we make two contributions towards understanding of in-context learning of linear models by transformers. First, we investigate the adversarial robustness of in-context learning in transformers to hijacking attacks — a type of adversarial attacks in which the adversary’s goal is to manipulate the prompt to force the transformer to generate a specific output. We show that both linear transformers and transformers with GPT-2 architectures are vulnerable to such hijacking attacks. However, adversarial robustness to such attacks can be significantly improved through adversarial training --- done either at the pretraining or finetuning stage --- and can generalize to stronger attack models. 
Our second main contribution is a comparative analysis of adversarial vulnerabilities across transformer models and other algorithms for learning linear models. This reveals two novel findings. First, adversarial attacks transfer poorly between larger transformer models trained from different seeds despite achieving similar in-distribution performance. This suggests that transformers of the same architecture trained according to the same recipe may implement different in-context learning algorithms for the same task. Second, we observe that attacks do not transfer well between classical learning algorithms for linear models (single-step gradient descent and ordinary least squares) and transformers. This suggests that there could be qualitative differences between the in-context learning algorithms that transformers implement and these traditional algorithms. 
\end{abstract}

\section{Introduction}

Transformers exhibit sophisticated in-context learning capabilities across a variety of settings such as language~\citep{brown2020language}, vision~\citep{kirsch2022general,bar2022visual,zhang2023makes}, tabular data~\citep{hollmann2022tabpfn,requeima2024llm,ashman2024context}, reinforcement learning and robotics~\citep{chen2021decision,raparthy2023generalization,team2023human,elawady2024relic}. The mechanisms underlying this behavior, however, remain poorly understood. While recent works have made progress by studying transformer behavior on supervised learning tasks, fundamental questions about how these models learn and implement algorithms in-context remain open~\citep[][Section 2.1]{anwar2024foundational}.

In this work, we investigate the mechanisms of in-context learning through the lens of adversarial robustness to hijacking attacks --- a threat model where an adversary manipulates examples in the in-context set to force the model to output specific target values~\citep{qiang2023hijacking,bailey2023image}. Beyond its direct practical relevance for deployed language models and emerging applications of in-context learning across various domains for sensitive applications like clinical decision-making~\citep{nori2023can} or robot control~\citep{elawady2024relic}, studying hijacking attacks provides a powerful tool for probing and understanding the algorithms that transformers learn to implement in-context.

Our investigation focuses on the setting of linear regression tasks, where we analyze two architecture classes: single-layer linear attention models and GPT-2 style transformers. Through a combination of theoretical analysis and extensive experiments, we show that both the linear transformer, and GPT-2 style transformers, do not learn adversarially robust in-context learning algorithms by default.
Specifically:

\begin{enumerate}[leftmargin=1em]
\item 
We prove that single-layer linear transformers, which prior work showed implement gradient descent on in-context data~\citep{von2022transformers}, can be hijacked through perturbation of just a single token (Theorem~\ref{thm:jailbreaking.yadv}).
We then empirically demonstrate that GPT-2 style transformers are also vulnerable to hijacking attacks, which we find through gradient-based methods. 

\item Encouragingly, we show that adversarial training can effectively improve robustness of GPT-2 style transformers, with impressive generalization: training on perturbations of $K$ examples yields robustness against manipulation of $K' > K$ tokens. This is particularly surprising given the historical difficulty of achieving robustness against adaptive adversaries in regression tasks~\citep{diakonikolas2019recent}.
\end{enumerate}

We then provide a comparative analysis of adversarial vulnerabilities of different transformers and traditional solvers.
Specifically, we analyze to what extent adversarial examples designed for a particular model or solver transfer to other models or solvers.
This comparative analysis yields multiple interesting insights:

\begin{enumerate}[leftmargin=1em]
\item We find that hijacking attacks transfer readily between smaller transformers but show poor transferability between larger transformers of identical architecture but trained from different random seeds, providing the first evidence that architecturally identical transformers trained with the same recipe may learn distinct in-context learning algorithms. 

\item We find that successful attacks against linear transformers (which perform gradient descent) fail to transfer to GPT2 architectures.
Similarly, we find that adversarial attacks transfer poorly between transformers and ordinary least squares (OLS).
This shows that the out-of-distribution behavior of transformers is mechanistically distinct from both gradient descent and OLS -- calling into question prior explanations about what algorithms these models implement to learn in-context~\citep{garg2022can,akyurek2022learning}.

\end{enumerate}

\section{Related Works}

\paragraph{In-Context Learning of Supervised Learning Tasks:} Our work is most closely related to prior works that have attempted to understand in-context learning of linear functions in transformers \citep{garg2022can,akyurek2022learning,von2022transformers,zhang2023trained,fu2023transformers,ahn2023transformers,vladymyrov2024linear}. 
\citet{von2022transformers} provided a construction of weights
of linear self-attention layers~\citep{schmidhuber1992learning,katharopoulos2020transformers,schlag2021linear} that allow the transformer to implement gradient descent over the in-context examples. They show that when optimized,
the weights of the linear self-attention layer closely match their
construction, indicating that linear transformers implicitly perform
mesa-optimization. This finding is corroborated by the works of 
\citet{zhang2023trained} and \citet{ahn2023transformers}.
A number of works have argued that when GPT2 transformers are trained on linear regression, they learn to implement ordinary least squares (OLS)~\citep{garg2022can,akyurek2022learning,fu2023transformers}.  More recently, \citet{vladymyrov2024linear} show that
linear transformers also implement other iterative algorithms
on noisy linear regression tasks with possibly different levels of noise.  
\citet{bai2024transformers} show that transformers can perform in-context algorithm selection: 
choosing different learning algorithms to solve different in-context learning tasks.
Other neural architectures such as recurrent neural networks have also been shown to implement in-context learning algorithms~\citep{hochreiter2001learning} such as bandit algorithms~\citep{wang2016learning} or gradient descent~\citep{kirsch2021meta}.

\paragraph{Hijacking Attacks:} While a considerable amount of
research has been conducted on the security aspects of LLMs,
most of the prior research has focused on jailbreaking
attacks. To the best of our knowledge, \citet{qiang2023hijacking} is the 
only prior that considers hijacking attack on LLMs or transformers
during in-context learning. They show that it is possible
to hijack LLMs to generate unwanted target outputs during
in-context learning by including adversarial tokens in the demos.
\citet{he2024data} also consider adversarial perturbations to in-context
data, however, their goal is to simply reduce the in-context learning
performance of the model in general, and not in a targeted way.
\citet{bailey2023image} demonstrate that vision-language
models can be hijacked through adversarial perturbations
to the vision modality alone. Similar to our work,
both \citet{qiang2023hijacking} and
\citet{bailey2023image} assume a white-box setup and use gradient-based methods
for finding adversarial perturbations to hijack the models.

\paragraph{Robust Supervised Learning Algorithms:} 
There are a number of frameworks for robustness in machine learning.
The framework we focus on in this work is data contamination/poisoning, where an adversary can manipulate the data in order to force predictions.
Surprisingly, designing efficient robust learning algorithms,
even for the relatively simple setting of linear regression,
has proved quite challenging, 
with significant progress only being made in the last decade~\citep{diakonikolas2023algorithmic}.
Different algorithms have been devised which work under a contamination model where only labels $y$ can be corrupted~\citep{bhatia2015robust, bhatia2017consistent, suggala2019adaptive} or when both features $x$ and labels $y$ can be corrupted~\citep{klivans2018efficient,diakonikolas2019efficient,cherapanamjeri2020optimal}. 
Note that all the aforementioned work focus on hand-designing
robust learning algorithms for each problem setting.
In contrast, we are concerned with understanding the propensity of the transformers
to learn to implement robust learning algorithms.

There are a number of other related frameworks for robustness in machine learning, e.g., robustness with respect to imperceptible (adversarial) perturbations of the input~\citep{goodfellow2015explaining,madry2017towards}.   We do not focus on these attack models in this work.

\section{Preliminaries}
\label{sec:preliminaries}
In this work, we investigate whether the learning algorithms that transformers
learn to implement in-context are adversarially robust.  We focus on the setting of in-context learning of linear models, a setting studied significantly in recent years~\citep{garg2022can,akyurek2022learning,von2022transformers,zhang2023trained,ahn2023transformers}.
We assume pre-training data that are sampled as follows. Each linear regression task is indexed by $\tau\in [B]$, with each task consisting of $N$ labeled examples $(x_{\tau,i}, y_{\tau, i})_{i=1}^{N}$, query example $x_{\tau,\query}$, parameters $w_\tau \iid \normal(0,I_d)$, features $x_{\tau,i}, x_{\tau,\query} \iid \normal(0,I_d)$ (independent of $w_\tau$), and labels $y_{\tau,i} = w_\tau^\top x_{\tau, i}$,\ $y_{\tau,\query} = w_\tau^\top x_{\tau,\query}$.   

The goal is to train a transformer on this data (by a method to be described shortly) and examine if, after pre-training, when we sample a new linear regression task (by sampling a new, independent $w\sim \normal(0,I_d)$ and features $x_i$, $i=1, \dots, M$), the transformer can formulate accurate predictions for new, independent query examples.  Note that 
the number of examples $M$ in a task at test time may differ from the number of examples $N$ per task observed during training. 

To feed data into the transformer, we need to decide on a tokenization mechanism, which requires some care since transformers map sequences of vectors of a fixed dimension into a sequence of vectors of the same length and dimension, while the features $x_i$ are $d$-dimensional and outputs $y_i$ are scalars.  That is, from a prompt of $N$ input-output pairs $(x_i, y_i)$ and a test example $x_{\query}$ for which we want to make predictions, the question is how to embed
\[ P = (x_1, y_1, \dots, x_N, y_N, x_\query),\]
into a matrix.  
We will consider two variants of tokenization: concatenation (denoted $\concat$), which concatenates $x_i$ and $y_i$ and stacks each sample into a column of an embedding matrix, and then appends $(x_\query, 0)^\top \in \R^{d+1}$ as the last column:
\begin{align*}
E(P) &= \begin{pmatrix}
    x_1 & x_2 & \cdots & x_N & x_\query \\
    y_1 & y_2 & \cdots & y_N & 0
\end{pmatrix} \in \R^{(d+1)\times (N+1)}.&(\concat) \numberthis \label{eq:embedding.matrix.concat}
\end{align*}
The notation $E(P)$ emphasizes that the embedding matrix is a function of the prompt $P$, and we shall sometimes denote this as $E$ for ease of notation.   This tokenization has been used in a number of prior works on in-context learning of function classes~\citep{von2022transformers,zhang2023trained,wu2023pretraining}.
With the $\concat$ tokenization the natural predicted value for $x_{M+1}$ appears in the $(d+1, M+1)$ entry of the transformer output.  This allows for a last-token prediction formulation of the squared-loss objective function: if $f(E;\theta)$ is a transformer, the objective function for $B$ batches of data consisting of $N+1$ samples $(x_{\tau, i}, y_{\tau,i})_{i=1}^{N}$, $(x_{\tau,\query}, y_{\tau,\query})$, each batch embedded into $E_\tau$, is
\begin{equation}\label{eqn:emp_loss_maintext}
    \widehat L(\params) = \f{1}{2B} \textstyle \sum_{\tau=1}^B \big([f(E_\tau; \theta)]_{d+1,N+1}- y_{\tau,\query} \big)^2.
\end{equation}

We will also consider an alternative tokenization method, $\interleave$, where features $x$ and $y$ are interleaved into separate tokens,
\begin{align*}
E(P) &= \begin{pmatrix}
    x_1 & 0 & x_2 & \cdots & x_N & 0 & x_\query \\
    0 & y_1 & 0 & \cdots & 0 & y_N & 0
\end{pmatrix} \in \R^{(d+1)\times (2N+1)}.&(\interleave) \numberthis \label{eq:embedding.matrix.interleave}
\end{align*}
By using causal masking, i.e. forcing the prediction for the $i$-th column of $E_\tau$ to depend only on columns $\leq i$, this tokenization allows for the formulation of a next-token prediction averaged across all $N$ pairs of examples,
\begin{equation}\label{eqn:emp_loss_standard_transformer}
    \widehat L(\params) = \f{1}{2B} \textstyle \sum_{\tau=1}^B \f{1}{N} \textstyle \sum_{i=1}^{N}\big( [f^{\mathsf{Mask}}(E_\tau; \theta)]_{d+1,2i+1} - y_{\tau, i+1}\big)^2,
\end{equation}
where we treat $y_{\tau, N+1} := y_{\tau,\query}$. 
This formulation was used in the original work by~\citet{garg2022can}

We consider in-context learning in two types of transformer models:
single-layer linear transformers, where we can theoretically analyze the behavior of the transformer, and standard GPT-2 style transformers, where we use experiments to probe their behavior.  In all experiments, we focus on the setting where $d=20$ and the number of examples per pre-training task is $N=40$. 

\subsection{Single-Layer Linear Transformer Setup}
\label{subsec:linear.transformer}
Linear transformers are a simplified transformer model in which the standard
self-attention layers are replaced by linear self-attention 
layers~\citep{katharopoulos2020transformers, von2022transformers, ahn2023transformers, zhang2023trained, vladymyrov2024linear}.
In this work, we specifically consider a single-layer linear self-attention (LSA) model,
\begin{equation} \label{eq:lsa}
    f_\lsa(E;\params) = f_\lsa(E; \WPV, \WKQ): = E + \WPV E \cdot \f{ E^\top \WKQ E}{N}. 
\end{equation}
This is a modified version of attention where we remove the softmax nonlinearity,
merge the projection and value matrices into a single matrix $\WPV \in\R^{d+1 \times d+1}$,
and merge the query and key matrices into a single matrix $\WKQ \in \R^{d+1 \times d+1}$.  For the linear transformer, we will assume the $\concat$ tokenization. 

Prior work by~\citet{zhang2023trained} developed an explicit formula for the predictions $f_\lsa$ when it is pre-trained on noiseless linear regression tasks (under the $\concat$ tokenization) by gradient flow with a particular initialization scheme.  This corresponds to gradient descent with an infinitesimal learning rate $\f{\mathrm{d}}{\mathrm dt} \params = - \nabla L(\params)$ in the infinite task limit $B\to \infty$ of the objective~\eqref{eqn:emp_loss_maintext},
\begin{equation}\label{eqn:population_loss}
    L(\params) = \lim_{B\to \infty} \widehat L(\params) = \f 12 \E_{w_\tau\sim \normal(0,I),\, x_{\tau,i},x_{\tau, \query}\iid \normal(0,I)} \l[ ([f(E_\tau; \theta)]_{d+1,N+1} - x_{\tau, \query}^\top w)^2 \r].
\end{equation}

\subsection{Standard Transformer Setup}
\label{subsec:standard.transformer.setup}
For studying the adversarial robustness of the in-context learning in
standard transformers, we use the same setup as described in \citet{garg2022can}.  Namely, we use a standard GPT2 architecture with the $\interleave$ tokenization. We provide details on the architecture and the training setup in Appendix~\ref{appx.training.details}.

\subsection{Hijacking Attacks}
We focus on a particular adversarial attack where the adversary's goal is to hijack the transformer.
Specifically, the aim of the adversary is to force the transformer to predict a specific output $y_\bad$ for $x_\query$ when given a prompt $P = (x_1, y_1, \dots, x_M, y_M, x_\query)$.  The adversary can choose one or more pairs $(x_i, y_i)$ to replace with an adversarial example $(x_\adv^{(i)}, y_\adv^{(i)})$.

We characterize hijacking attacks in this work along two axes:
$(i)$ the type of data being attacked 
$(ii)$ number of data-points or tokens being attacked.
The adversary may perturb either the $x$ feature $(x_i, y_i) \mapsto (x_\adv, y_i)$, which we call \xattack{},
or a label $y$, $(x_i, y_i)\mapsto (x_i, y_\adv)$, which we refer to as \yattack{}, or simultaneously perturb the pair $(x_i, y_i)\mapsto (x_\adv, y_\adv)$, which we refer to as \zattack{}.  We will primarily focus on \xattack{} and \yattack{} as the behavior of \zattack{} is qualitatively quite similar to \xattack{} (see Figures~\ref{fig:adv.training.y}~and~\ref{fig:adv.training.x}).
Furthermore, we allow for the adversary to perturb multiple tokens in the prompt $P$. 
A \texttt{k-token} attack means that the adversary can perturb at most
$k$ pairs $(x_i,y_i)$ in the prompt.\footnote{Note that for standard transformers with the $\interleave$ tokenization, a k-token attack corresponds to $2k$ tokens being manipulated (see~\eqref{eq:embedding.matrix.interleave}).}

We note that hijacking attacks are different from jailbreaks.
In jailbreaking, the adversary's goal is to bypass safety filters
instilled within the LLM~\citep{willison2023multi, kim2024jailbreaking}.
A jailbreak may be considered successful
if it can elicit \textit{any} unsafe response from the LLM.
While on the other hand, the goal of a hijacking attack
is to force the model to generate \textit{specific} outputs desired by
the adversary~\citep{bailey2023image}, which could potentially
be unsafe outputs, in which case the hijacking attack would
be considered a jailbreak as well.
A good analogy for jailbreaks and hijack attacks
is untargeted and targeted adversarial attacks as studied
in the context of image classification~\citep{liu2016delving}.


\section{Robustness of Transformers to Hijacking Attacks}
In this section, we first provide a theoretical result characterizing the lack of robustness of linear transformers to \textit{single-token} hijacking attacks.
We then empirically investigate the robustness of GPT-2 style transformers trained to solve linear regression in-context, and the efficacy of adversarial training towards improving the robustness of these transformers.

\subsection{Single-Layer Linear Transformers}
\label{subsec:linear.transformer.proof}
We first consider robustness of a linear transformer trained to solve linear regression in-context.
As reviewed previously in the Section~\ref{subsec:linear.transformer},
this setup has been considered in several prior 
works~\citep{von2022transformers,zhang2023trained,ahn2023transformers},
who all show that linear transformers learn to solve linear regression problems
in-context by implementing a (preconditioned) step of a gradient descent.
We build on this prior work to show that the solution learned by linear transformers is highly non-robust and that an adversary can hijack a linear transformer with very minimal perturbations to the in-context training set. Specifically, we show that throughout the training trajectory, an adversary can force the linear transformer to make any prediction it would like by simply adding a single $(x_\adv, y_\adv)$ pair to the input sequence.  We provide a constructive proof of this theorem in Appendix \ref{appx.proofs}.

\begin{restatable}{theorem}{linearhijack}
\label{thm:jailbreaking.yadv}
Let $t\geq 0$ and let $f_\lsa(\cdot\ ; \theta(t))$ be the linear transformer trained by gradient flow on the population loss using the initialization of~\citet{zhang2023trained}, and denote $\theta(\infty)$ as the infinite-time limit of gradient flow.  For any time $t\in \R_+ \cup \{\infty\}$ and prompt $\testprompt = (\testx_1, \testy_1, \dots, \testx_M, \testy_M, x_\query)$ with $x_\query \sim \normal(0,I)$, for any $y_{\bad}\in \R$, the following holds.
\begin{enumerate}
    \item If $x_\adv \sim \normal(0, I_d)$, there exists \changed{ $y_\adv = y_\adv(t, P, x_\query, y_\bad, x_\adv)\in \R$} s.t. with probability 1 over the draws of $x_\adv, x_\query$, by replacing any single example $(x_i, y_i)$, $i\leq M$, with $(x_\adv, y_\adv)$, the output on the perturbed prompt $P_\adv$ satisfies $\widehat y_\query(E(P_\adv);\theta(t)) = y_\bad.$
    \item If $y_\adv \neq 0$, there exists \changed{$x_\adv = x_\adv(t, P, x_\query, y_\bad, y_\adv)\in \R^d$ }s.t. with probability 1 over the draw of $x_\query$, by replacing any single example $(x_i, y_i)$, $i\leq M$, with $(x_\adv, y_\adv)$, the output on the perturbed prompt $P_\adv)$ satisfies $\widehat y_\query(E(P_\adv);\theta(t)) = y_\bad.$
\end{enumerate}
\end{restatable}

Theorem~\ref{thm:jailbreaking.yadv} demonstrates that throughout the training trajectory, by adding a single $(x_\adv, y_\adv)$ token an adversary can force the transformer to make any prediction the adversary would like.  Moreover, the $(x_\adv, y_\adv)$ pair can be chosen so that either $x_\adv$ is in-distribution (i.e., has the same distribution as the training data and other in-context examples) or $y_\adv$ is in-distribution.  We provide explicit formulas for each of these attacks in the Appendix (see~\eqref{eq:xadv.identity} and~\eqref{eq:yadv.identity}). 

At a high level, the non-robustness of the linear transformer is a consequence of the linear transformer implementing a learning algorithm -- one step gradient step -- that generalizes well but is inherently non-robust.
At a more mechanistic level, this non-robustness can be attributed to the learned in-context algorithm's inability to identify and remove outliers from the prompt.
This property is shared by many learning algorithms for regression problems: for instance, ordinary least squares, as an algorithm which is linear in the labels $y$, can also be shown to suffer similar problems as the linear transformer outlined in Theorem~\ref{thm:jailbreaking.yadv}.
While non-robustness of the transformers to hijacking attacks has been established in prior works~\citep{qiang2023hijacking,bailey2023image}, this is the first result that provides a mechanistic explanation as to \textit{why} transformers are vulnerable to hijacking attacks.

\changed{In Section~\ref{sec:attacks.transfer.from.theory.to.tfs}, we empirically verify that our attack can successfully force linear transformer to output the adversary's desired output $y_\bad$, and evaluate its transferability to GPT-2 style transformers as well.}

\subsection{GPT-2 Style Standard Transformers}
\label{subsec:gradient.attack}
In this section, we empirically investigate the robustness of GPT2-style standard transformers to gradient-based hijacking attacks,
and whether adversarial training (during pre-training or by fine-tuning) can improve the robustness of the transformers. 

\textit{Metrics}: To evaluate the impact of our adversarial attacks,
we use two metrics: \textit{ground truth error} (GTE), and \textit{targeted attack error} (TAE).
Ground-truth error measures mean-squared error (MSE) between the transformer's
prediction on the corrupted prompt $P_\adv$ 
and the ground-truth prediction, i.e., $y_{\text{clean}} = w^\top x_\query$.
Targeted attack error similarly measures mean-squared error (MSE) between the transformer's prediction on the corrupted prompt and $y_\bad$.
Let $\hat y$ be the transformer's prediction corresponding to $x_\query$, then:
\begin{align}
    \text{Ground Truth Error}  &= 
    \frac{1}{B}  \sum_{i=1}^B \left (\hat y_i - y_\text{clean}\right)^2,
    \quad 
    \text{Targeted Attack Error} = 
    \frac{1}{B} \sum_{i=1}^B \left (\hat y_i - y_\bad\right)^2.
\end{align}

\textit{Attack Methodology}: 
We choose $y_\bad$ according to the following formula, 
\begin{align}
    y_\bad = (1-\alpha) w_\tau^\top x_\query + \alpha w_{\perp}^\top x_{\query}
    \label{eq:y_bad.alpha}
\end{align}
Here $w_\tau$ is the underlying weight vector corresponding to the clean prompt $P$ and $w_{\perp} \perp w$, and $\alpha \in [0,1]$ is a parameter. \changed{We construct $w_{\perp}$ by first sampling a $w' \iid \normal(0,I_d)$ and applying Gram-Schmidt orthogonalization process to $w'$ to make it orthogonal to $w_\tau$. The resultant vector is then normalized to have L2 norm of $d$.} When $\alpha \to 0$, the target label $y_\bad$ is more similar to the in-distribution ground truth, while $\alpha \to 1$ represents a label which is more out-of-distribution.

To carry out a gradient-based attack, we randomly select a $k_\text{test}$ number of input examples---where $k_\text{test}$ is specified beforehand---and initialize their values to zero.
We then optimize these $k_\text{test}$ tokens by minimizing the targeted attack error, for target $y_\bad$ from~\eqref{eq:y_bad.alpha} for different values of $\alpha \in (0,1]$. 
Both during training and testing, we set the sequence length of
the transformer to be $40$.
See Appendix~\ref{appx.adversarial.attack.details} for further details on attack procedure.

Based on whether we are perturbing `features' or `labels' within the input prompt, we denote our attacks as either \xattack{} or \yattack{}.
When we simultaneously attack both the features and label, we call that \zattack.
Specifically, given a prompt $\testprompt = (\testx_1, \testy_1, \dots, \testx_M, \testy_M, \testx_\query)$: for \xattack{}, we replace $(x_1, y_1)$ with $(x_\adv, y_1)$, for \yattack{}, we replace $(x_1, y_1)$ with $(x_1, y_\adv)$, and for \zattack{}, we replace $(x_1, y_1)$ with $(x_\adv, y_\adv)$.

\textit{Results}: We find that gradient based adversarial attacks are generally successful.
Our main results appear in Figure~\ref{fig:adv.training.y}
under the label $k_\text{train}=0$, which show the targeted attack error for an $8$ layer transformer averaged over $1000$ prompts and 3 random initialization seeds when $\alpha=1$ from~\eqref{eq:y_bad.alpha}.  
We note that for \xattack{}, an adversary can achieve a very small
targeted attack error
with perturbing just a single token.
However, for \yattack{}, achieving low targeted attack
generally requires perturbing multiple y-tokens.
Note that this is in contrast with linear transformers,
for which we have previously shown that hijacking is possible with
perturbing just a single y-token.
Finally, \zattack{} behave in a qualitatively similar way to \xattack{} but are slightly more effective (this is most notable for $k_\text{test}=1$).  
Additional experiments investigating different choices of $\alpha$ appear in Appendix~\ref{appx.sec.adversarial.training}. 

\ifbool{iclrtemp}{\FigGradAttackYTraining}{}
\ifbool{iclrtemp}{\FigGradAttackXTraining}{}


\vspace{2em}

\subsection{Adversarial Training}
\label{sec.adversarial.training}
A common tactic to promote adversarial robustness of neural networks is to subject them to adversarial training --- i.e., train them on adversarially perturbed samples \citep{madry2017towards}.
In our setup, we create adversarially perturbed samples by carrying out the gradient-based attack outlined in Section~\ref{subsec:gradient.attack}
on the model undergoing training.  Namely, for the model $f_\theta^t$ at time $t$, for each standard prompt $P$, we take a target adversarial label $y_\bad$ and use the gradient-based attacks from Section~\ref{subsec:gradient.attack} to construct an adversarial prompt $P_\adv$.  

\ifbool{tmlrtemp}{\FigGradAttackYTraining}{}
\ifbool{tmlrtemp}{\FigGradAttackXTraining}{}

We consider two types of setups for adversarial training. In the first setup, we train the transformer model from scratch on
adversarially perturbed prompts. We call this \textit{adversarial pretraining}.
In the second setup, we first train the transformer model on standard (non-adversarial) prompts $P$ for $T_1$ number of steps; and then further train the transformer model for $T_2$ number of steps on adversarial prompts.
We call this setup \textit{adversarial fine-tuning}.
In our experiments, unless otherwise specified, we perform adversarial
pretraining for $5\cdot 10^5$ steps. For adversarial fine-tuning,
we perform $5\cdot 10^5$ steps of standard training and then
$10^5$ steps of adversarial training, i.e., $T_1=5\cdot 10^5$ and $T_2=10^5$.

\FigAccRobustnessTradeoff

The adversarial target value $y_\bad$ is constructed by sampling a weight vector $w\sim \normal(0,I)$ independent of the parameters $w_\tau$ which determine the labels for the task $\tau$ and setting $y_\bad=w^\top x_\query$ \changed{(i.e., $\alpha=1$)}.  
To keep training efficient, for each task we perform $5$ gradient steps to construct the adversarial prompt.
We denote the number of tokens attacked during training with $k_\text{train}$, and experiment with two values of $k_\text{train}=1$ and $k_\text{train}=3$. Unless stated otherwise, we use an $8$ layer transformer.

\textbf{Adversarial training improves robustness---even with only fine-tuning.}  In Figures~\ref{fig:adv.training.y}~and~\ref{fig:adv.training.x}, we show the robustness of transformers under $k$-token hijacking attacks when they are adversarially trained on either \xattack{} or \yattack.  We see that adversarial training against attacks of a fixed type (e.g. \xattack{} or \yattack) improves robustness to hijacking attacks of the same type, with robustness under \xattackkk{} seeing a particular improvement.  Notably, there is little difference between adversarial fine-tuning and pretraining, showing little benefit from the increased compute requirement of adversarial pretraining.

\textbf{Adversarial training against one attack model moderately improves robustness against another.}  Following adversarial training against \yattack, we see modest improvement in the robustness against \xattack{} and \zattack, while adversarial training against \xattack{} results in significant improvement against \zattack{} (as expected, given that 20 of the 21 dimensions \zattack{} uses is shared by \xattack) and modest improvement against \yattack.   We show in Fig.~\ref{appx.fig:adv.training.z.alpha.1} the results for adversarial training against \zattack.  

\ifbool{tmlrtemp}{\FigEffectofSeqLen}{}
 
\textbf{Adversarial training against $k$-token attacks can lead to robustness against $k'>k$ token attacks.}  In both Fig.~\ref{fig:adv.training.y} and~\ref{fig:adv.training.x} (as well as Fig.~\ref{appx.fig:adv.training.z.alpha.1}) we see that training against $k=3$ token attacks can lead to significant robustness against $k=7$ token attacks, especially in the case of models trained against \xattack{} and \zattack{}.  

\textbf{Minimal accuracy vs. robustness tradeoff.}  In many supervised learning problems, there is an inherent tradeoff between the robustness of a model and its (non-robust) accuracy~\citep{zhang2019theoreticallyprincipledtradeoffrobustness}.  In Fig.~\ref{fig.clean.performance} we compare the performance of models which undergo adversarial training vs. those which do not, and we find that while there is a moderate tradeoff when undergoing \yattack{} training, there is little tradeoff when undergoing \xattack{} and \zattack{} training.  

On the whole, given the challenging nature of robust regression problem \citep{diakonikolas2019recent}, the success of adversarial training is both surprising and remarkable, and hints at the ability of transformers to solve highly challenging non-convex optimization problems in context.

\subsection{Effect of Scaling Depth and Context Length}
\label{sec.effect.of.scaling.and.seqlen}
\ifbool{iclrtemp}{\FigEffectofSeqLen}{}
\ifbool{spencertemp}{\FigEffectofSeqLen}{}
Some recent works indicate that larger neural networks are naturally
more robust to adversarial attacks 
\citep{bartoldson2024adversarial,howe2024exploring}.
However, we did not observe any consistent improvement in
adversarial robustness of in-context learning in transformers in
our setup with scaling of the number of layers, as can be seen in Figure~\ref{appx.fig.effect.of.scale} in the appendix. 

We also studied the effect of sequence length, which scales the size of the in-
context training set.
We show in Figure~\ref{fig:results.seqlen} that for a fixed number of tokens attacked, longer context lengths can improve the robustness to hijacking attacks.
However, for a fixed \textit{proportion} of the context length attacked, the robustness to hijacking attacks is approximately the same across context lengths. 
\changed{Note that in this setup the transformers are evaluated at the same sequence length with which they were trained with (i.e., $M=N$).}
We explore this in more detail in the appendix (see Appendix~\ref{appx.sec.effect.of.seqlen}).

\section{Transferability of Hijacking Attacks}
Intuitively, if two predictive models are similar, they should be vulnerable to similar targeted adversarial attacks.
Thus, evaluating whether hijacking attacks transfer between two predictive models learned from distinct learning algorithms can help us evaluate to what extent the two models (and their underlying learning algorithms) could be considered similar.
Note that we are specifically interested in \textit{targeted} transfer;
i.e., we want adversarial samples generated by attacking a source model to predict $y_\bad$ to also cause a victim model to predict $y_\bad$.
Transfer of targeted attacks is generally less common than the transfer of untargeted attacks~\citep{liu2016delving}.

In this section, we first analyze the transferability of adversarial attacks between different transformers, and surprisingly find that adversarial attacks can fail to transfer between even transformers of the same architecture and trained with the same recipe---their sole difference being the random seeds used for training.  
Next, we investigate the transfer from linear transformers to GPT-2 style transformers,
and find that these attacks do not transfer.
Lastly, we investigate the transfer between ordinary least squares (OLS) and GPT-2 style transformers, and find that the hijacking attacks transfer poorly between them.

\begin{figure}[t]
    \centering
    \begin{subfigure}[b]{0.4\textwidth}
        \raggedleft
        \begin{subfigure}[b]{0.3\textwidth}
            \centering
            \includegraphics[width=\textwidth]{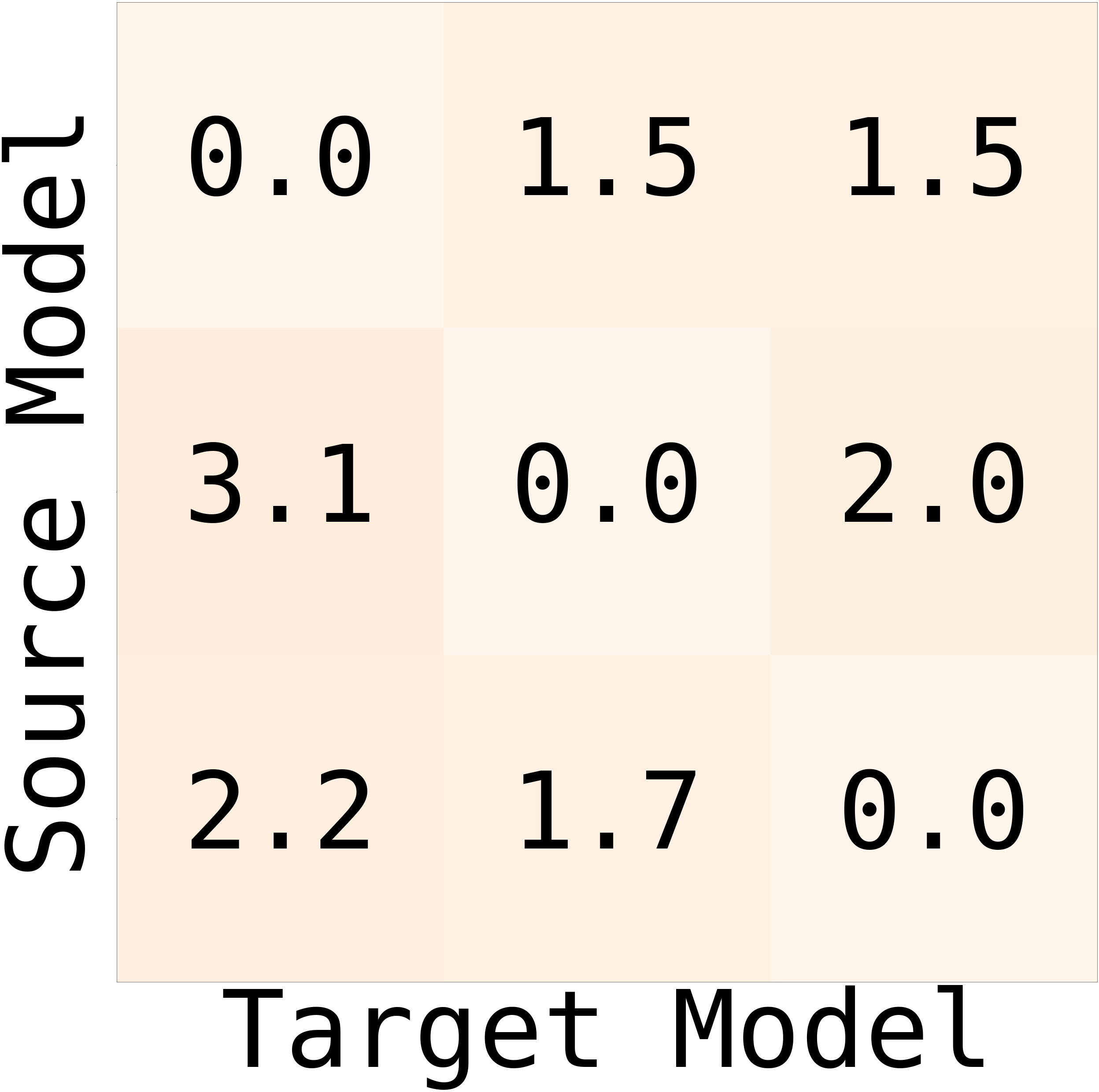}
            \caption{3 layers.}
        \end{subfigure}
        \hspace{1em}
        \begin{subfigure}[b]{0.3\textwidth}
            \centering
            \includegraphics[width=\textwidth]{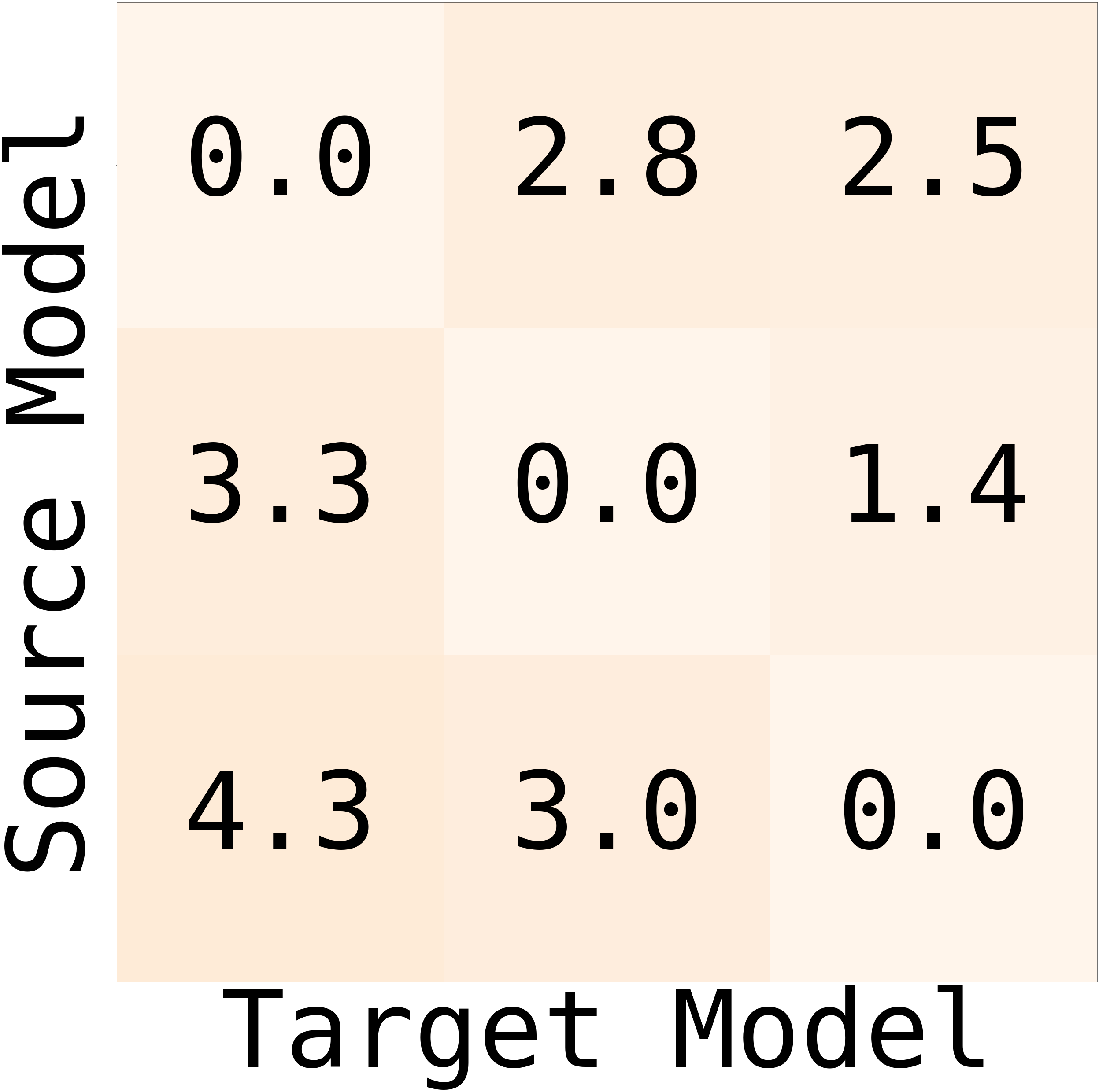}
            \caption{6 layers.}
        \end{subfigure}
        
        \vspace{0.3cm}
        
        \begin{subfigure}[b]{0.3\textwidth}
            \centering
            \includegraphics[width=\textwidth]{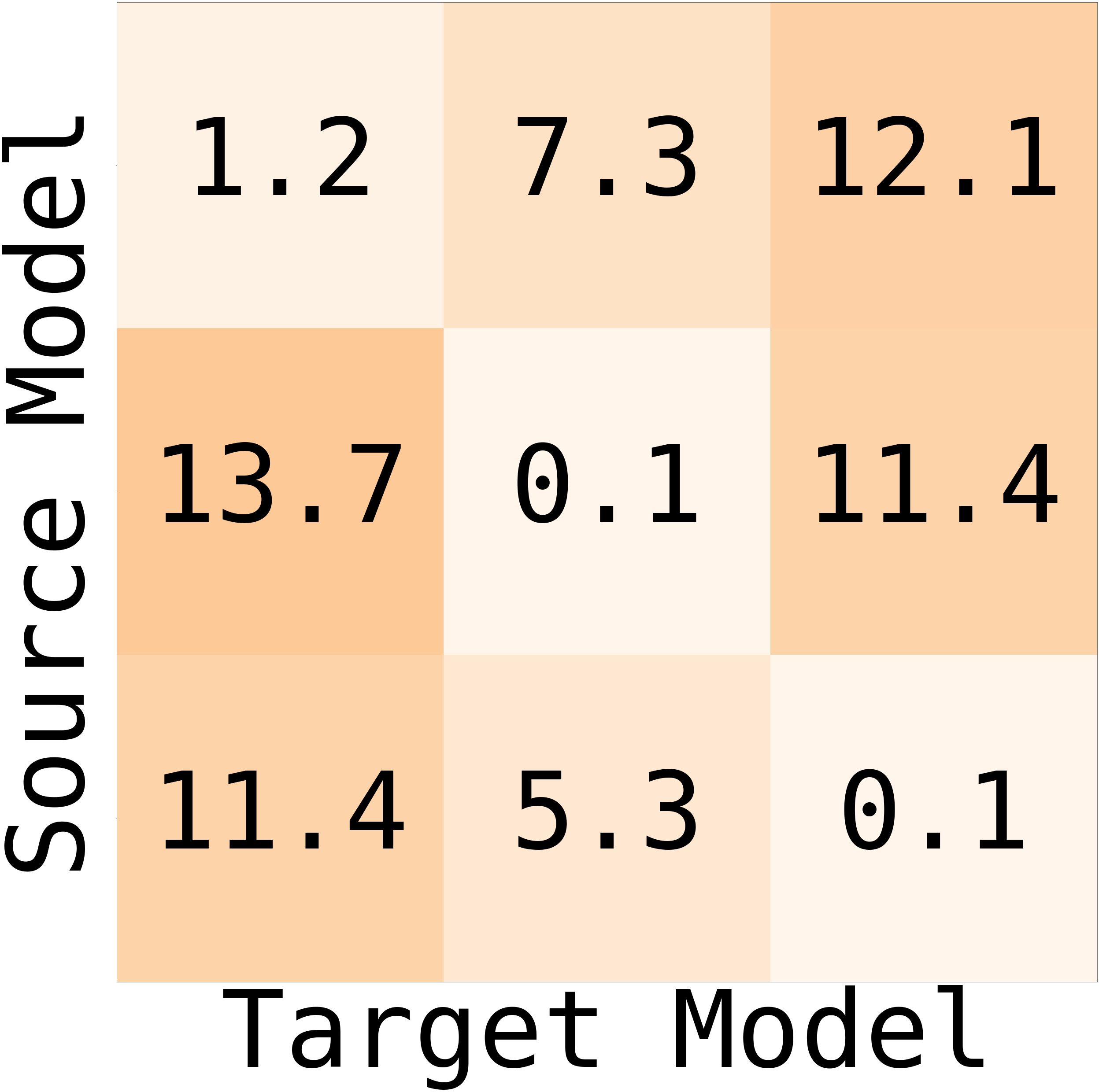}
            \caption{12 layers.}
        \end{subfigure}
       \hspace{1em} 
        \begin{subfigure}[b]{0.3\textwidth}
            \centering
            \includegraphics[width=\textwidth]{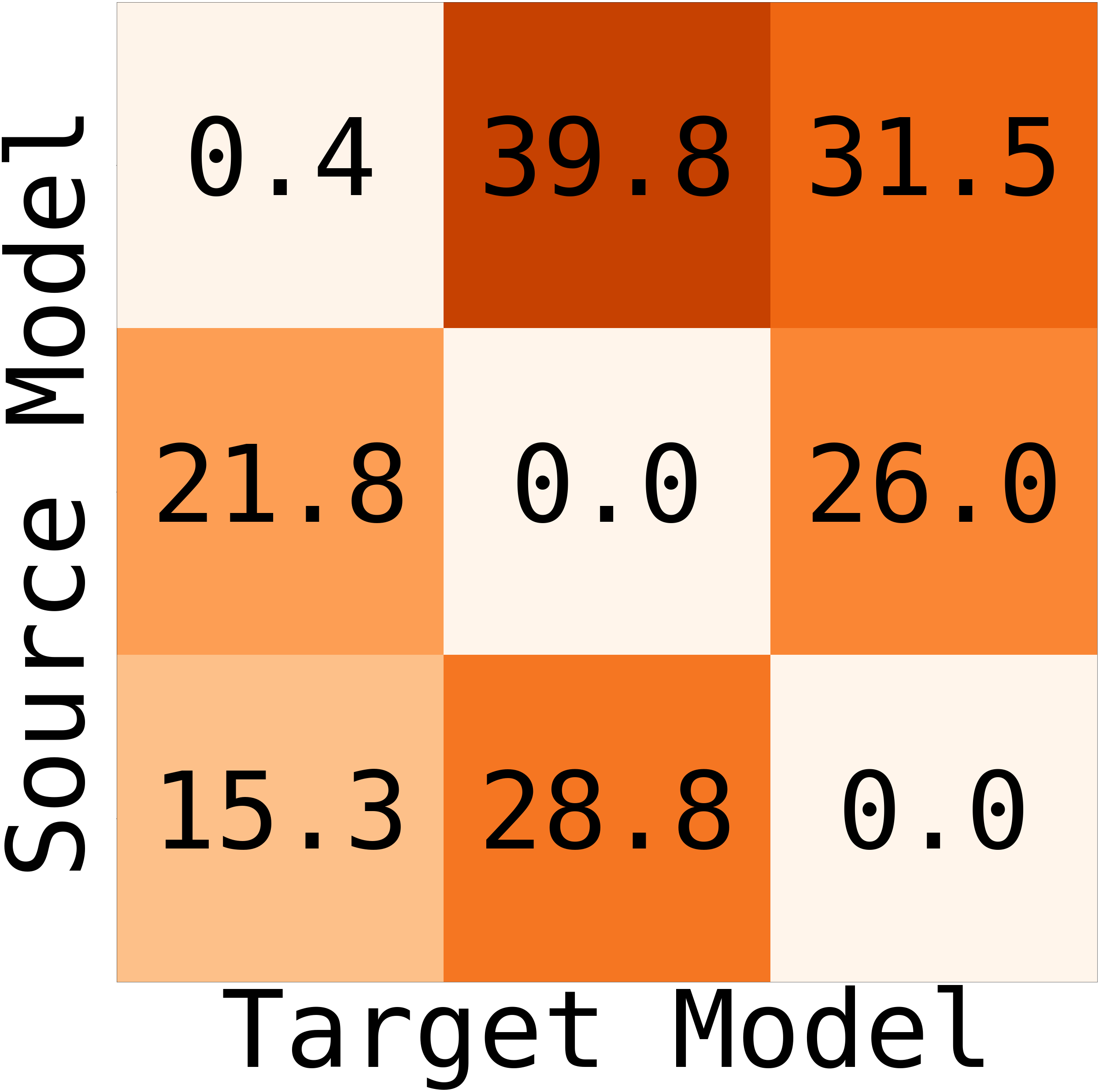}
            \caption{16 layers.}
        \end{subfigure}
    \end{subfigure}
    \begin{subfigure}[b]{0.58\textwidth}
        \hspace{3.5em}
        \includegraphics[width=0.48\textwidth]{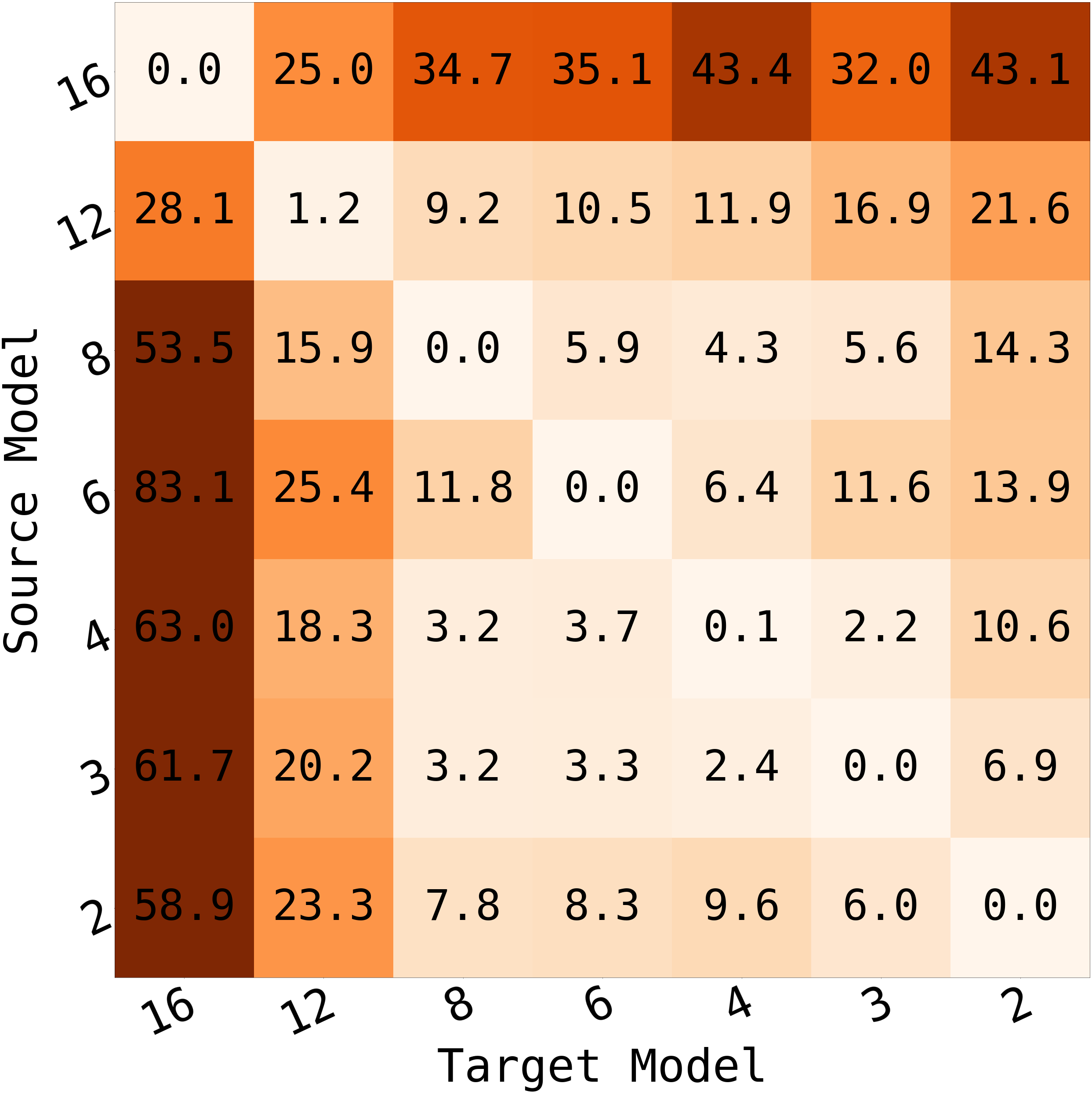}
        \caption{Source: $L$ layers, Target: $L'$ layers.}
    \end{subfigure}
    \caption{Targeted attack error when transferring an attack from a source model to a target models.  Attacks transfer better between smaller-scale models, but not to larger-scale models (right)---even across random seeds (left). Adversarial samples were generated using \xattack{} with $k=3$.}
    \label{fig.transfer.xadv.tf.to.tf.same.layers}
\end{figure}

\subsection{Transfer Between Different Transformers}
\label{sec.transfer.across.tf}
    Lastly, we look at the transferability of the adversarial attacks between different GPT-2 style transformers trained to solve linear regression in-context. 

For brevity, we restrict our focus to \xattack{} here; transferability of \yattack{} follows a similar pattern and is discussed in Appendix~\ref{appx.section.transfer}.
We first consider \textit{within-class transfer}, i.e., transfer from one transformer to another transformer with identical architecture but trained from a different random initialization.
We present these results in Fig.~\ref{fig.transfer.xadv.tf.to.tf.same.layers}(a-d),
averaged over $1000$ prompts.
We generate adversarial samples according to the methodology described in Sec.~\ref{subsec:gradient.attack}.
While sampling $y_\bad$, we use $\alpha=1$.
Diagonal entries correspond to self-transfer, i.e., TAE of the generated adversarial samples for the source model.
These results show that for transformers with smaller capacities ($3$ and $6$ layers) attacks transfer quite well, but transfers become progressively worse as the models become larger.  This suggests that higher-capacity transformers could implement different in-context learning algorithms when trained from different seeds.

We next consider \textit{across-class transfer}, i.e., transfer between transformers with different number of layers.  
We present these results in
Fig.~\ref{fig.transfer.xadv.tf.to.tf.same.layers}(e),
for one randomly chosen source and target model with given number of layers.
The results are also averaged over $1000$ prompts.
These results show a similar trend as within-class transfer: attacks from small-to-medium capacity models transfer better to other small-to-medium capacity models, while larger capacity models transfer poorly to all other capacity models.  

\ifbool{iclrtemp}{\vspace{-0.5em}}{}
\ifbool{tmlrtemp}{\vspace{-0.5em}}{}

\ifbool{iclrtemp}{\vspace{-0.5em}}{}
\ifbool{tmlrtemp}{\vspace{-0.5em}}{}

\ifbool{iclrtemp}{\FigTheoryAttack}{}
\ifbool{tmlrtemp}{\FigTheoryAttack}{}
\ifbool{spencertemp}{\FigGradAttackYTraining}{}
\ifbool{spencertemp}{\FigGradAttackXTraining}{}
\subsection{Transfer Between Linear Transformers and Standard Transformers}
\label{sec:attacks.transfer.from.theory.to.tfs}

We implement separate \xattack{} and \yattack{} based on formulas given in equations~\eqref{eq:xadv.identity}~and~~\eqref{eq:yadv.identity}.
Recall that these formulas presume that the network being attacked is a linear transformer which implements one step of gradient descent when trained to solve linear regression in context. 
As done previously, $y_\bad$ is chosen according to~\eqref{eq:y_bad.alpha}.

%
In Figure~\ref{fig:linear.transformer.attack.transfer}
we show the robustness of SGD-trained single-layer linear transformers and standard transformers of different depths as a function of $\alpha$ (cf.~\eqref{eq:y_bad.alpha}).
These results are averaged over $1000$ different samples and $3$ random initialization seeds for every model type (see Appendix~\ref{appx.training.details} for further details on training). 
We find that the gradient flow-derived attacks transfer to the SGD-trained single-layer linear transformers, as the targeted attack error is near zero for all values of $\alpha$.
Moreover, while standard (GPT2) transformers trained to solve linear regression in-context incur significant ground-truth error when the prompts are perturbed using the attacks from Theorem~\ref{thm:jailbreaking.yadv}, these attacks are not successful as \textit{targeted} attacks, since the targeted error is large. 
This behavior persists across GPT2 architectures of different depths.
As linear transformers implement one step of gradient descent, poor transferability of these attacks suggests that when trained on linear regression tasks, GPT2 architectures do not implement one step of gradient descent, as has been suggested in some prior works~\citep{von2022transformers}. 

\subsection{Transfer Between Transformers and Least Squares Solver}

\ifbool{iclrtemp}{\FigOLSTFTransferWrapped}{}
\ifbool{tmlrtemp}{\FigOLSTFTransfer}{}
\ifbool{spencertemp}{\FigOLSTFTransfer}{}

It has been argued that transformers trained to solve linear regression in-context implement ordinary least squares (OLS) \citep{garg2022can, akyurek2022learning}.
If so, adversarial (hijacking) attacks ought to transfer between transformers and OLS.
In Figure~\ref{fig:results.transfer.ols.tf}, we show mean squared error (MSE) between predictions of OLS and transformers on adversarial samples created by performing \xattackkk{} on OLS and transformers respectively.
It can be clearly observed that as the targeted prediction $y_\bad$ becomes more out-of-distribution ($\alpha\to 1$), MSE between predictions made by OLS and transformers also increases.
Furthermore, MSE is considerably larger when adversarial samples are created by attacking transformers.
This collectively indicates that the alignment between OLS and transformers is weaker out-of-distribution and that the transformers likely have additional adversarial vulnerabilities relative to OLS.
We provide additional results and expanded discussion in Appendix~\ref{appx.sec.transfer.ols.transfer}.

\section{Discussion}
This study makes several contributions to understanding how transformers perform in-context learning, particularly concerning adversarial robustness and the nature of the learned algorithms.

We conclusively establish that neither linear transformers nor standard GPT-2 style transformers inherently learn adversarially robust in-context learning algorithms. Specifically, within linear transformers, this vulnerability arises because they implement a standard, non-robust learning algorithm. 
This finding provides mechanistic insight into the adversarial non-robustness of transformers previously observed in other works~\citep{qiang2023hijacking,bailey2023image}.
On a positive note, we empirically demonstrate that adversarial training can improve robustness to hijacking attacks, and this improved robustness shows a degree of generalization. 
This is an encouraging result, especially as robust regression against adaptive adversaries is known to be challenging~\citep{diakonikolas2019recent}.

Furthermore, our analysis of the transferability of adversarial attacks reveals two novel phenomena. Firstly, hijacking attacks may fail to transfer effectively across larger transformer models trained with identical procedures but different random seeds, suggesting a non-universality of in-context learning mechanisms even within a fixed architecture and data distribution.
Secondly, adversarial samples sourced from traditional linear regression algorithms (like ordinary least squares and gradient descent) are largely ineffective against GPT-2 style transformers, and vice versa.
This provides indirect evidence that the in-context learning algorithms acquired by GPT-2 might differ qualitatively from these traditional learning algorithms.

\subsection{Limitations}
While the improvement in robustness through adversarial training is promising, the relative simplicity of the linear regression task used in our study means these results should be interpreted with caution.
Further investigations on more complex tasks are required to fully assess the capability of transformer architectures to learn adversarially robust learning algorithms.
Understanding and `reverse-engineering' the specific algorithms transformers implement in these robust settings could also provide novel insights for algorithm design.

An important limitation of our transferability analysis is its non-mechanistic nature.
While the differing behavior on adversarial (out-of-distribution) sample strongly hints at qualitative differences between learned algorithms, it leaves open questions about the precise mechanistic nature of these differences.
Specifically, it remains unclear whether these differences arise from different parameterizations of same underlying algorithm, or transformers learning variety of different types of learning algorithms, or due to some other reason.

\subsection{Future Work}
Prior work has shown that gradient descent on neural network parameters can exhibit an implicit bias towards solutions that generalize well but lack adversarial robustness~\citep{frei2023doubleedgedswordimplicitbias}.
Our work shows that a similar bias exists for linear transformers trained to solve linear regression in-context.
Future work could investigate the extent to which this bias affects the in-context learning algorithms discovered by transformers in general.

Our work, on the whole, highlights that developing a thorough understanding of in-context learning within transformers may be more challenging than previously thought, and emphasises the need of developing mechanistic understanding of these transformers.
Further mechanistic analyses of these transformers is needed to better understand what in-context learning algorithms these transformers learn.
Such reverse engineering efforts could also help inspire better algorithm design for various robust learning problems.

%% file: figures.tex
\NewDocumentCommand{\FigTheoryAttack}{}{%
\begin{figure}[t]
    \centering
    \begin{subfigure}{\textwidth}
        \centering
        \begin{subfigure}{0.24\textwidth}
           \includegraphics[width=\textwidth]{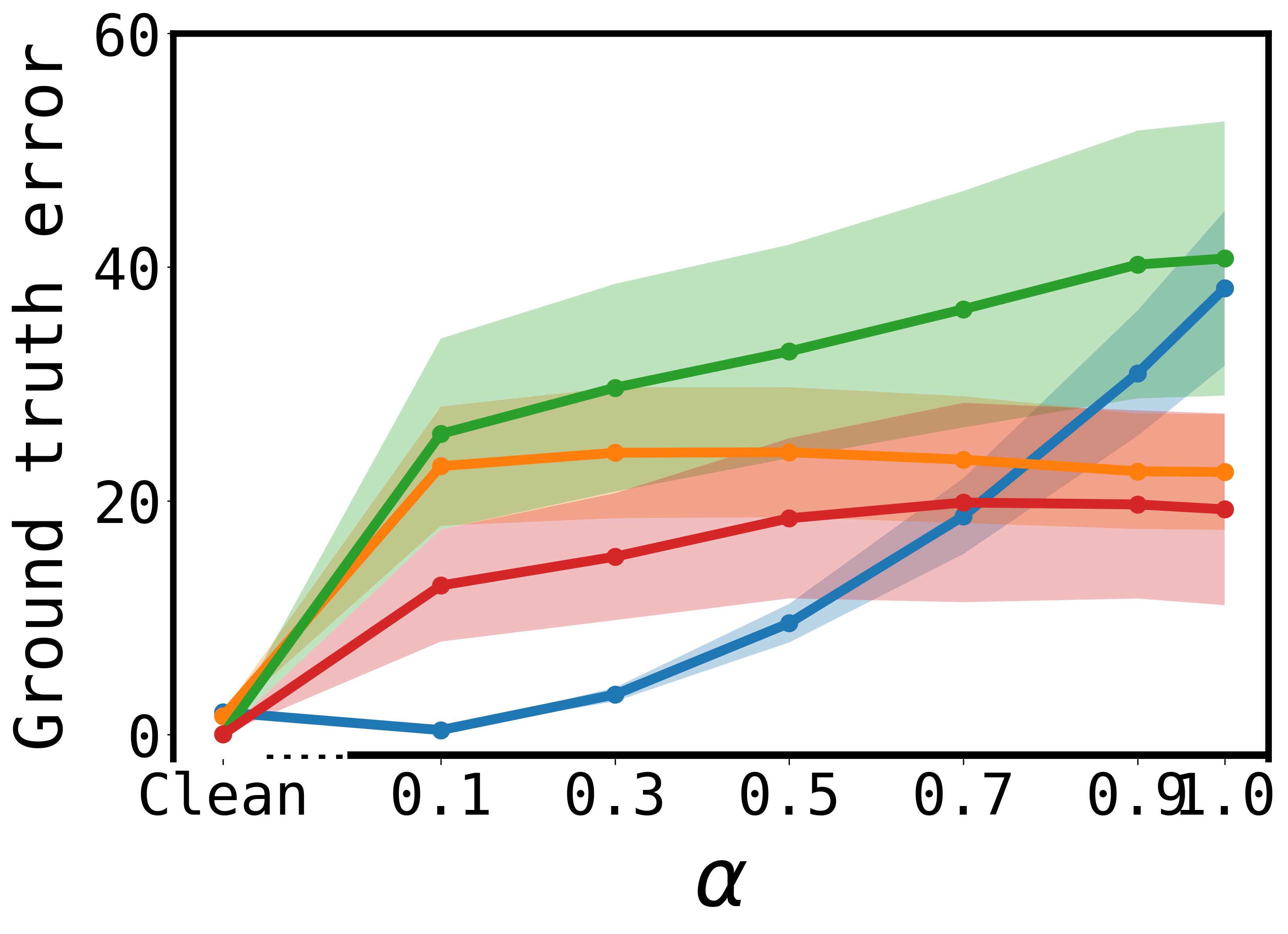}
        \caption{\xattackkk}
        \end{subfigure}
        \hfill
        \begin{subfigure}{0.24\textwidth}
           \includegraphics[width=\textwidth]{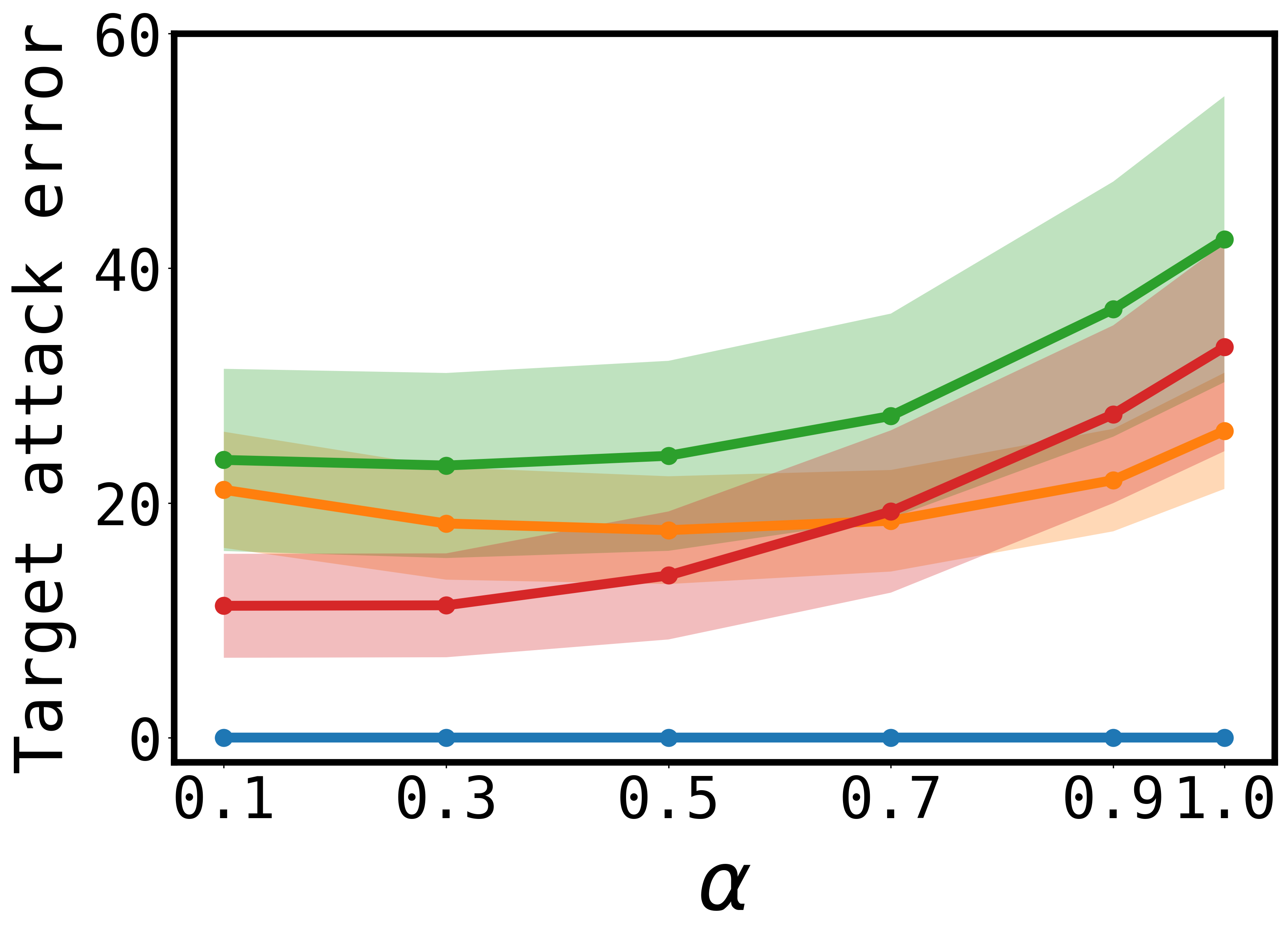}
        \caption{\xattackkk}
        \end{subfigure}
        \hfill
        \begin{subfigure}{0.24\textwidth}
           \includegraphics[width=\textwidth]{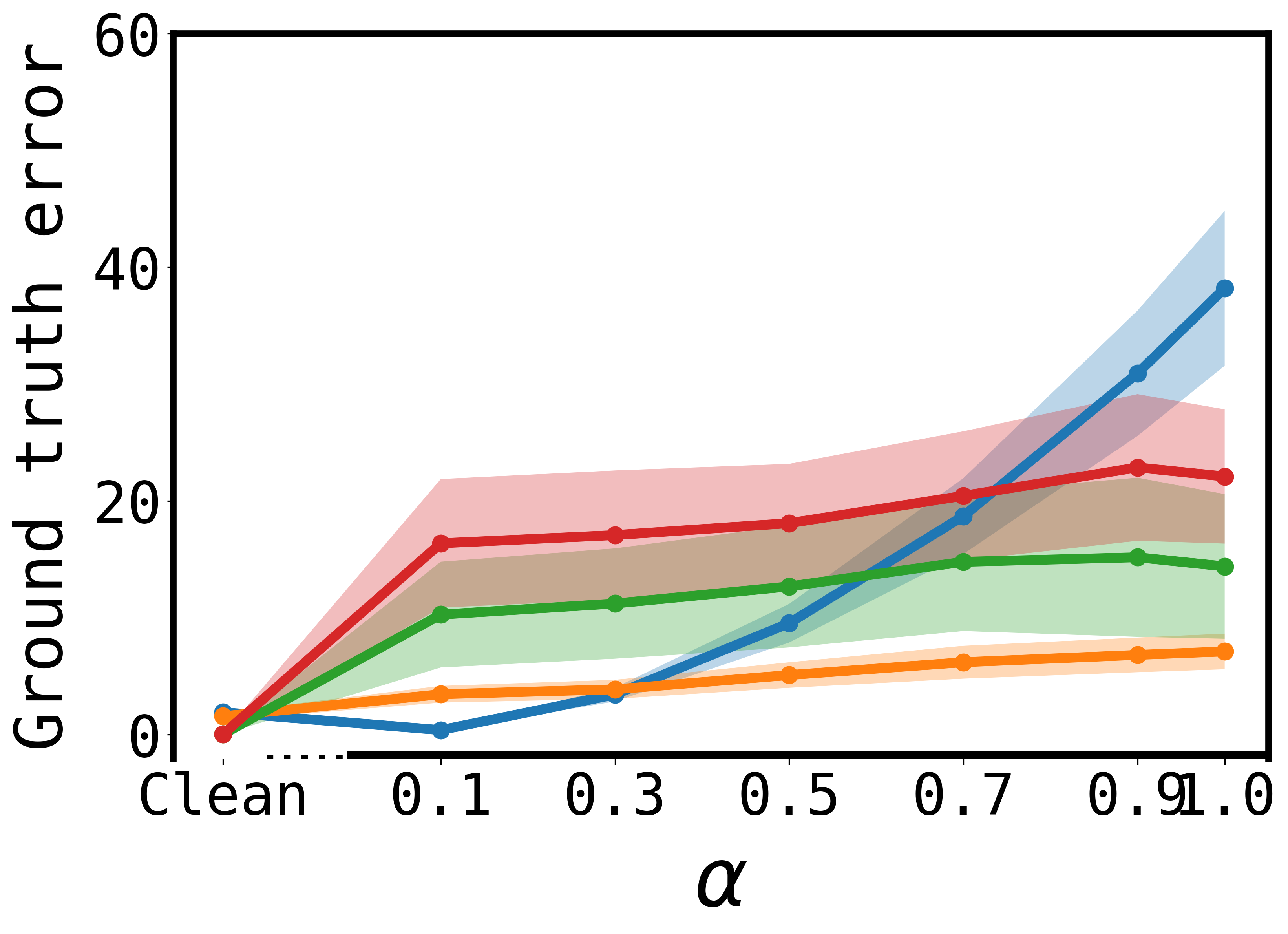}
        \caption{\yattackkk}
        \end{subfigure}
        \hfill
        \begin{subfigure}{0.24\textwidth}
           \includegraphics[width=\textwidth]{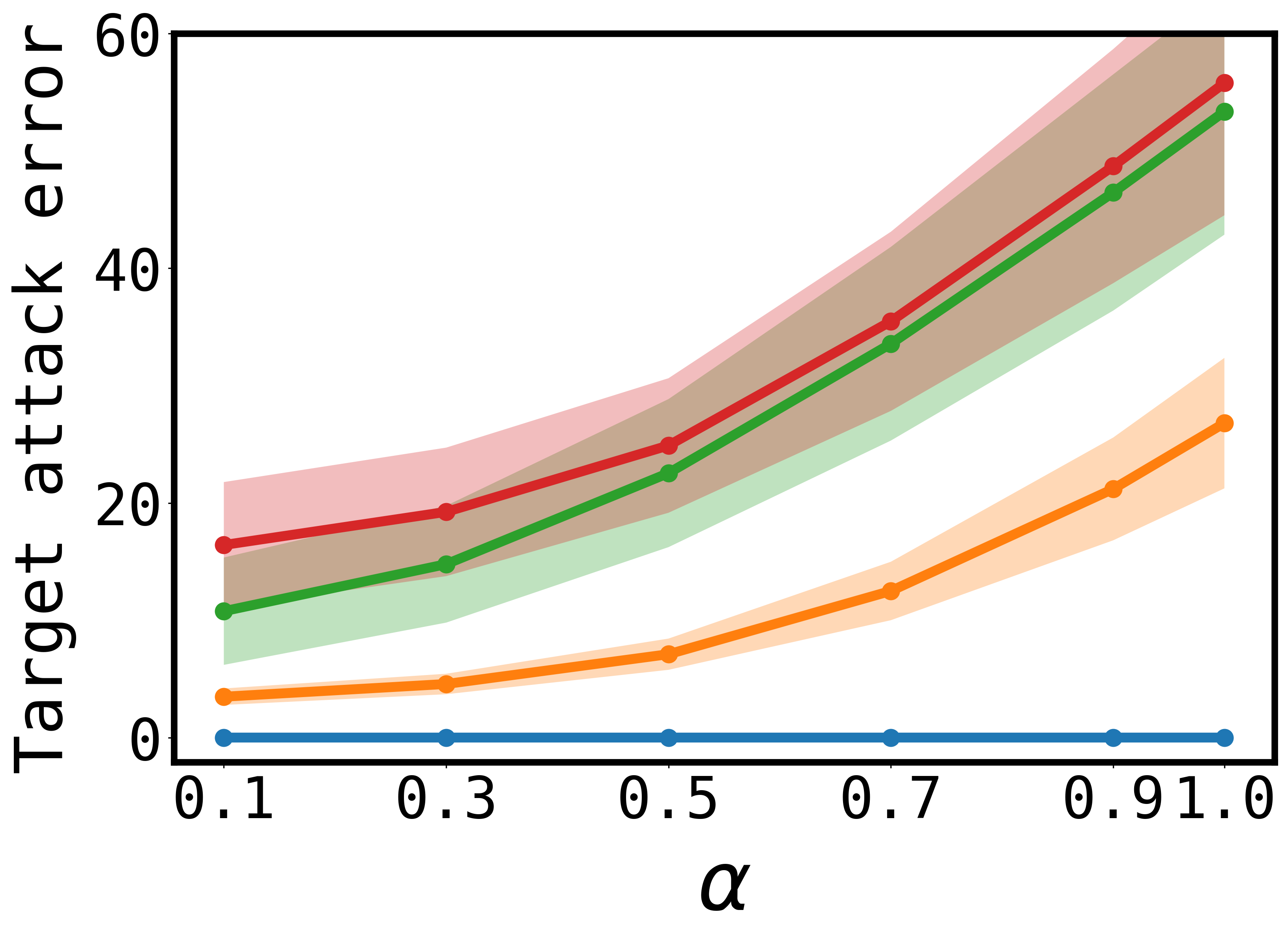}
        \caption{\yattackkk}
        \end{subfigure}
    \end{subfigure}
    \begin{subfigure}{\textwidth}
        \centering
        \includegraphics[width=0.99\textwidth]{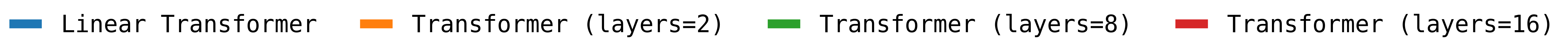}
    \end{subfigure}
    \caption{
    Robustness of different SGD-trained transformers when using attacks constructed from the \changed{infinite-time} gradient flow solution  via Theorem~\ref{thm:jailbreaking.yadv}, for different target values $y_\bad = (1 - \alpha)w^\top x_{\query} + \alpha w_{\perp}^\top x_{\query}$, where $w_\perp \perp w$.  While these attacks reduce ground truth error across all model classes, the \textit{targeted} attack error is only small for the linear transformer. Shaded area is standard error.}
    \label{fig:linear.transformer.attack.transfer}
\end{figure}}

\NewDocumentCommand{\FigGradAttackYTraining}{}{%
    \begin{figure}[b!]
        \centering
        \begin{subfigure}[b]{0.25\textwidth}
            \includegraphics[width=\textwidth]{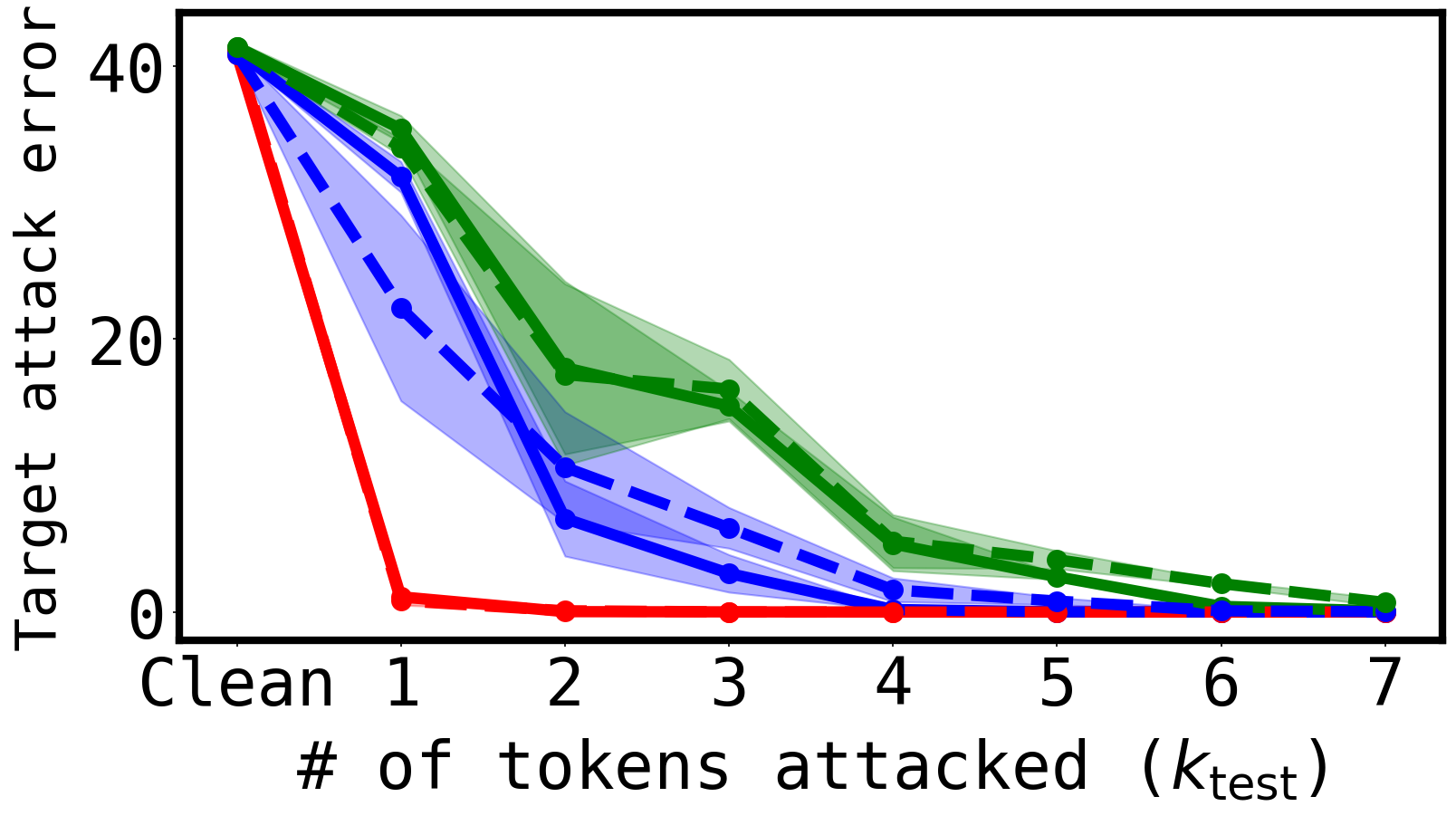}
            \caption{\xattackkk.}
            \label{subfig:x.attack.y.advtraining}
        \end{subfigure}
        \begin{subfigure}[b]{0.25\textwidth}
            \includegraphics[width=\textwidth]{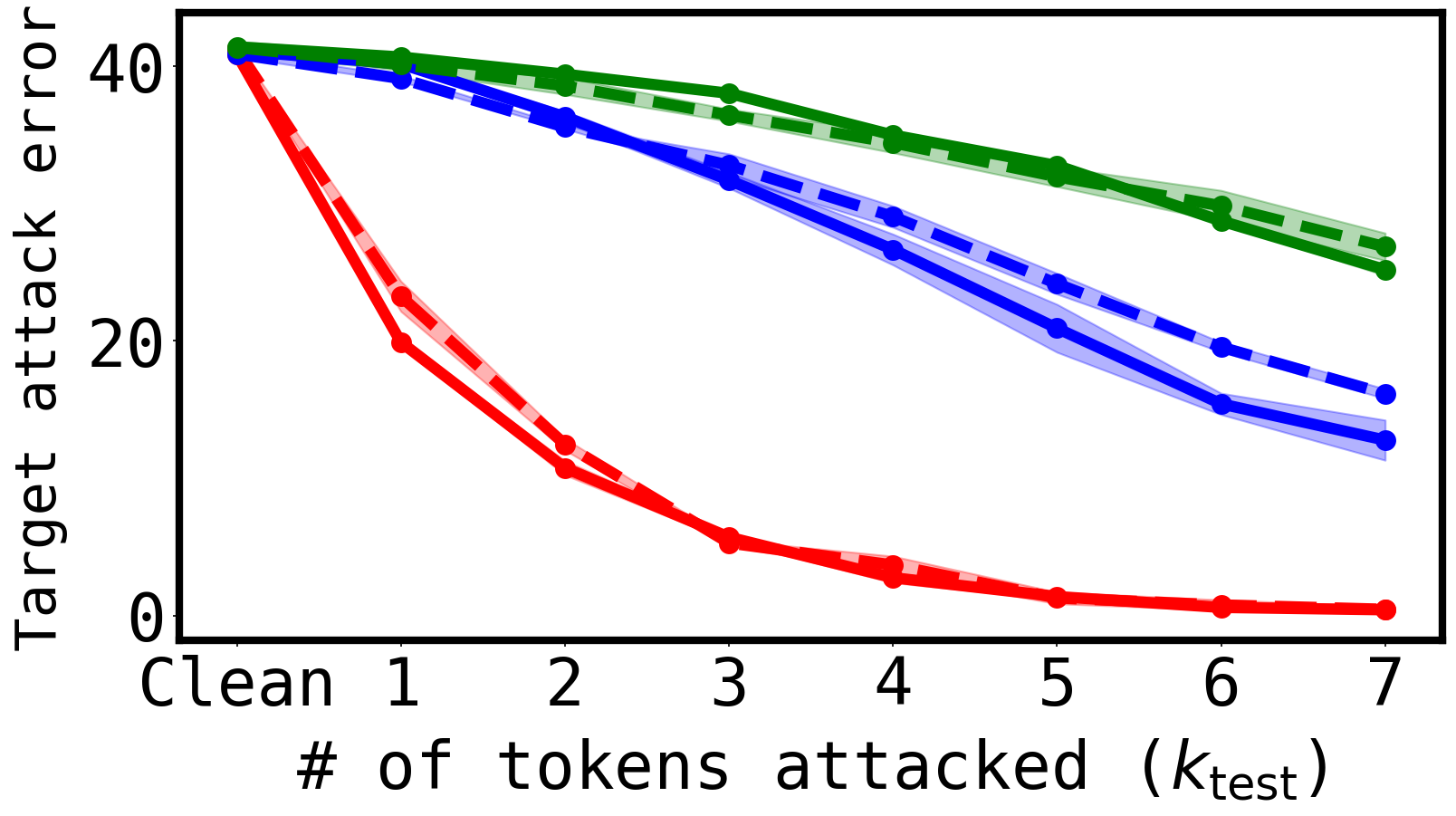}
            \caption{\yattackkk.}
        \end{subfigure}
        \begin{subfigure}[b]{0.25\textwidth}
            \includegraphics[width=\textwidth]{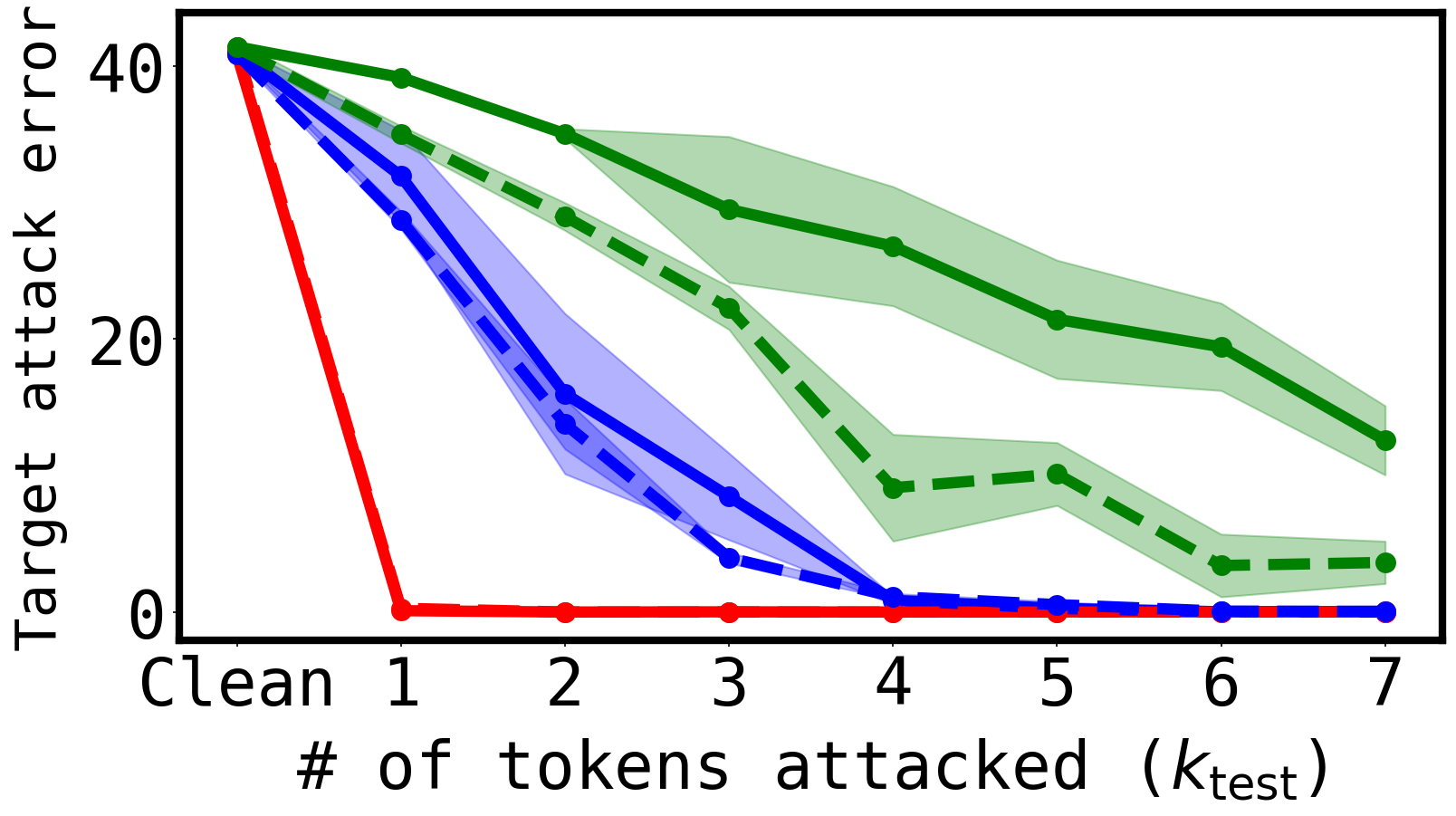}
            \caption{\zattackkk.}
        \end{subfigure}
        \begin{subfigure}[b]{0.8\textwidth}
            \includegraphics[width=\textwidth]{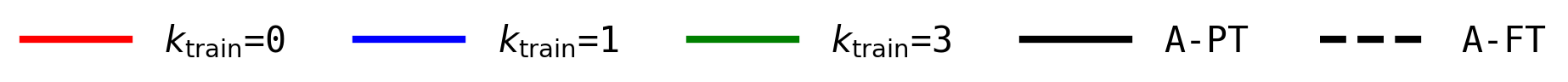}
        \end{subfigure}
        \caption{For both adversarial pretraining (A-PT) and fine-tuning (A-FT) against \yattack{}, robustness against $\yattack{}$ improves significantly, especially when trained on a budget of $k_\text{train}=3$ perturbed tokens. The results are shown for $8$ layer transformers with GPT-2 architecture.}
        \label{fig:adv.training.y}
    \end{figure}
}

\NewDocumentCommand{\FigGradAttackXTraining}{}{%
\begin{figure}[b!]
    \centering
    \begin{subfigure}[b]{0.25\textwidth}
        \includegraphics[width=\textwidth]{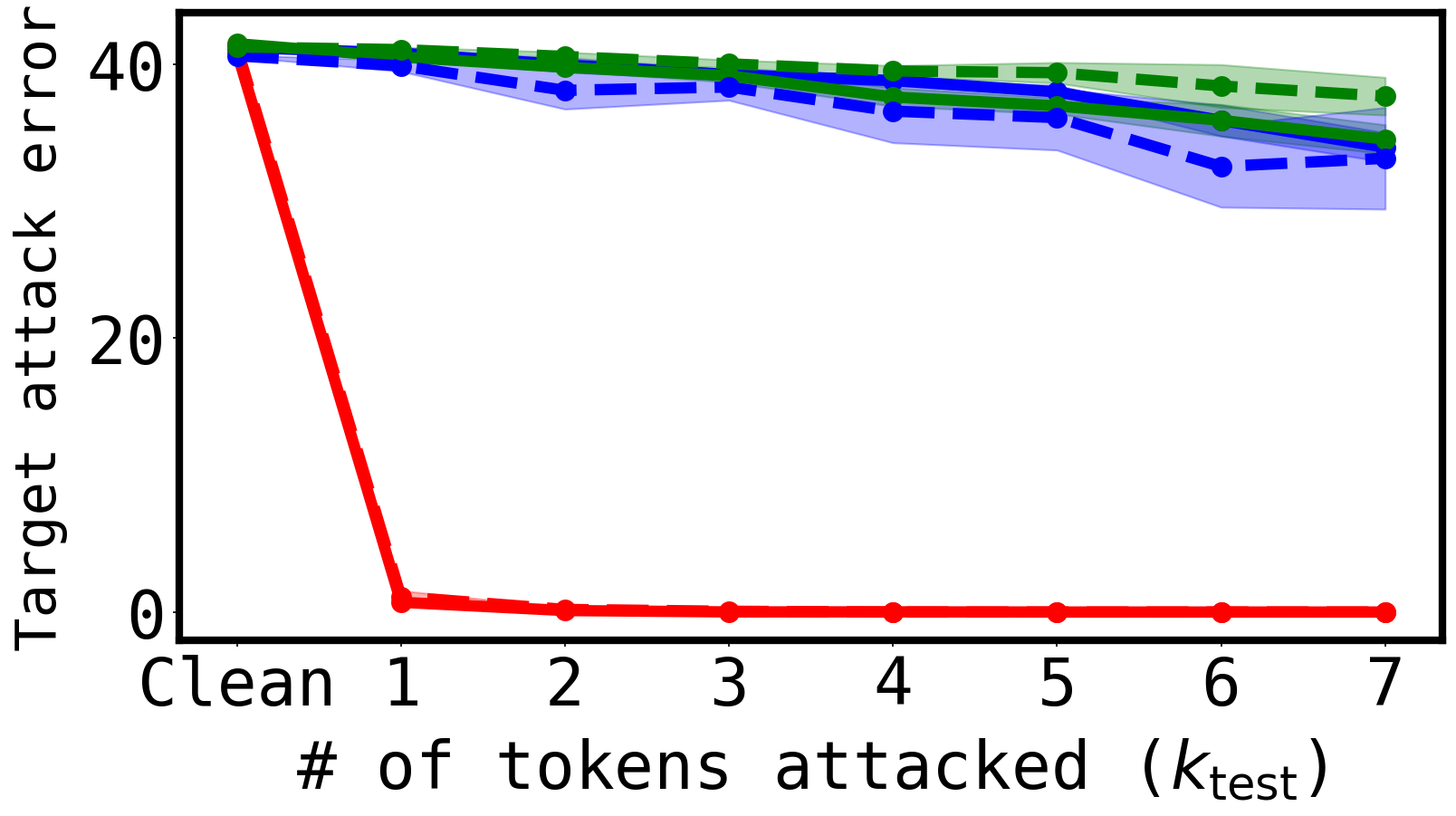}
        \caption{\xattackkk.}
    \end{subfigure}
        \begin{subfigure}[b]{0.25\textwidth}
        \includegraphics[width=\textwidth]{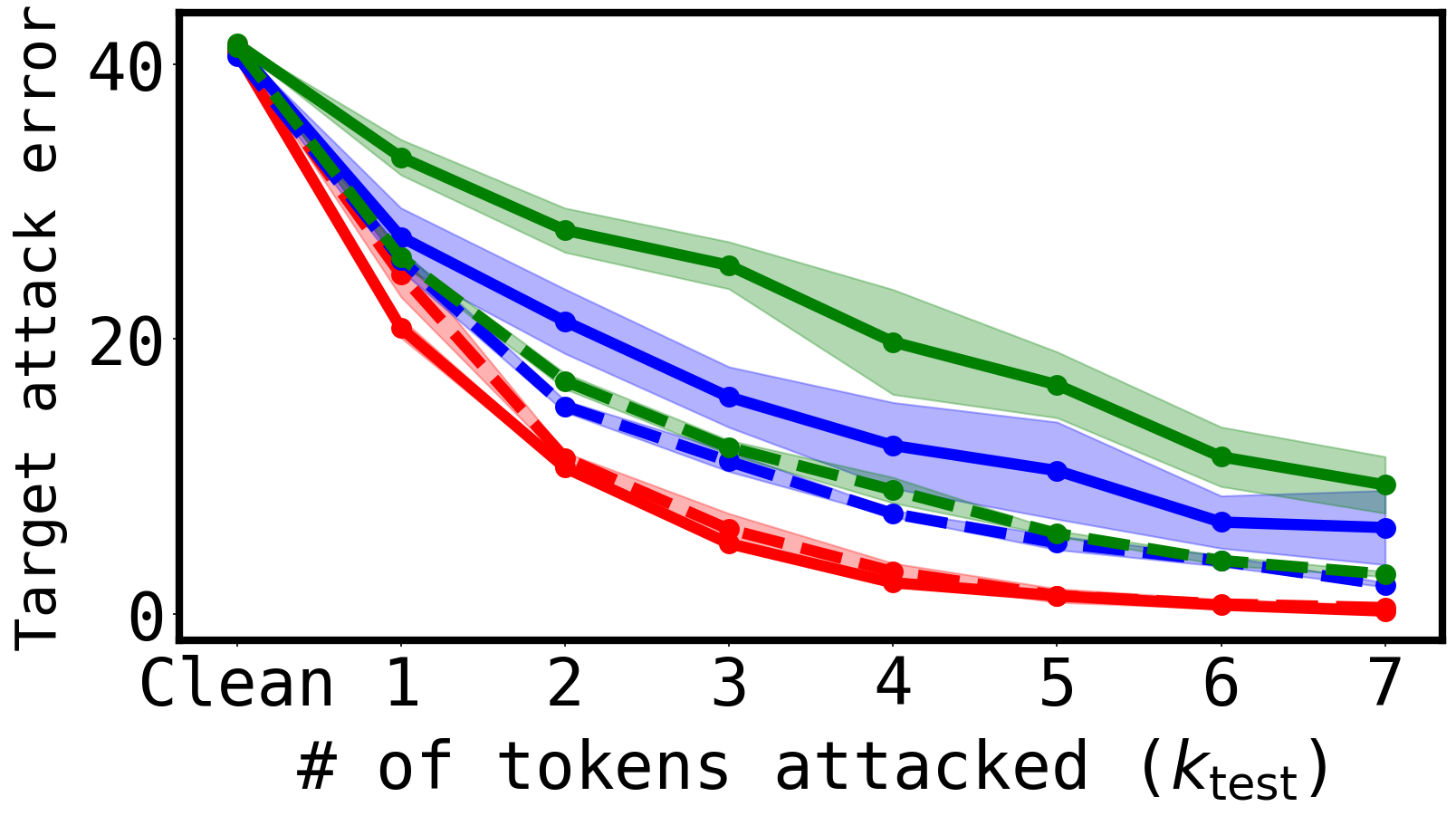}
        \caption{\yattackkk.}
    \end{subfigure}
        \begin{subfigure}[b]{0.25\textwidth}
        \includegraphics[width=\textwidth]{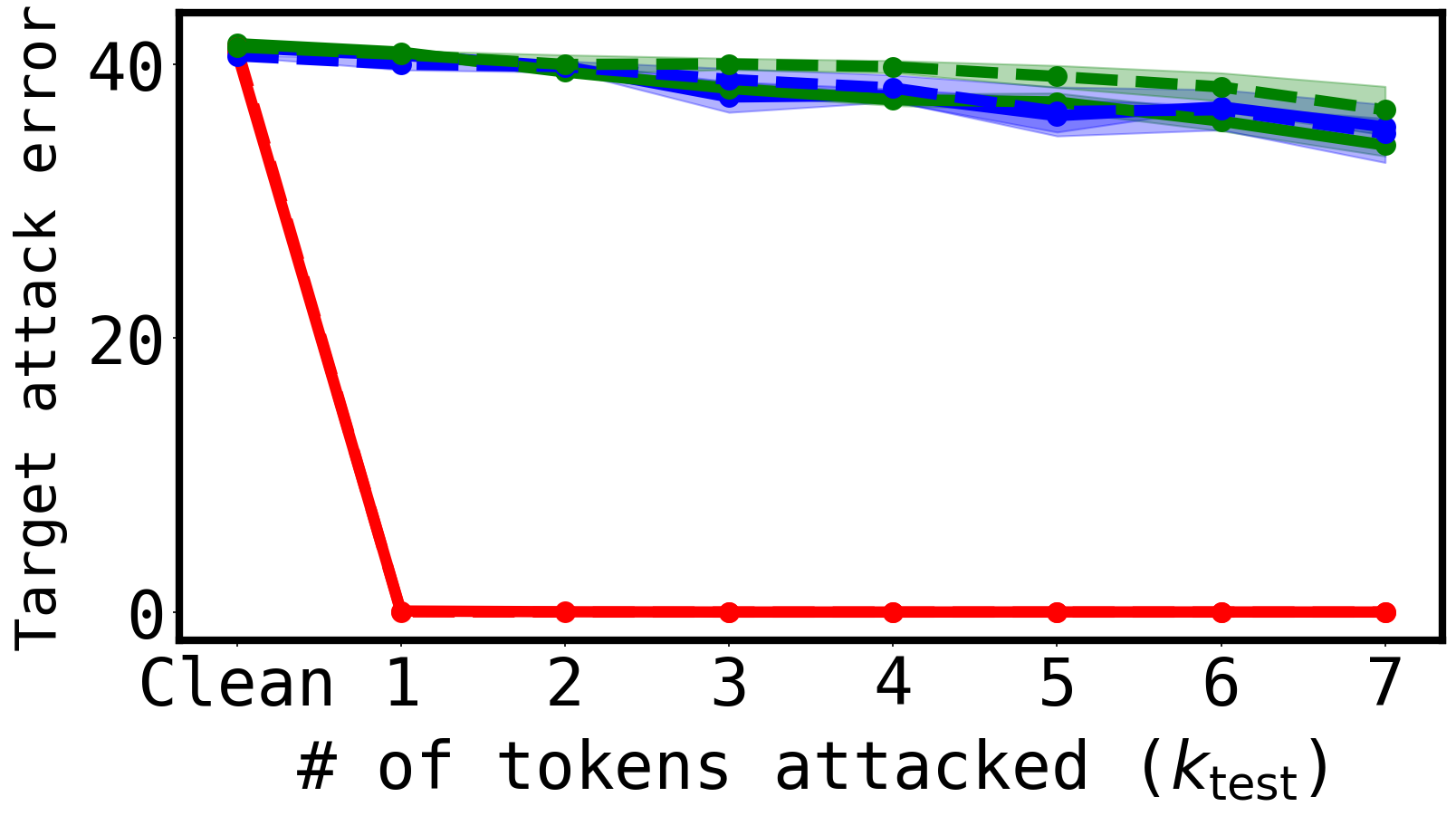}
        \caption{\zattackkk.}
    \end{subfigure}
    \begin{subfigure}[b]{0.8\textwidth}
        \includegraphics[width=\textwidth]{results/adv_training/adv_training_legend.png}
    \end{subfigure}
    \caption{For both adversarial pretraining (A-PT) and fine-tuning (A-FT) against \xattack{}, robustness against $\xattack{}$ \textit{and} \zattack{} improves for $7+$ token attacks when trained on $k_\text{train}=1$. The results are shown for $8$ layer transformers with GPT-2 architecture.}
    \label{fig:adv.training.x}
\end{figure}}

\NewDocumentCommand{\FigEffectofSeqLen}{}{%
\begin{wrapfigure}{r}{0.6\textwidth}
\raggedleft
    \begin{subfigure}[b]{0.29\textwidth}
        \centering
        \includegraphics[width=\textwidth]{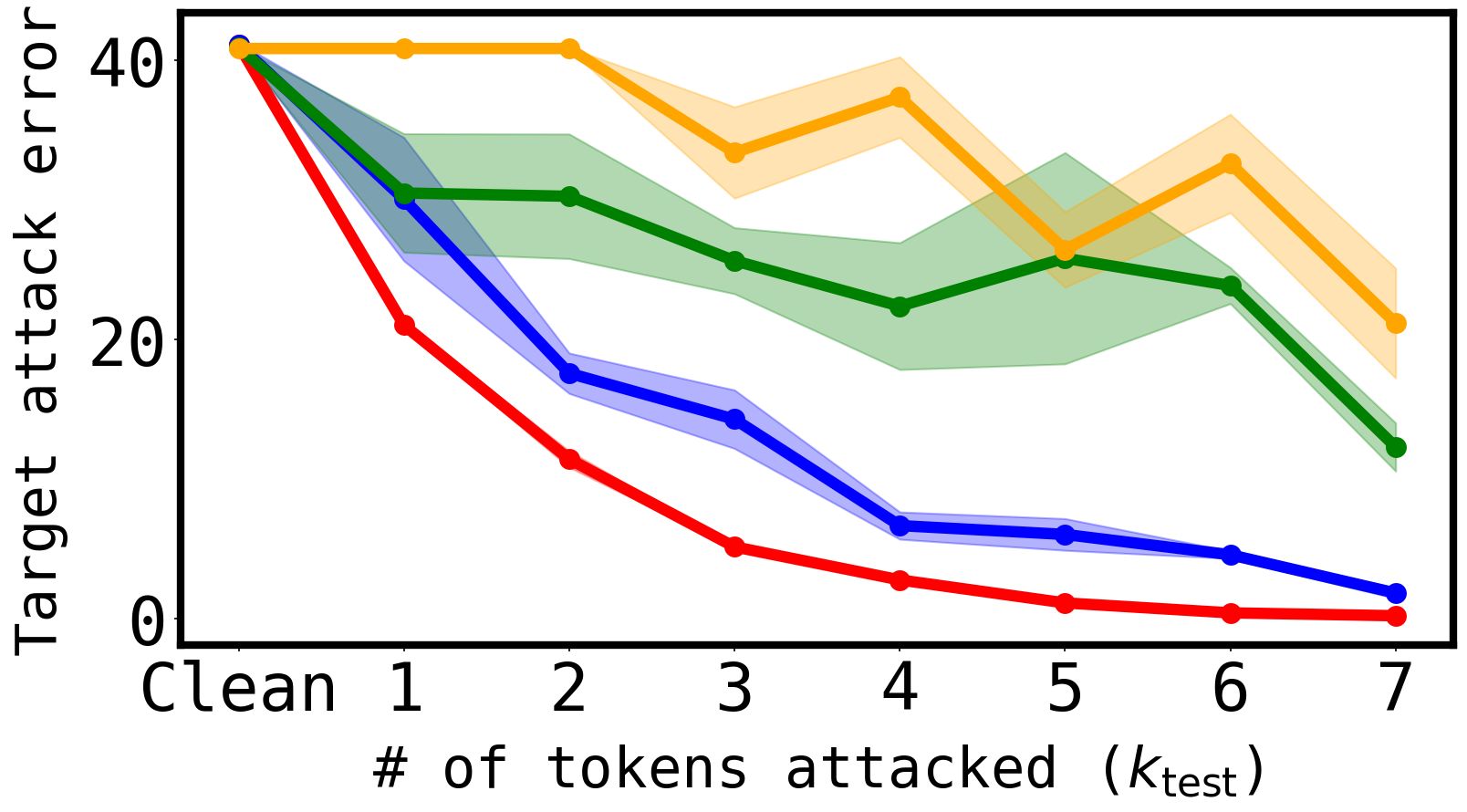}
    \end{subfigure}
    \begin{subfigure}[b]{0.29\textwidth}
        \centering
        \includegraphics[width=\textwidth]{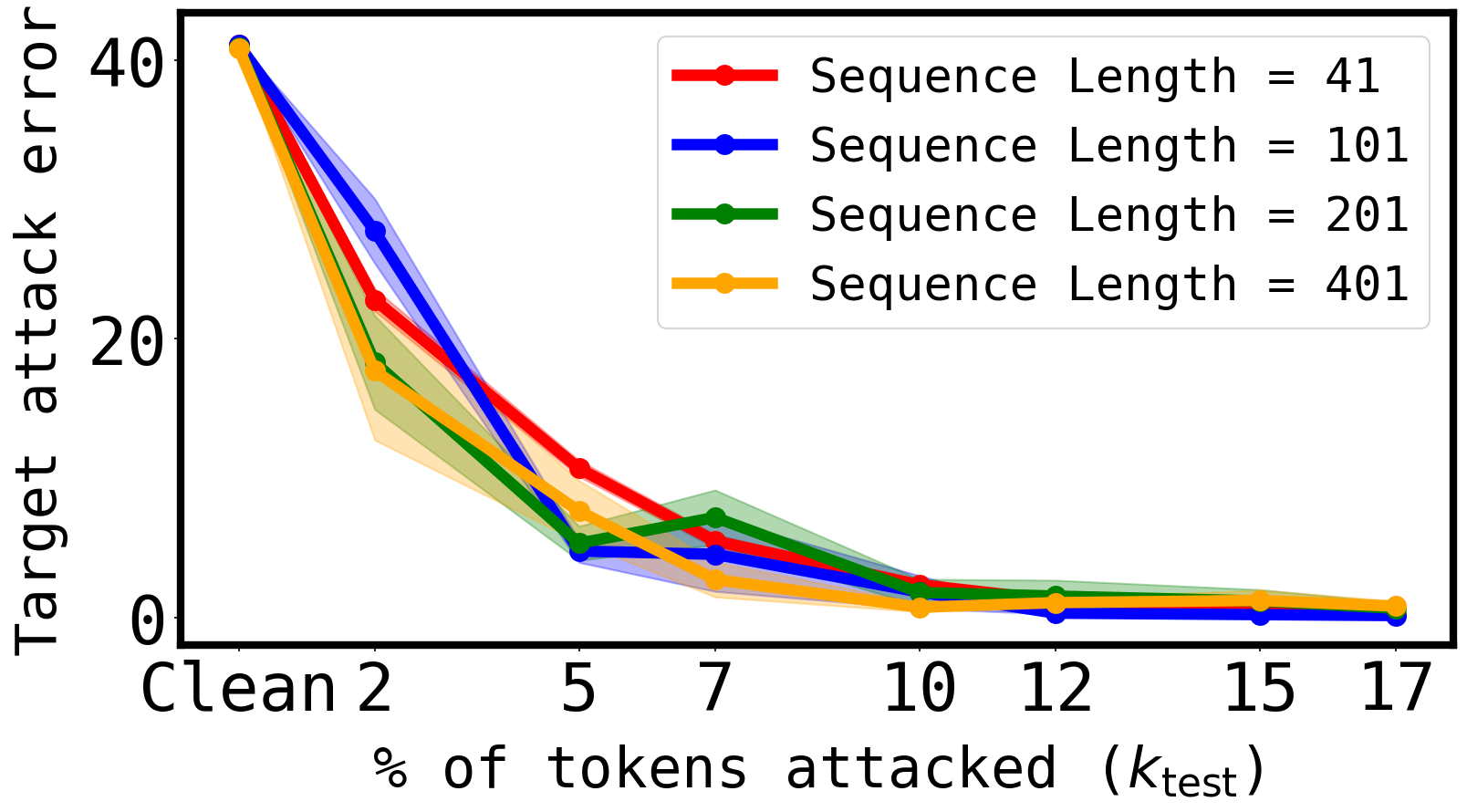}
    \end{subfigure}
    \caption{Larger context lengths can improve robustness for a fixed \textit{number} of tokens attacked, but not for a fixed \textit{proportion}. The number of layers is kept fixed at $8$ while varying the context length. \changed{The results are averaged across three independently trained transformers (with identical hyperparameters). The shaded area indicates standard error.}}
    \label{fig:results.seqlen}
\end{wrapfigure}}

\NewDocumentCommand{\FigAccRobustnessTradeoff}{}{%
\begin{wrapfigure}{r}{0.4\textwidth}
\raggedleft
        \begin{subfigure}[b]{0.4\textwidth}
        \includegraphics[width=\textwidth]{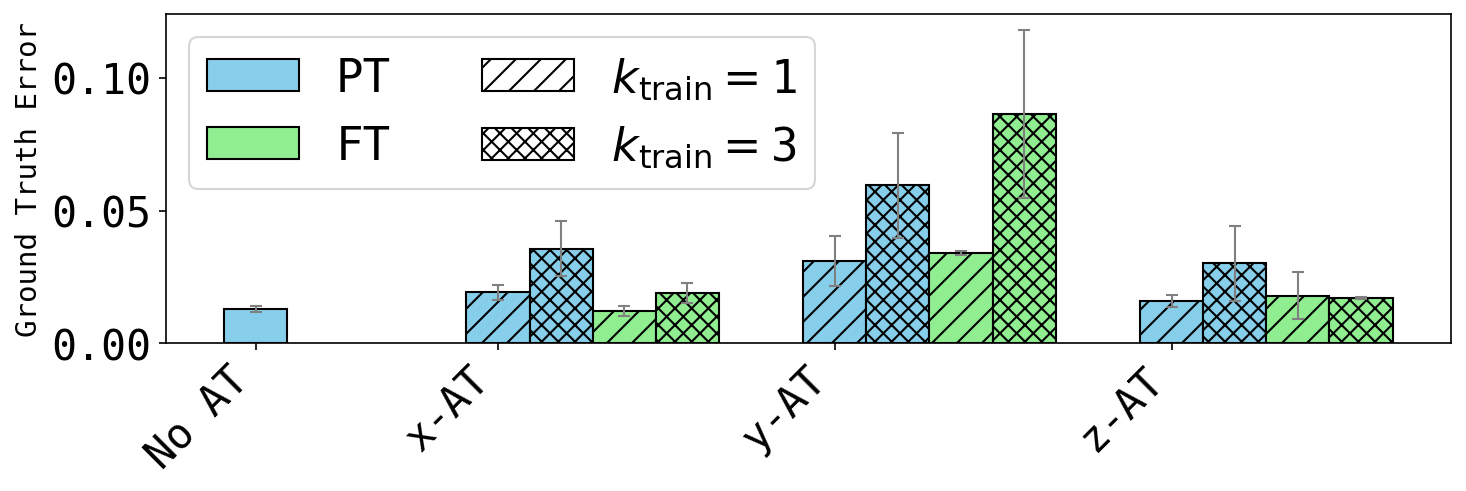}
        \end{subfigure}
        \caption{While there is a moderate tradeoff between robustness and (clean) accuracy when training against \yattack, the tradeoff is very small for \xattack{} and \zattack{} training.}
        \label{fig.clean.performance}
\end{wrapfigure}}

\NewDocumentCommand{\FigOLSTFTransfer}{}{%
\begin{figure}[t]
    \centering
    \begin{subfigure}[b]{0.39\textwidth}
        \centering
        \includegraphics[width=\textwidth]{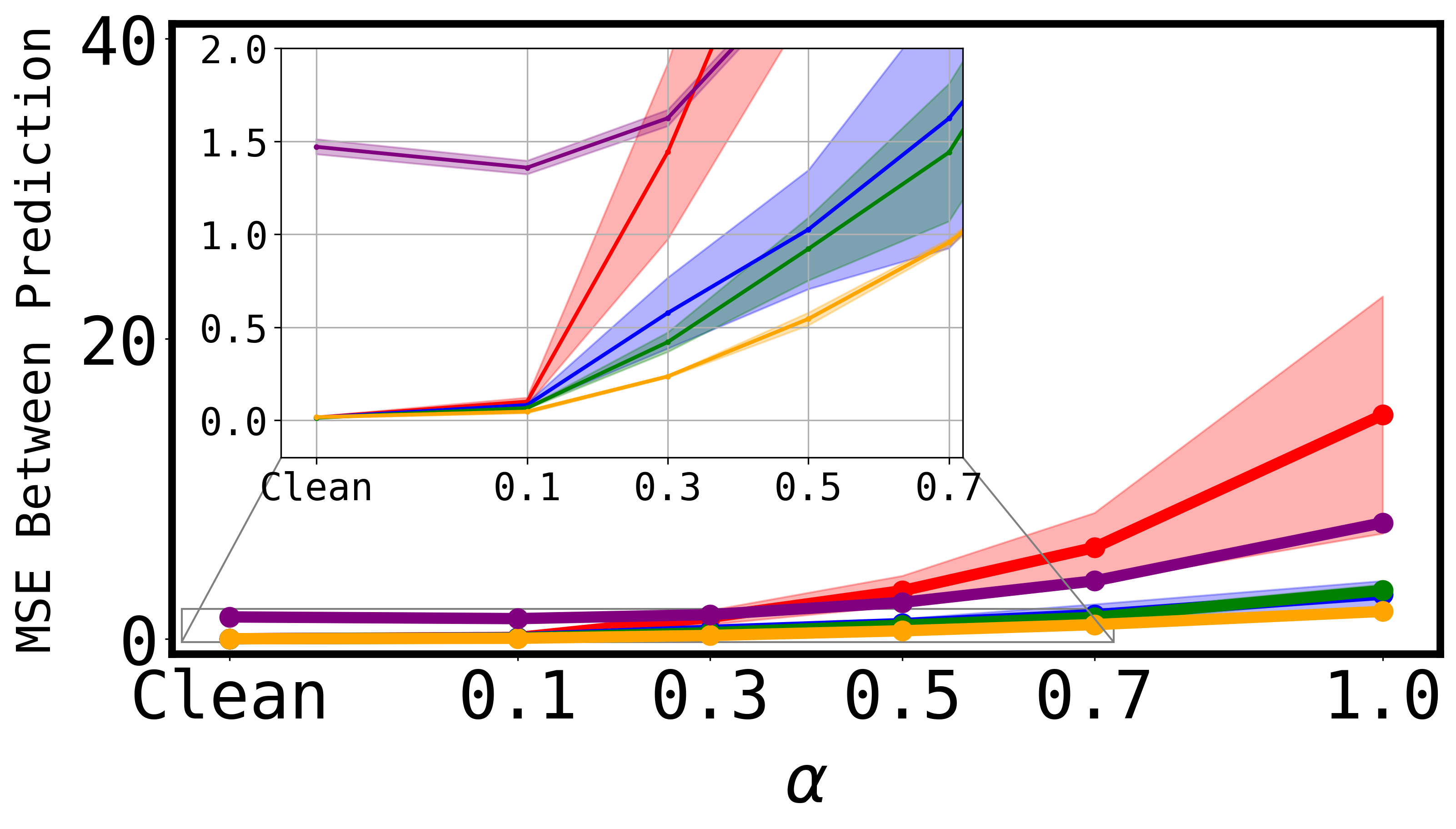}
        \caption{OLS $\rightarrow$ Transformers.}
    \end{subfigure}
    \hspace{5em}
    \begin{subfigure}[b]{0.39\textwidth}
        \centering
        \includegraphics[width=\textwidth]{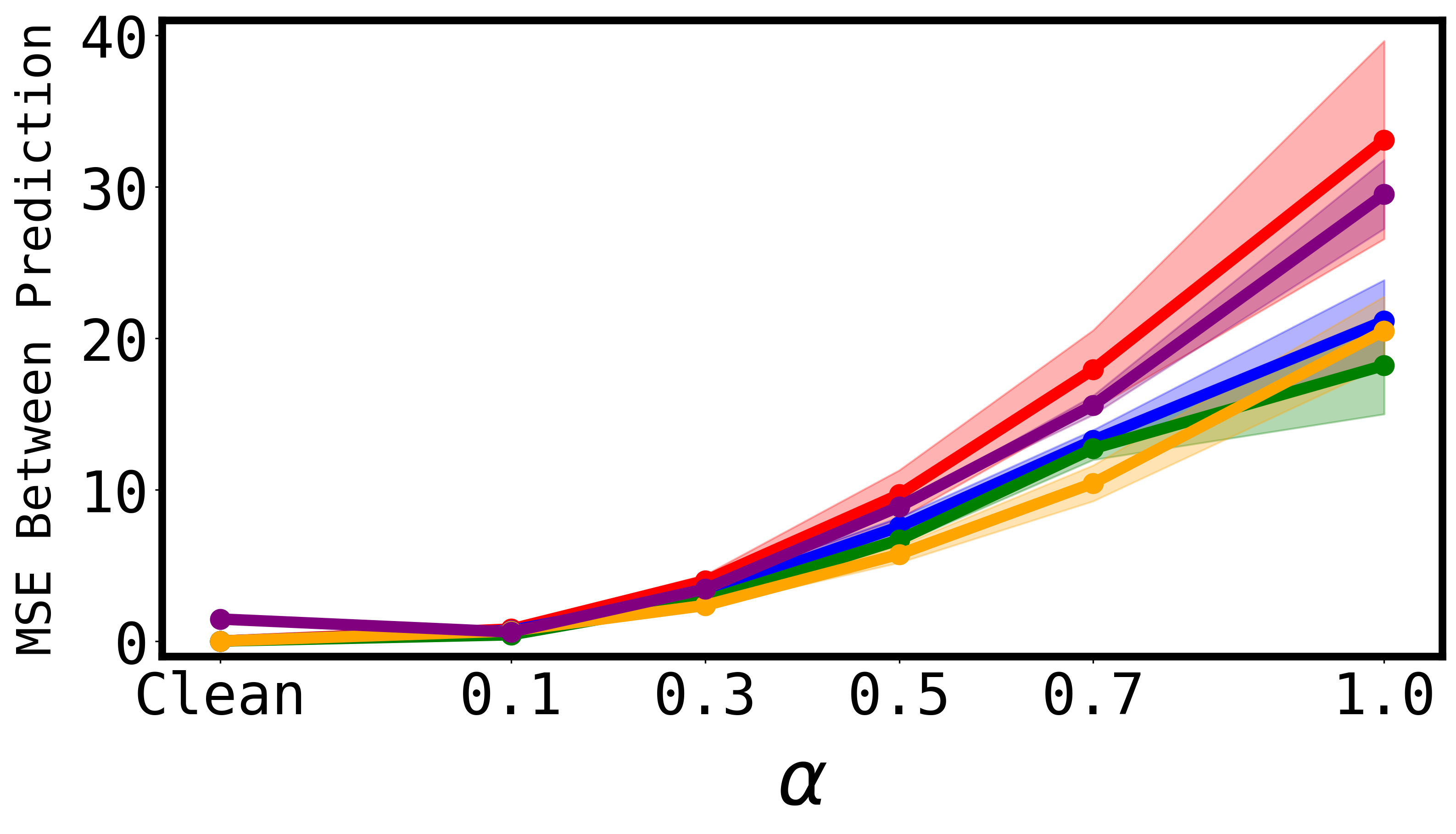}
        \caption{Transformers $\rightarrow$ OLS.}
    \end{subfigure}
    \begin{subfigure}[b]{0.9\textwidth}
        \centering
        \includegraphics[width=\textwidth]{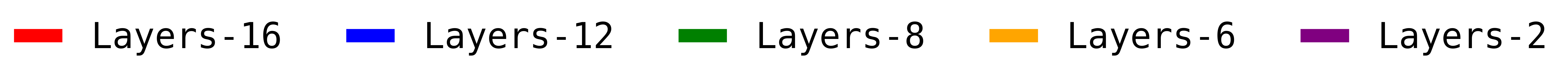}
    \end{subfigure}
    \caption{As the targeted prediction $y_\bad$ becomes more out-of-distribution ($\alpha \to 1$; cf.~\eqref{eq:y_bad.alpha}), we see (b) attacks derived from transformers result in significantly different predictions than that of OLS, while (a) attacks derived from OLS result in similar predictions for some classes of transformers but not all.}
    \label{fig:results.transfer.ols.tf}
\end{figure}}

\NewDocumentCommand{\FigOLSTFTransferWrapped}{}{%
\begin{wrapfigure}{r}{0.6\textwidth}
\vspace{-1em}
\raggedleft
    \begin{subfigure}[b]{0.28\textwidth}
        \centering
        \includegraphics[width=\textwidth]{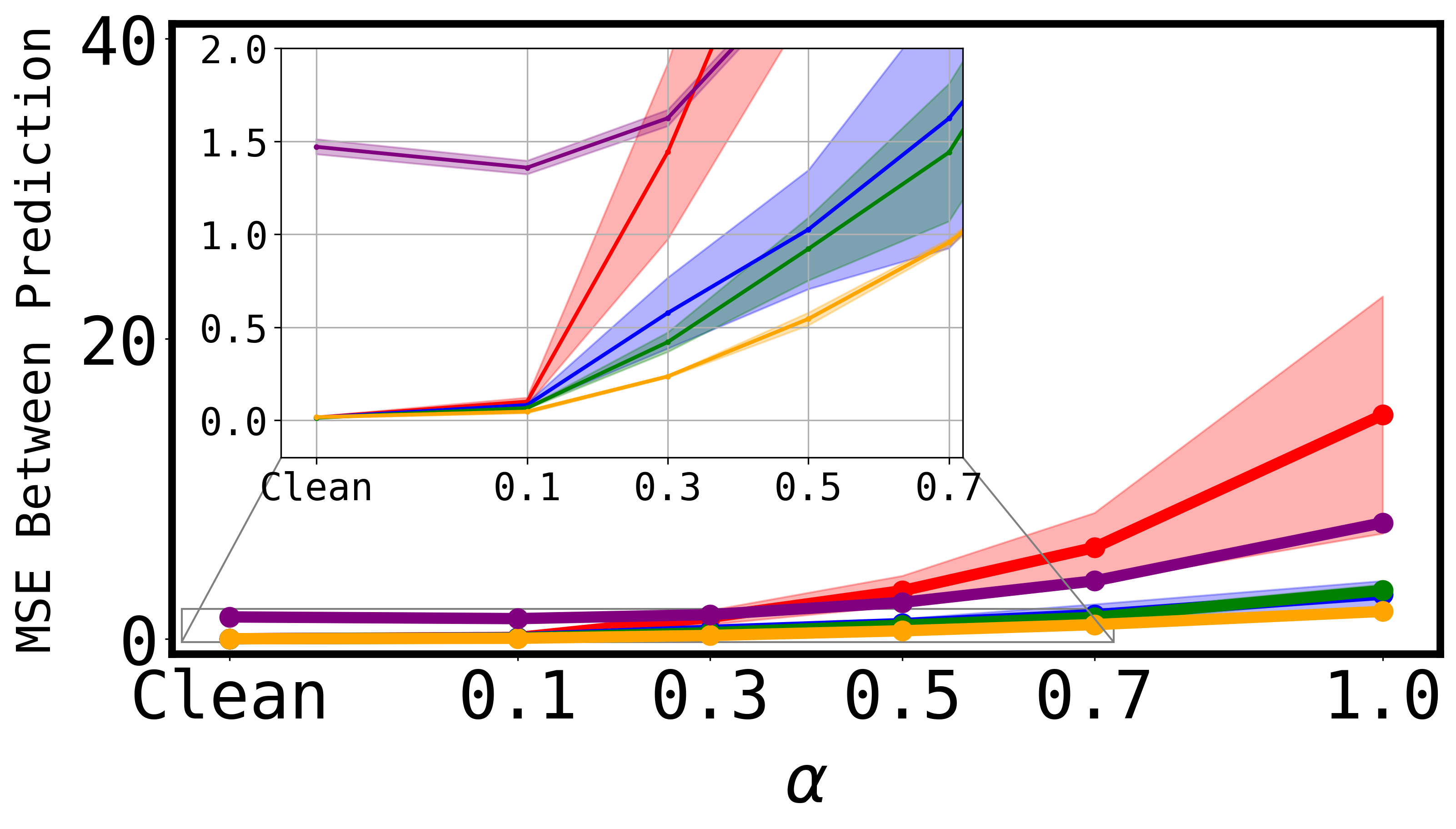}
        \caption{OLS $\rightarrow$ Transformers.}
    \end{subfigure}
    \begin{subfigure}[b]{0.28\textwidth}
        \centering
        \includegraphics[width=\textwidth]{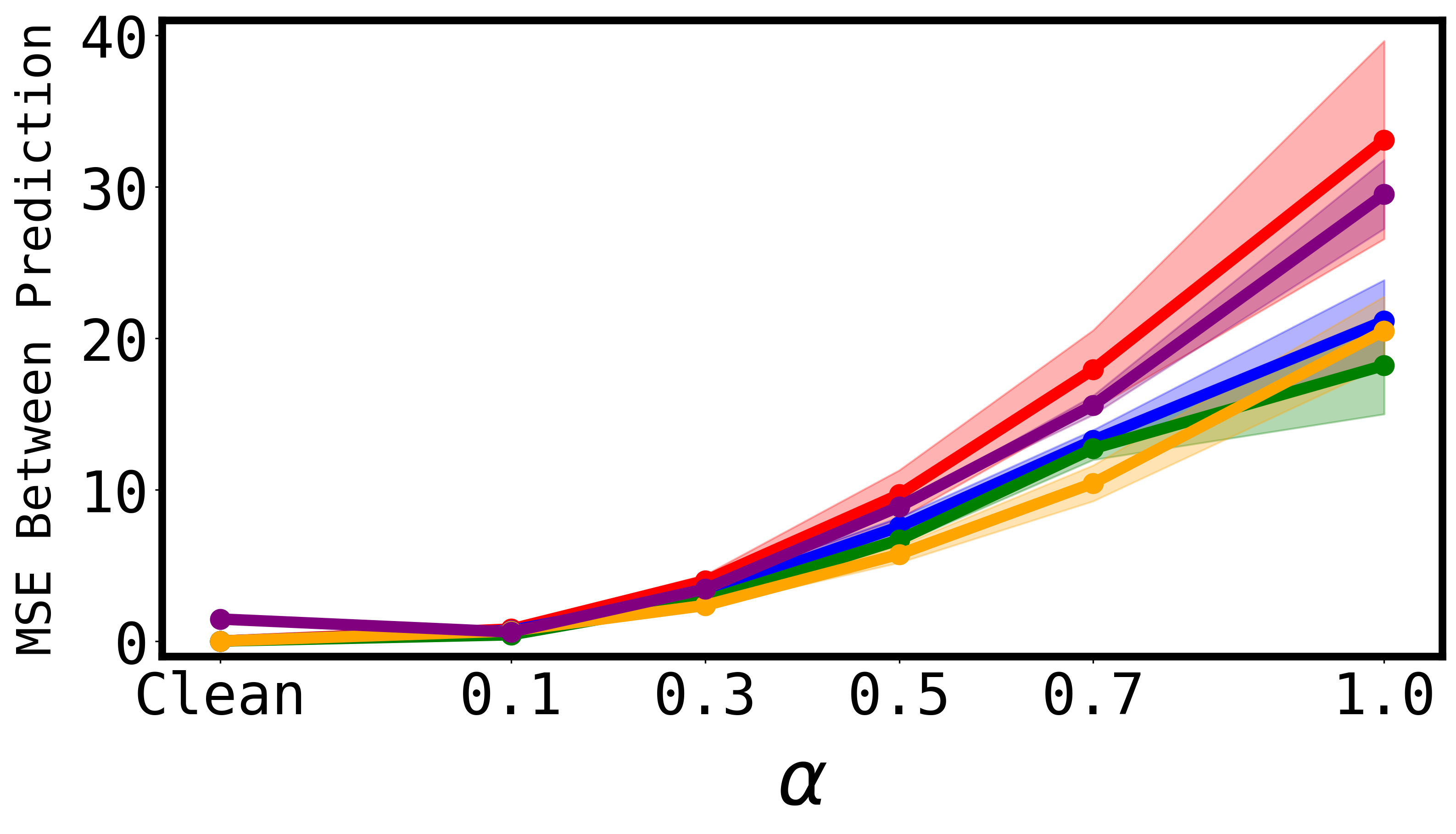}
        \caption{Transformers $\rightarrow$ OLS.}
    \end{subfigure}
    \begin{subfigure}[b]{0.55\textwidth}
        \centering
        \includegraphics[width=\textwidth]{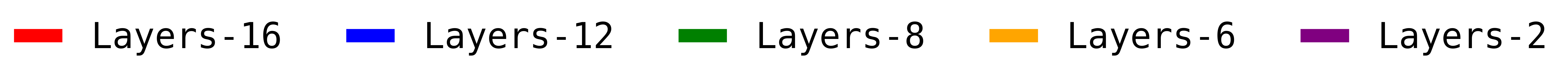}
    \end{subfigure}
    \caption{Mean squared error between predictions made by OLS and transformers on adversraial samples sourced respectively from OLS and transformers for different values of $\alpha$.}
    \label{fig:results.transfer.ols.tf}
\end{wrapfigure}}

%% file: appendix.tex
\appendix 
\section{Proofs}
\label{appx.proofs}
\paragraph{Notation:} We denote $[n] = \{1,2,...,n\}.$ 
We write the inner product of two matrices $A,B \in \mathbb{R}^{m \times n}$ as $\left\langle A,B\right\rangle = \operatorname{tr}(AB^\top).$ We use $0_{n}$ and $0_{m\times n}$ to denote the zero vector and zero matrix of size $n$ and $m\times n,$ respectively.  We denote the matrix operator norm and Frobenius norm as $\left\|\cdot\right\|_2$ and $\left\|\cdot\right\|_{F}$. We use $I_d$ to denote the $d$-dimensional identity matrix and sometimes we also use $I$ when the dimension is clear from the context.

\paragraph{Setup:} As described in the main text, we consider the setting of linear transformers trained on in-context examples of linear models, a setting considered in a number of prior theoretical works on transformers~\citep{von2022transformers,akyurek2022learning,zhang2023trained,ahn2023transformers,wu2023pretraining}. 
Let $x_i \in \R^d$ and $y_i \in \R$.  For a prompt $P =(x_1, y_1, \dots, x_N, y_N, x_{N+1})$, we say its \textit{length} is $N$.  For this prompt, we use an embedding which stacks $(x_i, y_i)^\top \in \R^{d+1}$ into the first $N$ columns with $(x_{N+1}, 0)^\top \in \R^{d+1}$ as the last column:
\begin{equation}
E = E(P) = \begin{pmatrix}
    x_1 & x_2 & \cdots & x_N & x_{N+1} \\
    y_1 & y_2 & \cdots & y_N & 0
\end{pmatrix} \in \R^{(d+1)\times (N+1)}.\label{eq:embedding.matrix.prompt}
\end{equation}
We consider a single-layer linear self-attention (LSA) model, which is a modified version of attention where we remove the softmax nonlinearity, merge the projection and value matrices into a single matrix $\WPV \in\R^{d+1, d+1}$, and merge the query and key matrices into a single matrix $\WKQ \in \R^{d+1, d+1}$.  Denote the set of parameters as $\params = (\WKQ, \WPV)$ and let
\begin{equation} \label{eq:lsa}
    f_\lsa(E;\params) = E + \WPV E \cdot \f{ E^\top \WKQ E}{N}. 
\end{equation}
The network's prediction for the query example $x_{N+1}$ is the bottom-right entry of matrix output by $f_\lsa$, 
\[ \widehat y_{\query}(E;\params) = [f_{\lsa}(E; \params)]_{(d+1), (N+1)}.\] 
We may occasionally use an abuse of notation by writing $\hat y_\query(E;\theta)$ as $\hat y_\query(P)$ or $\hat y_\query$ with the understanding that the transformer always forms predictions by embedding the prompt into the matrix $E$ and always depends upon the parameters $\theta$.

We assume training prompts are sampled as follows. Let $\Lambda$ be a positive definite covariance matrix.  Each training prompt, indexed by $\tau \in \N$, takes the form of $P_\tau = (x_{\tau, 1}, y_{\tau,1}, \dots, x_{\tau, N}, y_{\tau, N}), x_{\tau, N+1})$, where task weights $w_\tau \iid \normal (0, I_d)$, inputs $x_{\tau, i} \iid \normal(0, \Lambda)$, and labels $y_{\tau,i} = \sip{w_\tau}{x_i}$.  
The empirical risk over $B$ independent prompts is defined as
\begin{equation}\label{eqn:emp_loss_maintext}
    \widehat L(\params) = \f{1}{2B} \sum_{\tau=1}^B \bigg(\widehat y_{\tau, N+1}(E_\tau; \theta)- \sip{w_\tau}{x_{\tau, N+1}}\bigg)^2.
\end{equation}
We consider the behavior of gradient flow-trained networks over the population loss in the infinite task limit $B\to \infty$:
\begin{equation}\label{eqn:population_loss}
    L(\params) = \lim_{B\to \infty} \widehat L(\params) = \f 12 \E_{w_\tau\sim \normal(0,I_d),\, x_{\tau,i} x_{\tau, N+1}\iid \normal(0,\Lambda)} \l[ (\widehat y_{\tau, N+1}(E_\tau; \theta) - \sip{w_\tau}{x_{\tau, N+1}})^2 \r]
\end{equation}
Note that we consider the infinite task limit, but each task has a finite set of $N$ i.i.d. $(x_i, y_i)$ pairs.  
We consider the setting where $f_\lsa$ is trained by gradient flow on the population loss above.  
Gradient flow captures the behavior of gradient descent with infinitesimal step size and has dynamics $  \f{\mathrm{d}}{\mathrm dt} \params = - \nabla L(\params)$. 

We repeat Theorem~\ref{thm:jailbreaking.yadv} from the main section for convenience.

\linearhijack*

\begin{proof}
Let us partition the $\WKQ$ and $\WPV$ matrices as,
\[ \WKQ = \begin{pmatrix} W_{11}^{KQ} & w_{12}^{KQ}\\ (w_{21}^{KQ})^\top & w_{22}^{KQ} \end{pmatrix},\quad  \WPV = \begin{pmatrix} W_{11}^{PV} & w_{12}^{PV}\\ (w_{21}^{PV})^\top & w_{22}^{PV} \end{pmatrix},\]
where $W_{11}^{KQ}, W_{11}^{PV} \in \R^{d\times d}$, $w_{12}^{KQ}, w_{12}^{PV}, w_{21}^{KQ}, w_{21}^{PV}\in \R^d$, and $w_{22}^{KQ}, w_{22}^{PV}\in \R$.  Then by definition, for an embedding matrix $E$ with $M+1$ columns,
\begin{equation}\label{eqn:simple_expression_prediction}
    \hat y_{\query}(E;\theta) =    \begin{pmatrix}
         (w_{21}^{PV})^\top & w_{22}^{PV}\end{pmatrix}
     \cdot \left(\frac{EE^\top}{M}\right) 
     \begin{pmatrix}
   \WKQ_{11} \\ (w_{21}^{KQ})^\top \end{pmatrix}
  x_\query.
\end{equation}
Due to the lack of positional encoders, the prediction is the same when replacing $(x_k, y_k)$ with $(x_\adv, y_\adv)$ for any $k$, so for notational simplicity of the proof we will consider the case of replacing $(x_1, y_1)$ with $(x_\adv, y_\adv)$.  So, let us consider the embedding corresponding to $(x_\adv, y_\adv, x_2, y_2, \dots, x_M, y_M, x_\query)$, so that
\[ EE^\top = \frac 1 {M}  \begin{pmatrix} 
 x_\adv x_\adv^\top + \summm i 2M x_i x_i^\top + x_{\query} x_{\query}^\top & y_\adv x_\adv + \summm i2M y_i x_i \\
 y_\adv x_\adv^\top + \summm i 2M y_i x_i^\top & y_\adv^2 + \summm i2 M y_i^2 \end{pmatrix}.\]
Expanding, we have
\begin{align*}
\hat y_\query(E;\theta) &= \f {(w_{21}^\PV)^\top}{M} \l( x_\adv x_\adv^\top + \summm i2M x_i x_i^\top + x_{\query} x_{\query}^\top \r) W_{11}^\KQ x_{\query} \\
&\qquad + \f{ (w_{21}^\PV)^\top}{M} \l( y_\adv x_\adv + \summm i2M y_i x_i \r) (w_{21}^\KQ)^\top x_{\query} \\
&\qquad + \f{ w_{22}^\PV}{M} \l( y_\adv x_\adv^\top + \summm i2M y_i x_i^\top \r) W_{11}^\KQ x_{\query} \\
&\qquad + \f {w_{22}^\PV}{M} \l( y_\adv^2 + \summm i2M y_i^2 \r)( w_{21}^\KQ)^\top x_{\query}.
\end{align*}
When training by gradient flow over the population using the initialization of~\cite[Assumption 3.3]{zhang2023trained}, by Lemmas C.1, C.5, and C.6 of~\citep{zhang2023trained} we know that for all times $t\in \R_+\cup \{\infty\}$, it holds that $w_{21}^\PV(t) = w_{12}^\PV(t) = w_{21}^\KQ(t) = 0$ and $\WKQ_{11}(t)\neq 0$ and $w_{22}^{PV}(t)\neq 0$.  In particular, the prediction formula above simplifies to
\begin{align*}\numberthis \label{eq:prediction.formula}
\hat y_\query (E;\theta(t)) &= \f{w_{22}^{PV}(t)}{M} \l( y_\adv x_\adv^\top + \summm i2M y_i x_i^\top \r) \WKQ_{11}(t) x_{\query}.
\end{align*}
For notational simplicity let us denote $W(t) = w_{22}^{\PV}(t) \WKQ_{11}(t)$, so that 
\[ \hat y(E;\theta(t)) = \f 1{M} \l( y_{\adv} x_{\adv}^\top + \summm i2M y_i x_i^\top \r) W(t) x_{\query}.\]
The goal is to take $y_\bad\in \R$ and find $(x_\adv, y_\adv)$ such that $\hat y(E;\theta(t))=y_\bad$.  Rewriting the above equation we see that this is equivalent to finding $(x_\adv, y_\adv)$ such that

\begin{equation} \label{eq:xadv.yadv.preliminary}
y_\adv x_\adv^\top W(t) x_{\query} = M \l( y_\bad - \f 1 {M} \summm i2M y_i x_i^\top W(t) x_{\query} \r).
\end{equation}
From here we see that if $W(t) x_{\query} \neq 0$ then by setting
\begin{equation} \label{eq:xadv.yadv.identity}
x_\adv y_\adv = \f{ MW(t) x_{\query}}{\snorm{W(t) x_{\query}}^2} \cdot \l( y_\bad - \f 1 {M} \summm i2M y_i x_i^\top W(t) x_{\query}\r),
\end{equation}
we guarantee that $\hat y(E;\theta(t))=y_\bad$.  By~\citet[Lemmas A.3 and A.4]{zhang2023trained}, we know $W(t)\neq 0$ for all $t$.  Since $W(t)\neq 0$ and $x_{\query}\sim \normal(0, I)$ is independent of $W(t)$, we know $W(t) x_{\query} \neq 0$ a.s.   Therefore the identity~\eqref{eq:xadv.yadv.identity} suffices for constructing adversarial tokens, and indeed for any choice of $y_{\adv}\neq 0$ this directly allows for constructing $x$-based adversarial tokens,
\begin{equation} \label{eq:xadv.identity}
x_\adv = \f{ MW(t) x_{\query}}{y_\adv \snorm{W(t) x_{\query}}^2} \cdot \l( y_\bad - \f 1 {M} \summm i2M y_i x_i^\top W(t) x_{\query}\r),
\end{equation}

On the other hand, if we want to construct an adversarial token by solely changing the label $y$, we can return to~\eqref{eq:xadv.yadv.preliminary}.  Clearly, as long as $x_{\adv}^\top W(t) x_{\query}\neq 0$, then dividing both sides by this quantity allows for solving $y_\adv$.  If we assume $x_\adv$ is another in-distribution independent $N(0,I)$ sample, then since $W(t)\neq 0$ guarantees that $x_{\adv}^\top W(t) x_{\query} \neq 0$ and so we can construct 
\begin{equation} \label{eq:yadv.identity}
y_\adv = \frac{ M \l( y_\bad - \f 1 {M} \summm i2M y_i x_i^\top W(t) x_{\query} \r)}{x_\adv^\top W(t) x_{\query} }.
\end{equation}

\end{proof}

\clearpage
\section{Additional Results}
In this section, we provide additional results that expand on some of the results given in the main text.
\subsection{Effect of Scale}
\label{appx.sec.effect.of.scale}
We conducted experiments with transformers with different number of layers to evaluate whether scale has any effect on adversarial robustness of the transformer or not.  We observed no meaningful improvement in the adversarial robustness
of the transformers with increase in the number of layers.
This is shown in the figure below for $y_\bad$ chosen
with $\alpha=1$. See Section~\ref{sec.effect.of.scaling.and.seqlen} in the main text for relevant discussion.

\begin{figure}[!h]
    \centering
    \begin{subfigure}[b]{0.45\textwidth}
    \centering
        \includegraphics[width=0.6\textwidth]{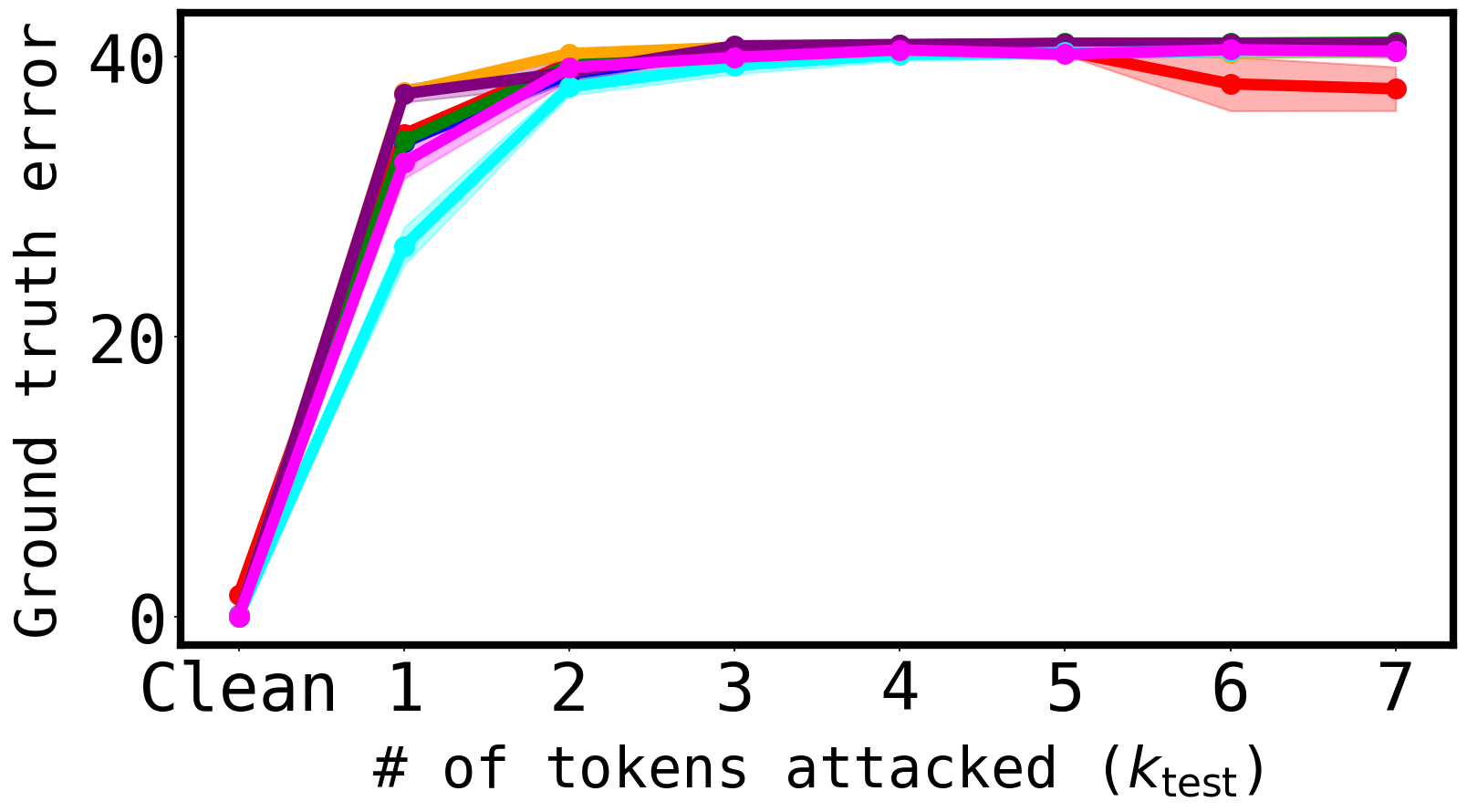}
        \caption{\xattackkk.}
    \end{subfigure}
    \begin{subfigure}[b]{0.45\textwidth}
    \centering
        \includegraphics[width=0.6\textwidth]{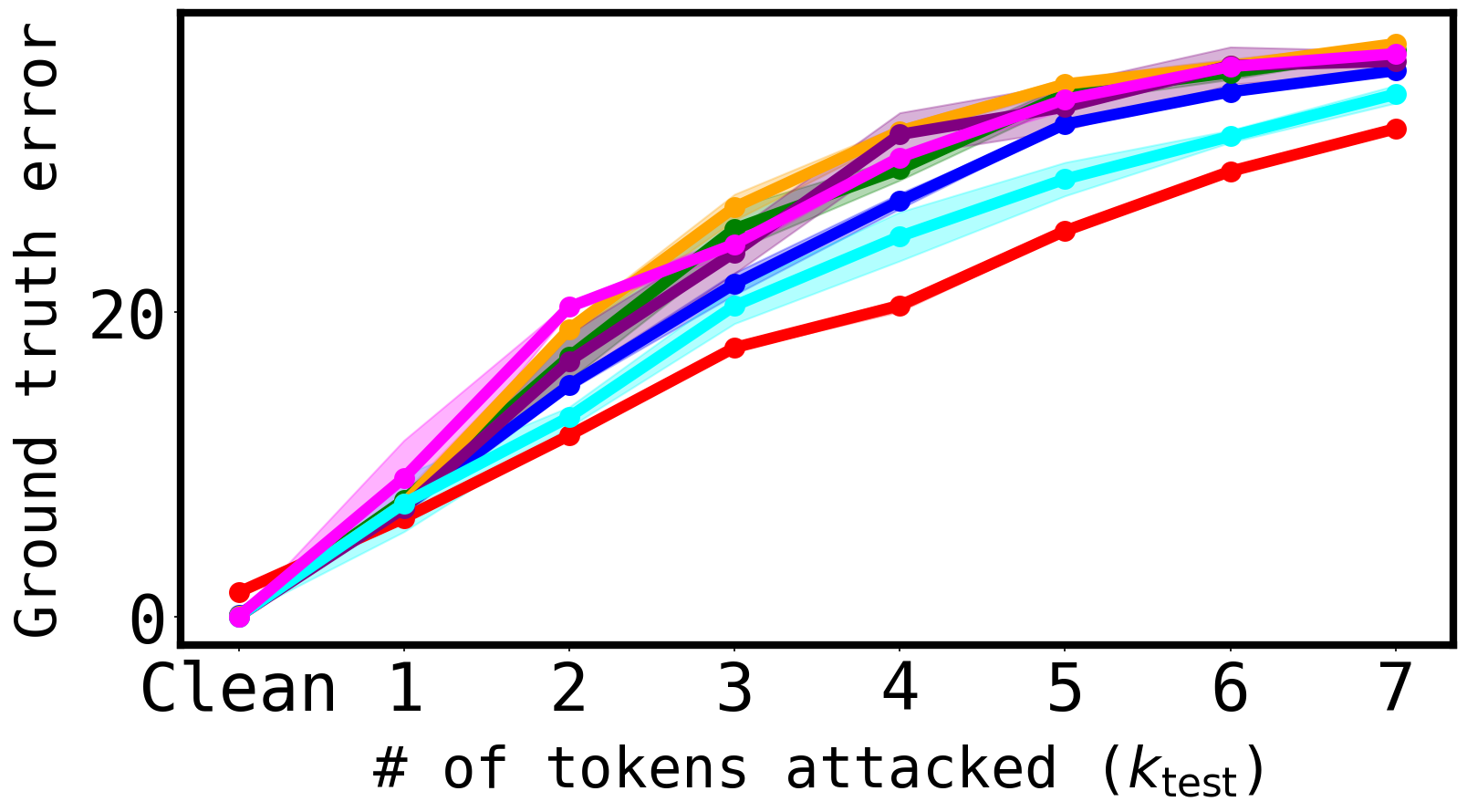}
        \caption{\yattackkk.}
    \end{subfigure}
    \begin{subfigure}[b]{0.45\textwidth}
    \centering
        \includegraphics[width=0.6\textwidth]{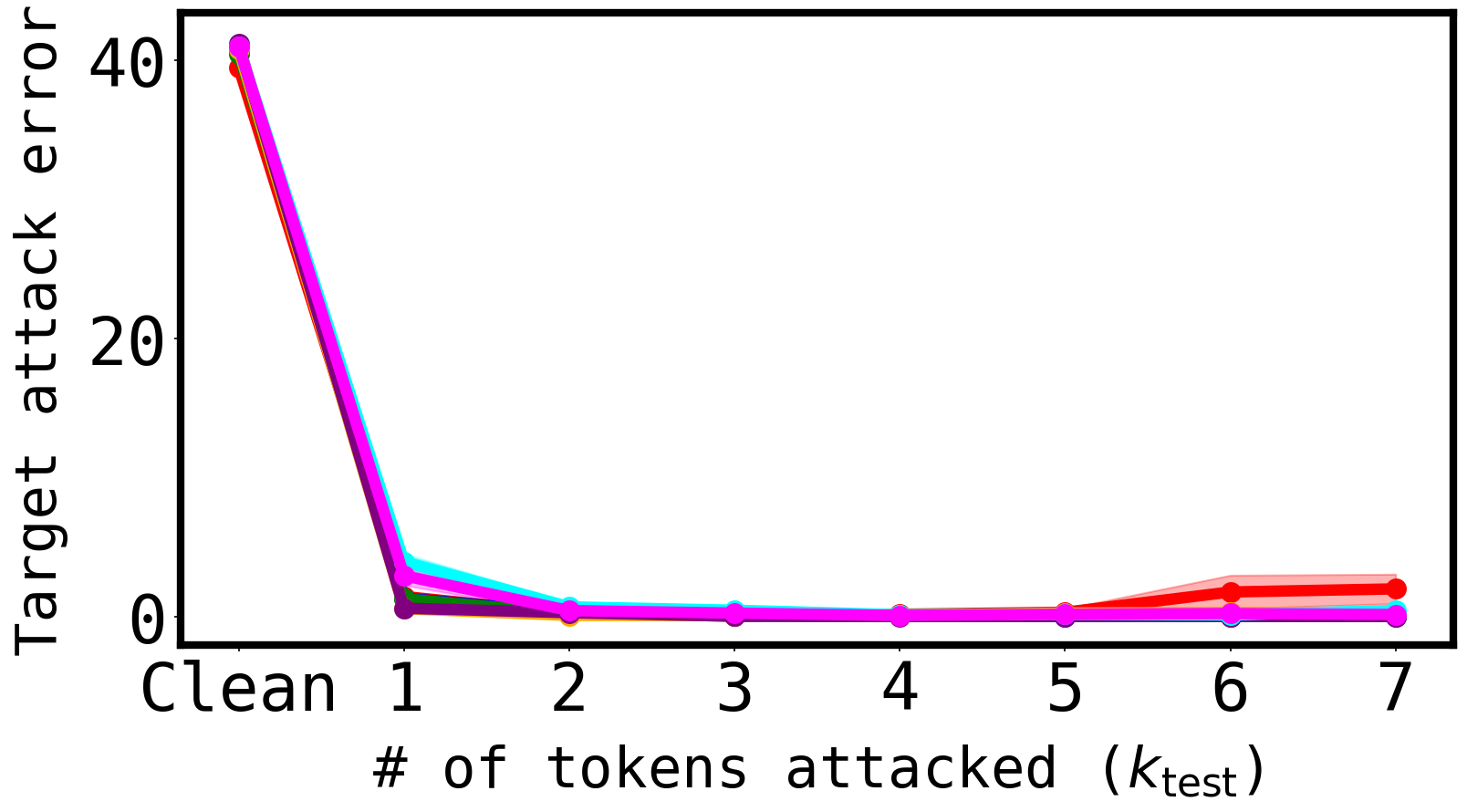}
        \caption{\xattackkk.}
    \end{subfigure}
    \begin{subfigure}[b]{0.45\textwidth}
    \centering
        \includegraphics[width=0.6\textwidth]{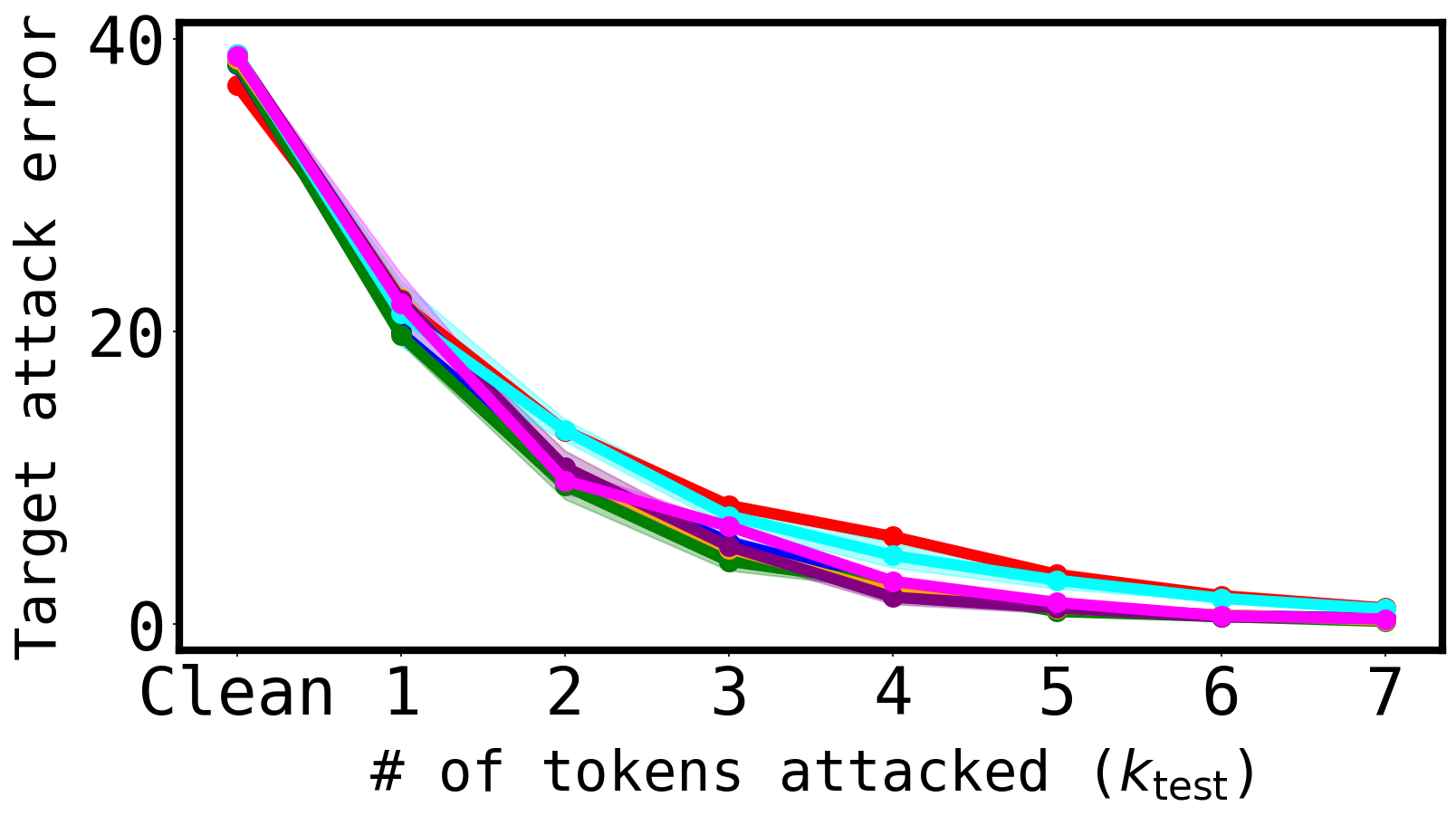}
        \caption{\yattackkk.}
    \end{subfigure}
    \begin{subfigure}[b]{0.6\textwidth}
        \includegraphics[width=\textwidth]{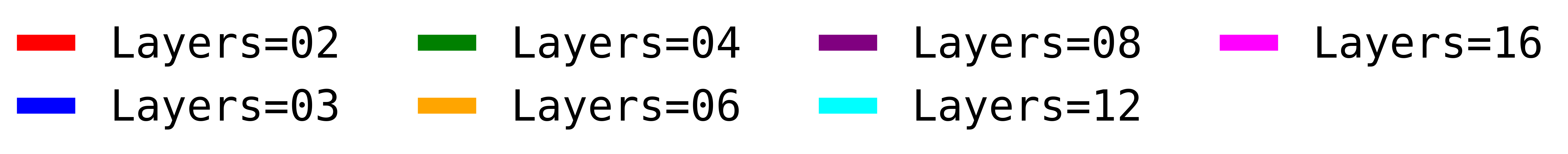}
    \end{subfigure}
    \caption{Increasing the scale of the transformer does not improve the adversarial robustness of in-context learning in transformers.}
    \label{appx.fig.effect.of.scale}
\end{figure}

\ifbool{spencertemp}{\clearpage}{}
\subsection{Effect of Sequence Length}
\label{appx.sec.effect.of.seqlen}
We show here the complete set of results, for both \xattack{} and \yattack{},  on how an increase in sequence length positively impacts adversarial robustness if adversary can manipulate the same number of tokens (for all sequence lengths), but if the adversary can manipulate the same proportion of tokens (which would amount to different number of tokens for different sequence lengths), increase in sequence length has a negligble effect on the adversarial robustness. See Section~\ref{sec.effect.of.scaling.and.seqlen} in the main text for relevant discussion.
\begin{figure}[!h]
    \centering
    \begin{subfigure}[b]{0.45\textwidth}
    \centering
        \includegraphics[width=0.6\textwidth]{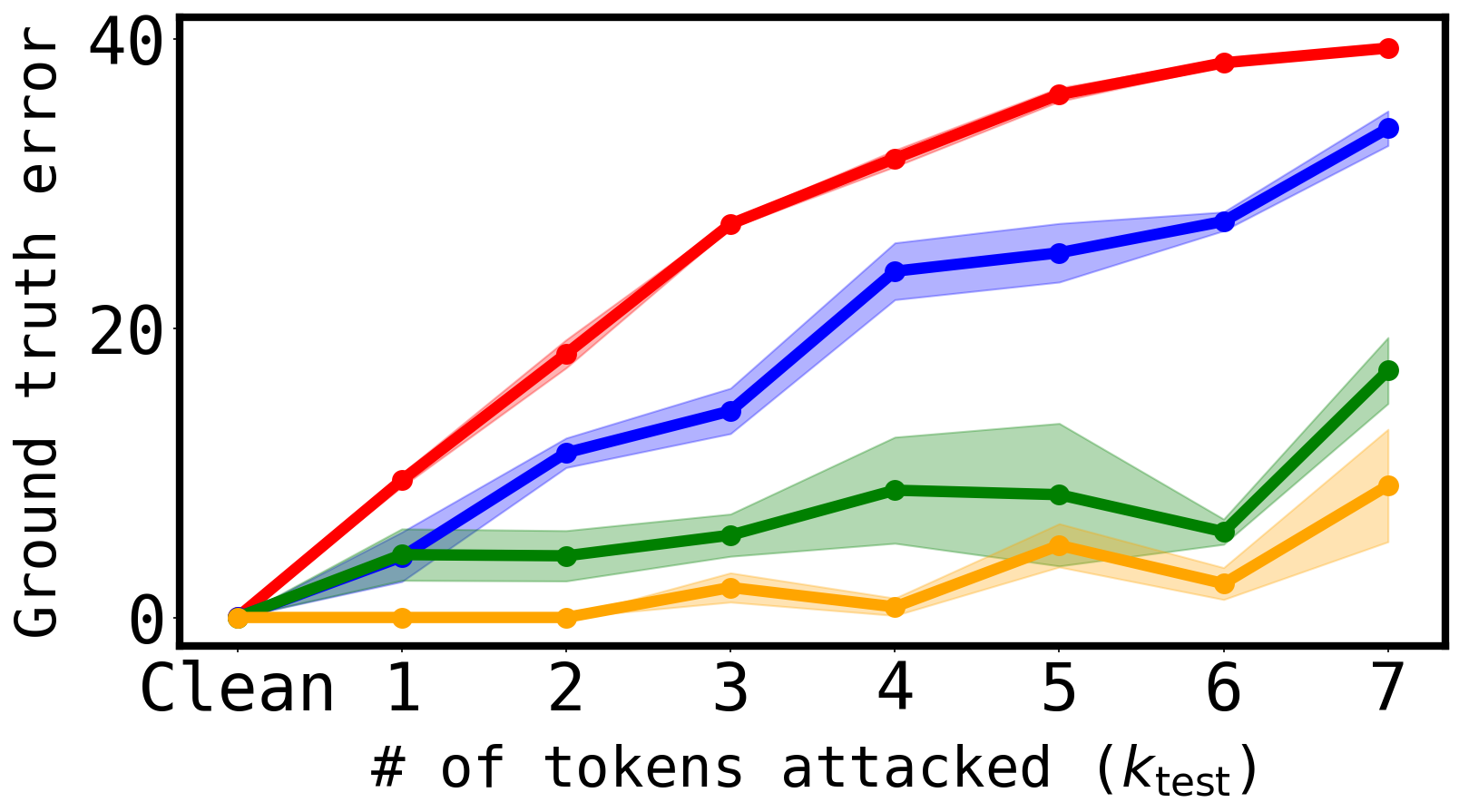}
        \caption{\xattackkk}
    \end{subfigure}
    \begin{subfigure}[b]{0.45\textwidth}
    \centering
        \includegraphics[width=0.6\textwidth]{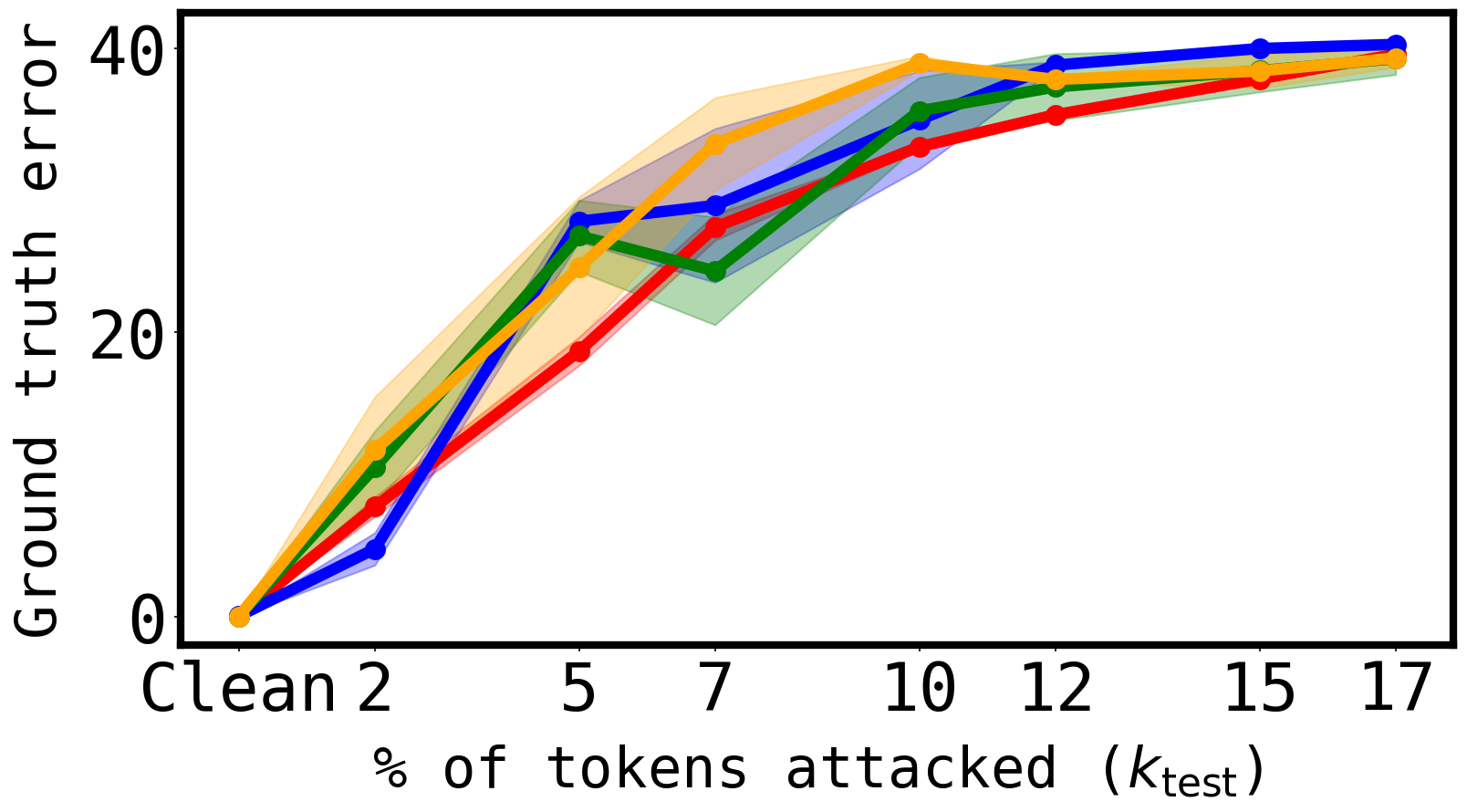}
        \caption{\xattackkk}
    \end{subfigure}
    \begin{subfigure}[b]{0.45\textwidth}
    \centering
        \includegraphics[width=0.6\textwidth]{results/effect_ofseqlen/alpha_1.0/tae_attack_y.png}
        \caption{\yattackkk.}
    \end{subfigure}
    \begin{subfigure}[b]{0.45\textwidth}
    \centering
        \includegraphics[width=0.6\textwidth]{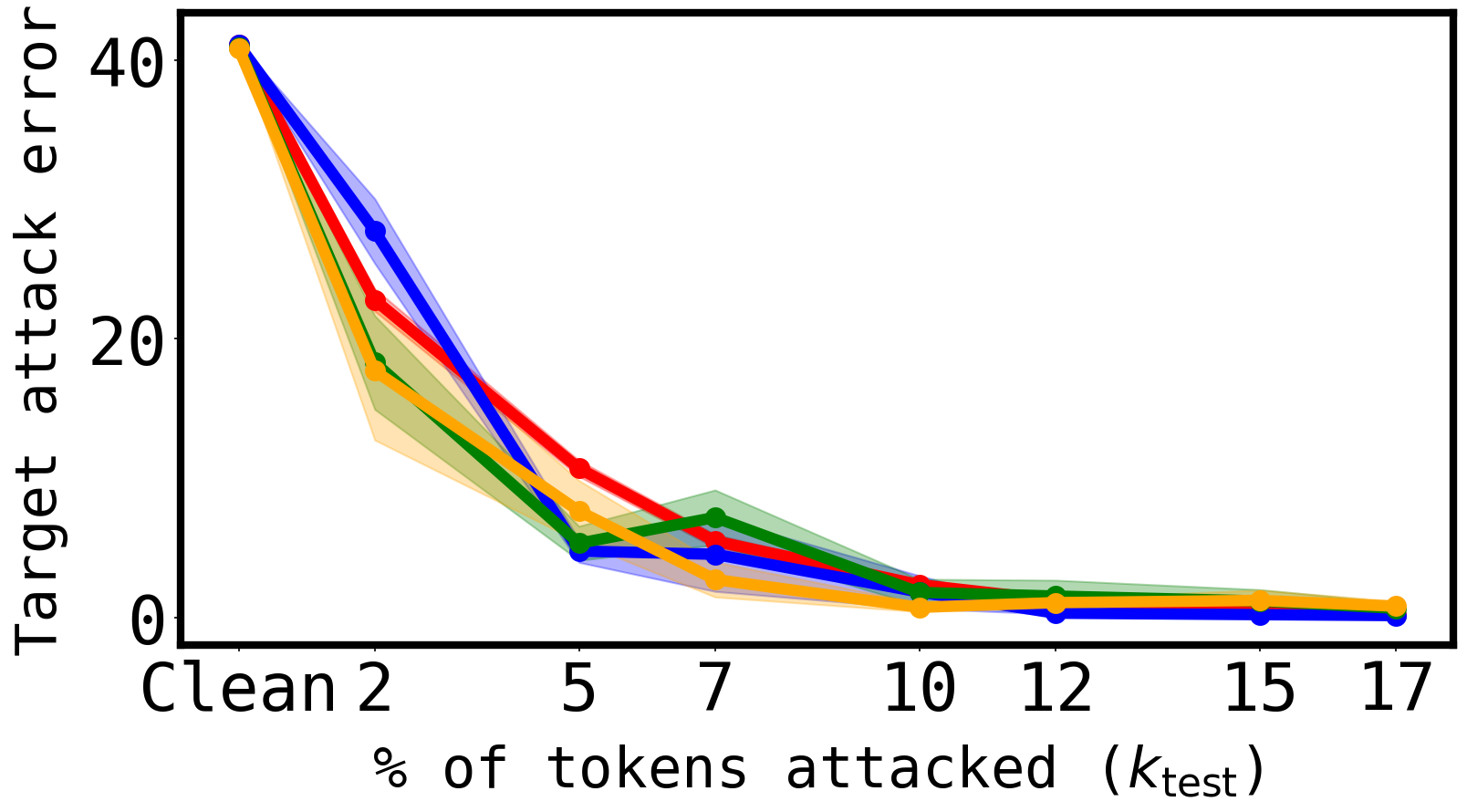}
        \caption{\yattackkk.}
    \end{subfigure}
    \begin{subfigure}[b]{0.98\textwidth}
    \centering
        \fbox{\includegraphics[width=0.9\textwidth]{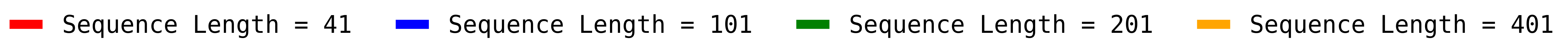}}
    \end{subfigure}
    \caption{Effect of increase in sequence length.}
    \label{appx.fig.effect.of.seqlen.y.attack}
\end{figure}

\ifbool{iclrtemp}{\clearpage}{}
\ifbool{tmlrtemp}{\FloatBarrier}{}
\ifbool{spencertemp}{\FloatBarrier}{}
\subsection{Gradient-Based Adversarial Attacks \& Adversarial Training}
\label{appx.sec.adversarial.training}
In the main text (in Sections~\ref{subsec:gradient.attack}~and~\ref{sec.adversarial.training}), we gave results for attacks performed with $y_\bad$ chosen by setting $\alpha=1$ in equation~\ref{eq:y_bad.alpha}.
Here, we present results for $\alpha=0.5$ and $\alpha=0.1$.
These results are qualitatively similar to the case of $\alpha=1$ and are presented only for completeness.
Furthermore, in the main text, we showed only target attack error for our attacks due to space constraints, while here we present results for both ground truth error and target attack error.

\ifbool{spencertemp}{\clearpage}{}

\subsubsection{$\alpha=1.0$}
\begin{figure}[h]
    \centering
    \begin{subfigure}[b]{0.32\textwidth}
        \includegraphics[width=\textwidth]{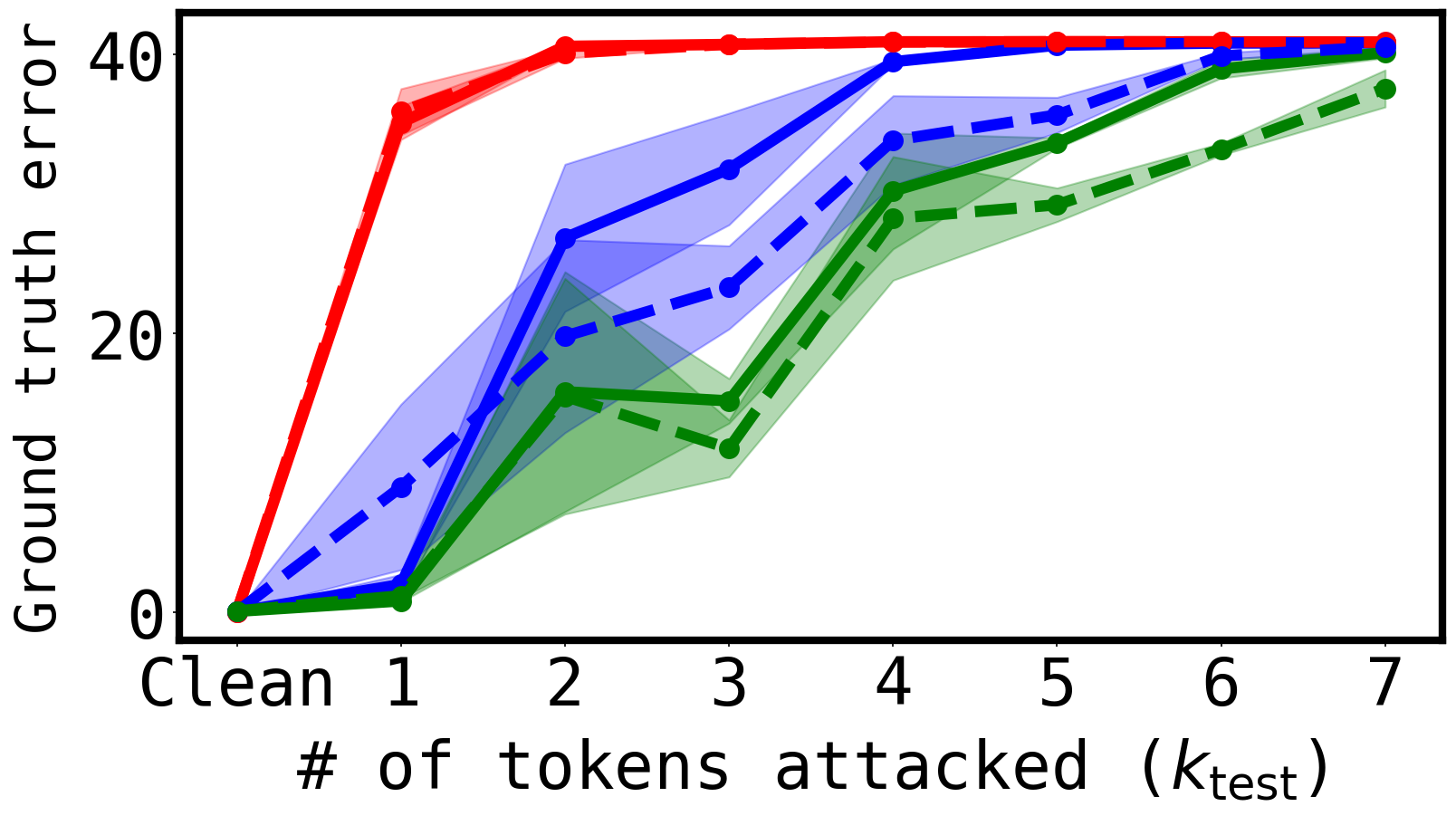}
        \caption{\xattackkk.}
    \end{subfigure}
    \begin{subfigure}[b]{0.32\textwidth}
        \includegraphics[width=\textwidth]{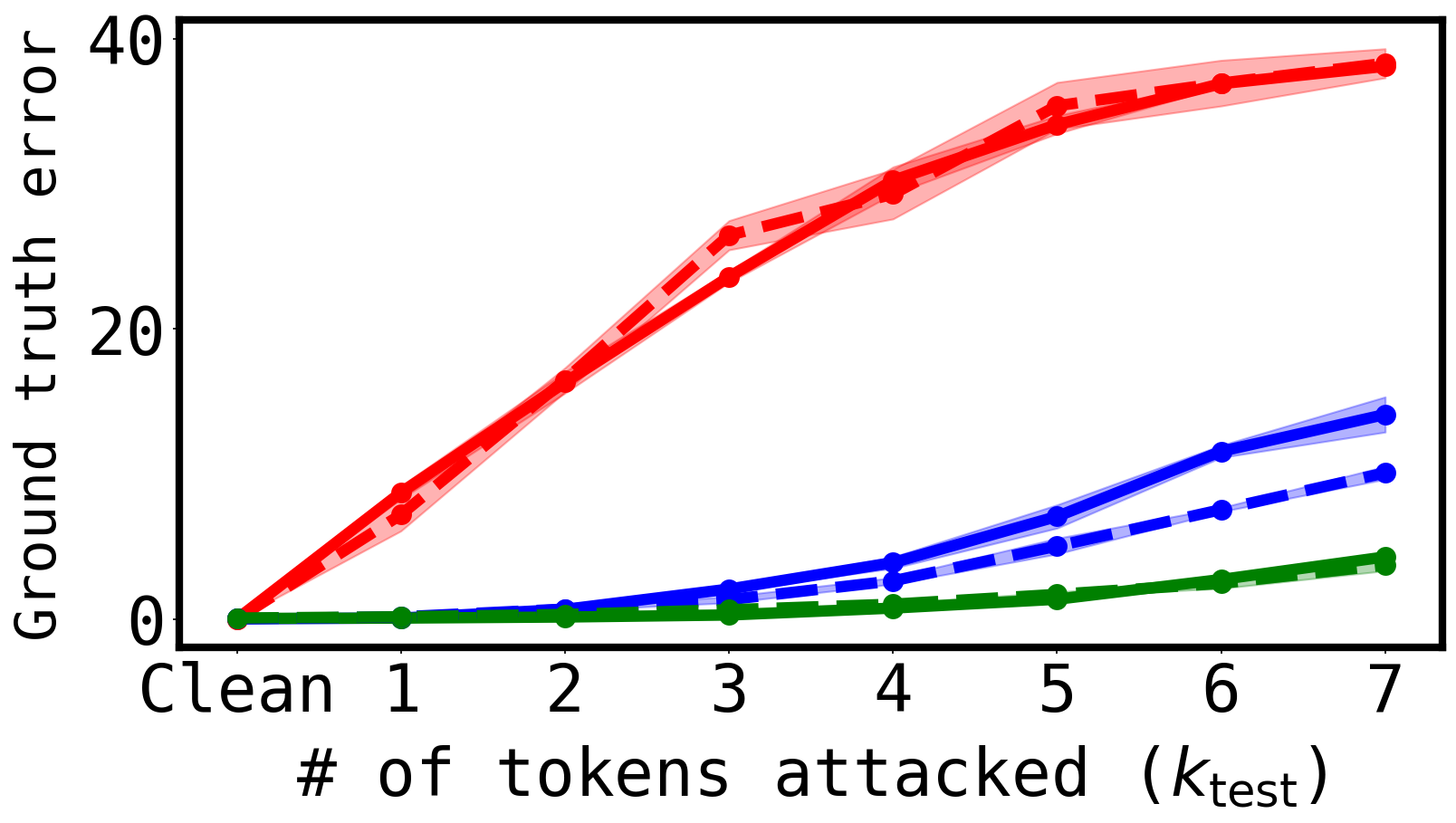}
        \caption{\yattackkk.}
    \end{subfigure}
    \begin{subfigure}[b]{0.32\textwidth}
        \includegraphics[width=\textwidth]{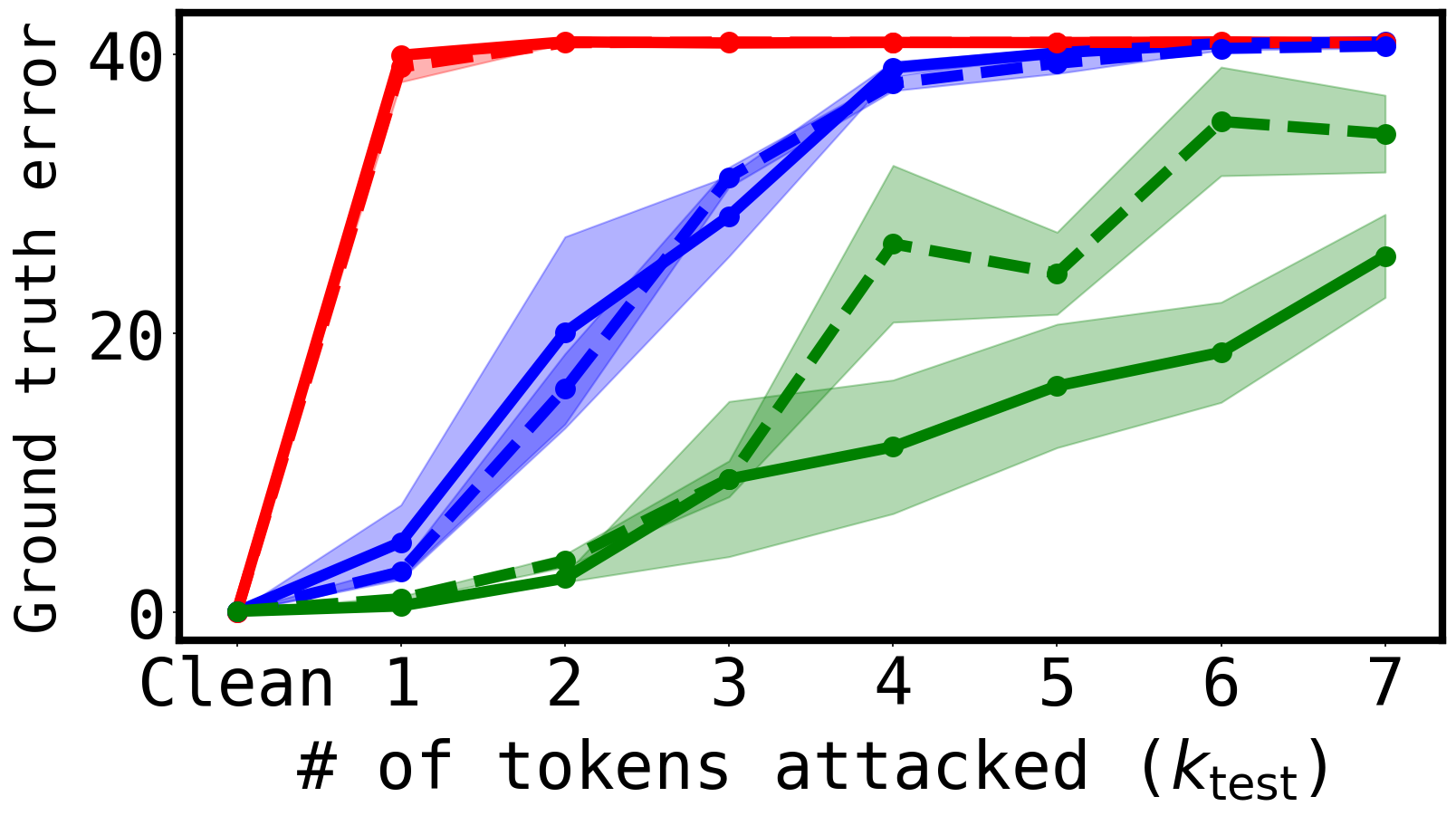}
        \caption{\zattackkk.}
    \end{subfigure}
    \begin{subfigure}[b]{0.32\textwidth}
        \includegraphics[width=\textwidth]{results/adv_training/alpha_1.0/tae_training_y_attack_x.png}
        \caption{\xattackkk.}
        \label{subfig:x.attack.y.advtraining}
    \end{subfigure}
        \begin{subfigure}[b]{0.32\textwidth}
        \includegraphics[width=\textwidth]{results/adv_training/alpha_1.0/tae_training_y_attack_y.png}
        \caption{\yattackkk.}
    \end{subfigure}
        \begin{subfigure}[b]{0.32\textwidth}
        \includegraphics[width=\textwidth]{results/adv_training/alpha_1.0/tae_training_y_attack_z.png}
        \caption{\zattackkk.}
    \end{subfigure}
    \begin{subfigure}[b]{0.98\textwidth}
        \includegraphics[width=\textwidth]{results/adv_training/adv_training_legend.png}
    \end{subfigure}
    \caption{Adversarial training against \yattack{}. A-PT denotes adversarial
    pretraining and A-FT denotes adversarial finetuning. $k_\text{train}$ denotes
    the number of tokens attacked during training and $k_\text{train}=0$ corresponds
    to a model that has not undergone adversarial training at all.}
    \label{appx.fig:adv.training.y.alpha.1}
\end{figure}
\begin{figure}[!h]
    \centering
    \begin{subfigure}[b]{0.32\textwidth}
        \includegraphics[width=\textwidth]{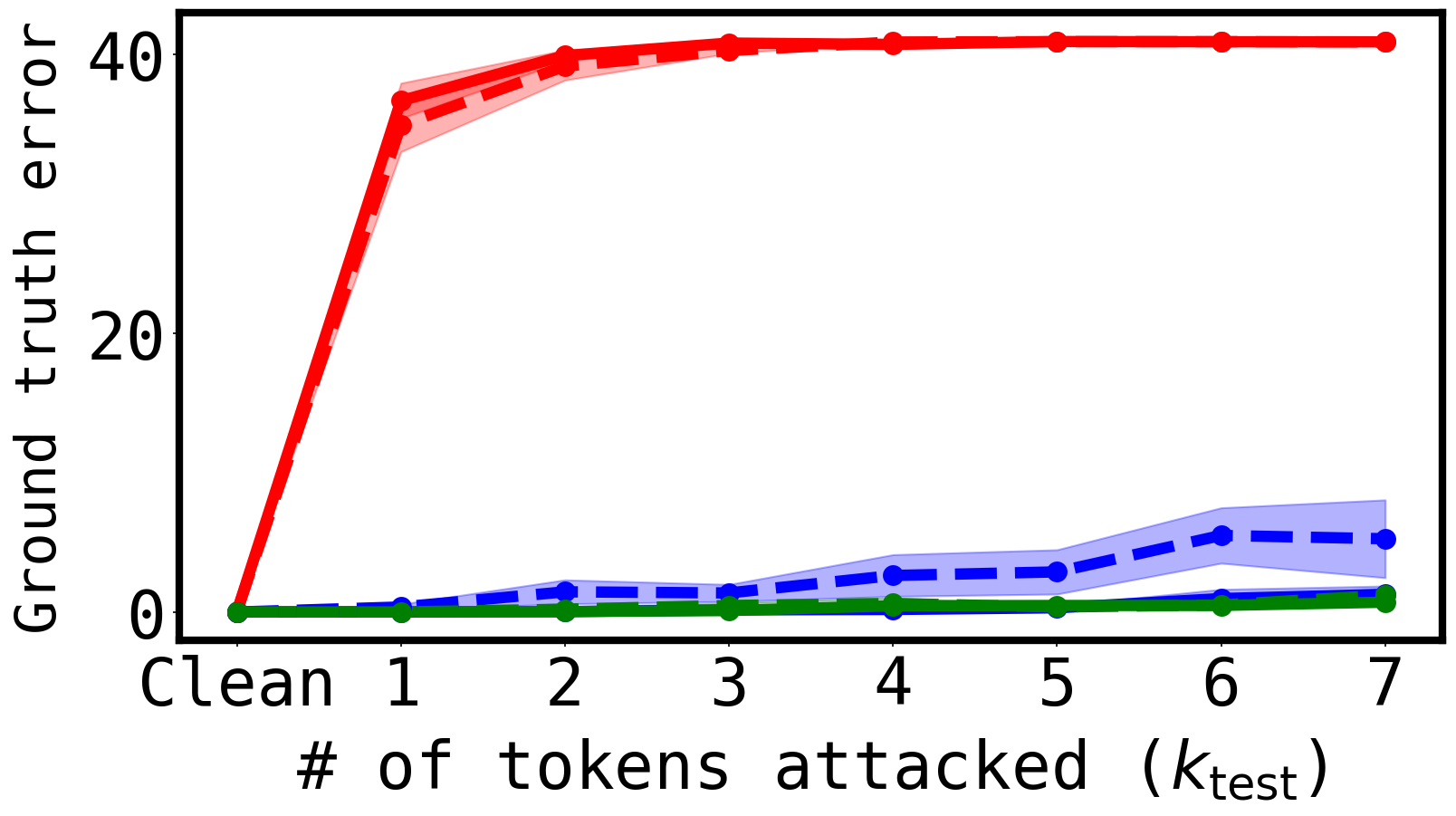}
        \caption{\xattackkk.}
    \end{subfigure}
    \begin{subfigure}[b]{0.32\textwidth}
        \includegraphics[width=\textwidth]{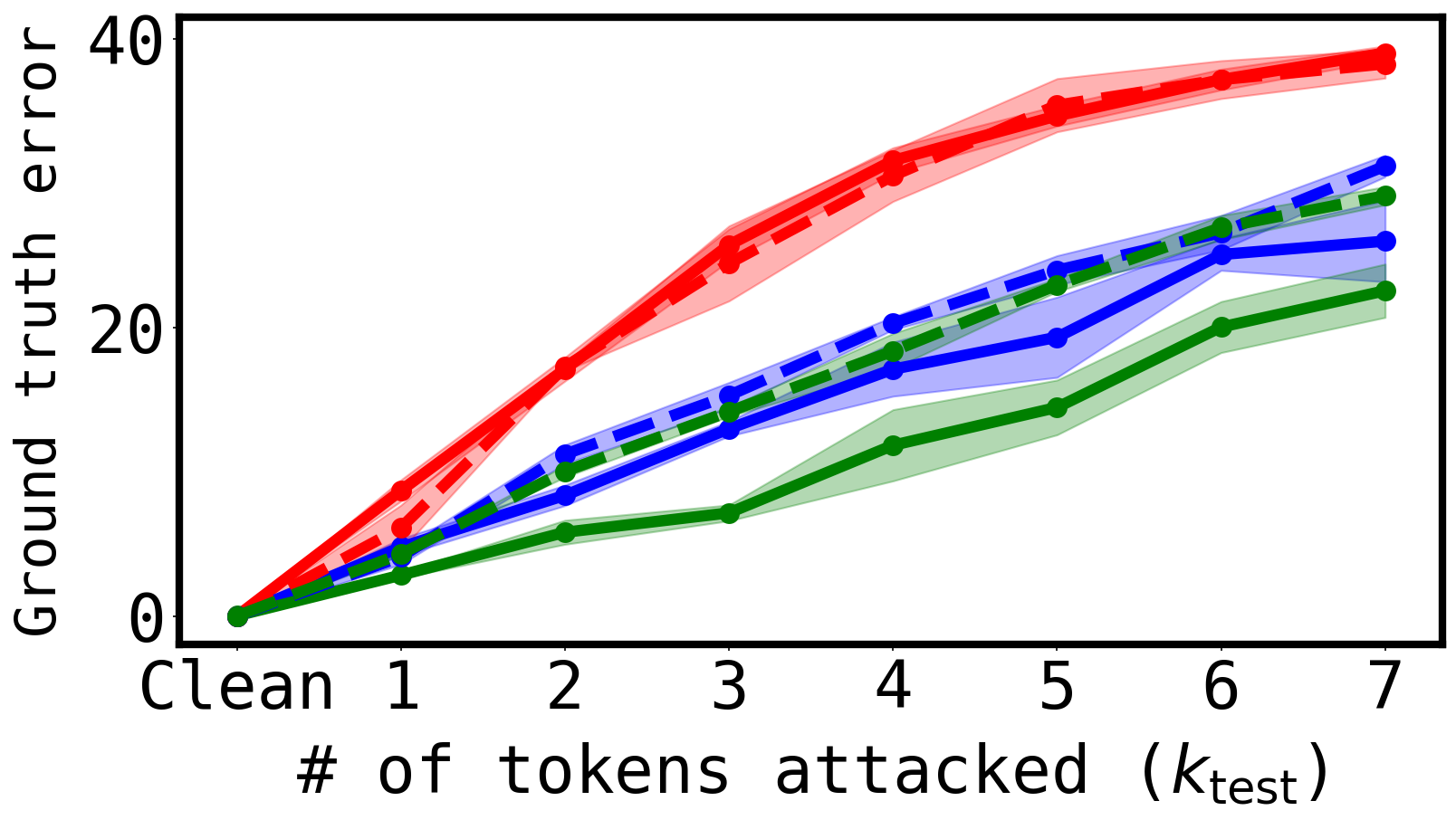}
        \caption{\yattackkk.}
    \end{subfigure}
    \begin{subfigure}[b]{0.32\textwidth}
        \includegraphics[width=\textwidth]{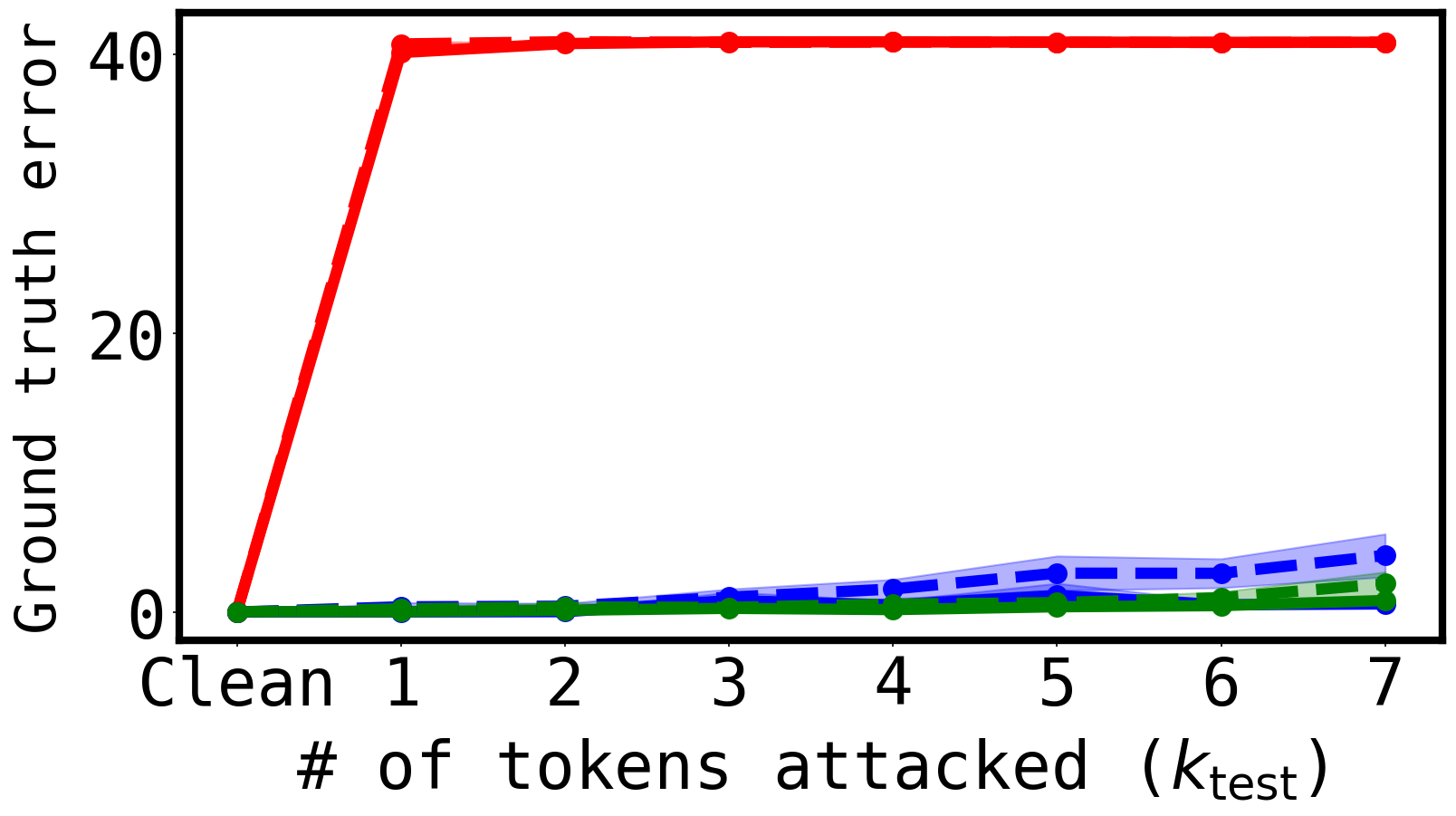}
        \caption{\zattackkk.}
    \end{subfigure}
    \begin{subfigure}[b]{0.32\textwidth}
        \includegraphics[width=\textwidth]{results/adv_training/alpha_1.0/tae_training_x_attack_x.png}
        \caption{\xattackkk.}
    \end{subfigure}
        \begin{subfigure}[b]{0.32\textwidth}
        \includegraphics[width=\textwidth]{results/adv_training/alpha_1.0/tae_training_x_attack_y.png}
        \caption{\yattackkk.}
    \end{subfigure}
        \begin{subfigure}[b]{0.32\textwidth}
        \includegraphics[width=\textwidth]{results/adv_training/alpha_1.0/tae_training_x_attack_z.png}
        \caption{\zattackkk.}
    \end{subfigure}
    \begin{subfigure}[b]{0.98\textwidth}
        \includegraphics[width=\textwidth]{results/adv_training/adv_training_legend.png}
    \end{subfigure}
    \caption{Adversarial training against \xattack.}
    \label{appx.fig:adv.training.x.alpha.1}
\end{figure}
\begin{figure}[!h]
    \centering
    \begin{subfigure}[b]{0.32\textwidth}
        \includegraphics[width=\textwidth]{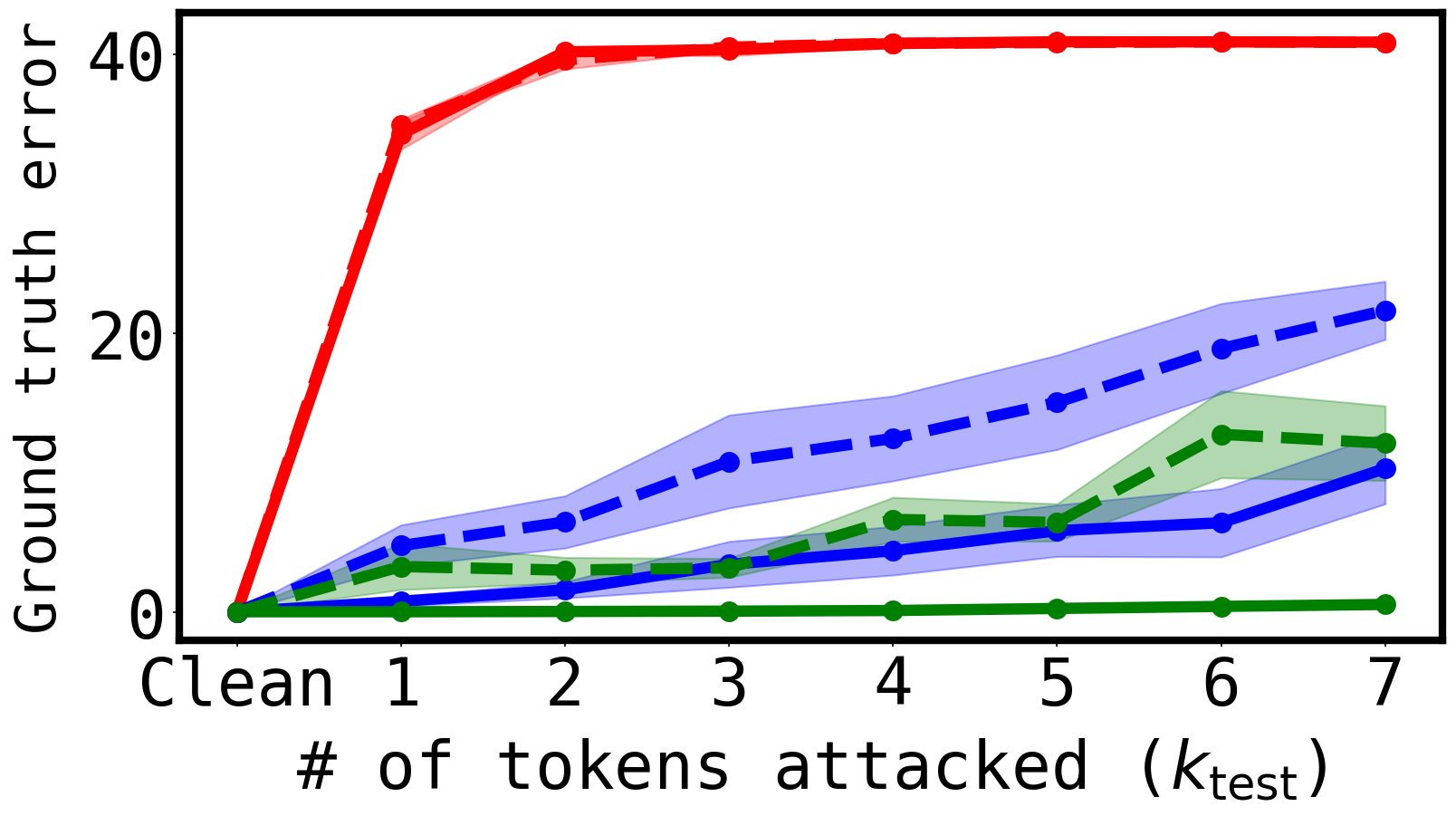}
        \caption{\xattackkk.}
    \end{subfigure}
    \begin{subfigure}[b]{0.32\textwidth}
        \includegraphics[width=\textwidth]{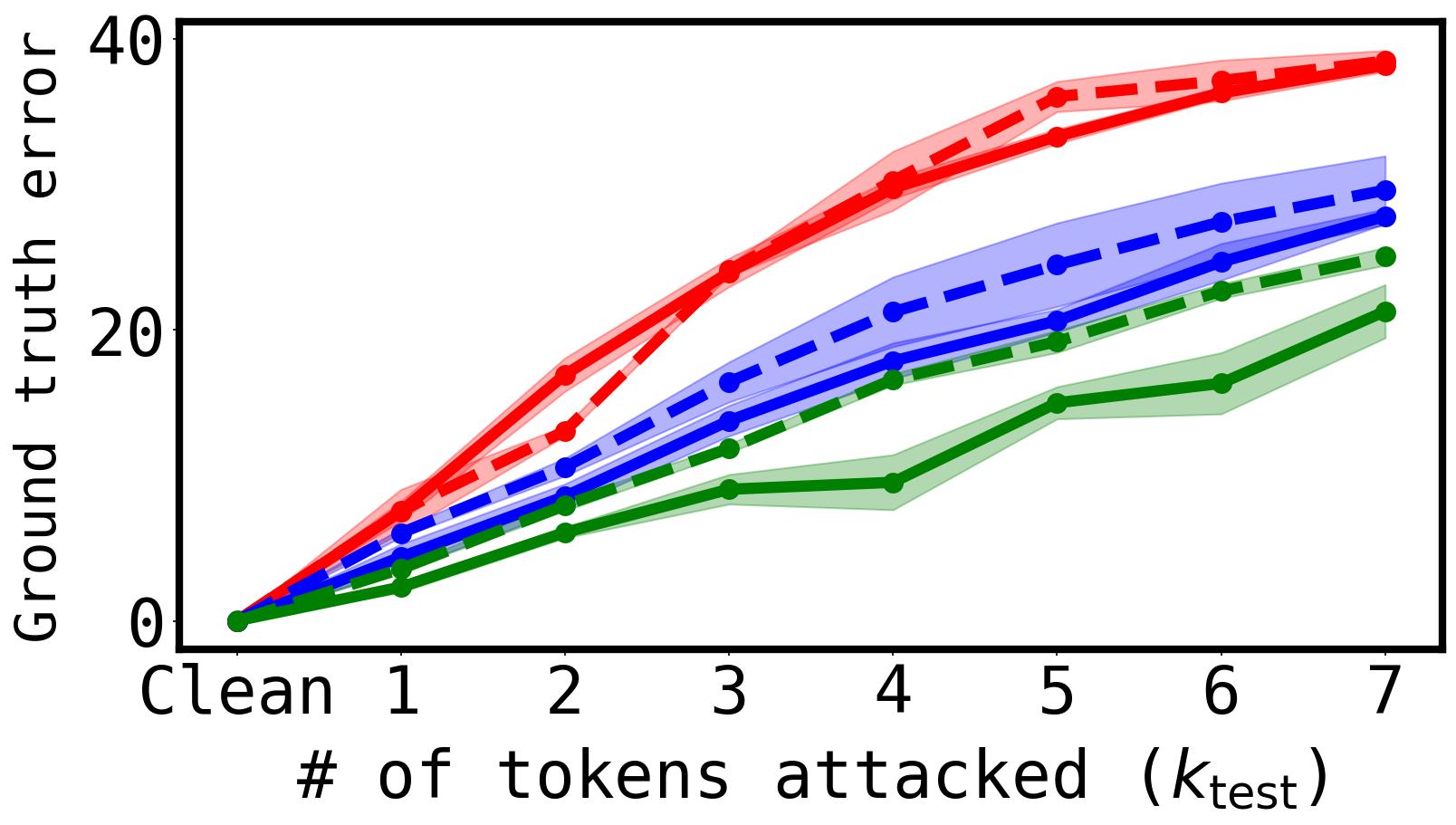}
        \caption{\yattackkk.}
    \end{subfigure}
    \begin{subfigure}[b]{0.32\textwidth}
        \includegraphics[width=\textwidth]{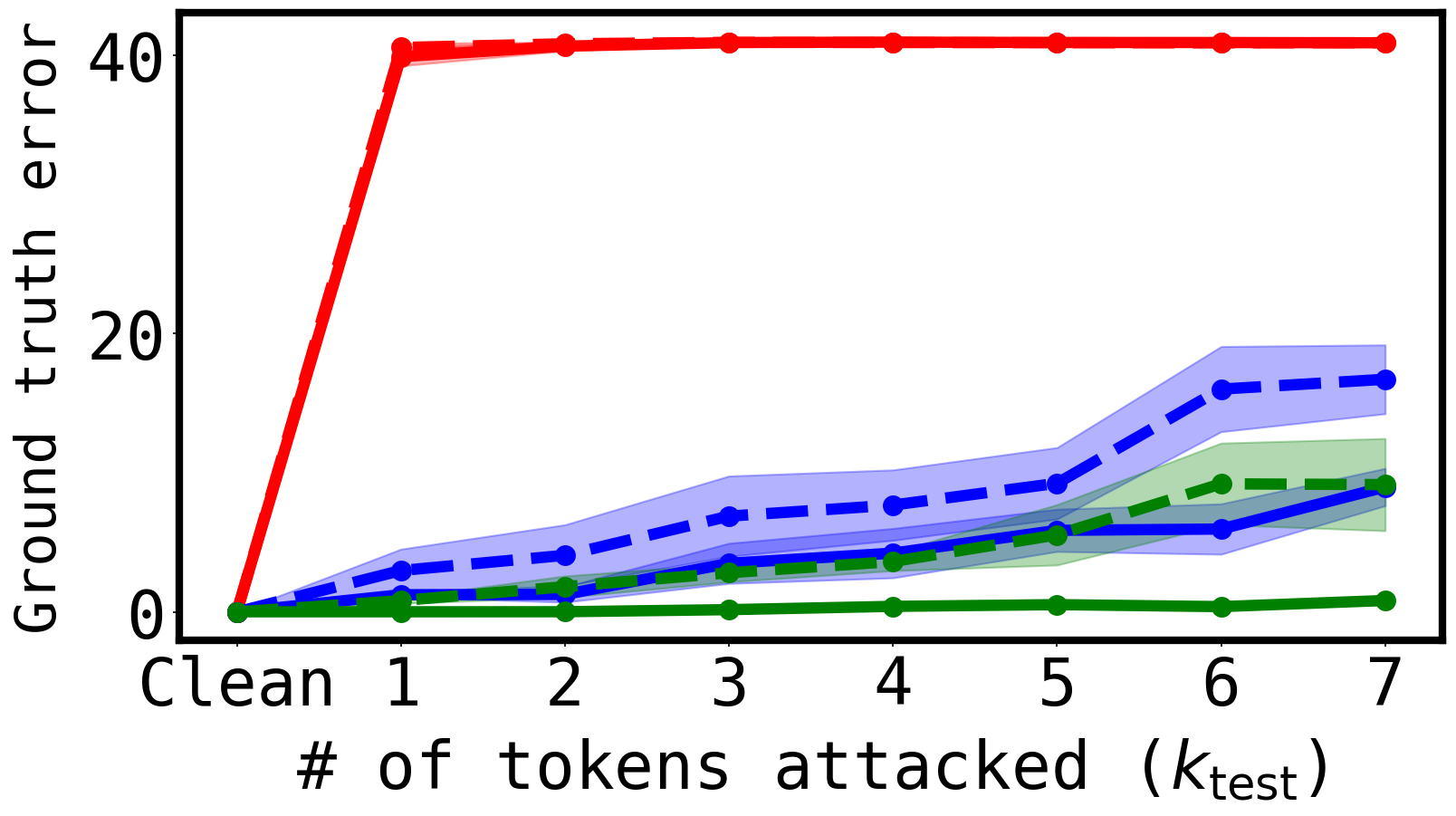}
        \caption{\zattackkk.}
    \end{subfigure}
    \begin{subfigure}[b]{0.32\textwidth}
        \includegraphics[width=\textwidth]{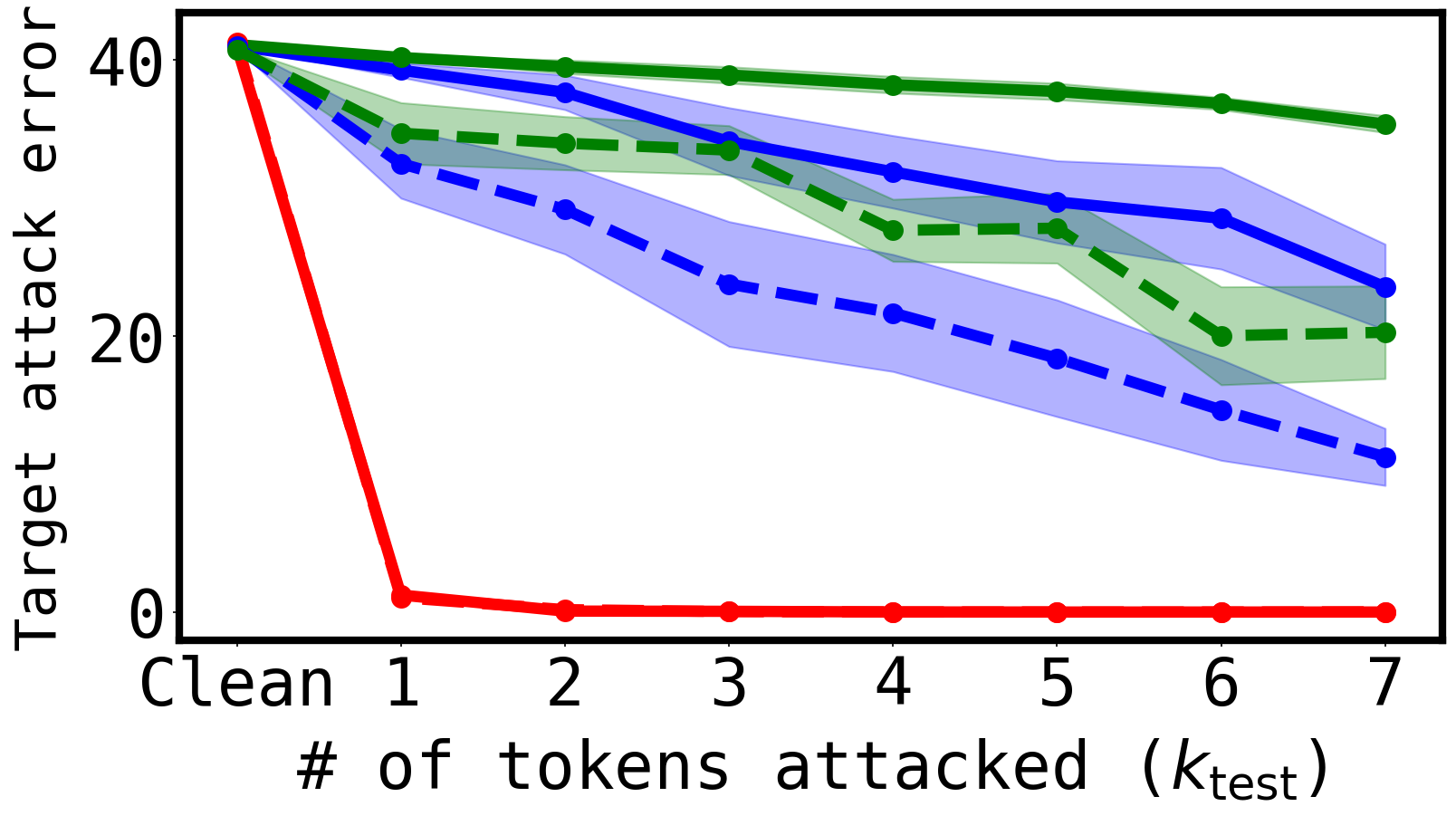}
        \caption{\xattackkk.}
    \end{subfigure}
        \begin{subfigure}[b]{0.32\textwidth}
        \includegraphics[width=\textwidth]{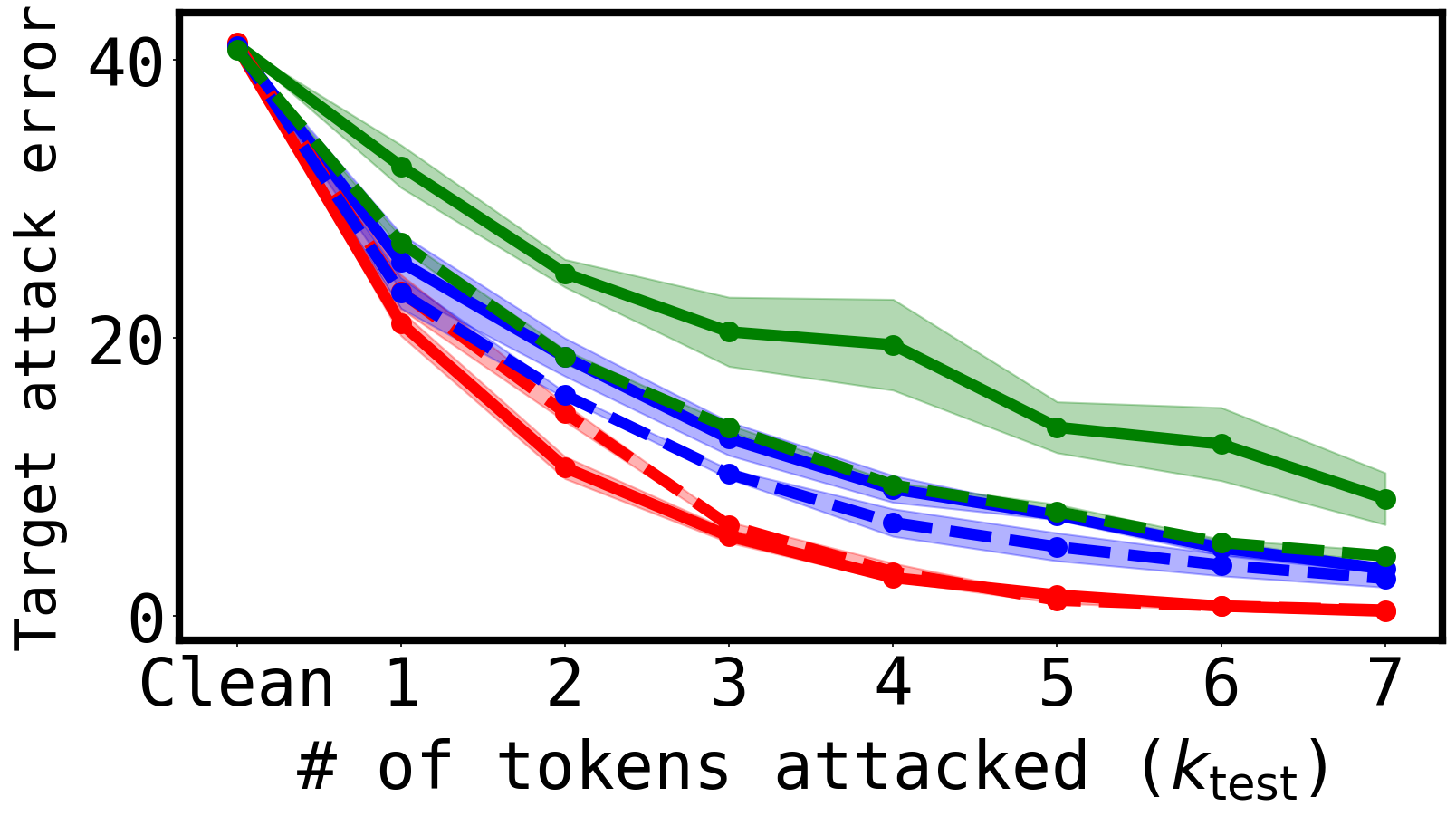}
        \caption{\yattackkk.}
    \end{subfigure}
        \begin{subfigure}[b]{0.32\textwidth}
        \includegraphics[width=\textwidth]{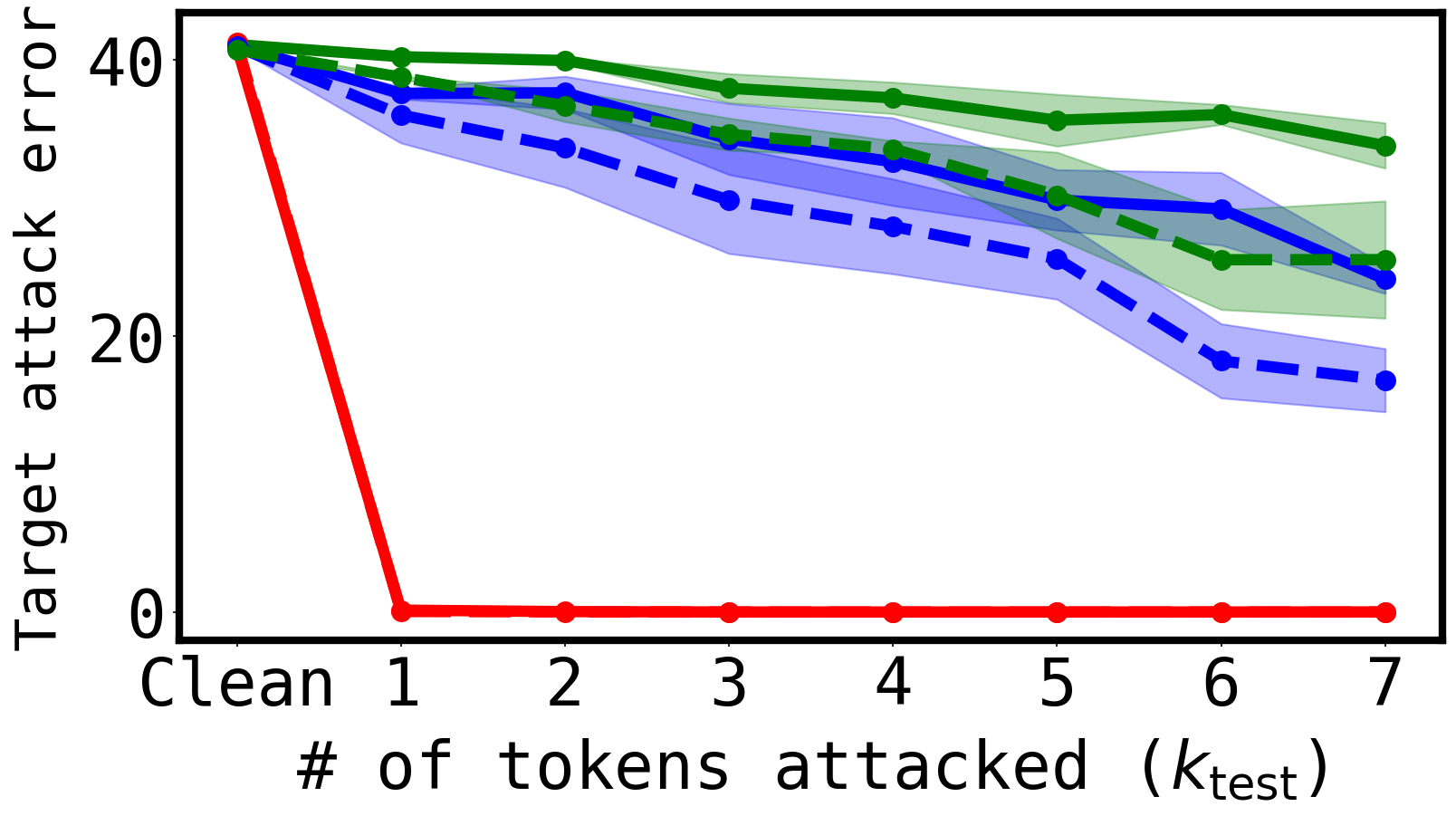}
        \caption{\zattackkk.}
    \end{subfigure}
    \begin{subfigure}[b]{0.98\textwidth}
        \includegraphics[width=\textwidth]{results/adv_training/adv_training_legend.png}
    \end{subfigure}
    \caption{Adversarial training against \zattack.}
    \label{appx.fig:adv.training.z.alpha.1}
\end{figure}

\FloatBarrier
\ifbool{tmlrtemp}{\clearpage}{}
\subsubsection{$\alpha=0.5$}
\begin{figure}[!h]
    \centering
    \begin{subfigure}[b]{0.32\textwidth}
        \includegraphics[width=\textwidth]{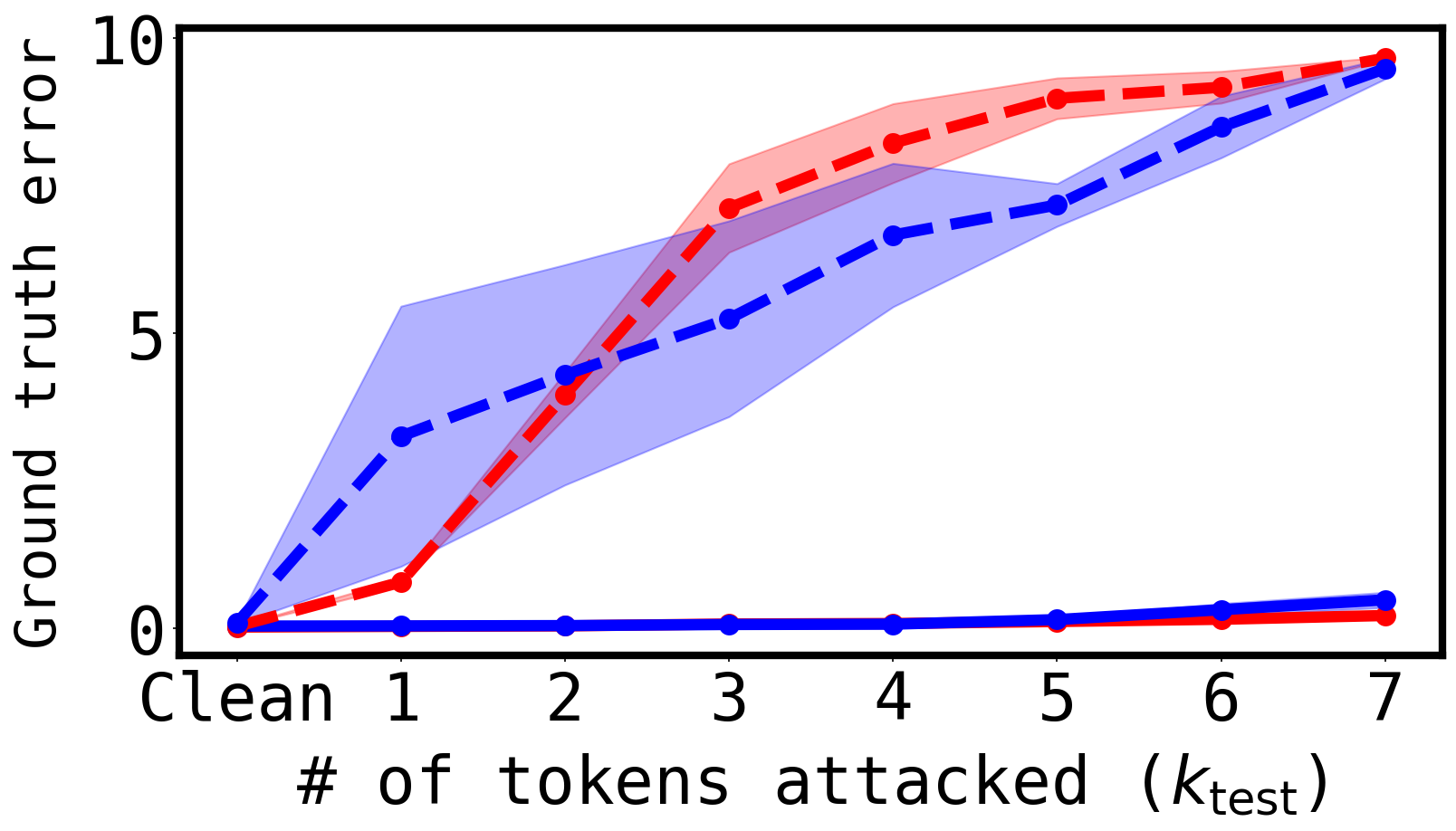}
        \caption{\xattackkk.}
    \end{subfigure}
    \begin{subfigure}[b]{0.32\textwidth}
        \includegraphics[width=\textwidth]{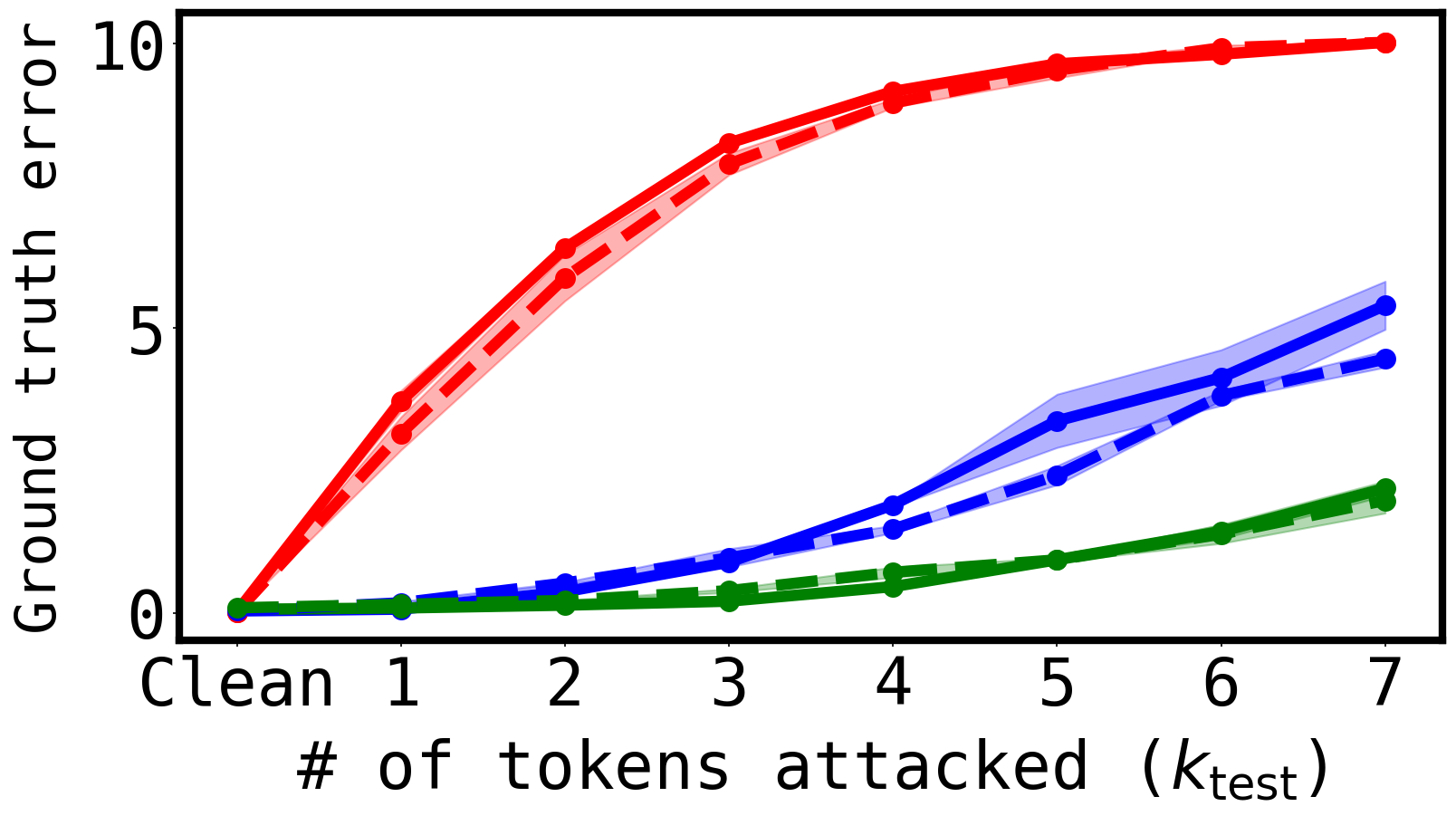}
        \caption{\yattackkk.}
    \end{subfigure}
    \begin{subfigure}[b]{0.32\textwidth}
        \includegraphics[width=\textwidth]{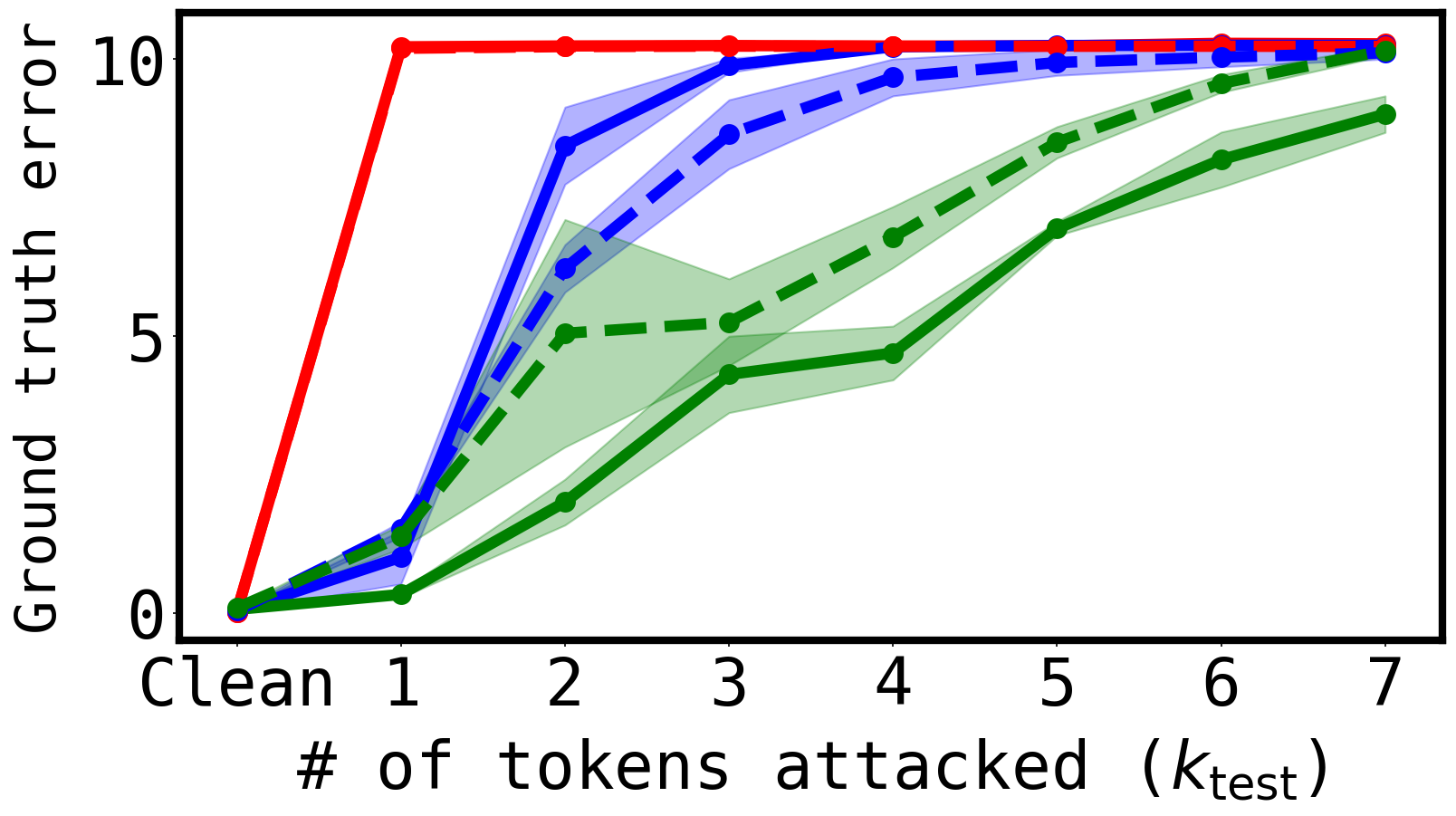}
        \caption{\zattackkk.}
    \end{subfigure}
    \begin{subfigure}[b]{0.32\textwidth}
        \includegraphics[width=\textwidth]{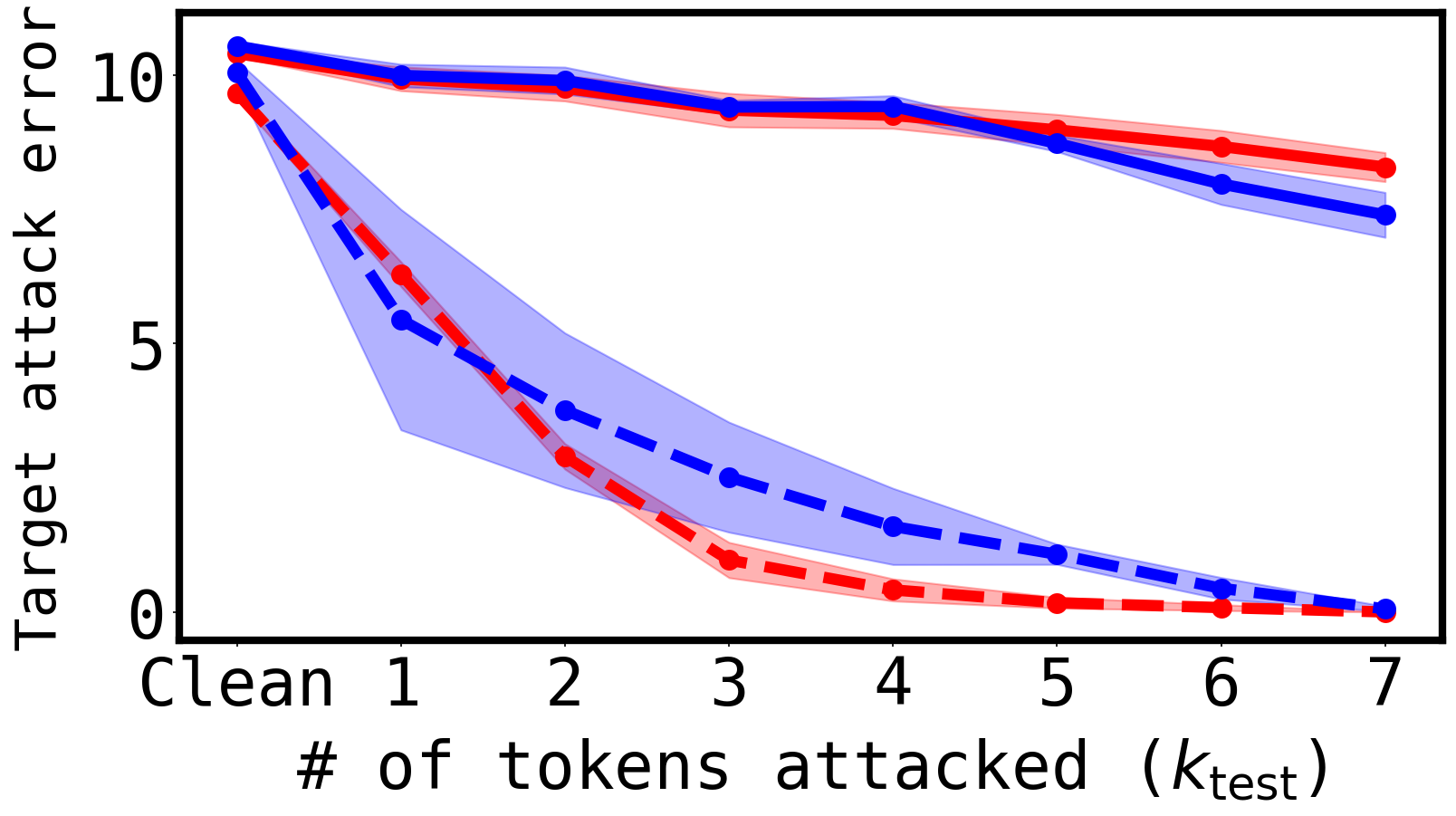}
        \caption{\xattackkk.}
    \end{subfigure}
        \begin{subfigure}[b]{0.32\textwidth}
        \includegraphics[width=\textwidth]{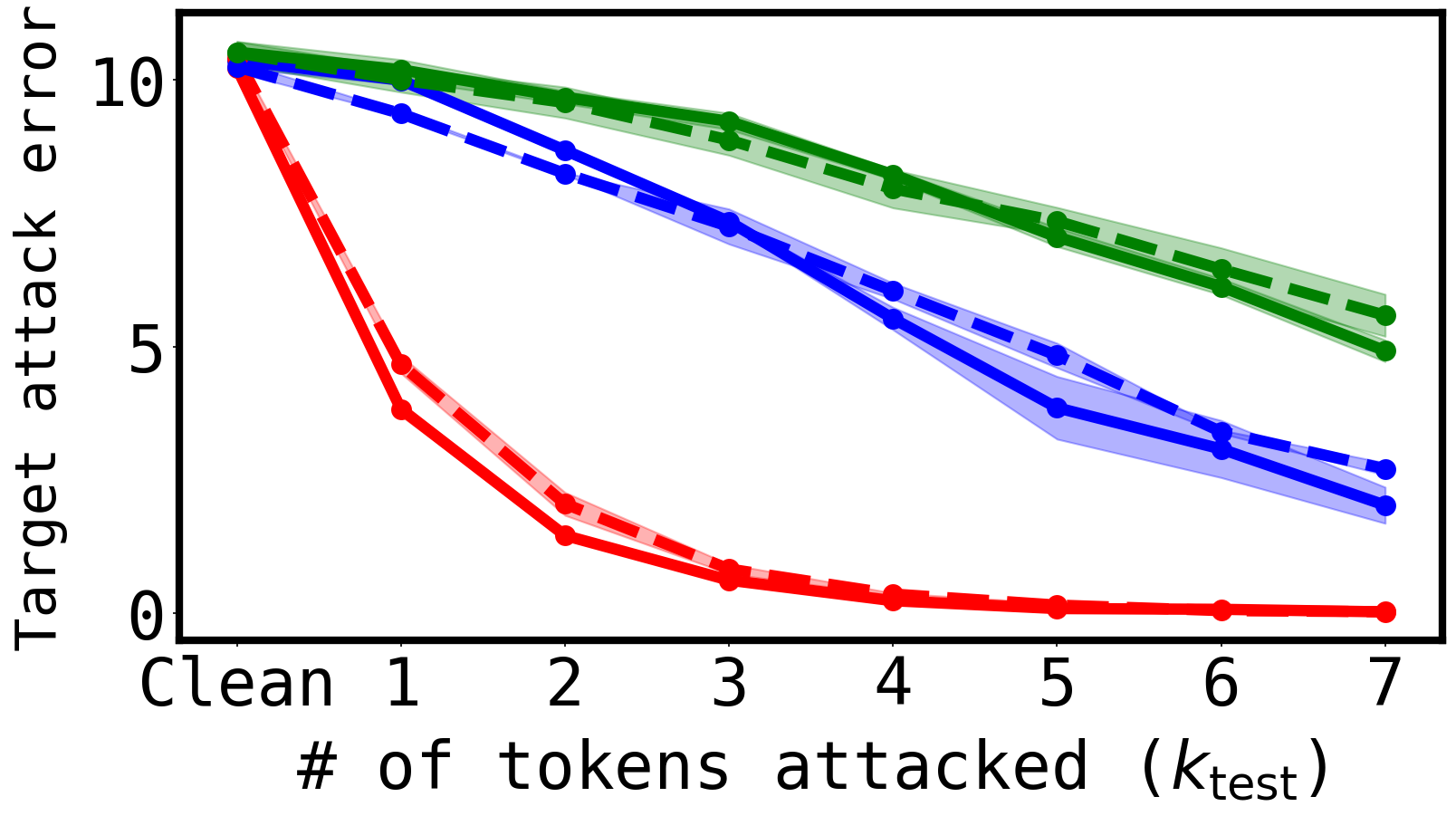}
        \caption{\yattackkk.}
    \end{subfigure}
        \begin{subfigure}[b]{0.32\textwidth}
        \includegraphics[width=\textwidth]{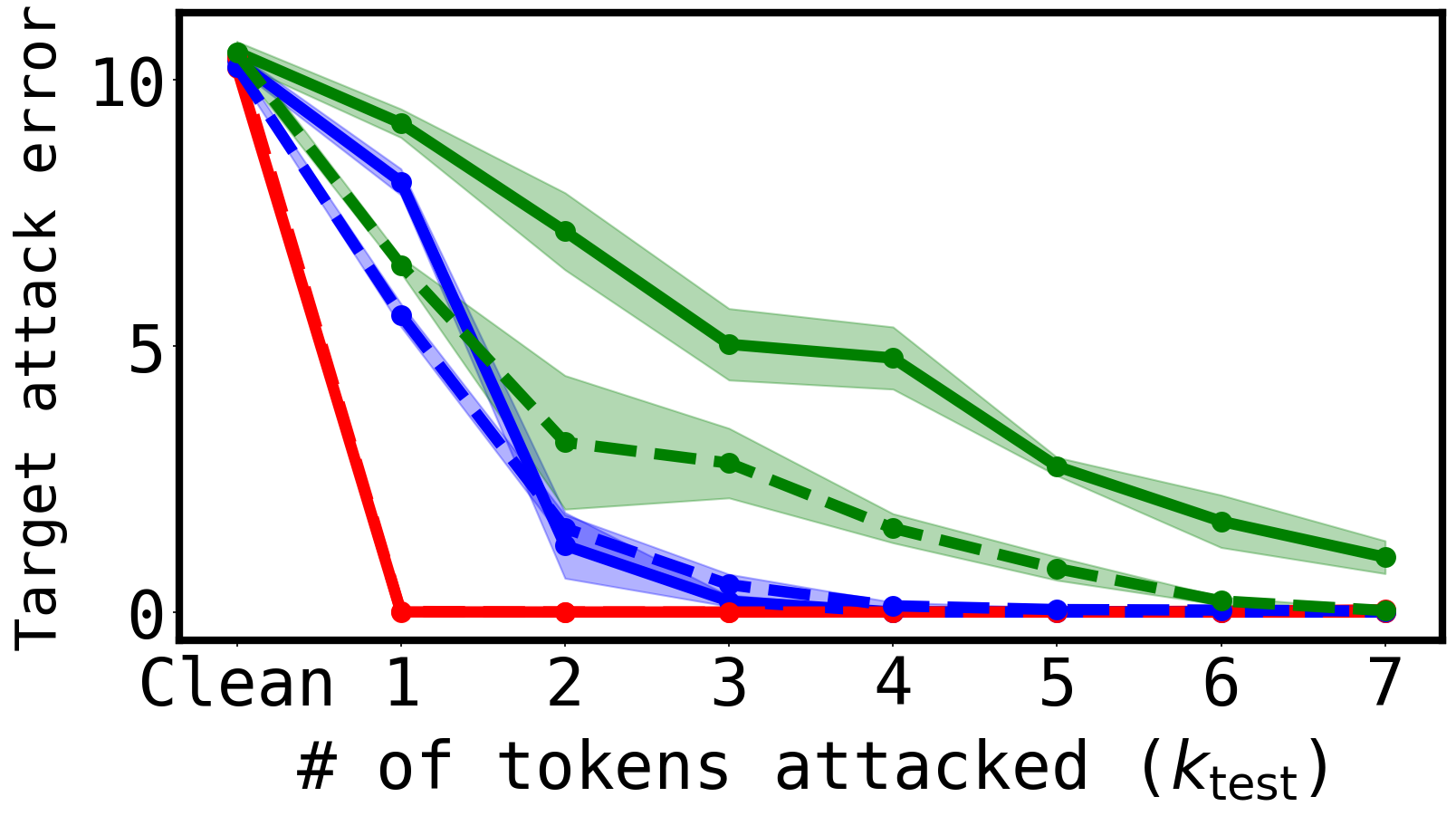}
        \caption{\zattackkk.}
    \end{subfigure}
    \begin{subfigure}[b]{0.98\textwidth}
        \includegraphics[width=\textwidth]{results/adv_training/adv_training_legend.png}
    \end{subfigure}
    \caption{Adversarial training against \yattack.}
    \label{appx.fig:adv.training.y.alpha.0.5}
\end{figure}
\begin{figure}[!h]
    \centering
    \begin{subfigure}[b]{0.32\textwidth}
        \includegraphics[width=\textwidth]{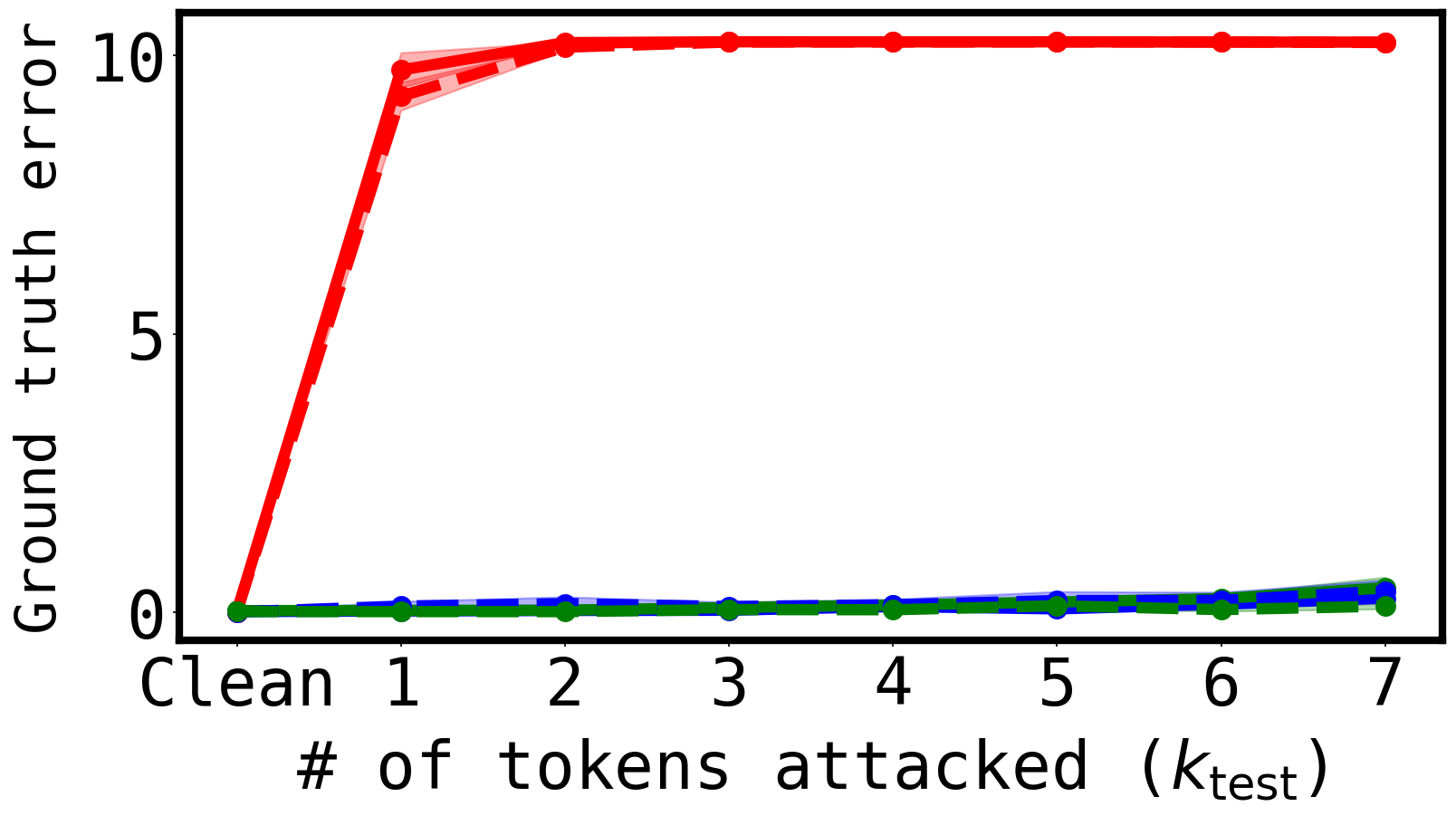}
        \caption{\xattackkk.}
    \end{subfigure}
    \begin{subfigure}[b]{0.32\textwidth}
        \includegraphics[width=\textwidth]{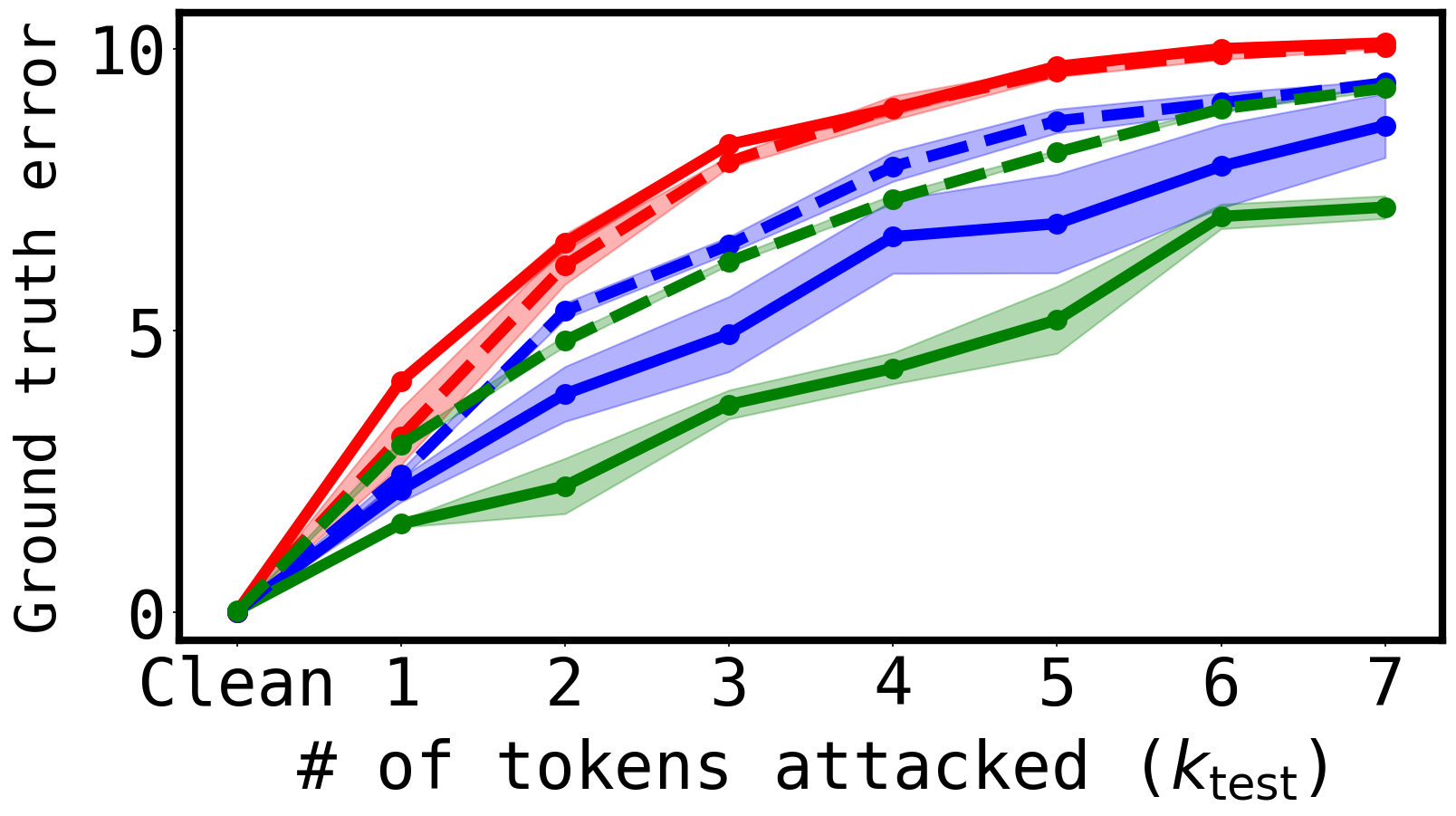}
        \caption{\yattackkk.}
    \end{subfigure}
    \begin{subfigure}[b]{0.32\textwidth}
        \includegraphics[width=\textwidth]{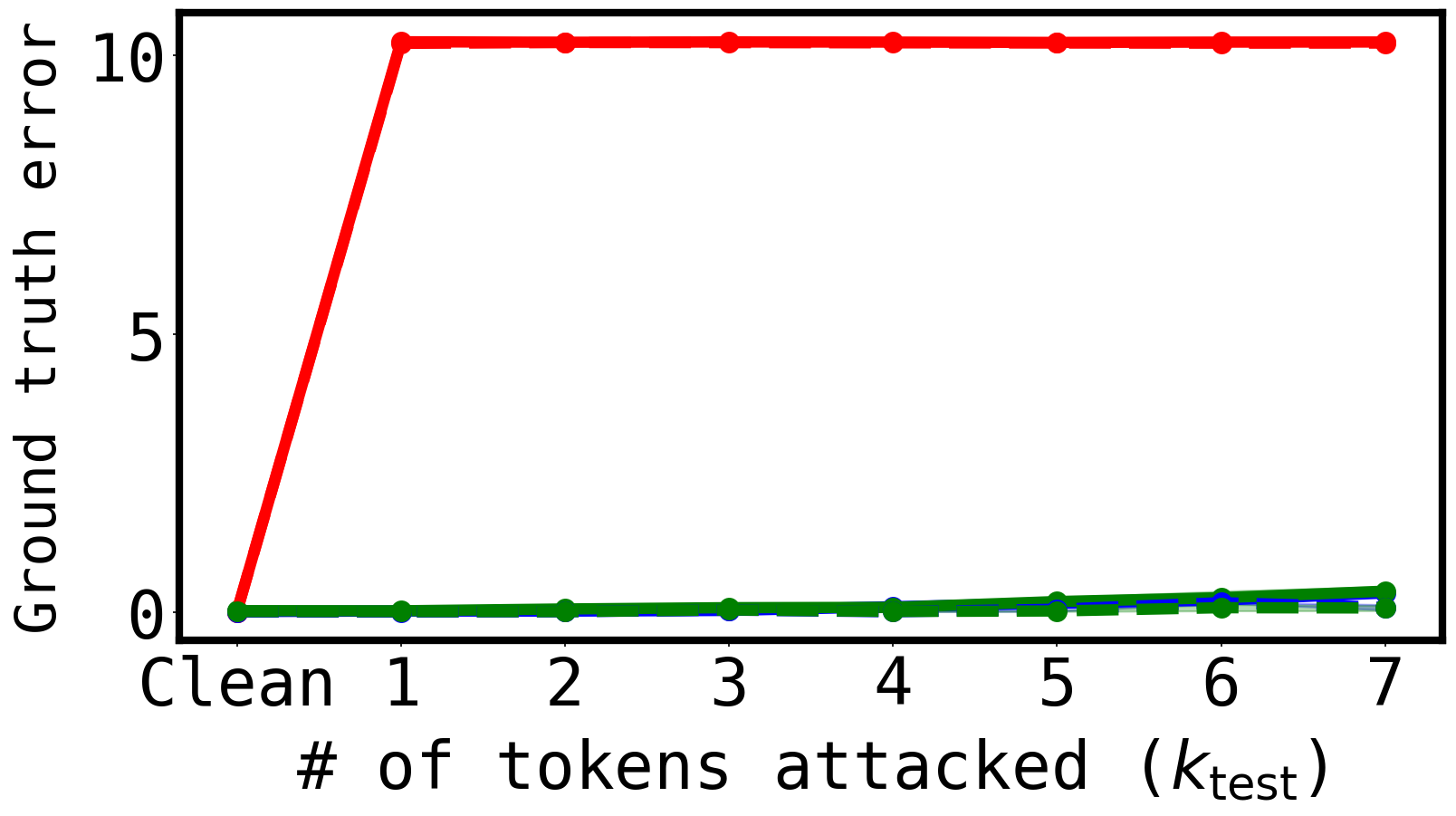}
        \caption{\zattackkk.}
    \end{subfigure}
    \begin{subfigure}[b]{0.32\textwidth}
        \includegraphics[width=\textwidth]{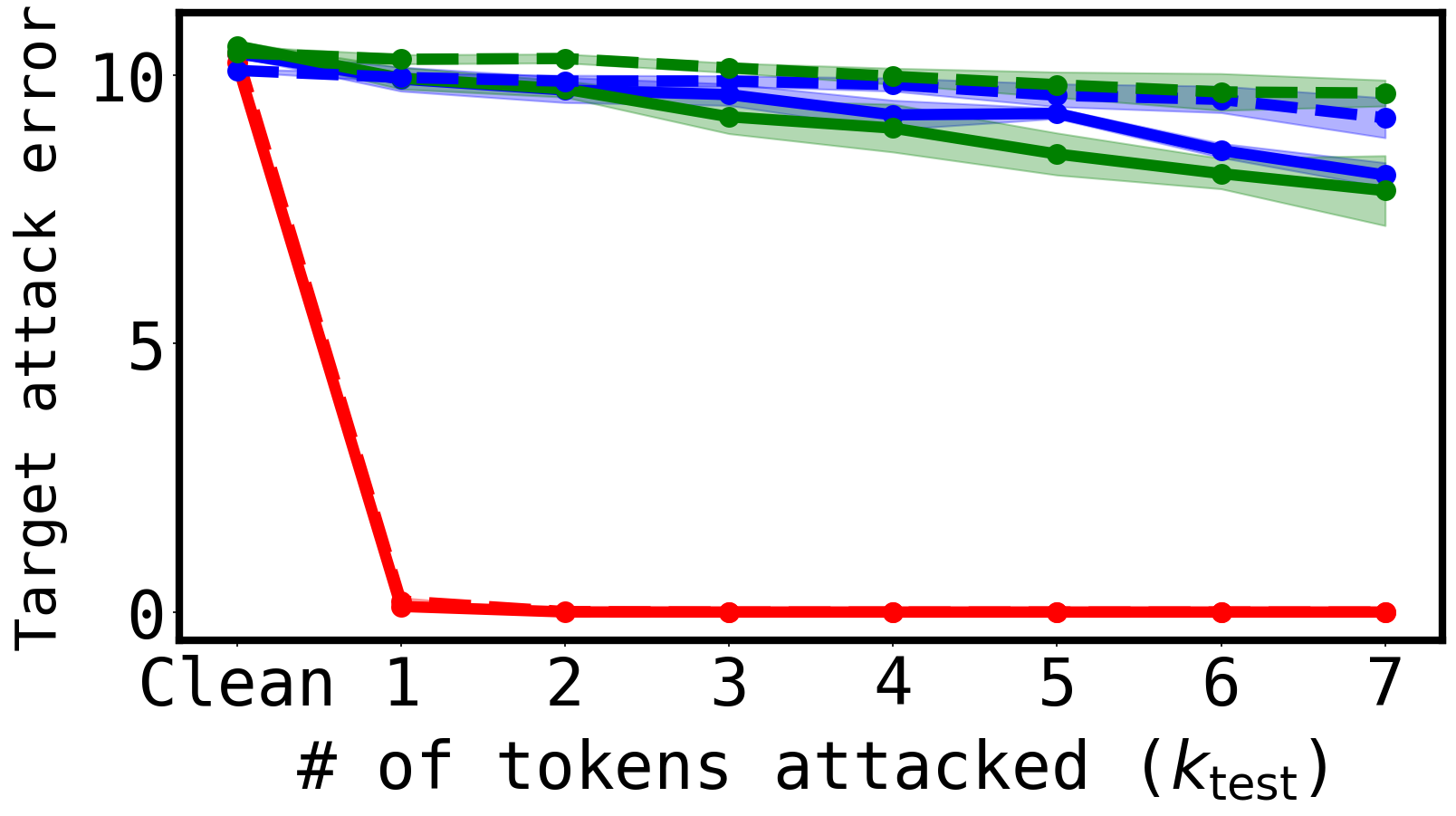}
        \caption{\xattackkk.}
    \end{subfigure}
        \begin{subfigure}[b]{0.32\textwidth}
        \includegraphics[width=\textwidth]{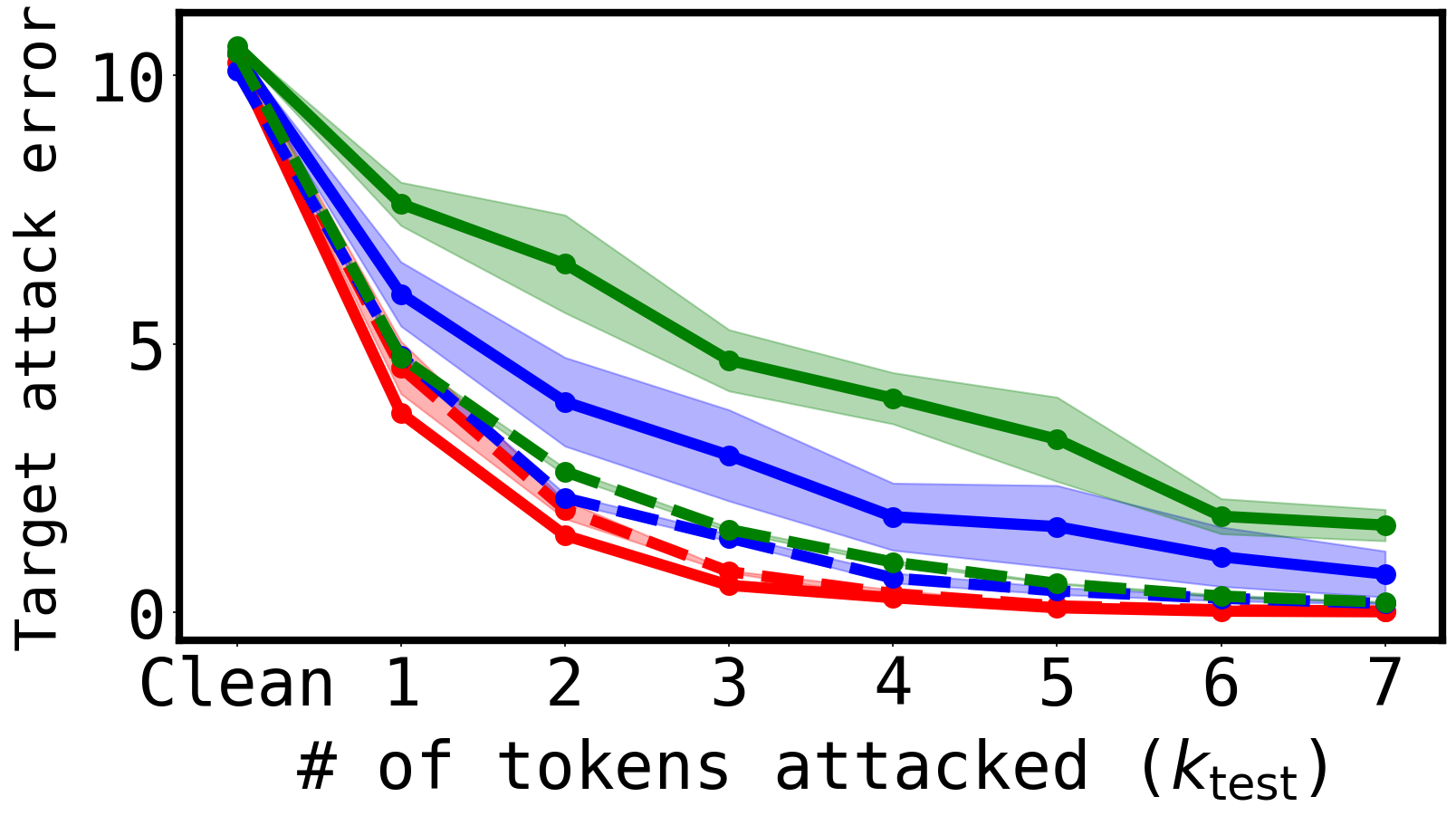}
        \caption{\yattackkk.}
    \end{subfigure}
        \begin{subfigure}[b]{0.32\textwidth}
        \includegraphics[width=\textwidth]{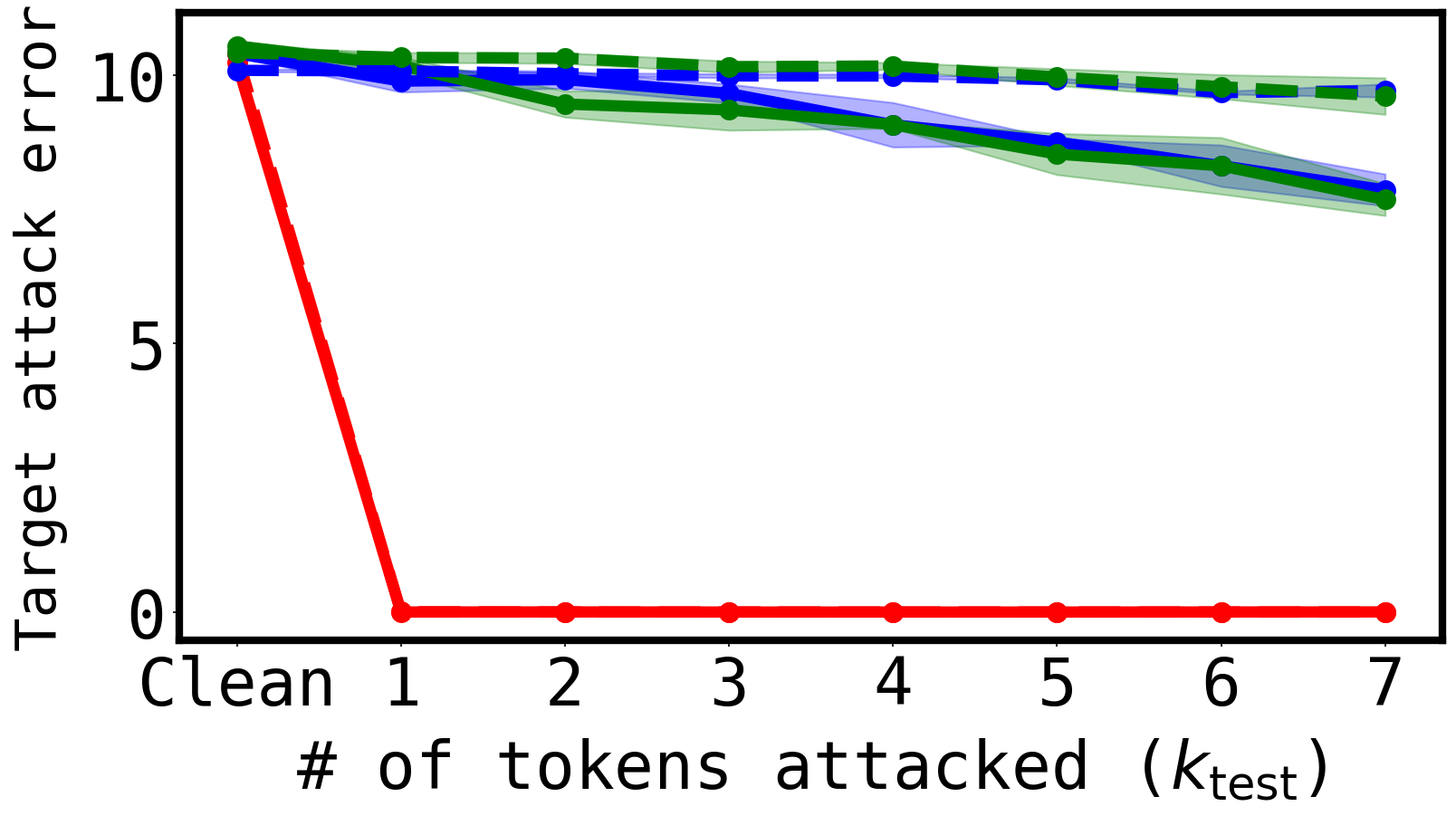}
        \caption{\zattackkk.}
    \end{subfigure}
    \begin{subfigure}[b]{0.98\textwidth}
        \includegraphics[width=\textwidth]{results/adv_training/adv_training_legend.png}
    \end{subfigure}
    \caption{Adversarial training against \xattack.}
    \label{appx.fig:adv.training.x.alpha.0.5}
\end{figure}
\begin{figure}[!h]
    \centering
    \begin{subfigure}[b]{0.32\textwidth}
        \includegraphics[width=\textwidth]{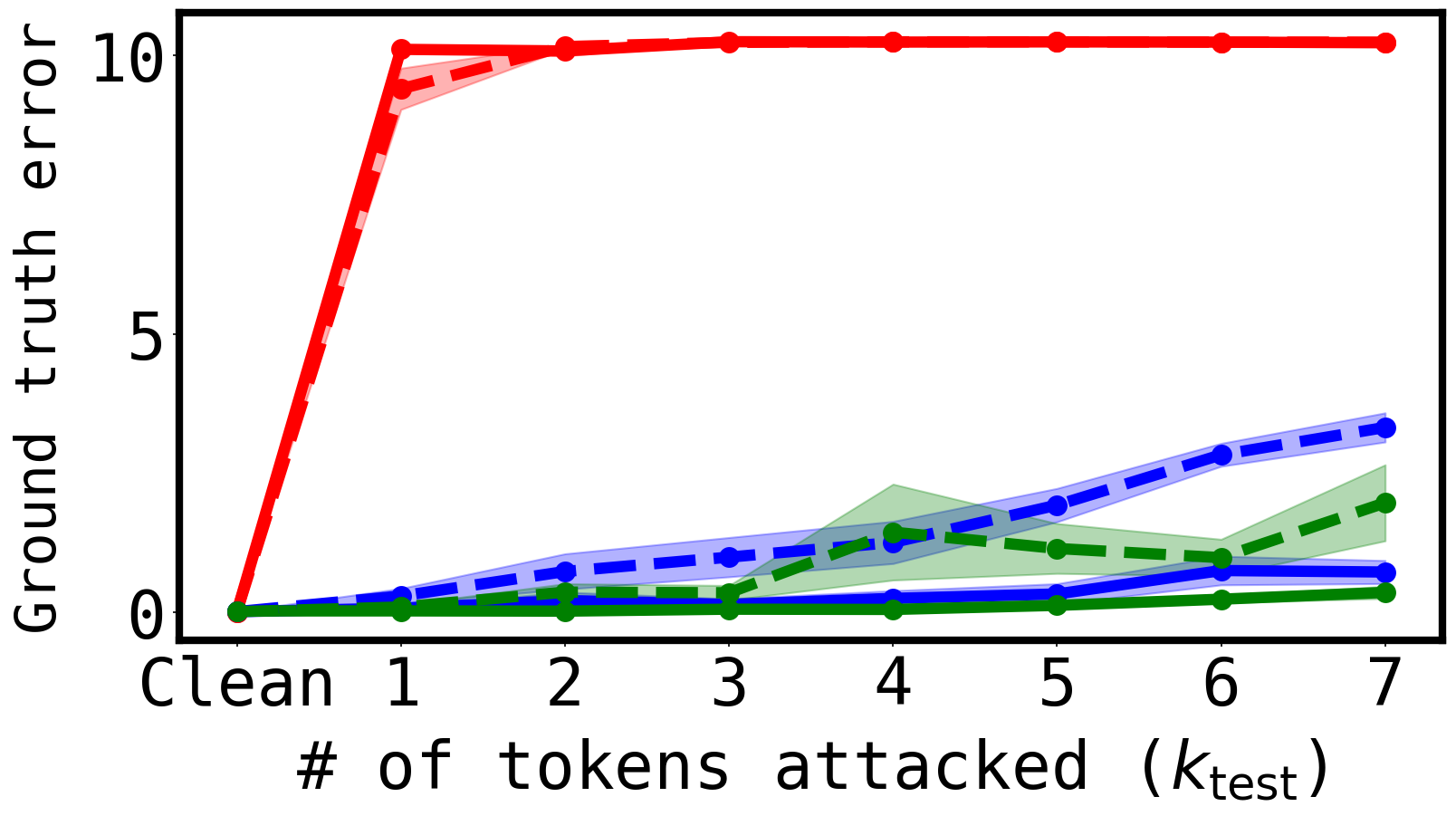}
        \caption{\xattackkk.}
    \end{subfigure}
    \begin{subfigure}[b]{0.32\textwidth}
        \includegraphics[width=\textwidth]{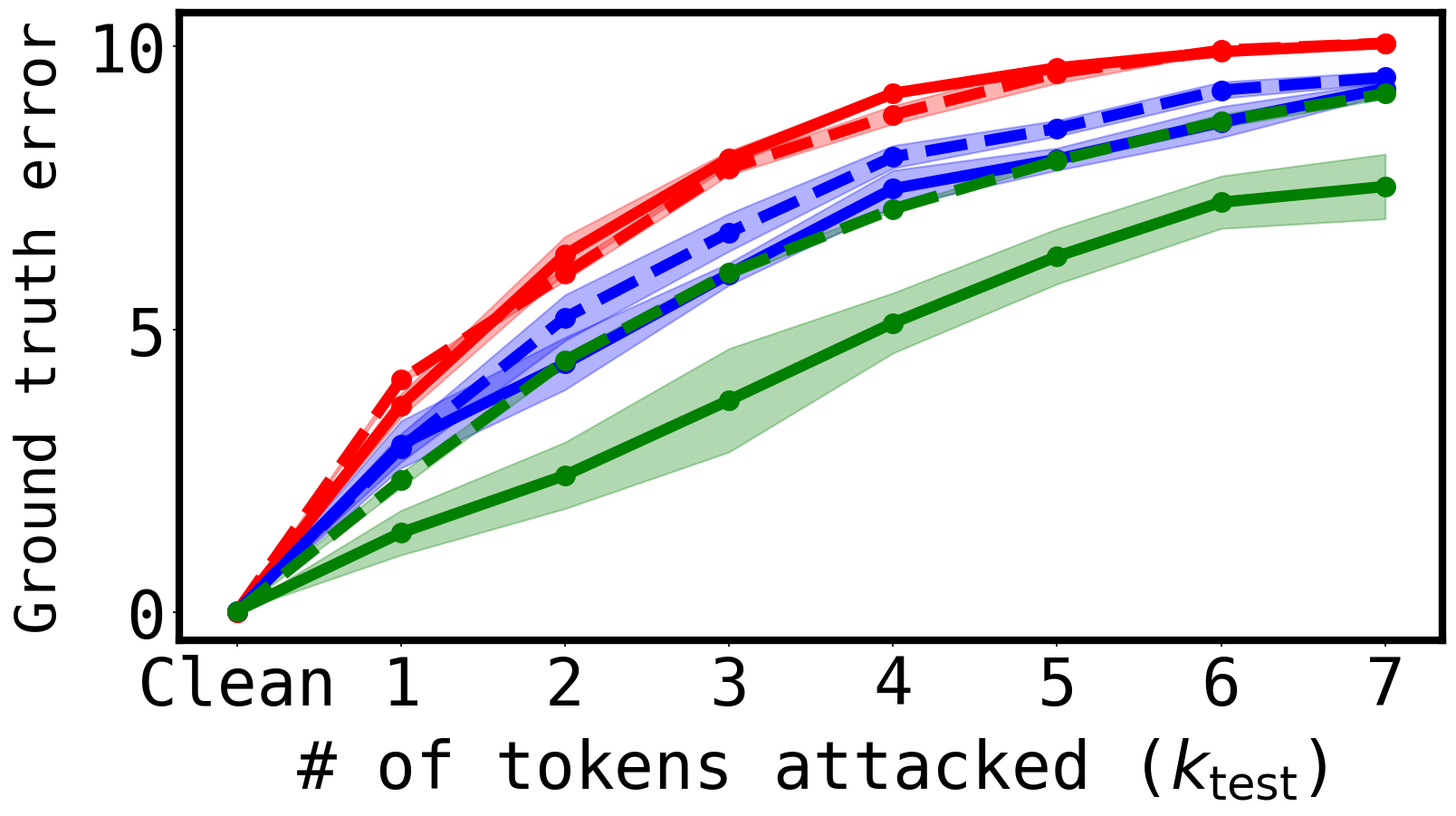}
        \caption{\yattackkk.}
    \end{subfigure}
    \begin{subfigure}[b]{0.32\textwidth}
        \includegraphics[width=\textwidth]{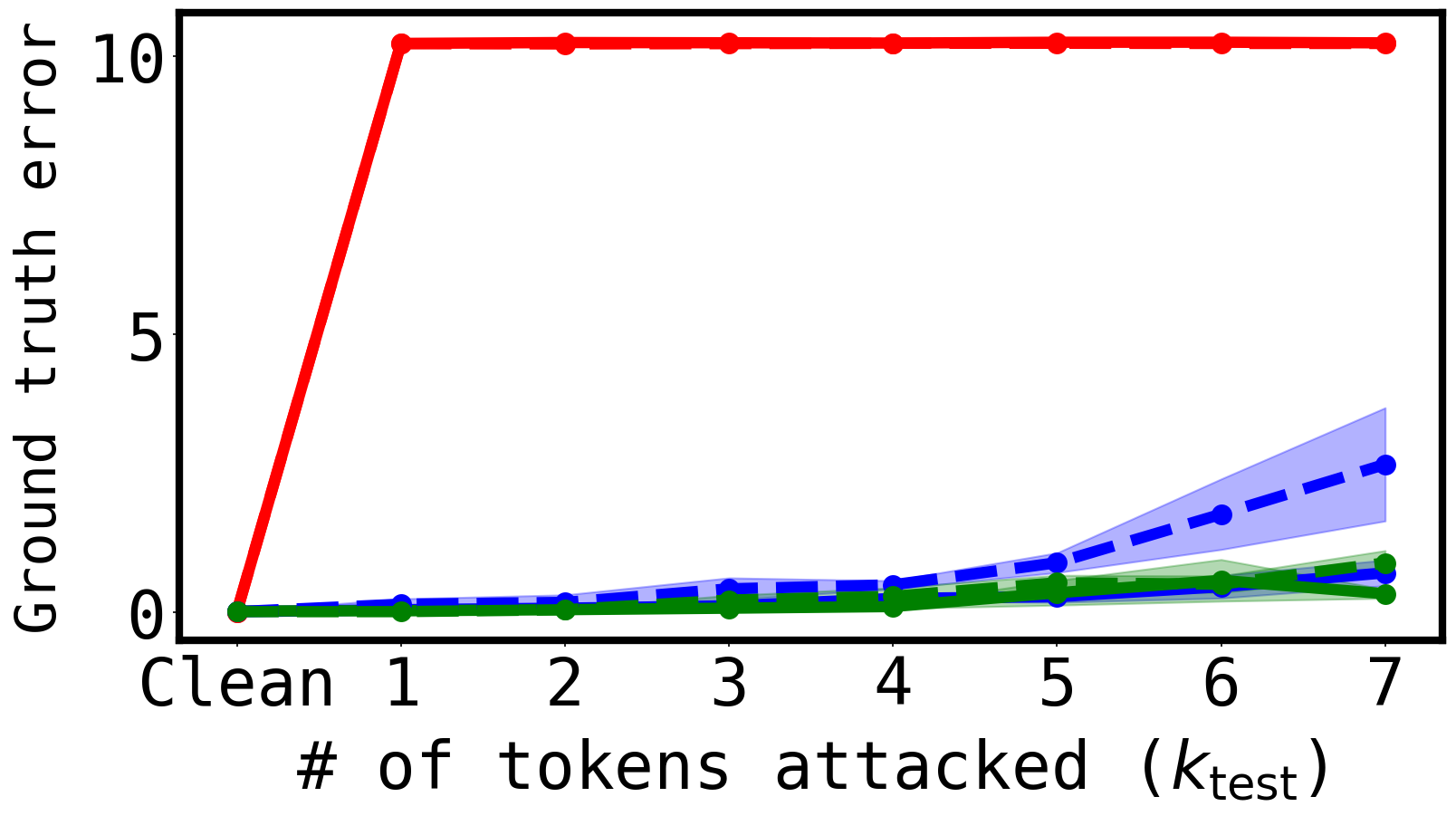}
        \caption{\zattackkk.}
    \end{subfigure}
    \begin{subfigure}[b]{0.32\textwidth}
        \includegraphics[width=\textwidth]{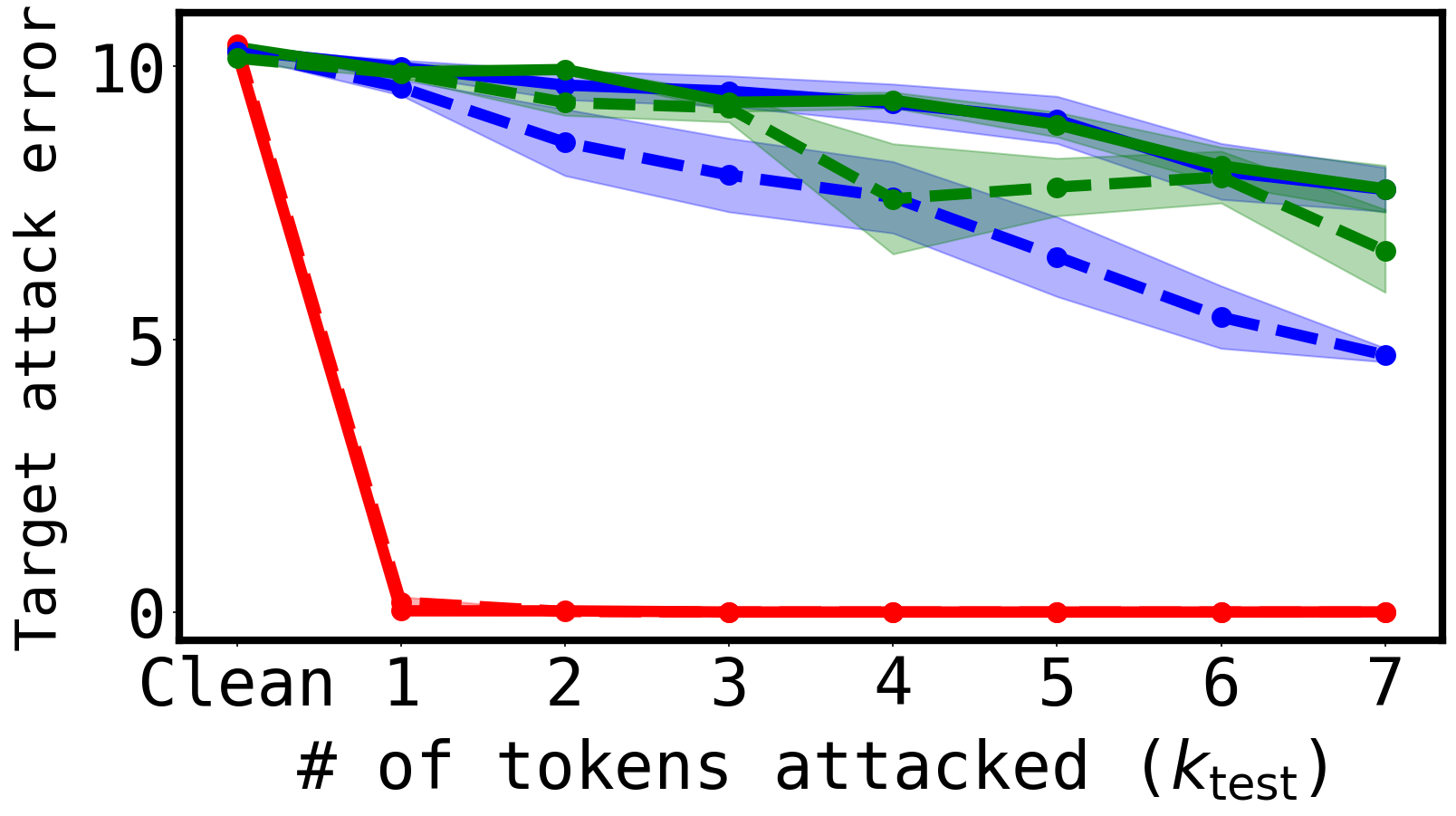}
        \caption{\xattackkk.}
    \end{subfigure}
        \begin{subfigure}[b]{0.32\textwidth}
        \includegraphics[width=\textwidth]{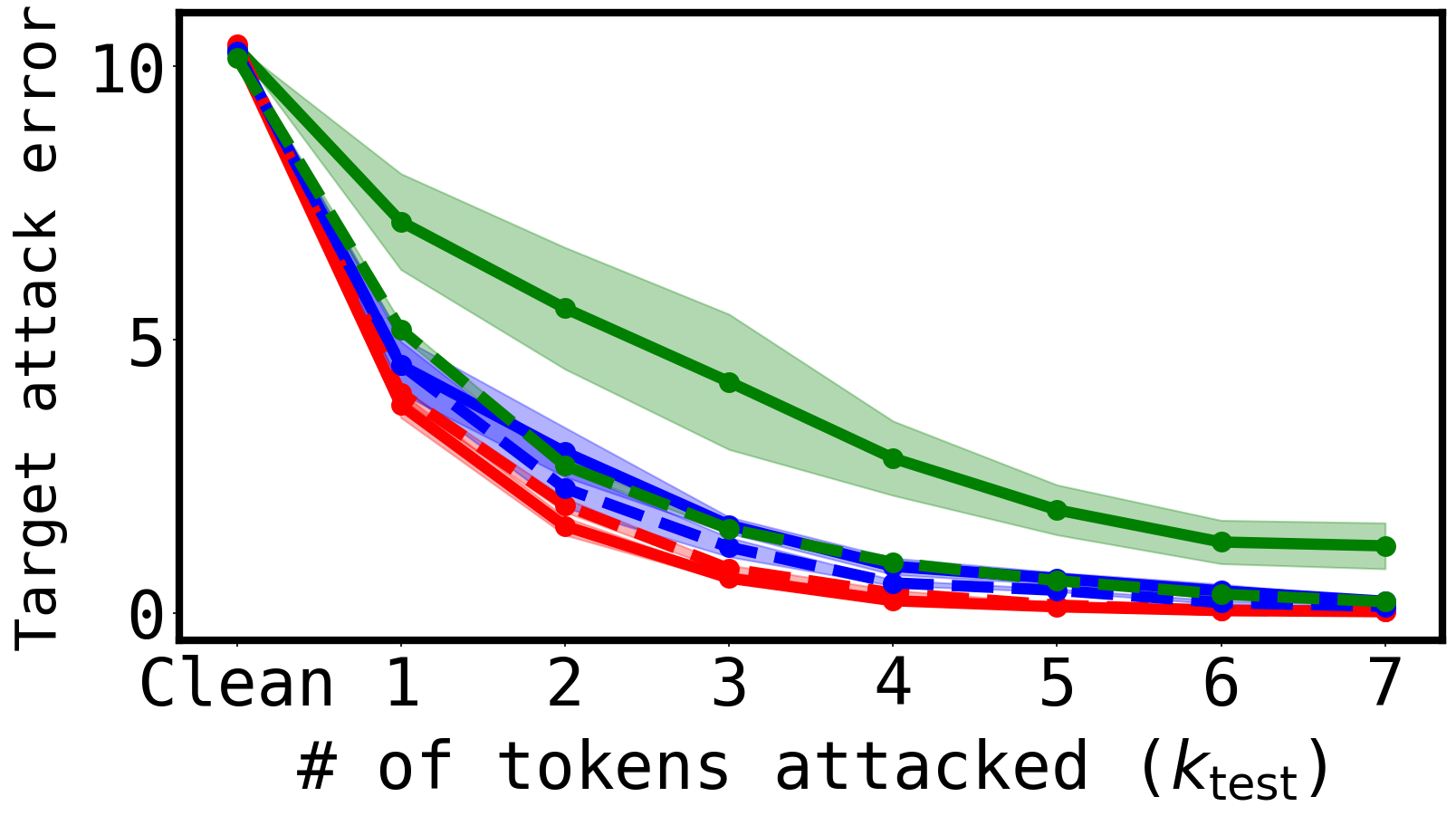}
        \caption{\yattackkk.}
    \end{subfigure}
        \begin{subfigure}[b]{0.32\textwidth}
        \includegraphics[width=\textwidth]{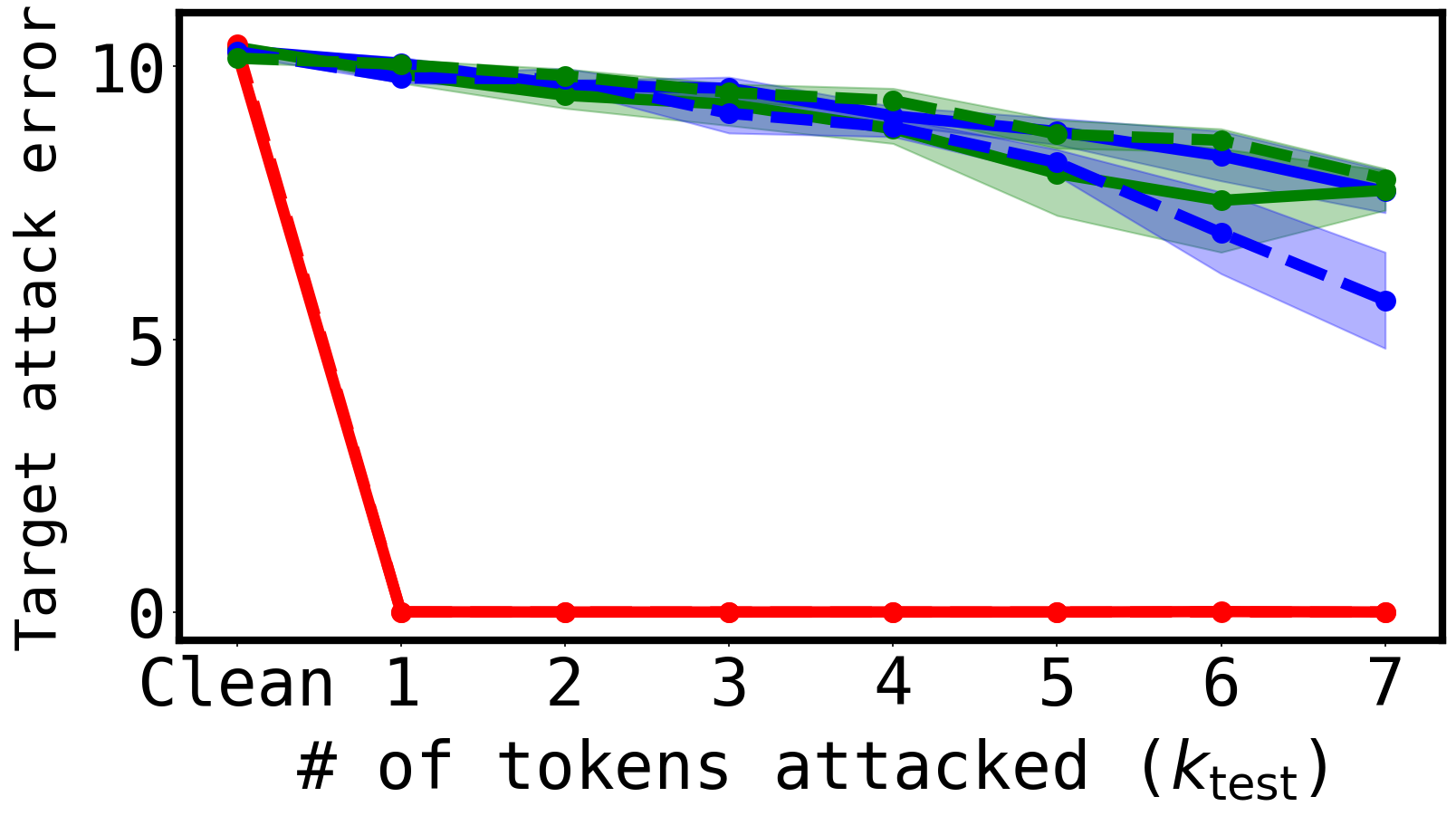}
        \caption{\zattackkk.}
    \end{subfigure}
    \begin{subfigure}[b]{0.98\textwidth}
        \includegraphics[width=\textwidth]{results/adv_training/adv_training_legend.png}
    \end{subfigure}
    \caption{Adversarial training against \zattack.}
    \label{appx.fig:adv.training.z.alpha.0.5}
\end{figure}

\ifbool{tmlrtemp}{\FloatBarrier}{\clearpage}
\subsubsection{$\alpha=0.1$}
\begin{figure}[!h]
    \centering
    \begin{subfigure}[b]{0.32\textwidth}
        \includegraphics[width=\textwidth]{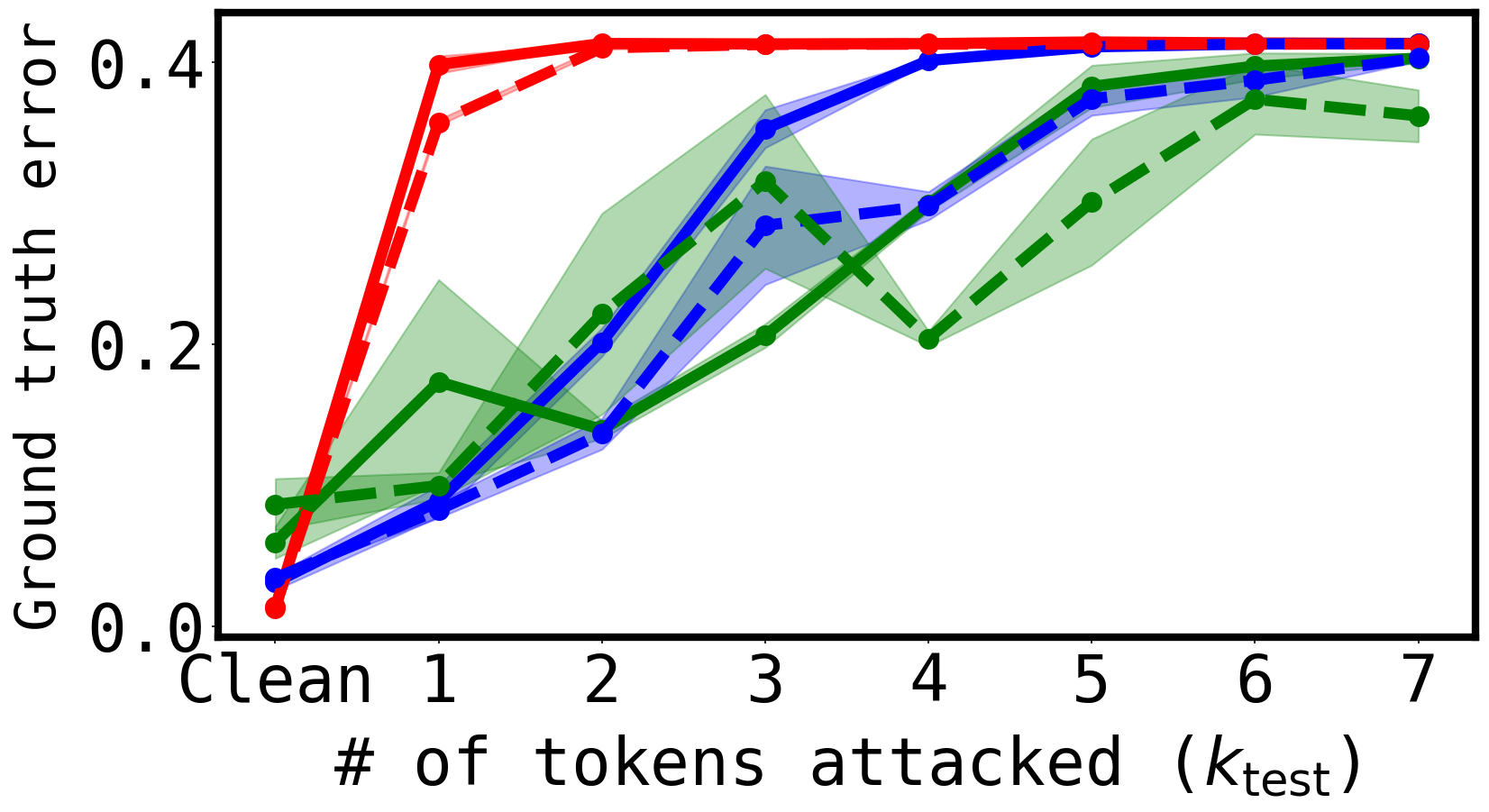}
        \caption{\xattackkk.}
    \end{subfigure}
    \begin{subfigure}[b]{0.32\textwidth}
        \includegraphics[width=\textwidth]{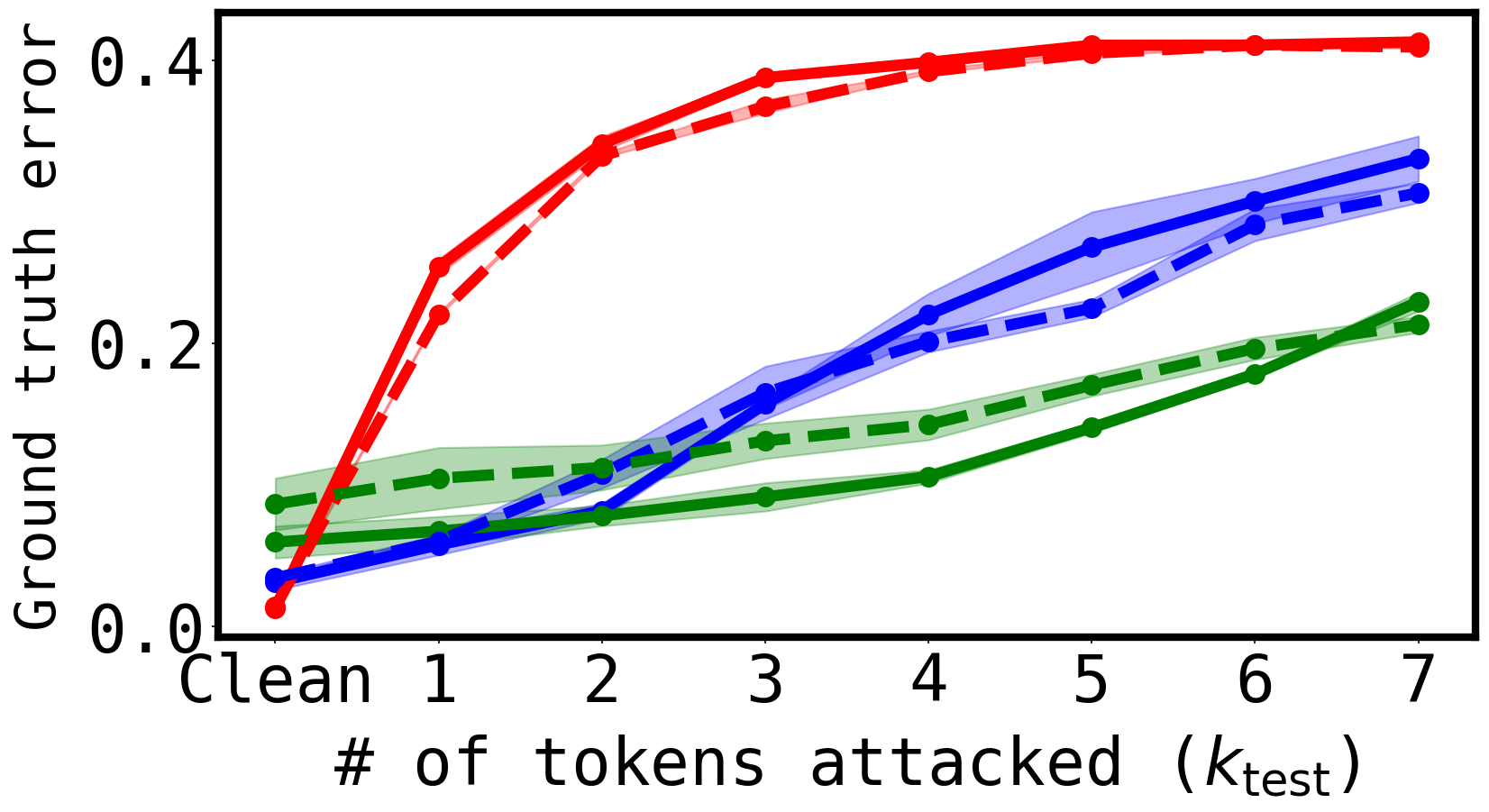}
        \caption{\yattackkk.}
    \end{subfigure}
    \begin{subfigure}[b]{0.32\textwidth}
        \includegraphics[width=\textwidth]{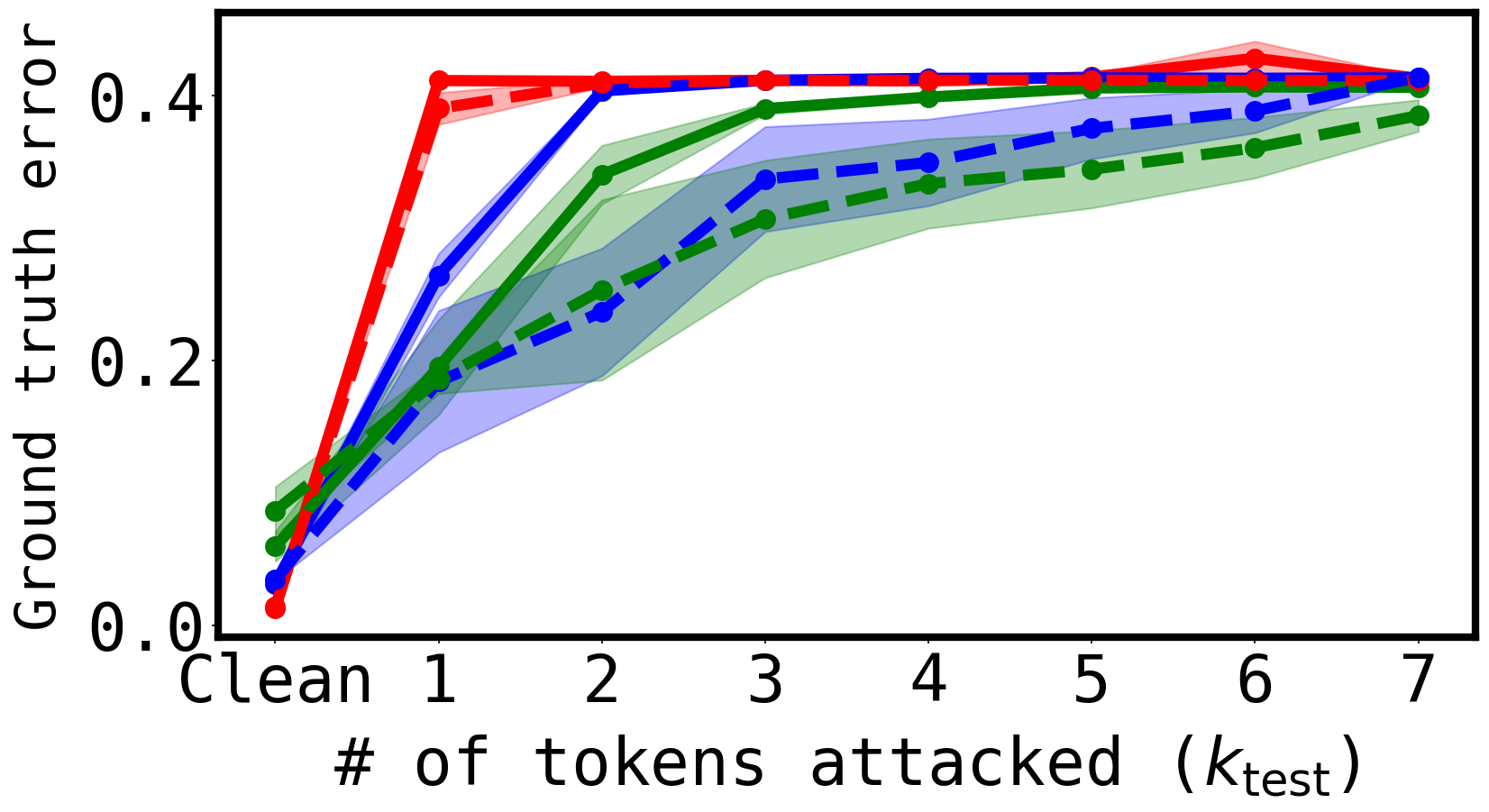}
        \caption{\zattackkk.}
    \end{subfigure}
    \begin{subfigure}[b]{0.32\textwidth}
        \includegraphics[width=\textwidth]{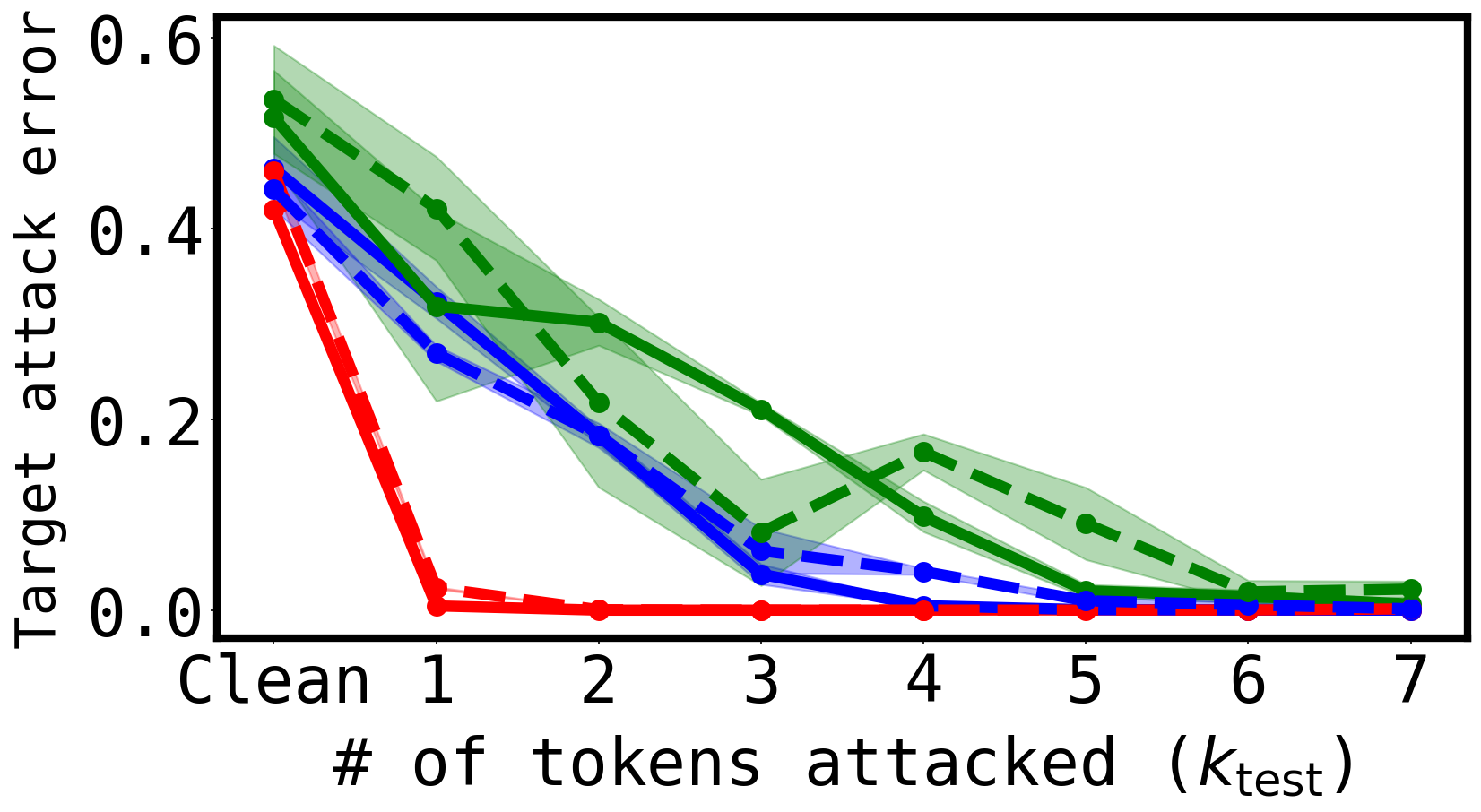}
        \caption{\xattackkk.}
    \end{subfigure}
        \begin{subfigure}[b]{0.32\textwidth}
        \includegraphics[width=\textwidth]{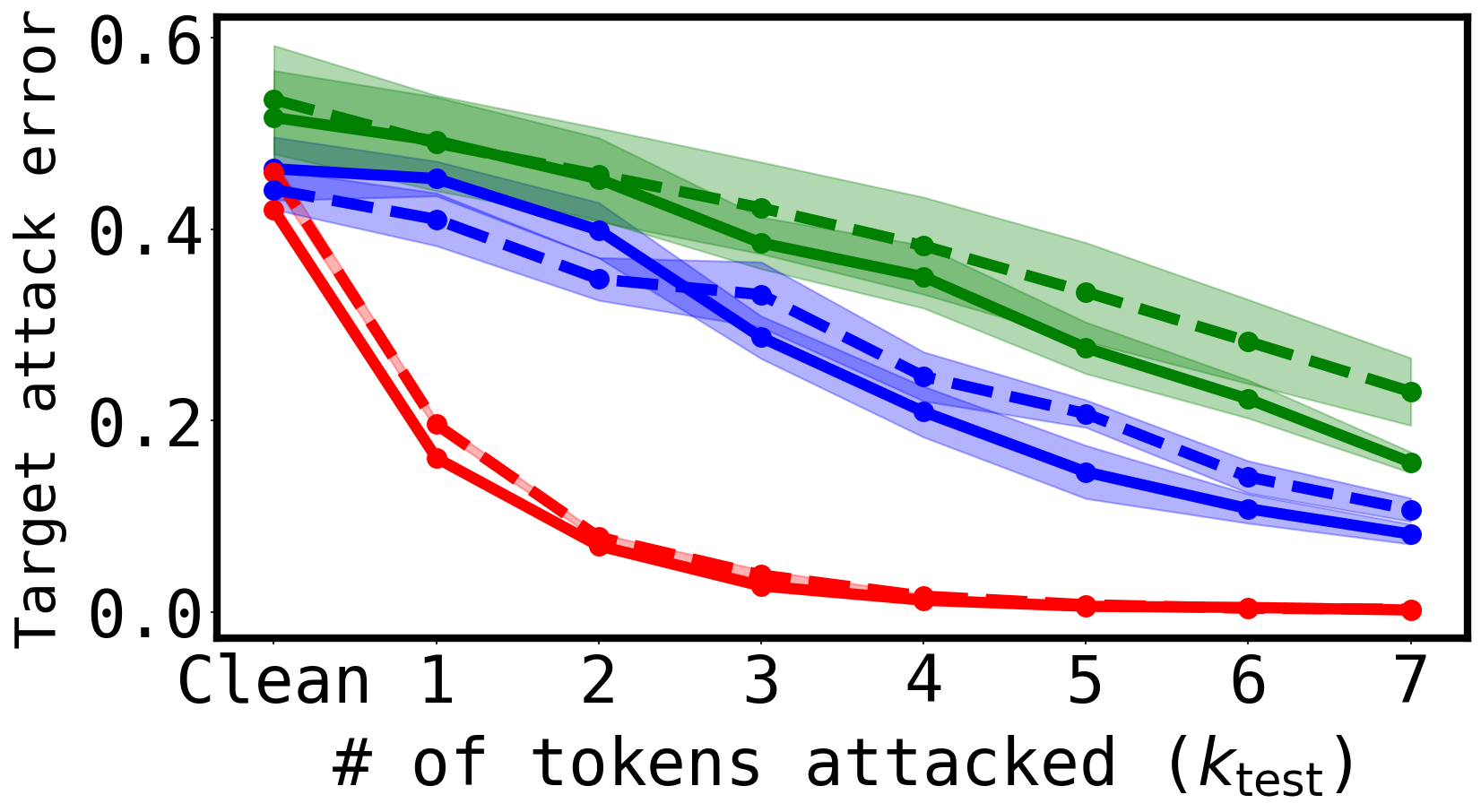}
        \caption{\yattackkk.}
    \end{subfigure}
        \begin{subfigure}[b]{0.32\textwidth}
        \includegraphics[width=\textwidth]{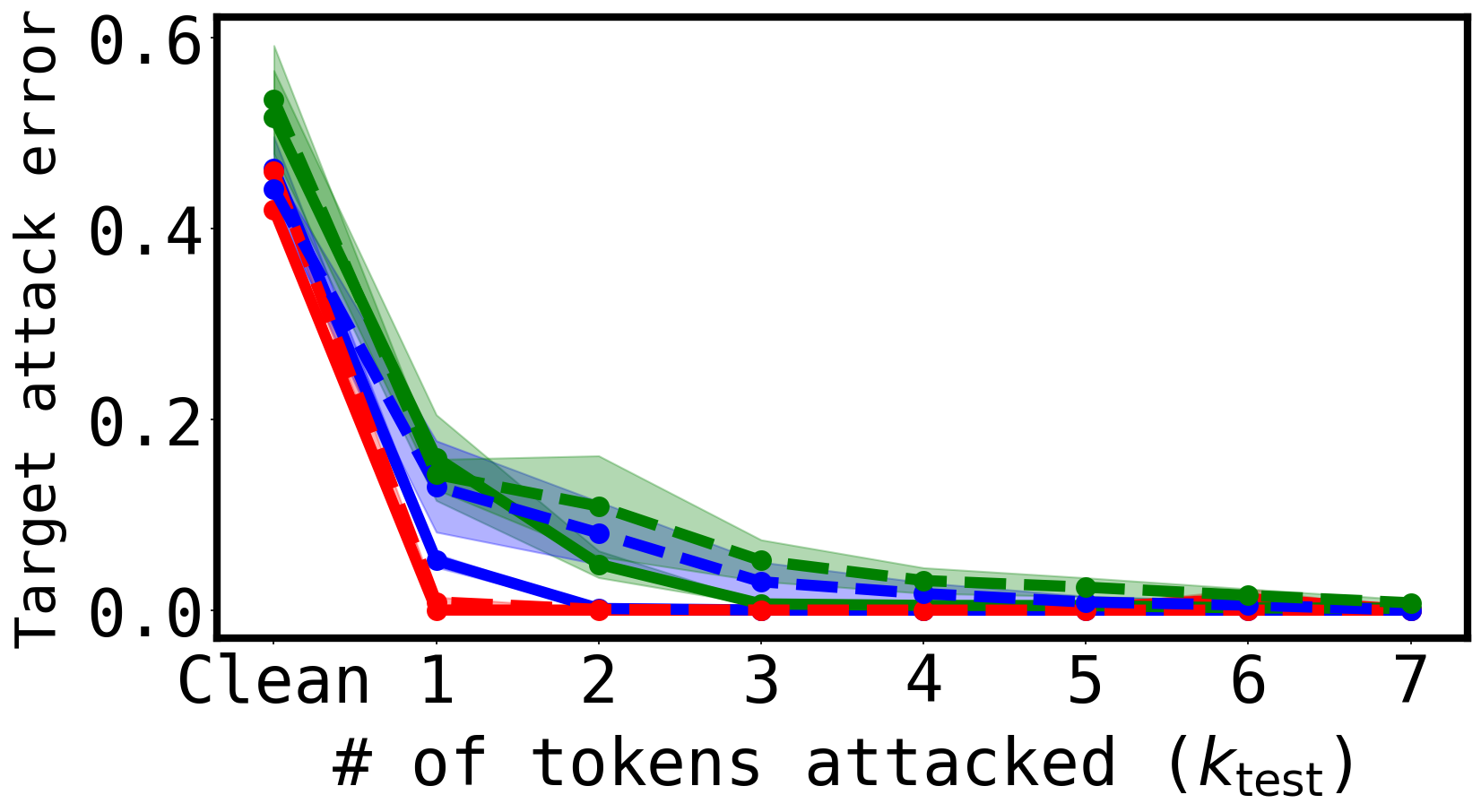}
        \caption{\zattackkk.}
    \end{subfigure}
    \begin{subfigure}[b]{0.98\textwidth}
        \includegraphics[width=\textwidth]{results/adv_training/adv_training_legend.png}
    \end{subfigure}
    \caption{Adversarial training against \yattack. A-PT denotes adversarial
    pretraining and A-FT denotes adversarial finetuning. $k_\text{train}$ denotes
    the number of tokens attacked during training and $k_\text{train}=0$ corresponds
    to a model that has not undergone adversarial training at all.}
    \label{appx.fig:adv.training.y.alpha.0.1}
\end{figure}
\begin{figure}[!h]
    \centering
    \begin{subfigure}[b]{0.32\textwidth}
        \includegraphics[width=\textwidth]{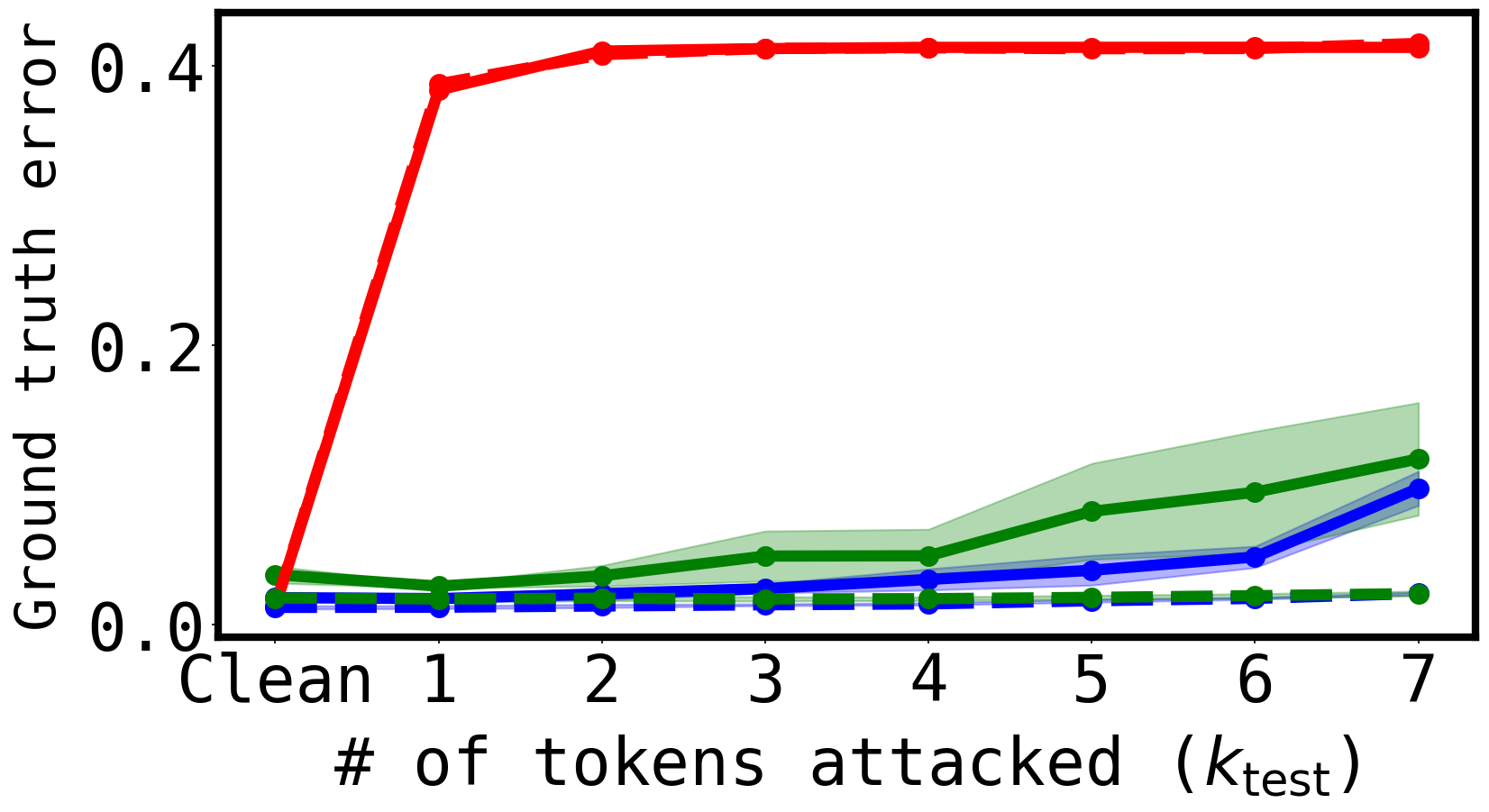}
        \caption{\xattackkk.}
    \end{subfigure}
    \begin{subfigure}[b]{0.32\textwidth}
        \includegraphics[width=\textwidth]{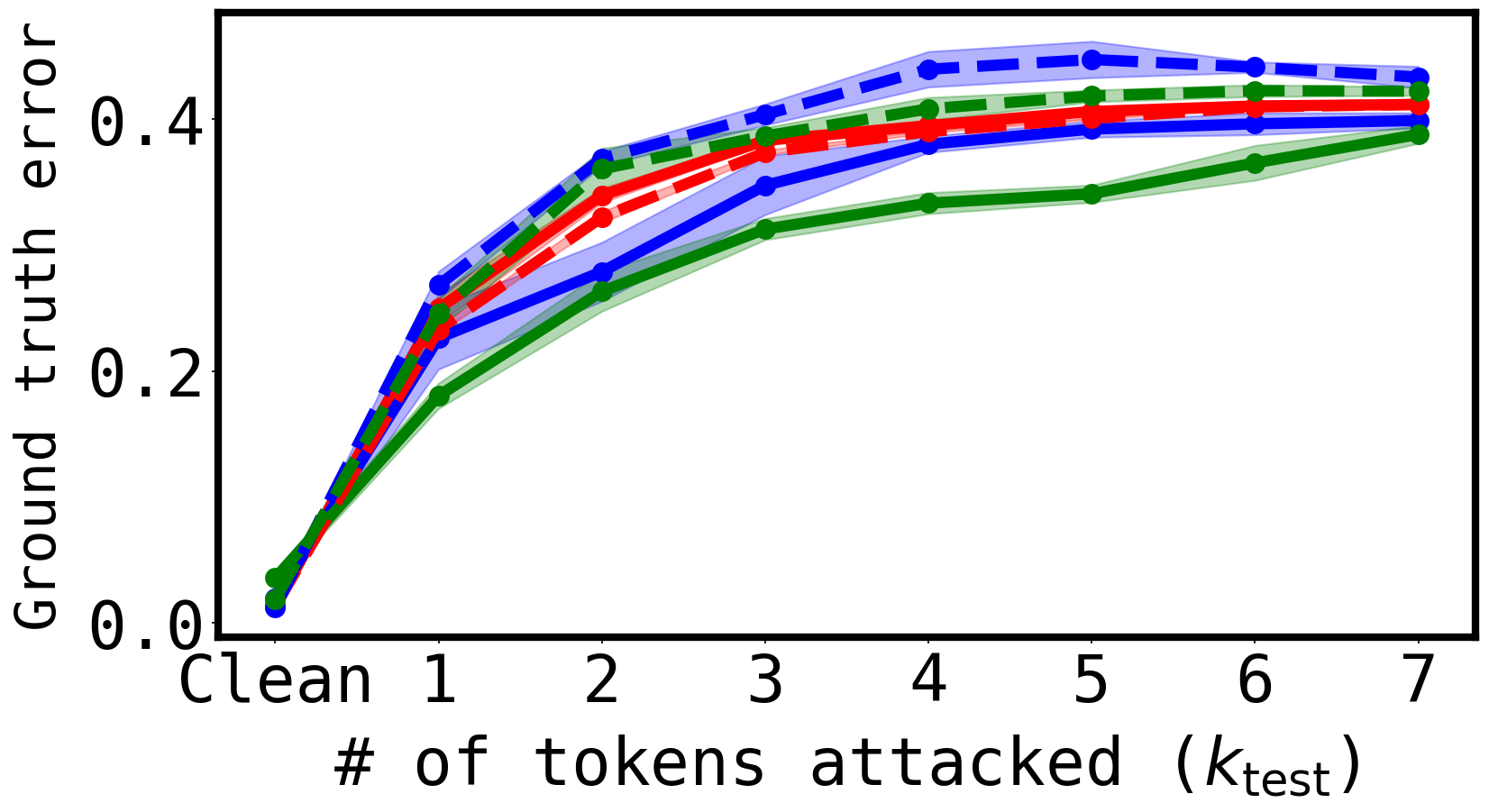}
        \caption{\yattackkk.}
    \end{subfigure}
    \begin{subfigure}[b]{0.32\textwidth}
        \includegraphics[width=\textwidth]{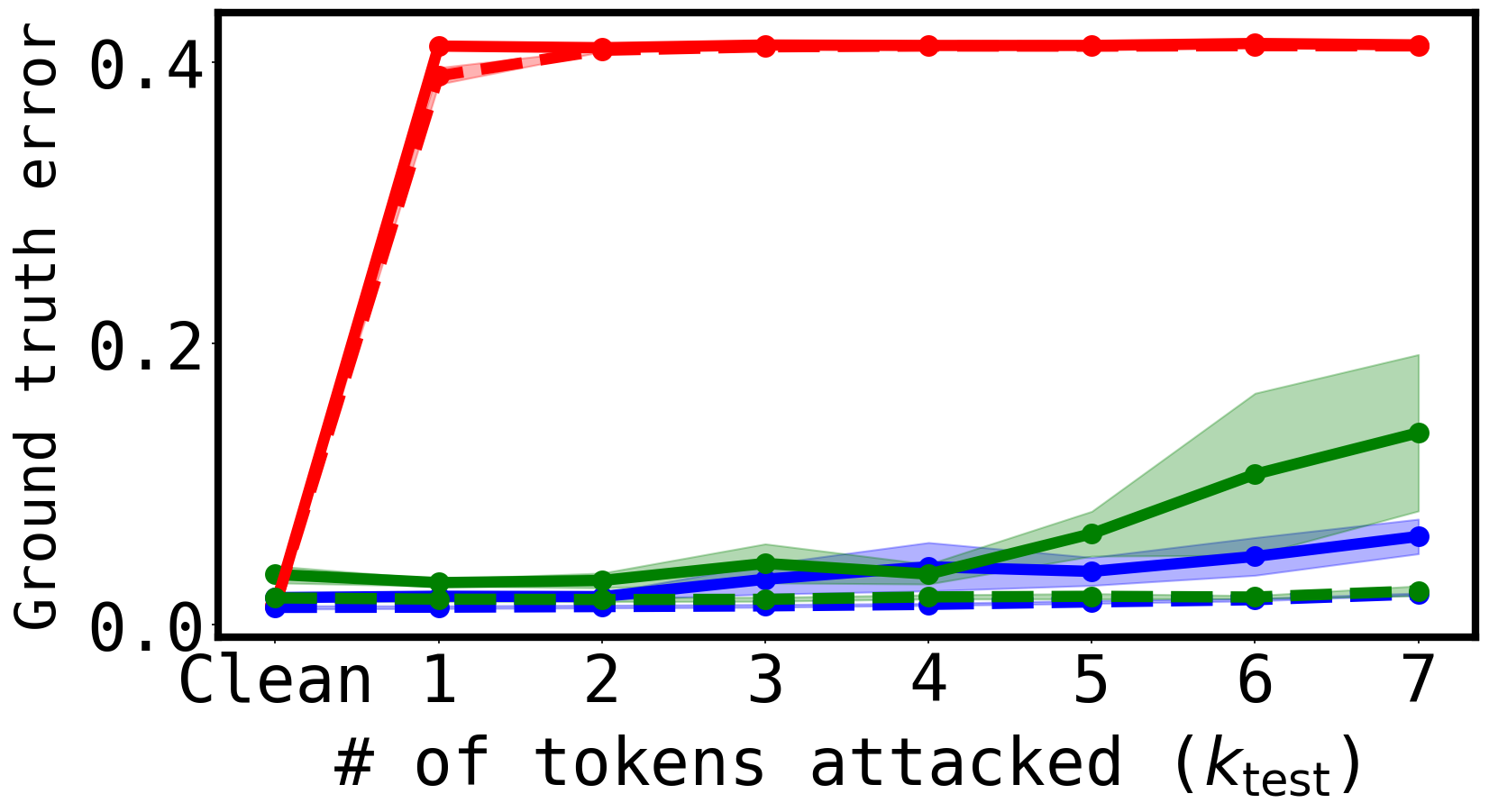}
        \caption{\zattackkk.}
    \end{subfigure}
    \begin{subfigure}[b]{0.32\textwidth}
        \includegraphics[width=\textwidth]{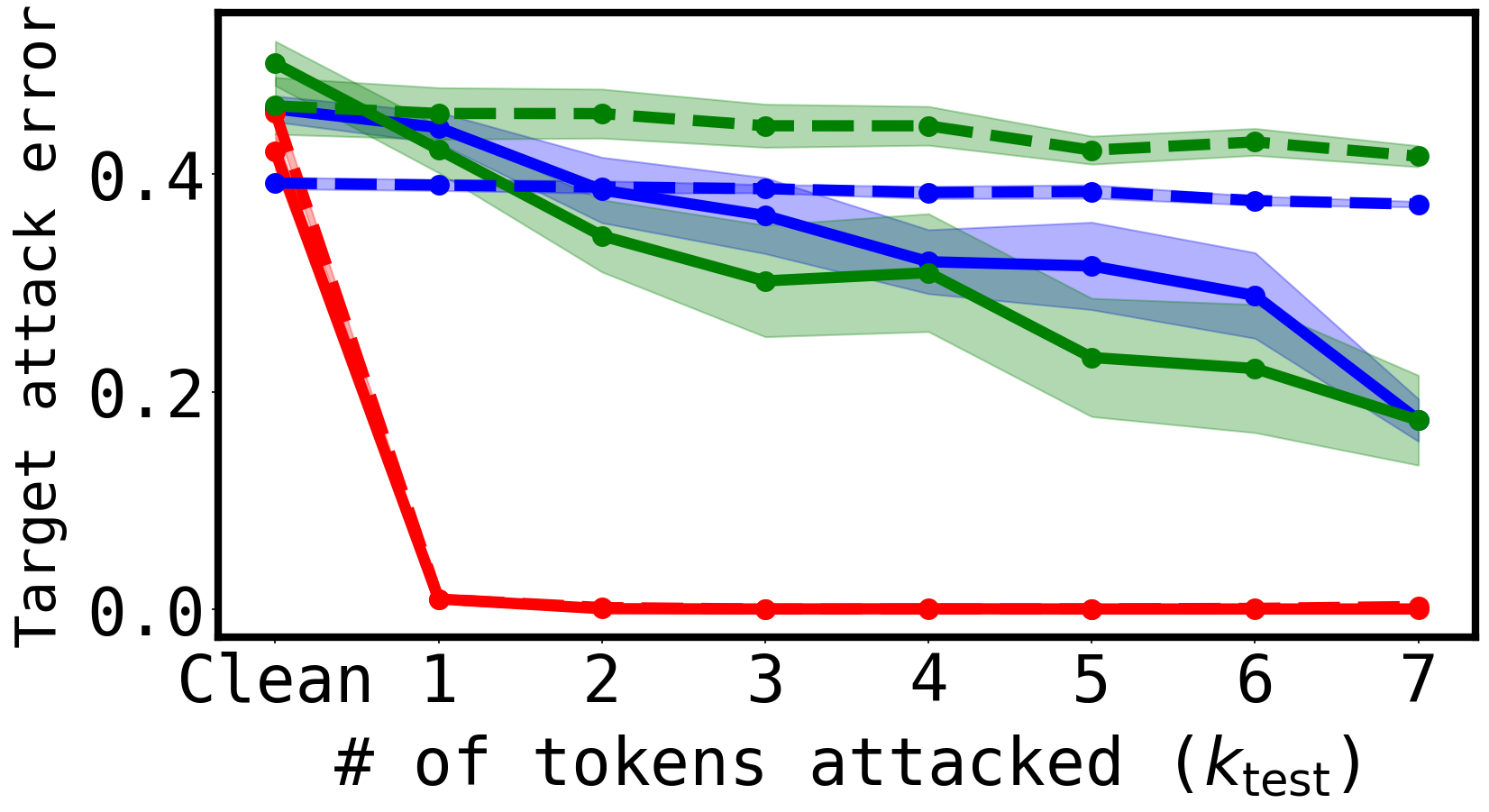}
        \caption{\xattackkk.}
    \end{subfigure}
        \begin{subfigure}[b]{0.32\textwidth}
        \includegraphics[width=\textwidth]{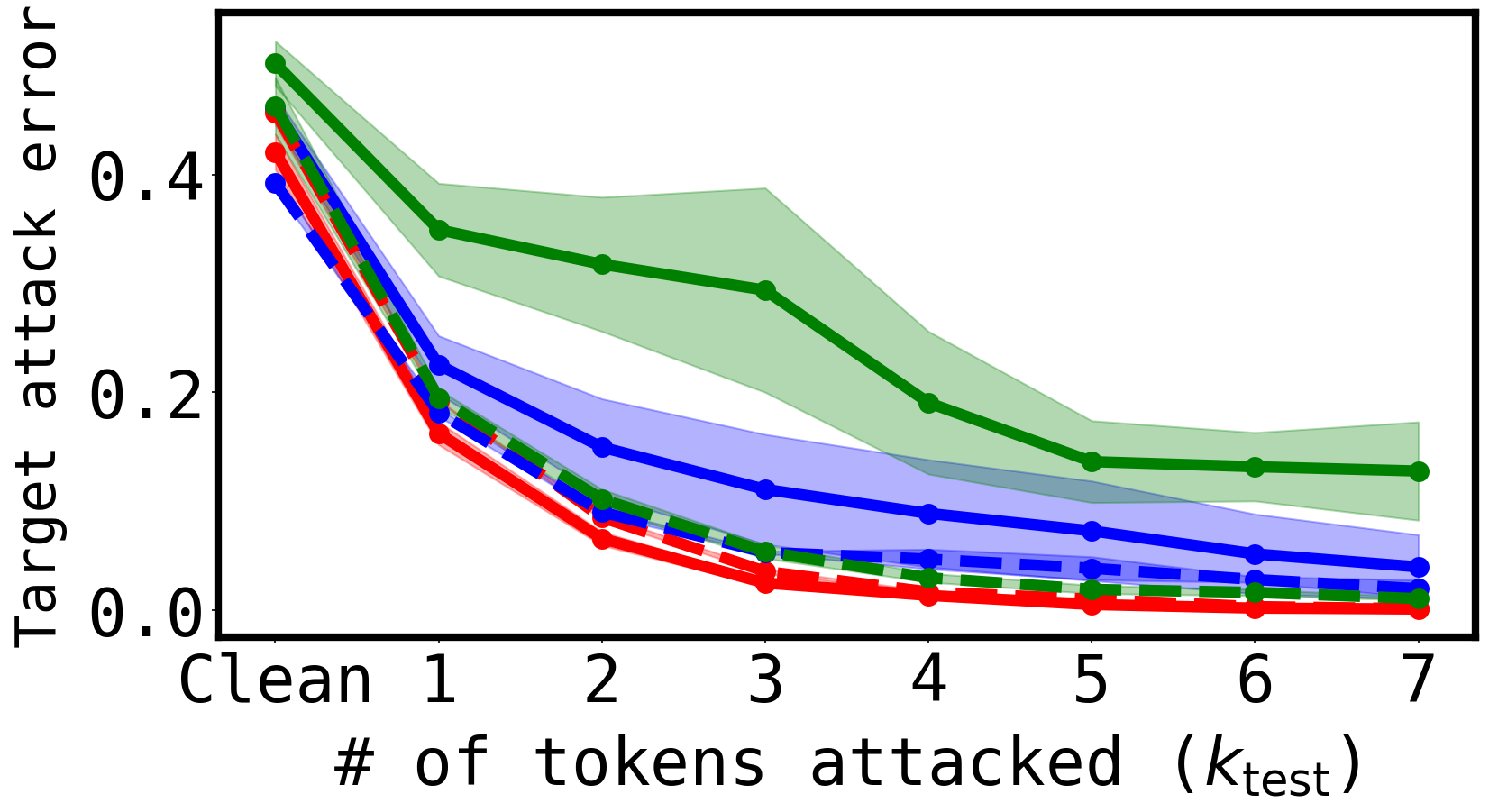}
        \caption{\yattackkk.}
    \end{subfigure}
        \begin{subfigure}[b]{0.32\textwidth}
        \includegraphics[width=\textwidth]{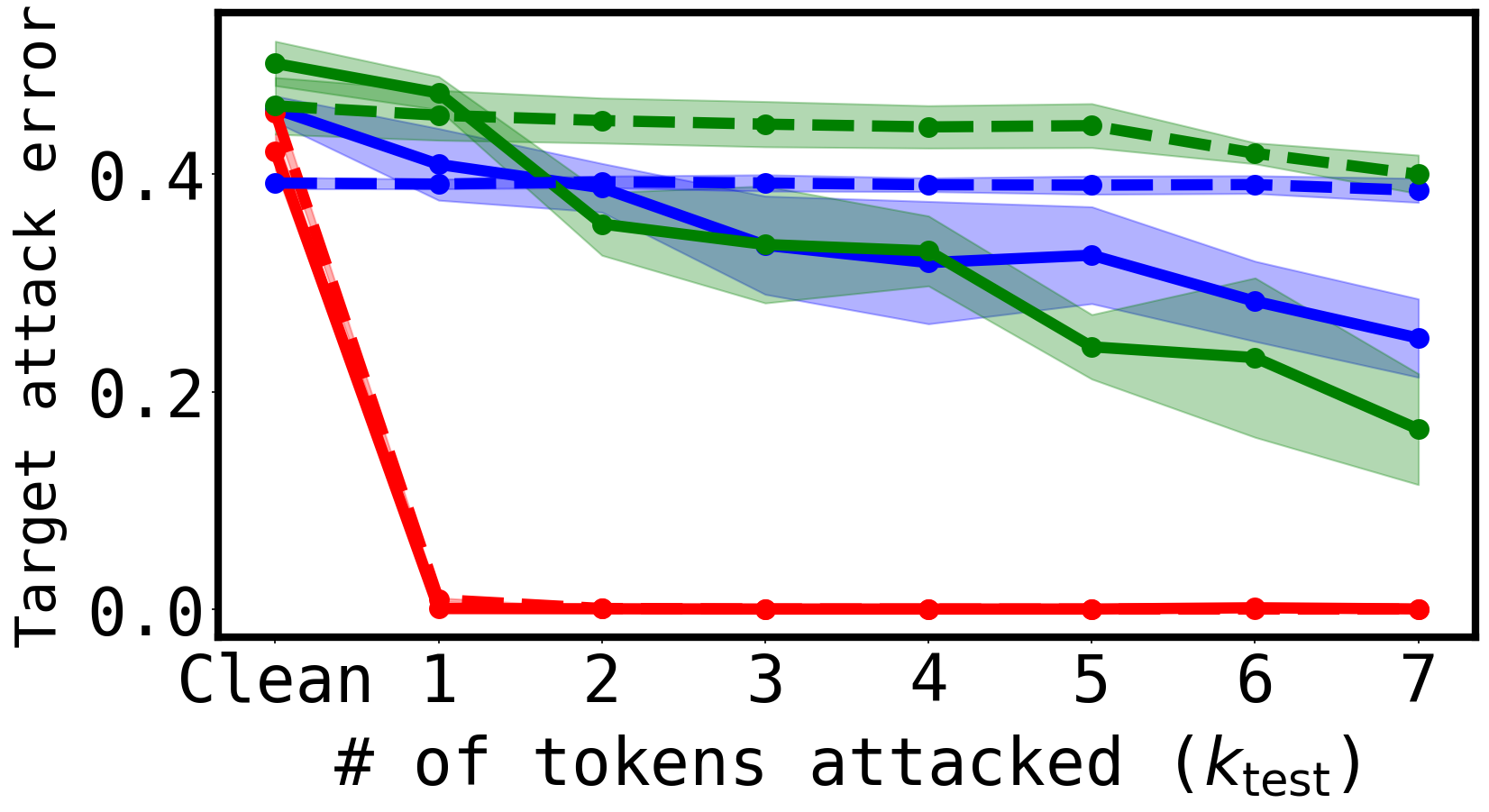}
        \caption{\zattackkk.}
    \end{subfigure}
    \begin{subfigure}[b]{0.98\textwidth}
        \includegraphics[width=\textwidth]{results/adv_training/adv_training_legend.png}
    \end{subfigure}
    \caption{Adversarial training against \xattack.}
    \label{appx.fig:adv.training.x.alpha.0.1}
\end{figure}
\begin{figure}[!h]
    \centering
    \begin{subfigure}[b]{0.32\textwidth}
        \includegraphics[width=\textwidth]{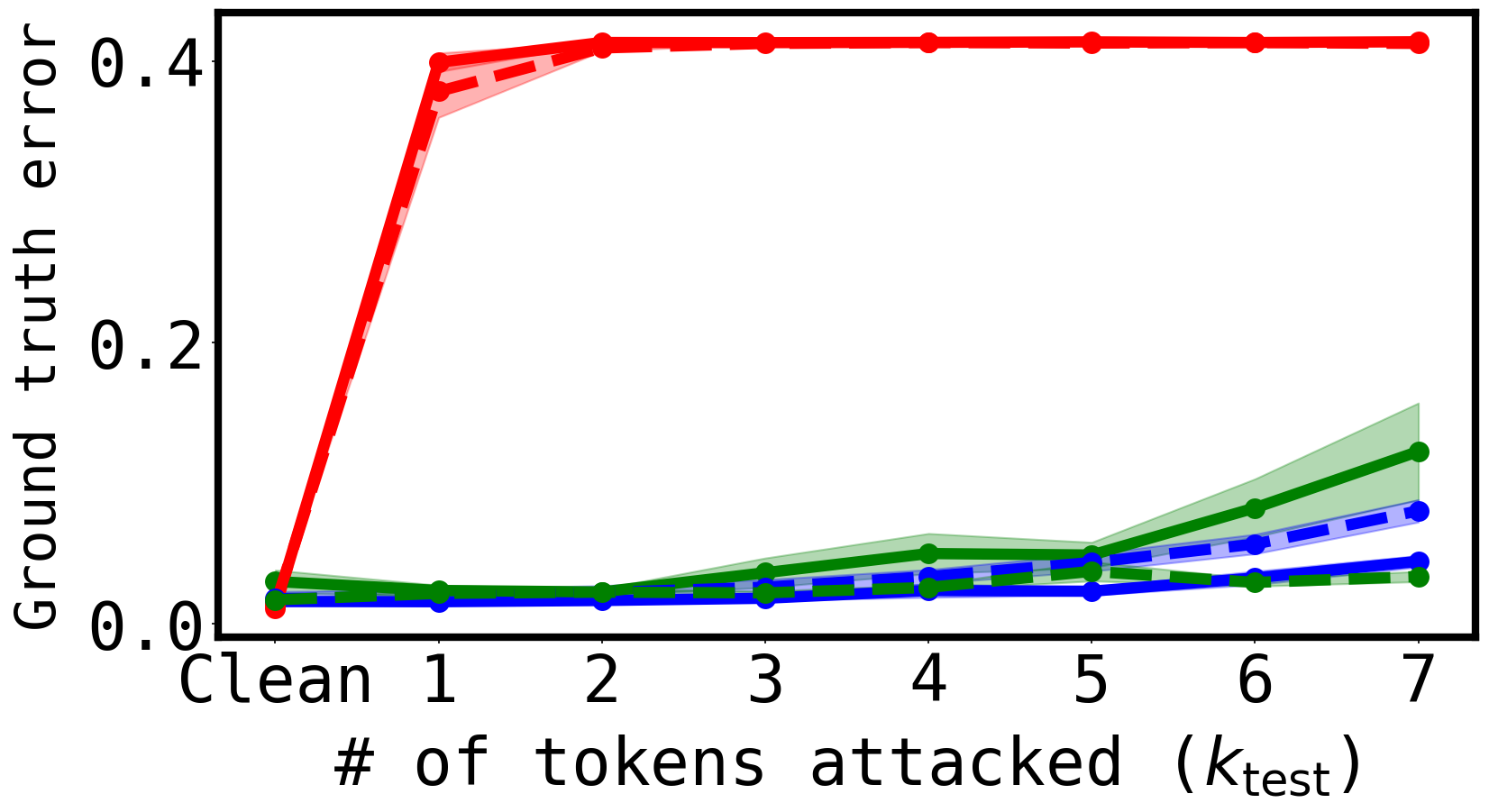}
        \caption{\xattackkk.}
    \end{subfigure}
    \begin{subfigure}[b]{0.32\textwidth}
        \includegraphics[width=\textwidth]{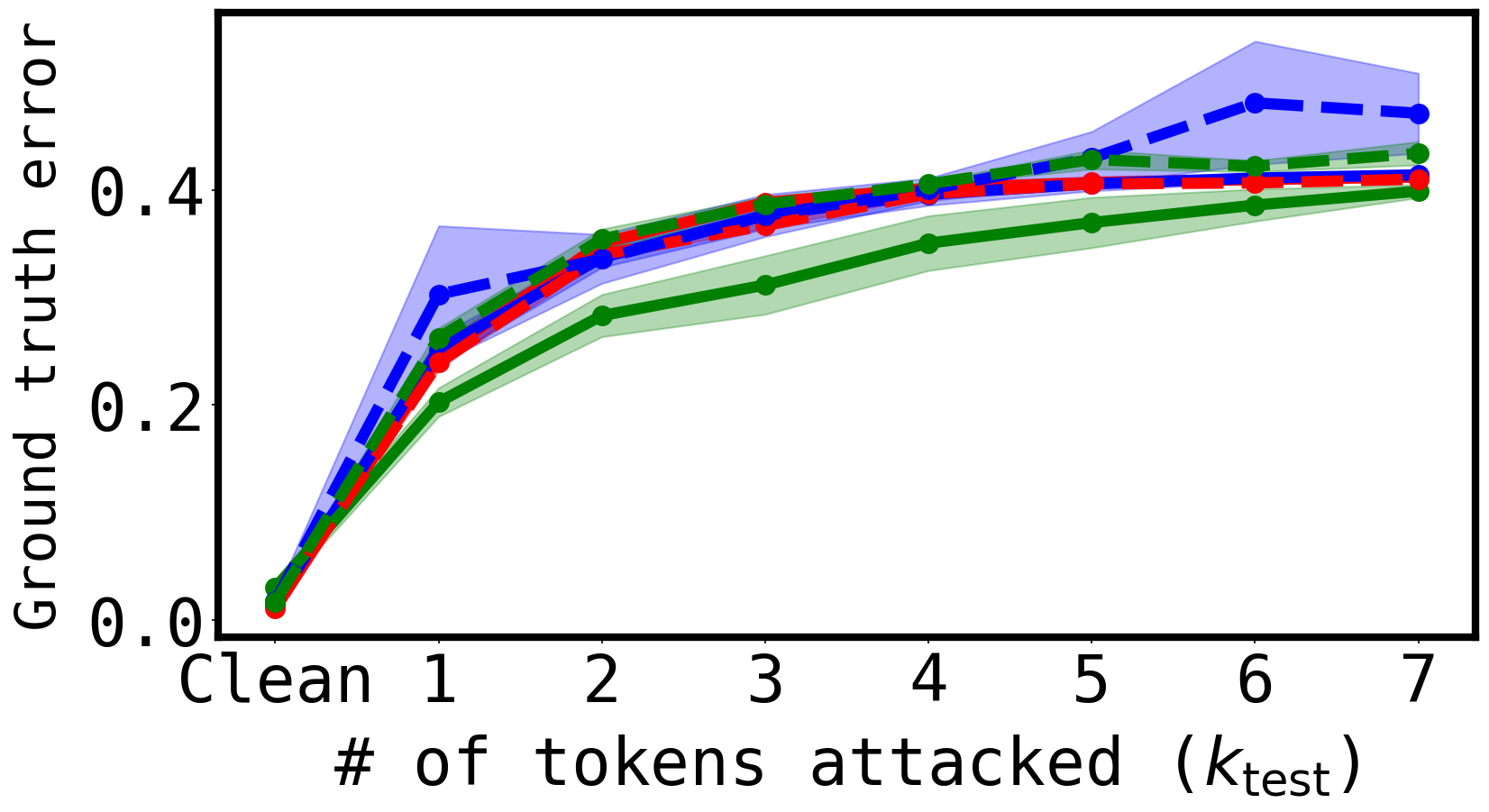}
        \caption{\yattackkk.}
    \end{subfigure}
    \begin{subfigure}[b]{0.32\textwidth}
        \includegraphics[width=\textwidth]{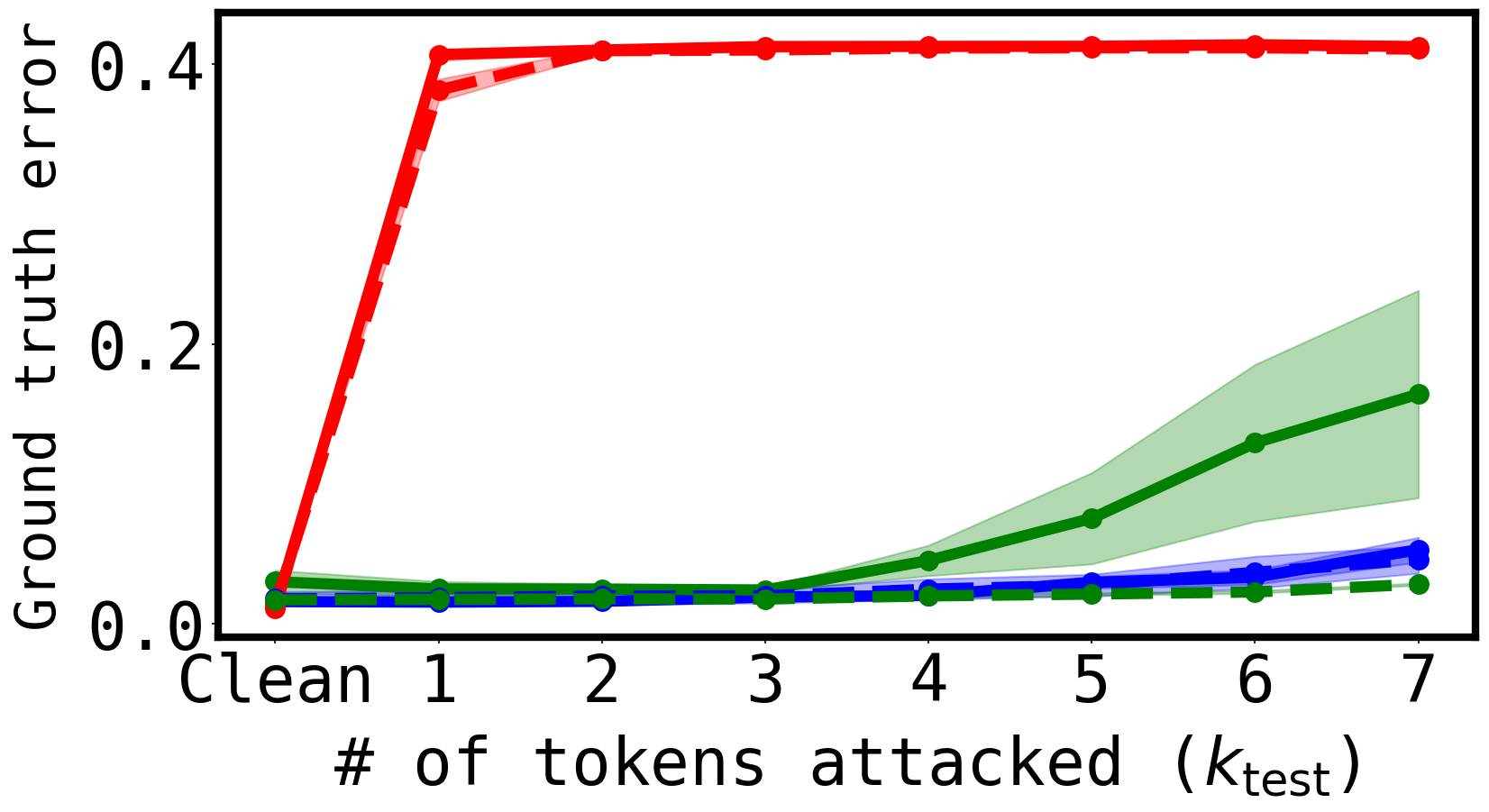}
        \caption{\zattackkk.}
    \end{subfigure}
    \begin{subfigure}[b]{0.32\textwidth}
        \includegraphics[width=\textwidth]{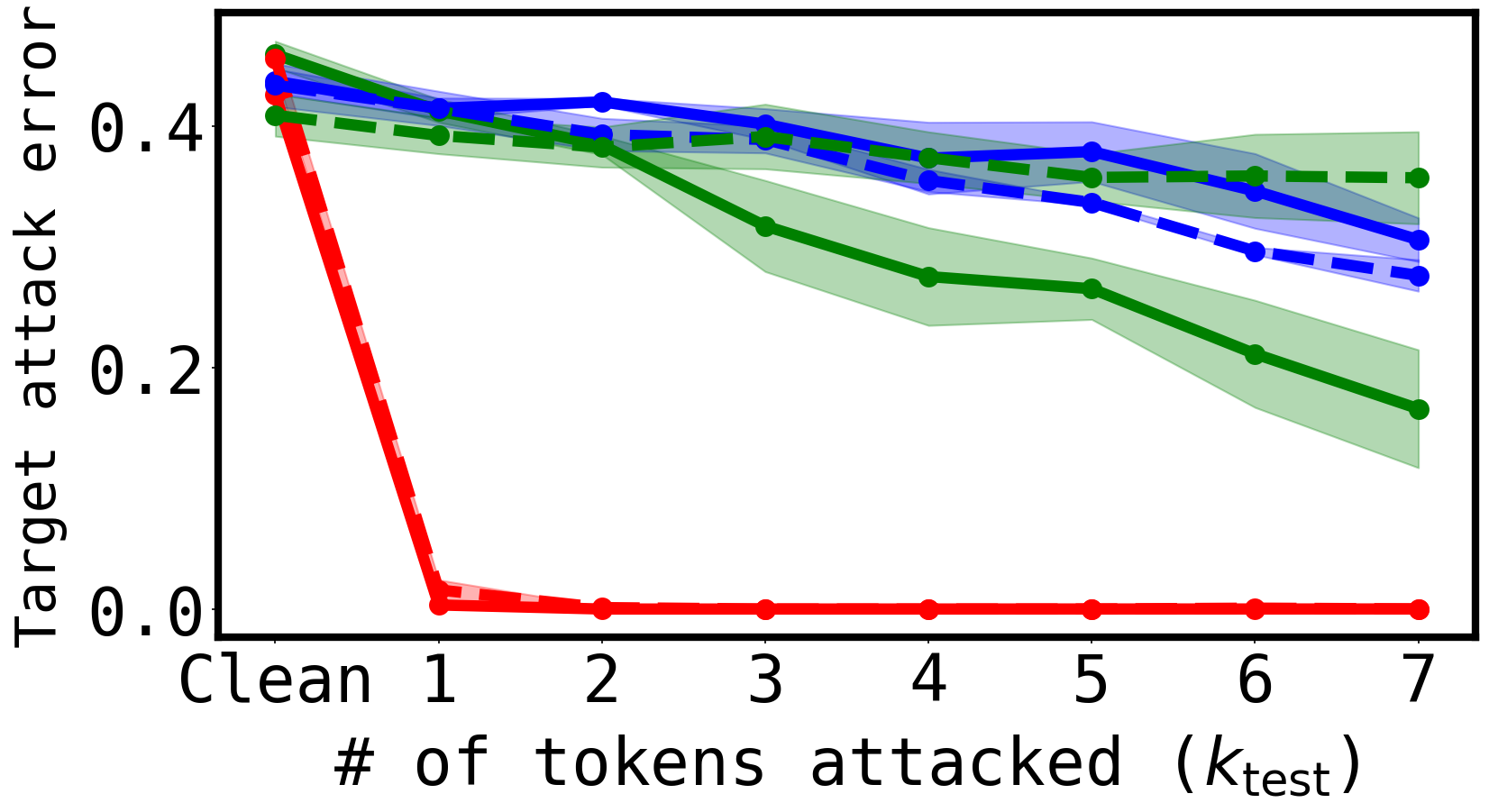}
        \caption{\xattackkk.}
    \end{subfigure}
        \begin{subfigure}[b]{0.32\textwidth}
        \includegraphics[width=\textwidth]{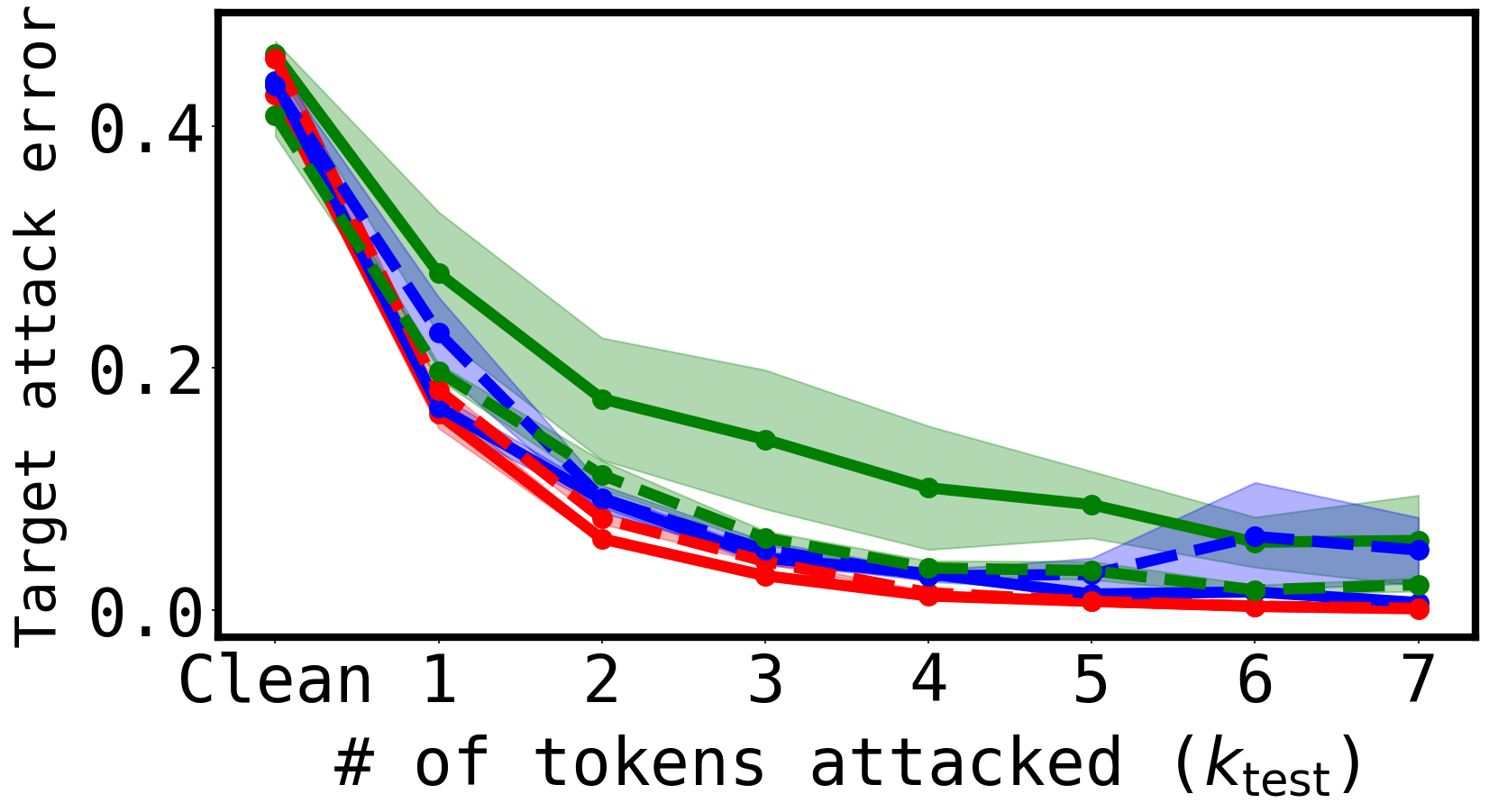}
        \caption{\yattackkk.}
    \end{subfigure}
        \begin{subfigure}[b]{0.32\textwidth}
        \includegraphics[width=\textwidth]{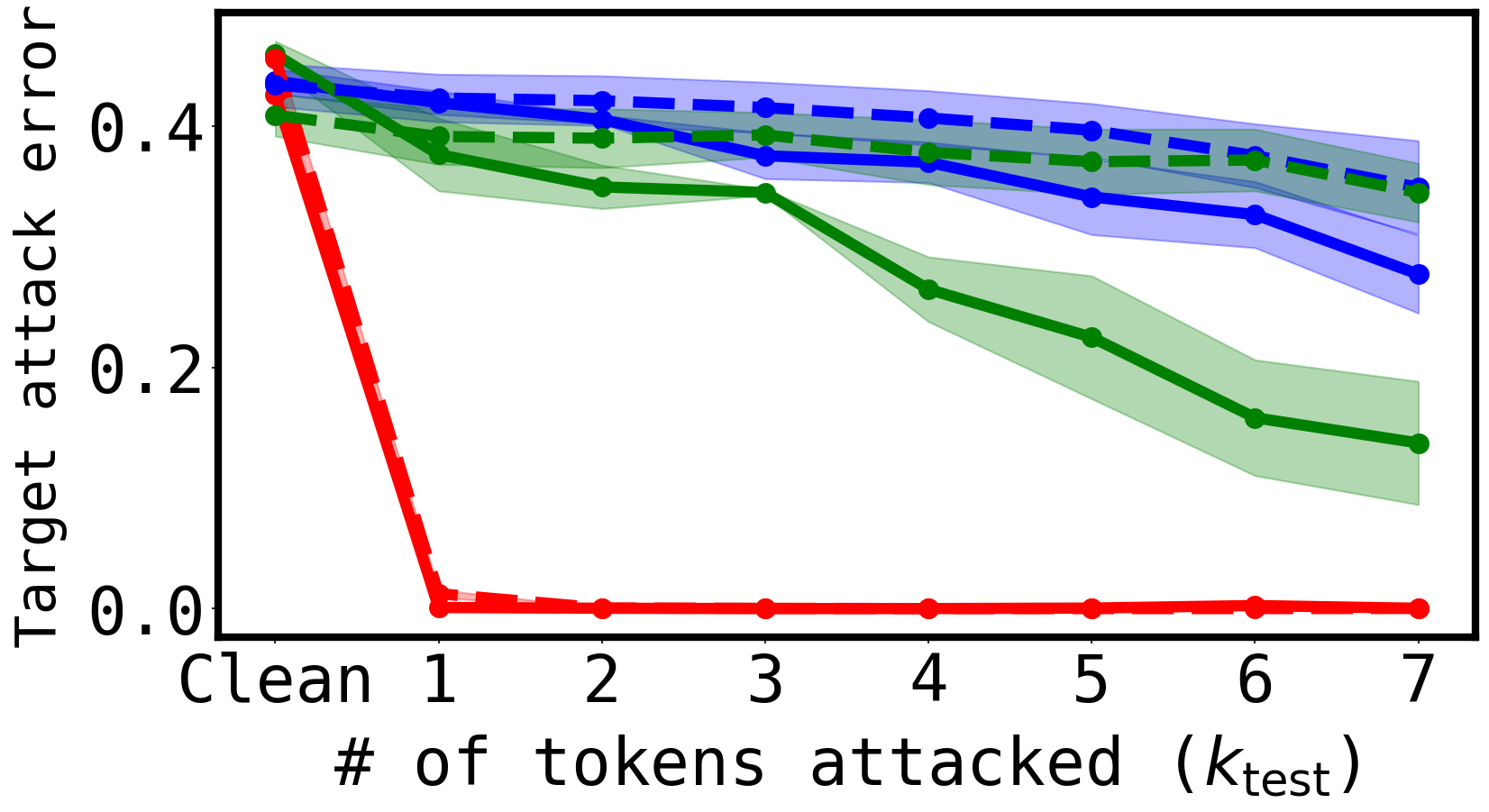}
        \caption{\zattackkk.}
    \end{subfigure}
    \begin{subfigure}[b]{0.98\textwidth}
        \includegraphics[width=\textwidth]{results/adv_training/adv_training_legend.png}
    \end{subfigure}
    \caption{Adversarial training against \zattack.}
    \label{appx.fig:adv.training.z.alpha.0.1}
\end{figure}

\FloatBarrier

\subsection{Transferability}
\label{appx.section.transfer}
In Section~\ref{sec.transfer.across.tf}, we briefly presented some results around transfer of adversarial examples generated using one transformer to other transformers -- either with the same architecture or different architecture.
We present complete results here, for both \xattack{} and \yattack{}.
As in the main text, we first present results for transfer across same class of transformers, i.e., transformers with same number of layers and then present results for transfer across different classes of transformers.
\begin{figure}[h]
    \centering
    \begin{subfigure}[b]{0.13\textwidth}
        \includegraphics[width=\textwidth]{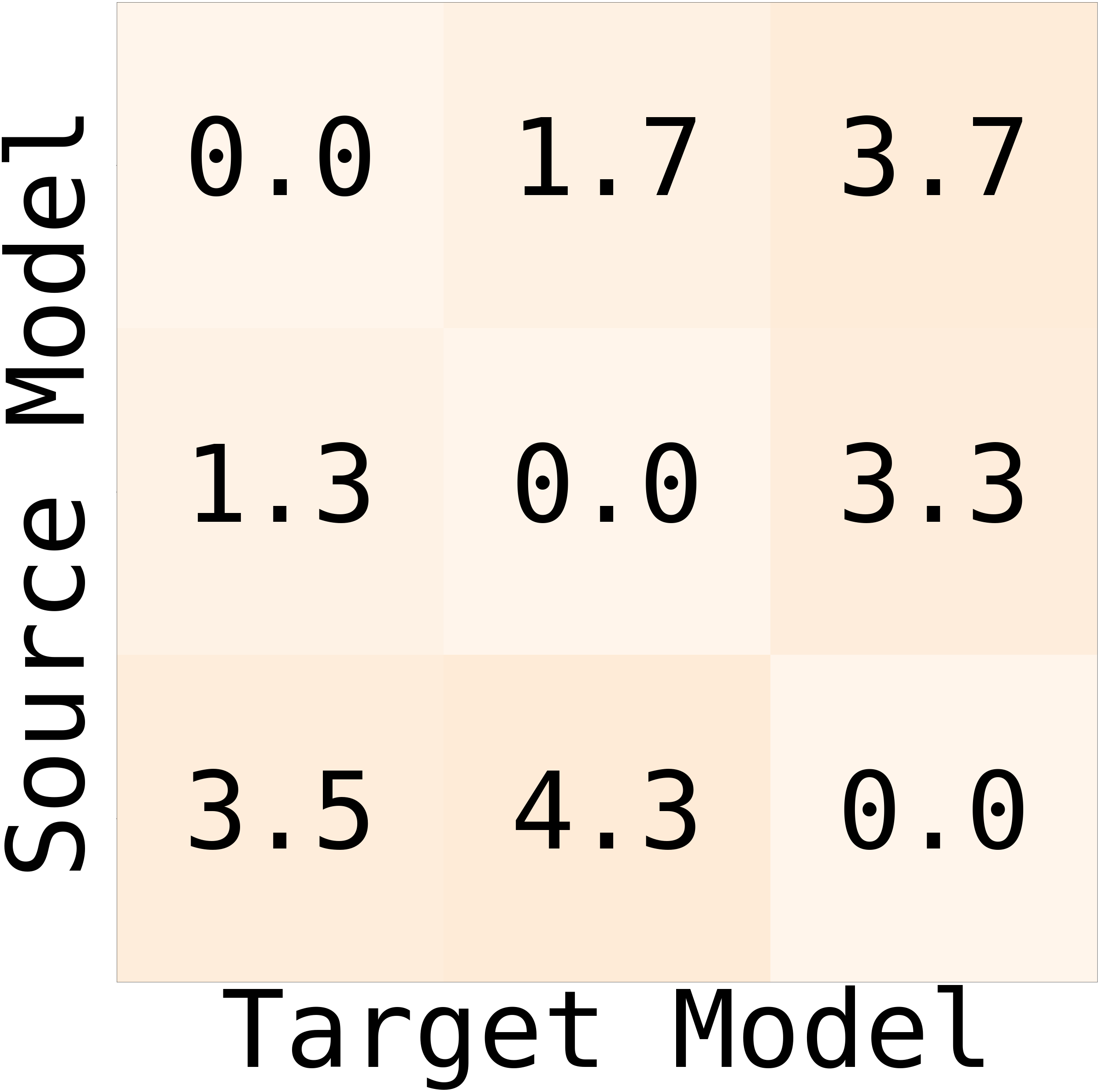}
        \caption{2 layers.}
    \end{subfigure}
    \hfill
    \begin{subfigure}[b]{0.13\textwidth}
        \includegraphics[width=\textwidth]{results/tf_to_tf_transfer_x/x_garg-scale_idxs_3_iters_100_single_modelFalselayers_3.png}
        \caption{3 layers.}
    \end{subfigure}
    \hfill
    \begin{subfigure}[b]{0.13\textwidth}
        \includegraphics[width=\textwidth]{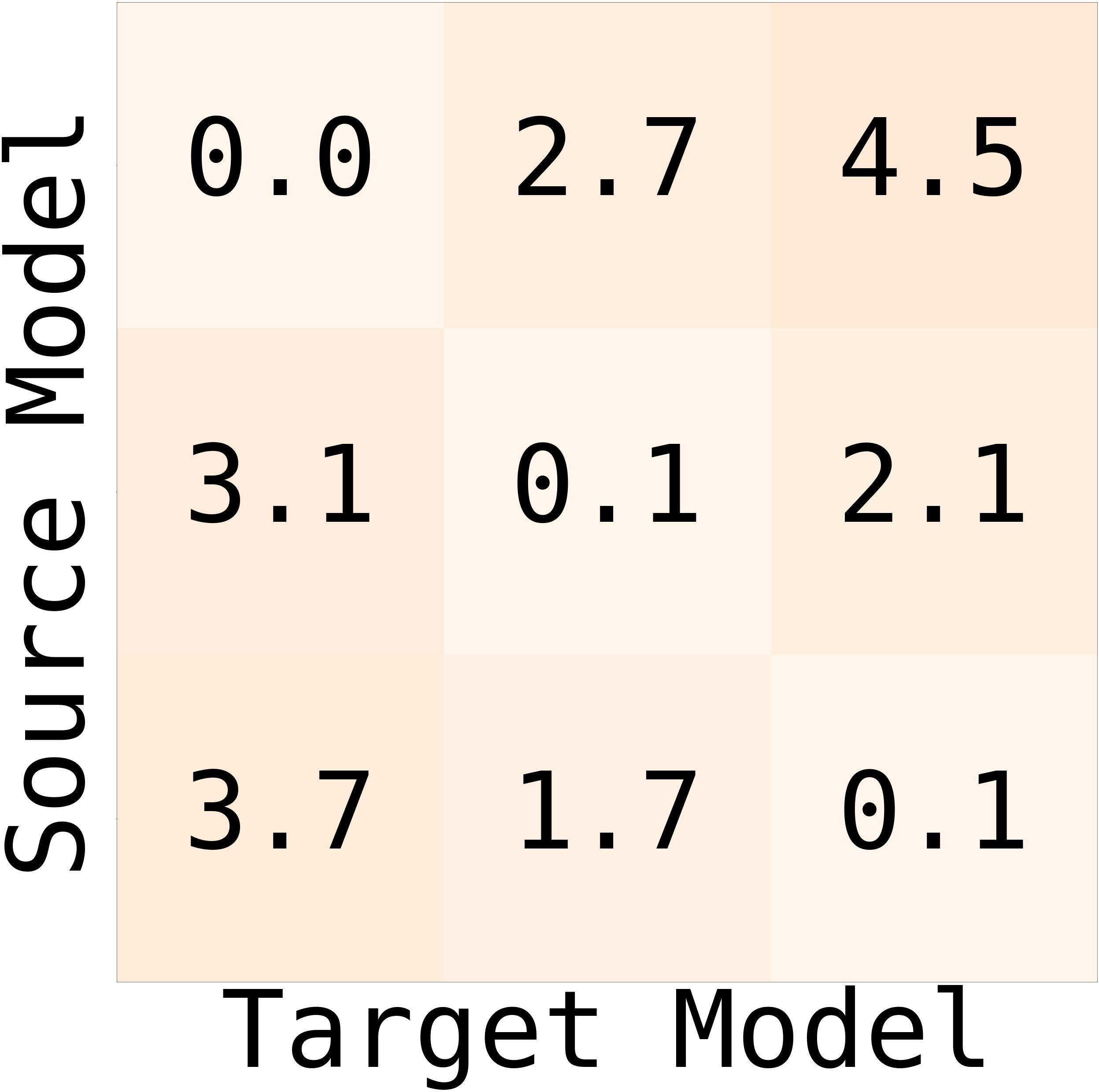}
        \caption{4 layers.}
    \end{subfigure}
    \hfill
    \begin{subfigure}[b]{0.13\textwidth}
        \includegraphics[width=\textwidth]{results/tf_to_tf_transfer_x/x_garg-scale_idxs_3_iters_100_single_modelFalselayers_6.png}
        \caption{6 layers.}
    \end{subfigure}
    \begin{subfigure}[b]{0.13\textwidth}
        \includegraphics[width=\textwidth]{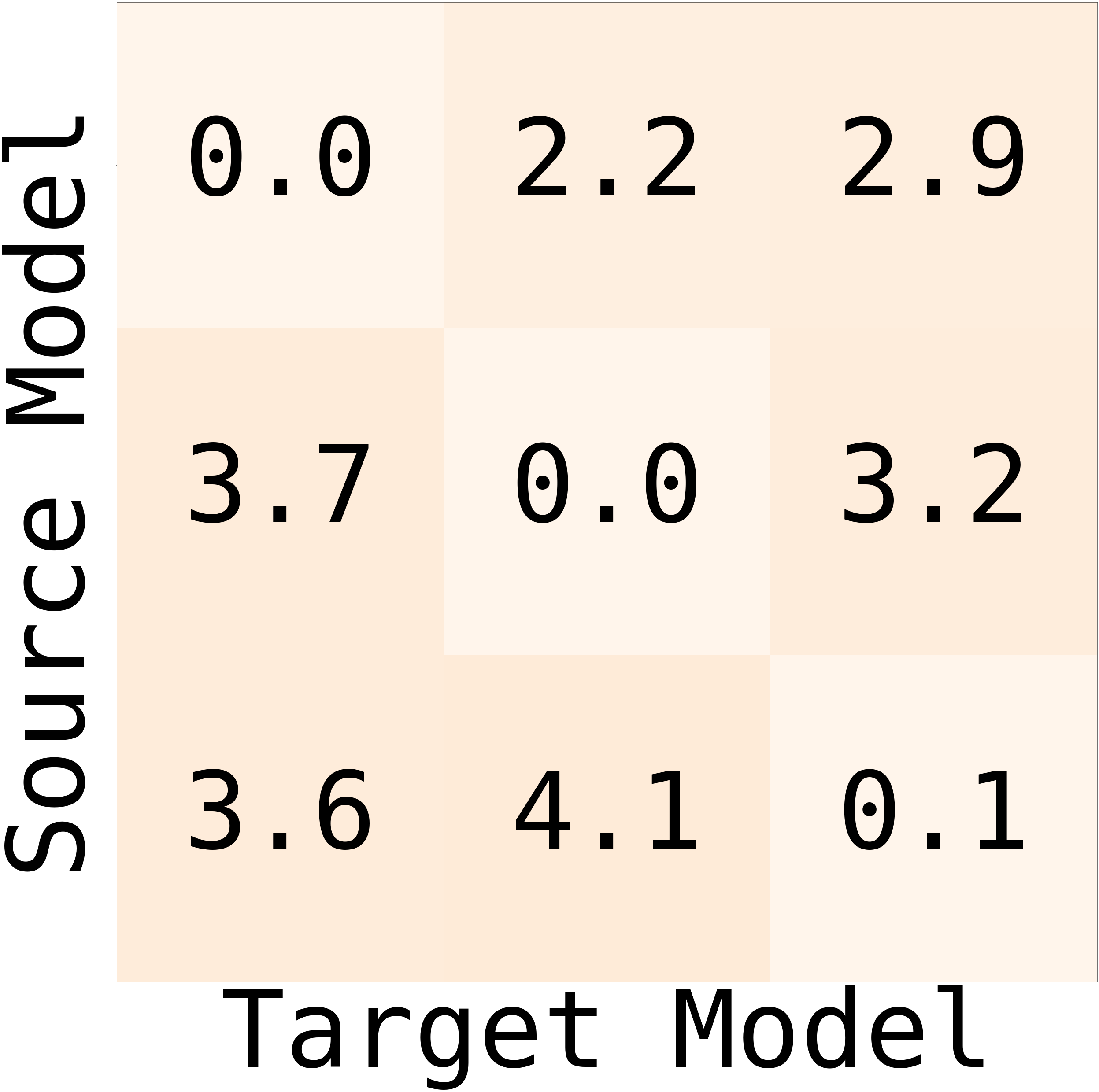}
        \caption{8 layers.}
    \end{subfigure}
    \hfill
    \begin{subfigure}[b]{0.13\textwidth}
        \includegraphics[width=\textwidth]{results/tf_to_tf_transfer_x/x_garg-scale_idxs_3_iters_100_single_modelFalselayers_12.png}
        \caption{12 layers.}
    \end{subfigure}
    \hfill
    \begin{subfigure}[b]{0.13\textwidth}
        \includegraphics[width=\textwidth]{results/tf_to_tf_transfer_x/x_garg-scale_idxs_3_iters_100_single_modelFalselayers_16.png}
        \caption{16 layers.}
    \end{subfigure}
    \caption{\textit{Target Attack Error} for different target models on adversarial samples generated using a source model with the same number of layers. Adversarial samples were generated using \xattack{} with $k=3$. Transfer of adversarial samples across transformers progressively becomes poorer as number of layers increases.}
    \label{appx.fig.transfer.xadv.tf.to.tf.same.layers}
\end{figure}
\begin{figure}[h]
    \centering
    \begin{subfigure}[b]{0.13\textwidth}
        \includegraphics[width=\textwidth]{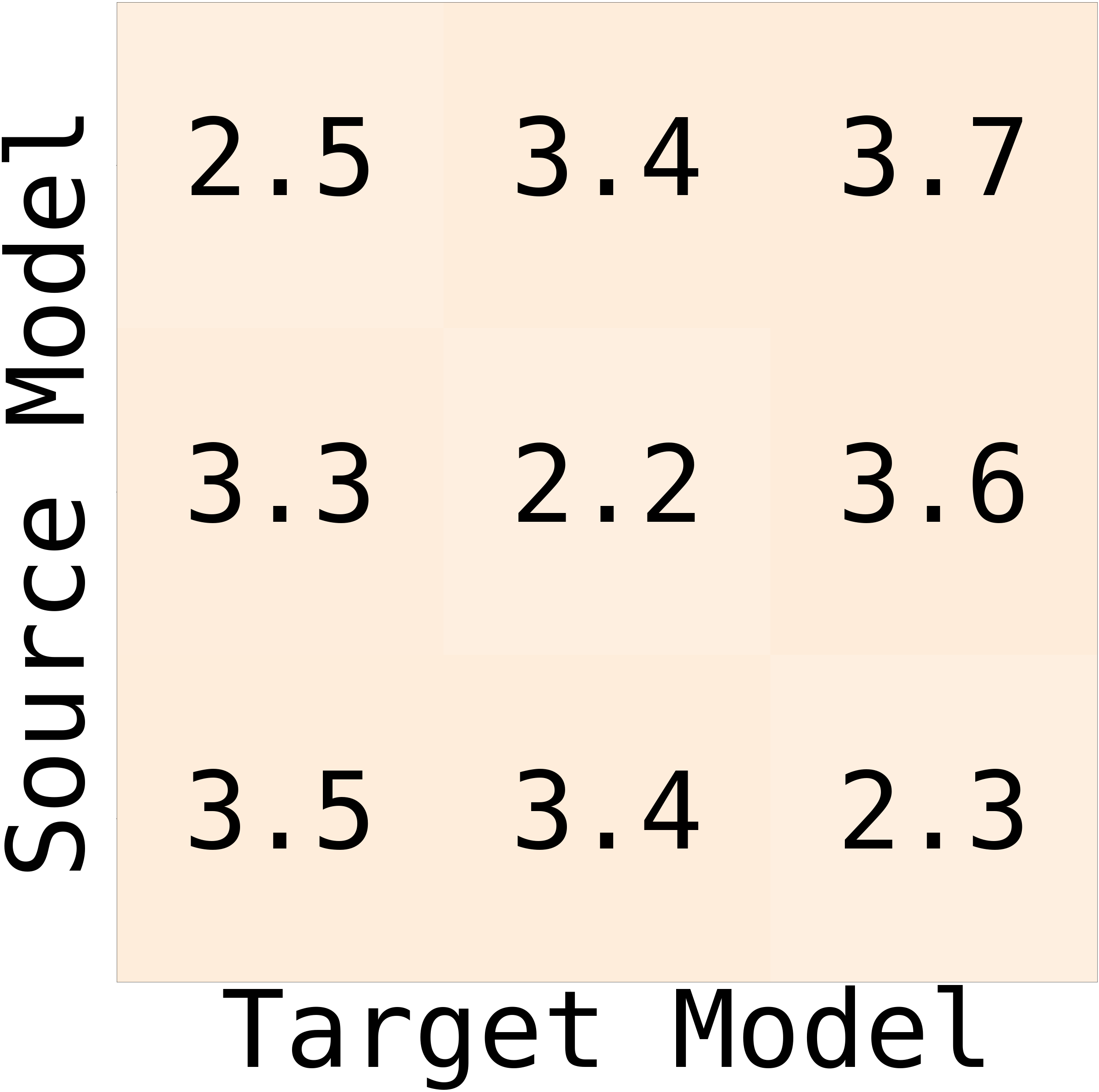}
        \caption{2 layers.}
    \end{subfigure}
    \hfill
    \begin{subfigure}[b]{0.13\textwidth}
        \includegraphics[width=\textwidth]{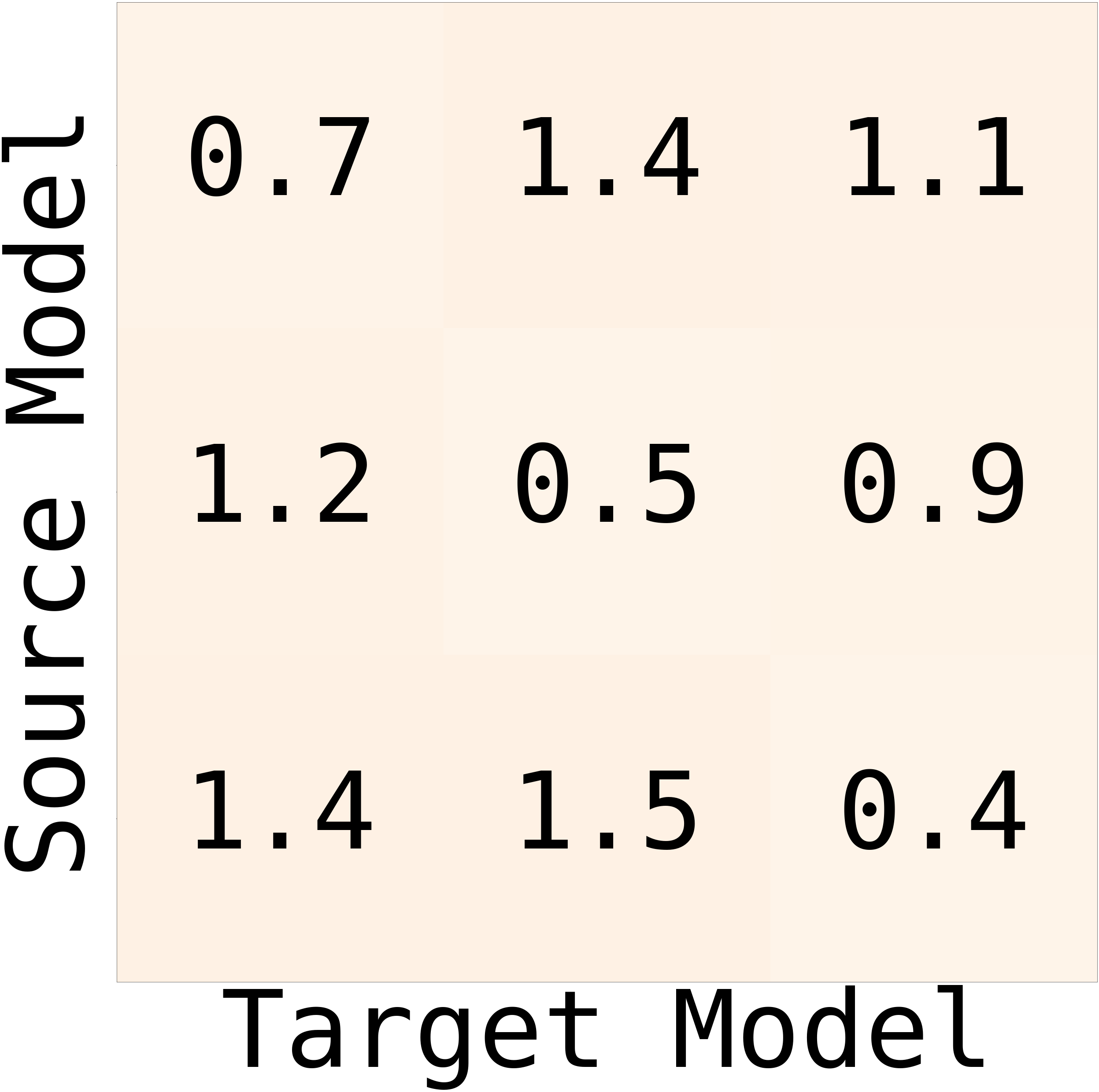}
        \caption{3 layers.}
    \end{subfigure}
    \hfill
    \begin{subfigure}[b]{0.13\textwidth}
        \includegraphics[width=\textwidth]{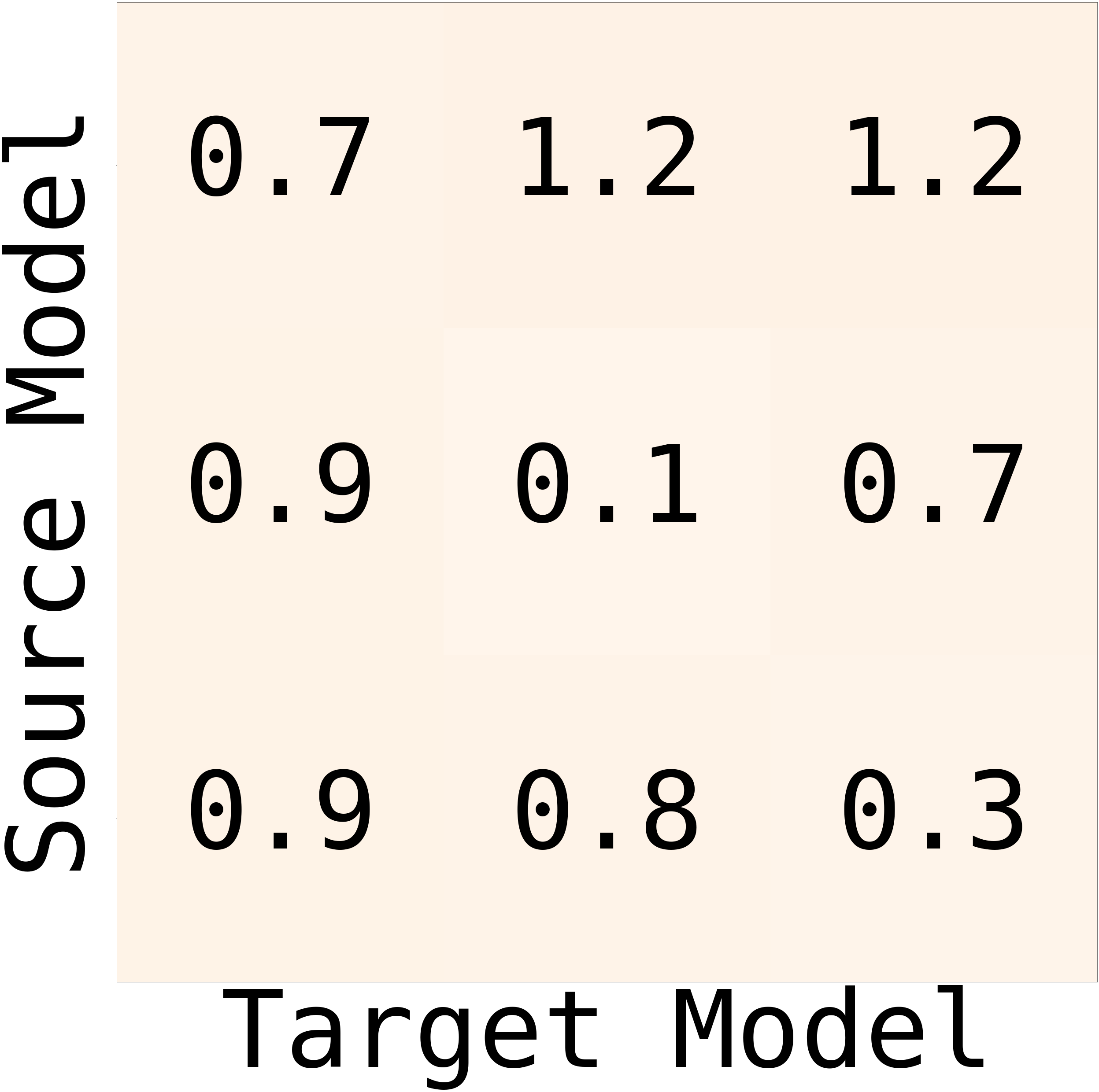}
        \caption{4 layers.}
    \end{subfigure}
    \hfill
    \begin{subfigure}[b]{0.13\textwidth}
        \includegraphics[width=\textwidth]{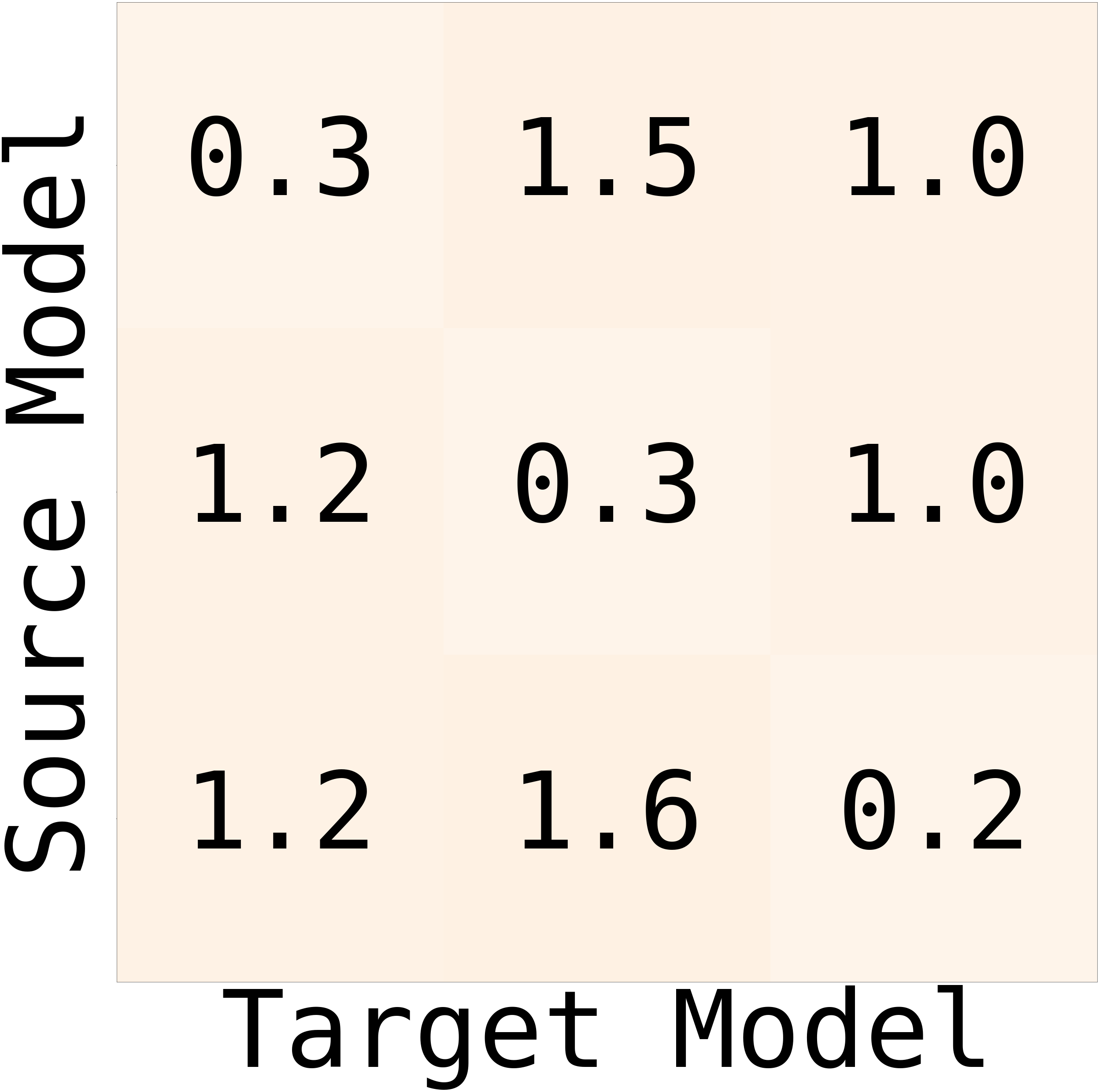}
        \caption{6 layers.}
    \end{subfigure}
    \begin{subfigure}[b]{0.13\textwidth}
        \includegraphics[width=\textwidth]{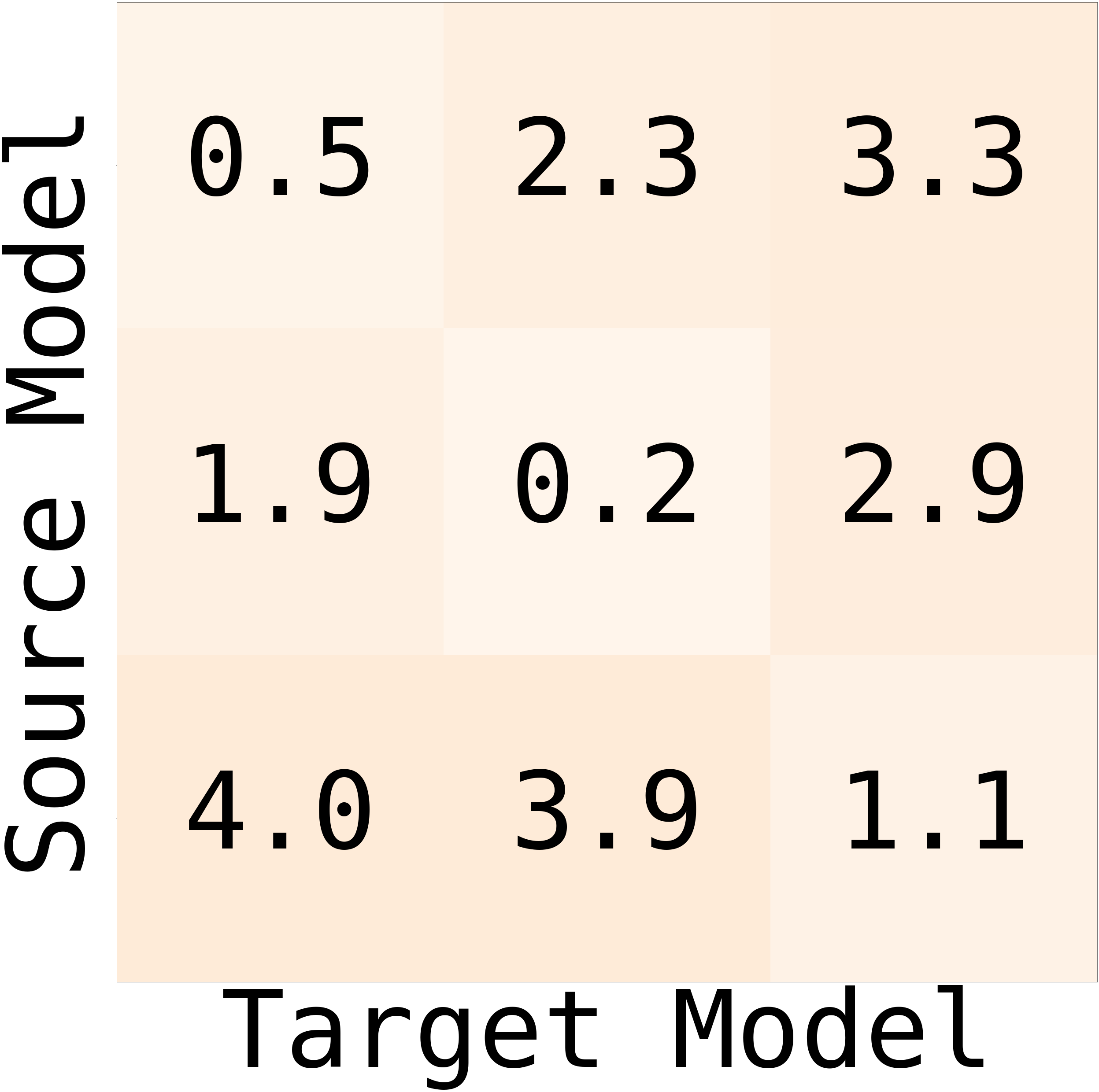}
        \caption{8 layers.}
    \end{subfigure}
    \hfill
    \begin{subfigure}[b]{0.13\textwidth}
        \includegraphics[width=\textwidth]{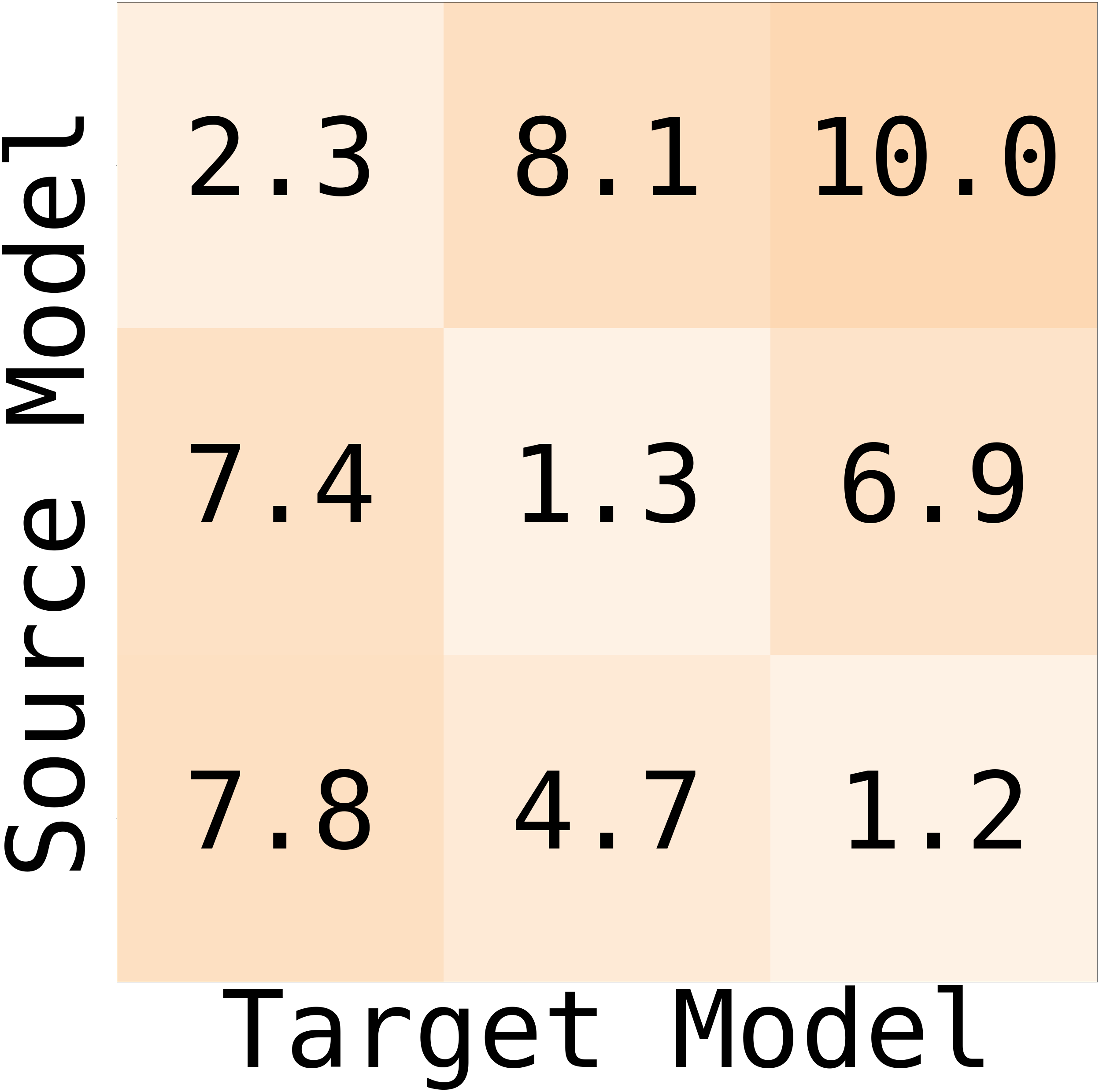}
        \caption{12 layers.}
    \end{subfigure}
    \hfill
    \begin{subfigure}[b]{0.13\textwidth}
        \includegraphics[width=\textwidth]{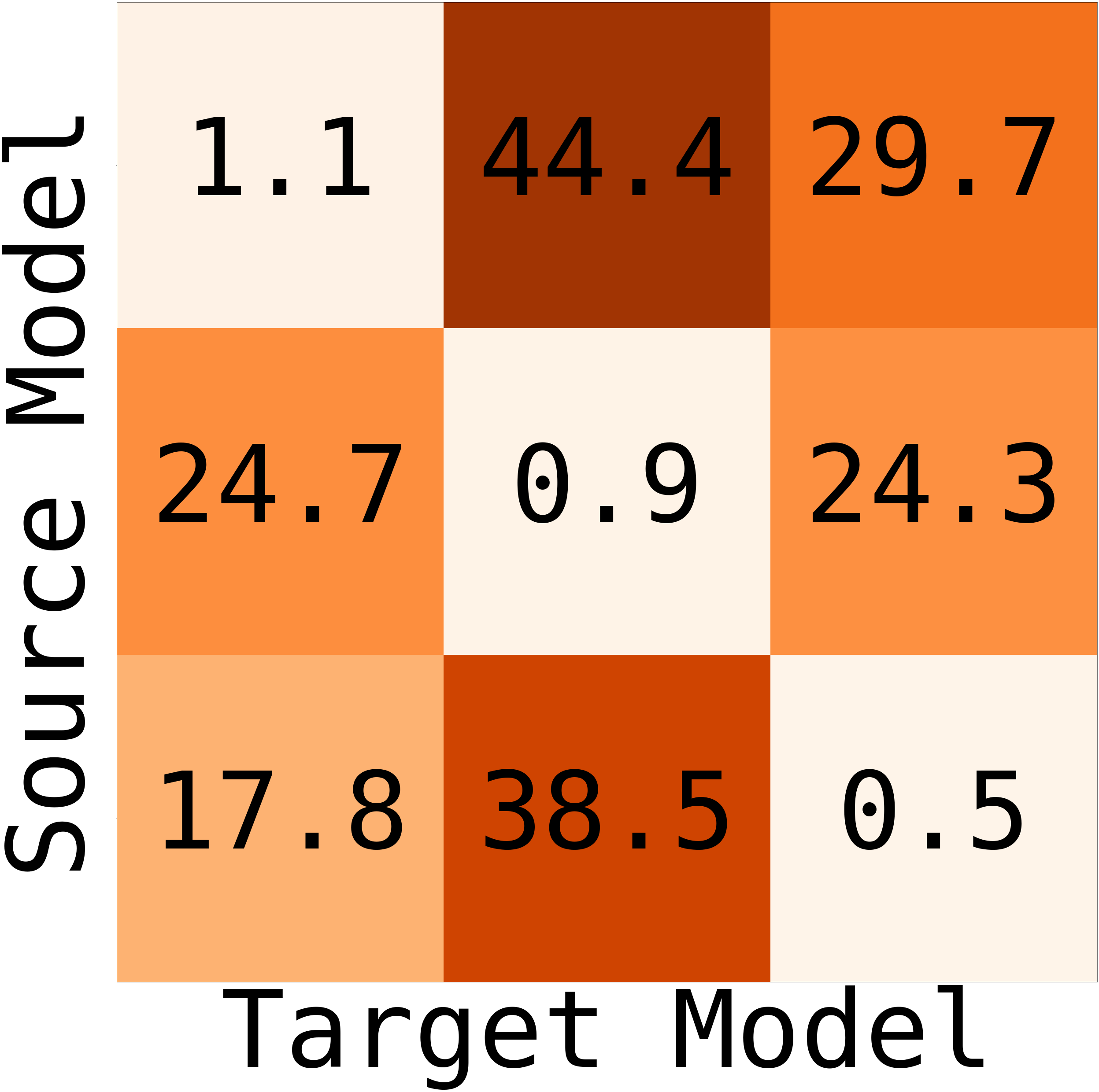}
        \caption{16 layers.}
    \end{subfigure}
    \caption{Same as above figure (\ref{appx.fig.transfer.xadv.tf.to.tf.same.layers}) but adversarial samples were generated using \yattack{} with $k=7$. As with \xattack{}, transfer of adversarial samples samples across transformers progressively becomes poorer as number of layers increases.}
    \label{appx.fig.transfer.yadv.tf.to.tf.same.layers}
\end{figure}

\begin{figure}
    \begin{subfigure}[b]{0.49\textwidth}
        \includegraphics[width=\textwidth]{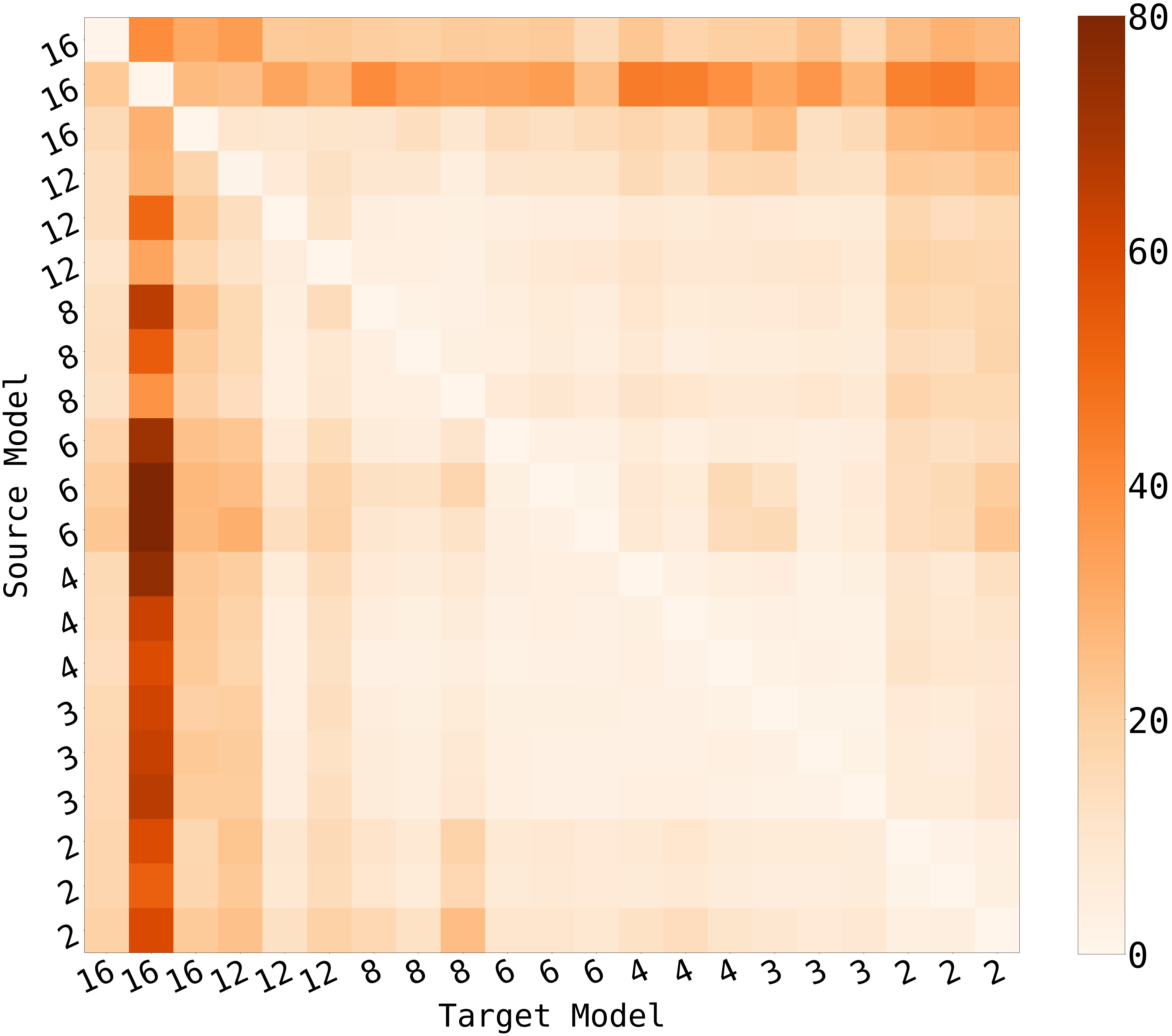}
        \caption{}
    \end{subfigure}
    \hfill
    \begin{subfigure}[b]{0.49\textwidth}
        \includegraphics[width=\textwidth]{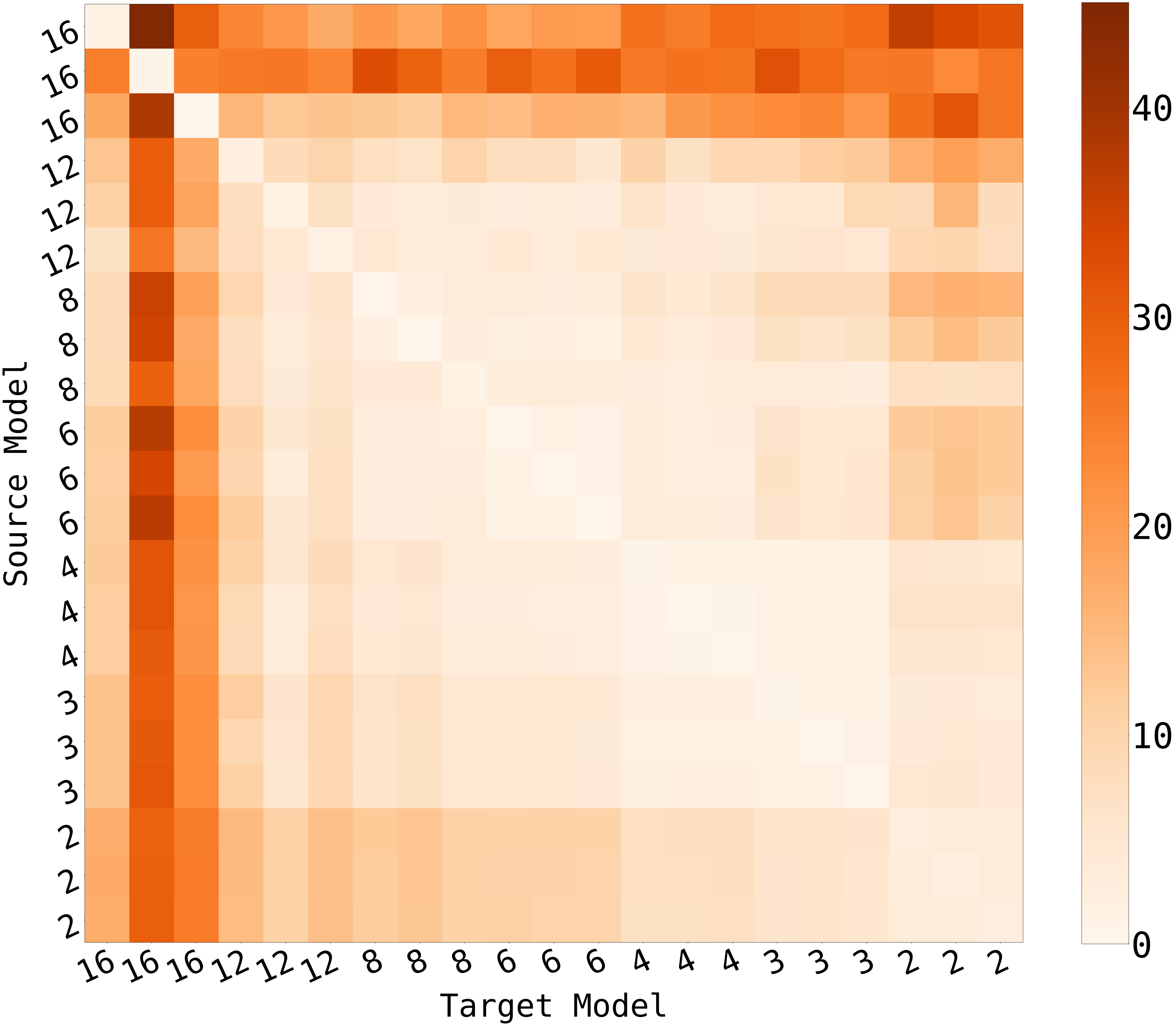}
        \caption{}
    \end{subfigure}
        \caption{Target Attack Error for different target models on adversarial samples possibly generated using a source model with a different number of layers. In (a) adversarial samples were generated using \xattack{} with $k=3$. In (b) adversarial samples were generated using \yattack{} with $k=7$. Transfer is generally worse when}
        \label{appx.fig.transfer.tf.to.tf.across.layers}
\end{figure}

\FloatBarrier

\subsection{Hijacking Attacks on Ordinary Least Square}
\label{appx.sec.transfer.ols.transfer}
Linear regression can be solved using ordinary least square.
This solution can be written in closed-form as follow:
\begin{align}
    \hat y &= f(X, Y, x_\query) = \left(X^\top X\right)^{-1}X^\top Y x_\query
\end{align}
where $X = [x_1^\top; x_2^\top; \cdots;x_N^\top]$ and $Y = [y_1, ..., y_N]$.
We implement a gradient-based adversarial attack on this solver by using Jax autograd to calculate the gradients $\nabla_{X} f(X, Y, x_\query)$ and $\nabla_{Y} f(X, Y, x_\query)$. 
Similar to our gradient-based attack on the transformer, we only update a randomly chosen subset of entries withing $X$ and $Y$.
In OLS, $X$ and $Y$ are not tokenized, however, for consistency of language, we will continue to refer to the individual entries of these matrices, i.e., $x_i,y_i$ as tokens.
We perform $1000$ iterations and use a learning rate of $0.01$ for both \xattack{} and \yattack.

Figure \ref{appx.fig.ols.attack} shows results for \xattack{} and y-attack respectively on OLS for $y_\bad$ chosen by using $\alpha=1.0$.
The adversarial robustness of OLS is qualitatively similar to that of the transformer;
for a fixed compute budget, single-token \yattack{} are much less successful compared to single-token \xattack,
and target attack error is lower when greater number of tokens are attacked.

\ifbool{iclrtemp}%
{\begin{figure}[b]
    \centering
    \begin{subfigure}[b]{0.39\textwidth}
       \includegraphics[width=\textwidth]{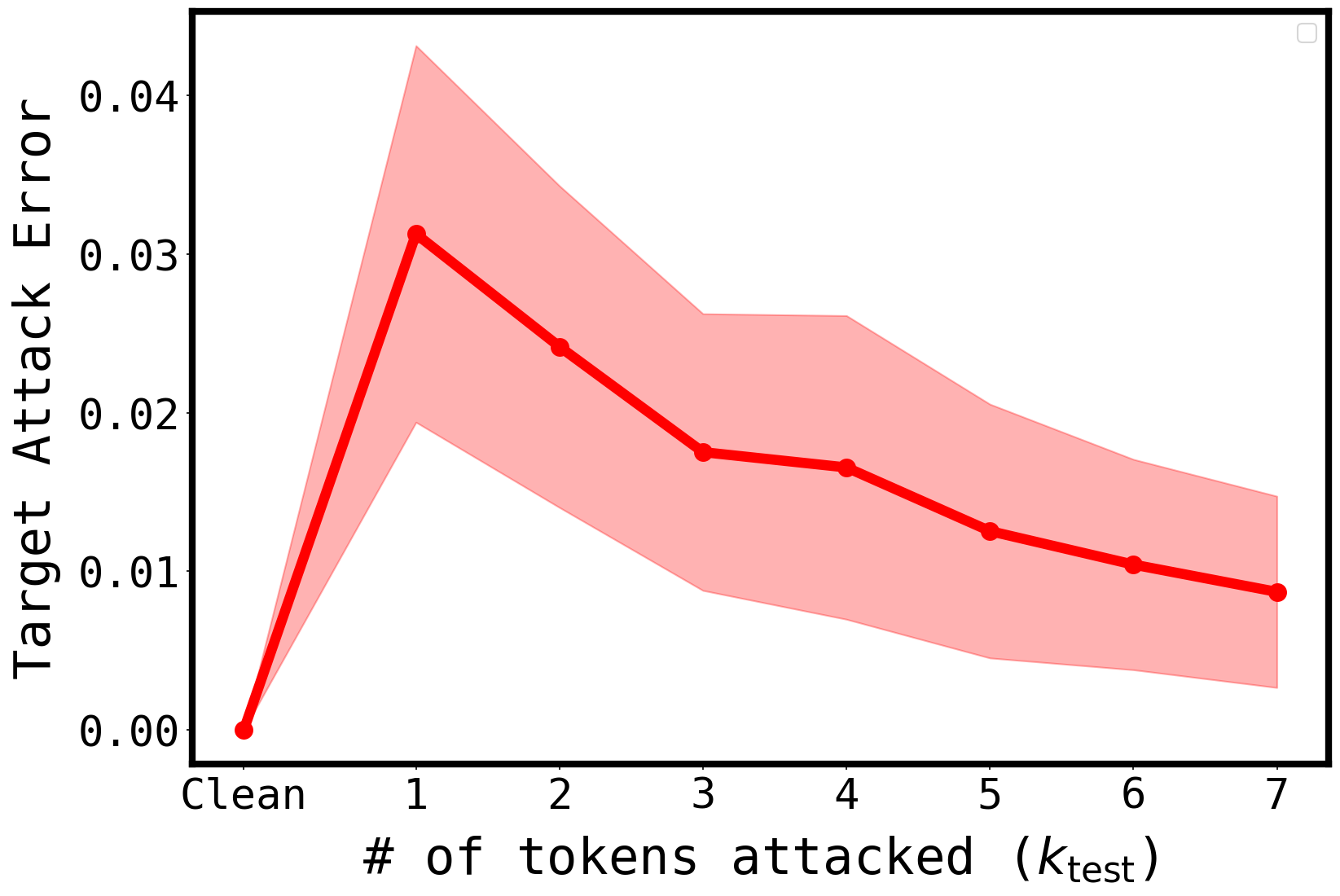}
    \caption{\xattackkk.}
    \end{subfigure}
    \hspace{5em}
    \begin{subfigure}[b]{0.38\textwidth}
       \includegraphics[width=\textwidth]{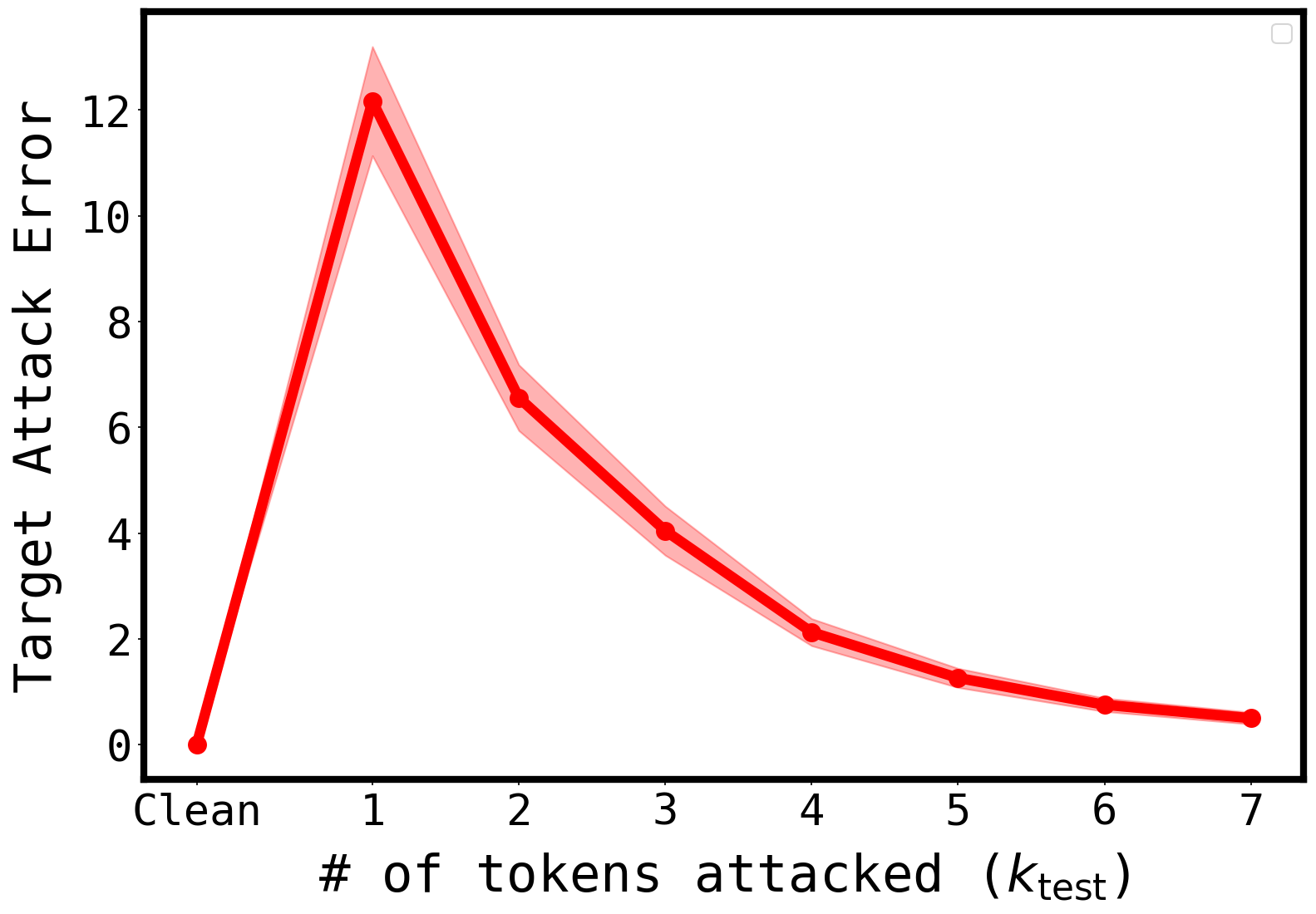}
    \caption{\yattackkk.}
    \end{subfigure}
    \caption{The adversarial robustness of ordinary least squares to gradient-based hijacking attacks is qualitatively similar to that of the transformers.}
    \label{appx.fig.ols.attack}
\end{figure}}%
{\begin{figure}[t]
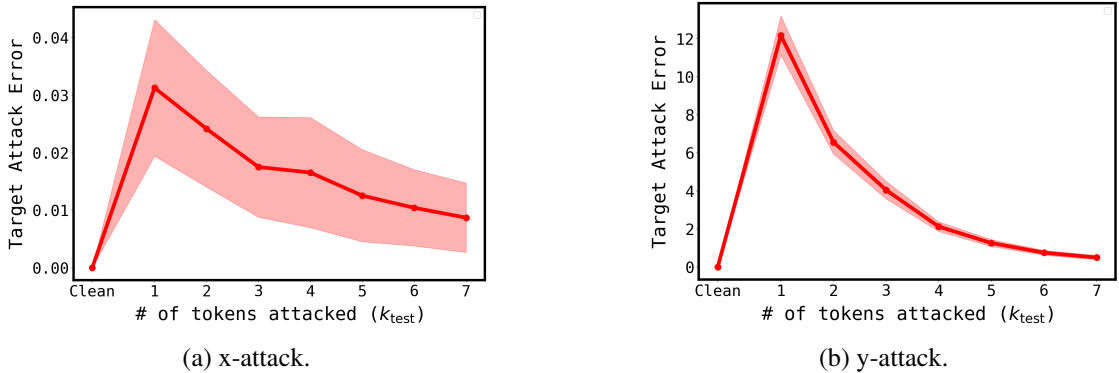

    \centering
    \begin{subfigure}[b]{0.39\textwidth}
       \includegraphics[width=\textwidth]{results/ols/OLS_attack_x_idxs_8_alpha_1.0_max_iters_1000.png}
    \caption{\xattackkk.}
    \end{subfigure}
    \hspace{5em}
    \begin{subfigure}[b]{0.38\textwidth}
       \includegraphics[width=\textwidth]{results/ols/OLS_attack_y_idxs_8_alpha_1.0_max_iters_1000.png}
    \caption{\yattackkk.}
    \end{subfigure}
    \caption{The adversarial robustness of ordinary least squares to gradient-based hijacking attacks is qualitatively similar to that of the transformers.}
    \label{appx.fig.ols.attack}
\end{figure}}

We further look at the transfer of adversarial attacks between transformers and OLS.
Specifically, by attacking OLS we create a set of adversarial samples and then measure the mean squared error (MSE) between the predictions of OLS and different transformers on these adversarial samples, and vice versa.
Figure~\ref{appx.fig.ols.tf.transfer.alpha.plot} shows the transfer for adversarial samples for different values of $\alpha$ for sampling $y_\bad$.
For \xattack, we attack $3$ indices and for y-attack, we attack $7$ indices.
We can make following observations from this figure: (i) the predictions made by OLS and transformers tend to diverge as $\alpha$ increases. This indicates lack of alignment between the predictions made by OLS and transformers OOD. 
(ii) For \xattack, MSE between predictions is significantly lower when adversarial samples are sourced by attacking OLS relative to when adversarial samples are sourced by attacking the transformers. In other words, adversarial samples transfer better from OLS to transformers but not vice versa. This hints at the fact that adversarial robustness of the transformers is worse than that of OLS.
(iii) For y-attack, the aforementioned asymmetry in transfer above does not exist except for transformers with layers $16$ and $12$.
(iv) Finally, we note that transformer with $16$ layers clearly always behaves in an anomalous fashion, with transformers with layers $12$ and $2$ also sometimes behaving anomalously, which is in line with the discussion in previous section on intra-transformer transfer of adversarial samples.

In Figure~\ref{appx.fig.ols.tf.transfer.idxs.plot}, we present complementary results showing MSE between predictions of OLS and transformers on adversarial samples when different number of tokens are attacked for $\alpha=1.0$. These results further support the observations made in the previous paragraph.

\newcommand{\figurewidth}{\ifbool{iclrtemp}{0.39\textwidth}{0.3\textwidth}}

\begin{figure}
    \centering
    \begin{subfigure}[b]{\figurewidth}
       \includegraphics[width=\textwidth]{results/OLS_TF_Transfer/ols_to_tf_attack_x_idxs_3_max_iters_1000.png}
       \caption{\xattackkk: OLS $\rightarrow$ Transformers.}
    \end{subfigure}
    \hspace{5em}
    \begin{subfigure}[b]{\figurewidth}
       \includegraphics[width=\textwidth]{results/OLS_TF_Transfer/tf_to_ols_attack_x_idxs_3.png}
       \caption{\xattackkk: Transformers $\rightarrow$ OLS.}
    \end{subfigure}
    \vspace{2em}
    \begin{subfigure}[b]{\figurewidth}
       \includegraphics[width=\textwidth]{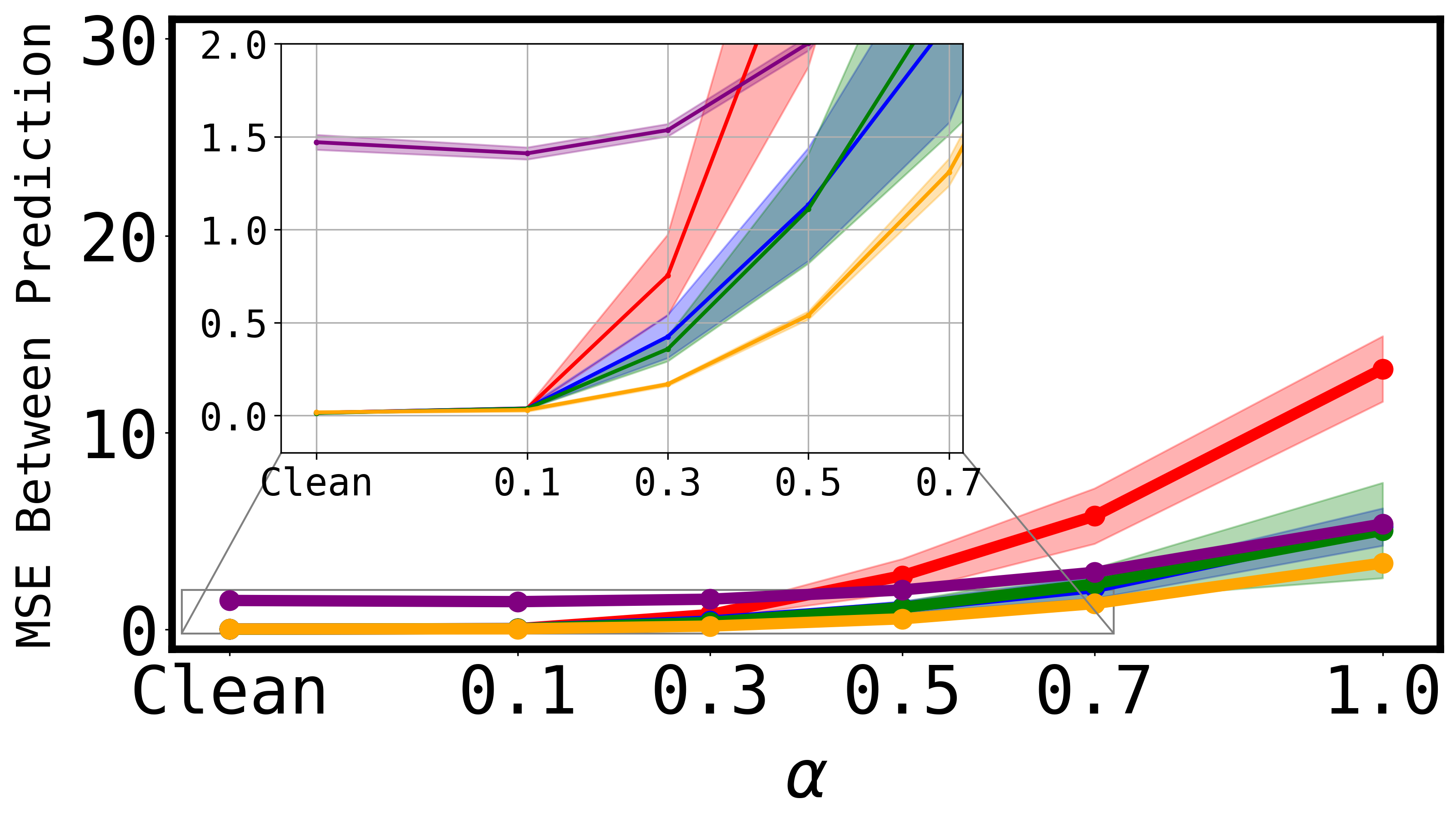}
       \caption{\yattackkk: OLS $\rightarrow$ Transformers.}
    \end{subfigure}
    \hspace{5em}
    \begin{subfigure}[b]{\figurewidth}
       \includegraphics[width=\textwidth]{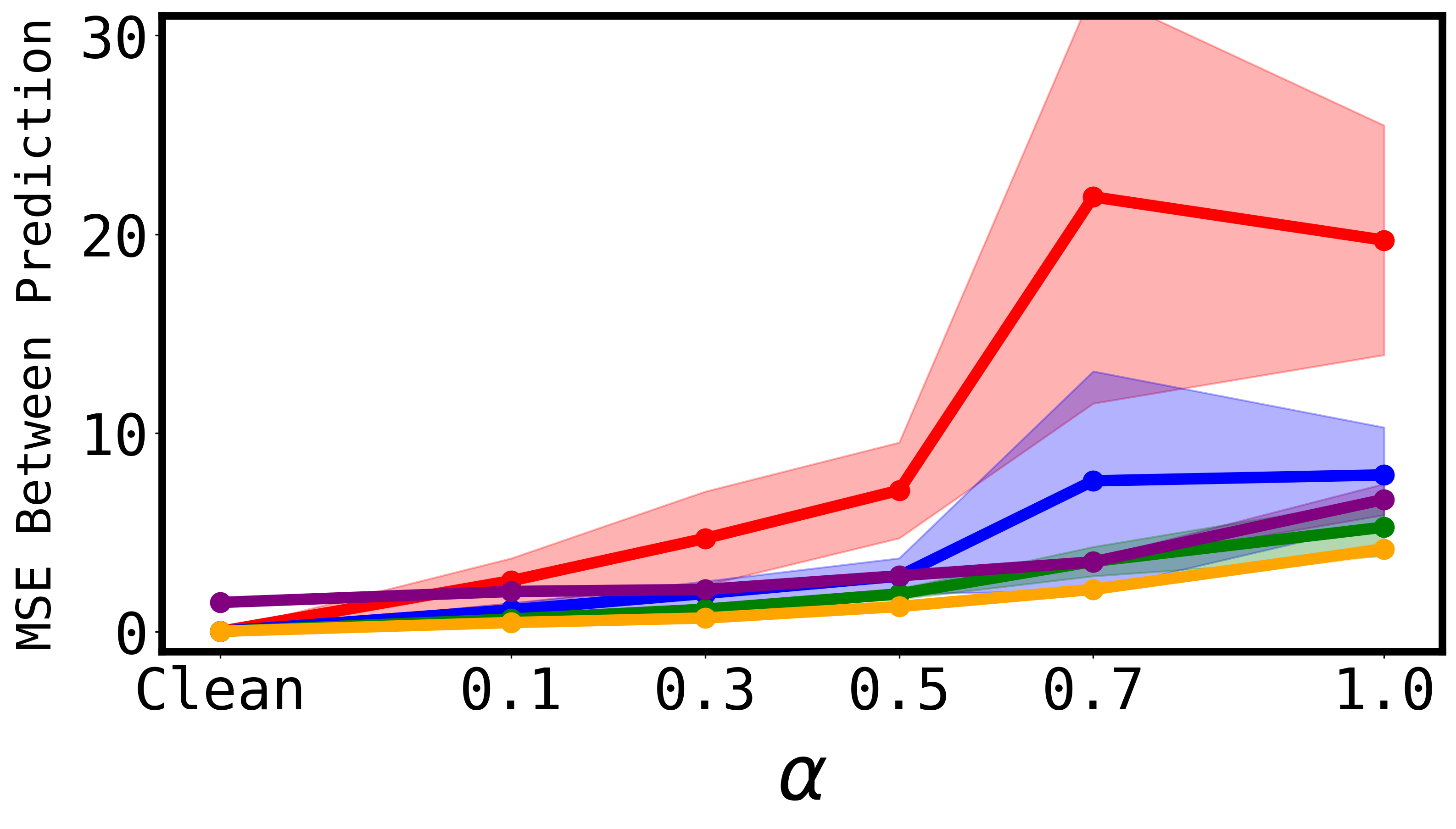}
       \caption{\yattackkk: Transformers $\rightarrow$ OLS.}
    \end{subfigure}
     \begin{subfigure}[b]{0.8\textwidth}
       \includegraphics[width=\textwidth]{results/OLS_TF_Transfer/ols_attack_x_idxs_3_legend.png}
    \end{subfigure}   
    \caption{The mean squared error between the predictions being made by the transformer and OLS on adversarial samples tends to increase as the `OOD-ness' of the $y_\bad$ increases. Furthermore, the difference in prediction is generally higher when the hijacking attacks are derived using the transformer (notice the differences in scale). For \xattack, we attack 3 tokens and for y-attack we attack 7 tokens when creating adversarial samples.}
    \label{appx.fig.ols.tf.transfer.alpha.plot}
\end{figure}

\begin{figure} 
    \centering
    \begin{subfigure}[b]{\figurewidth}
       \includegraphics[width=\textwidth]{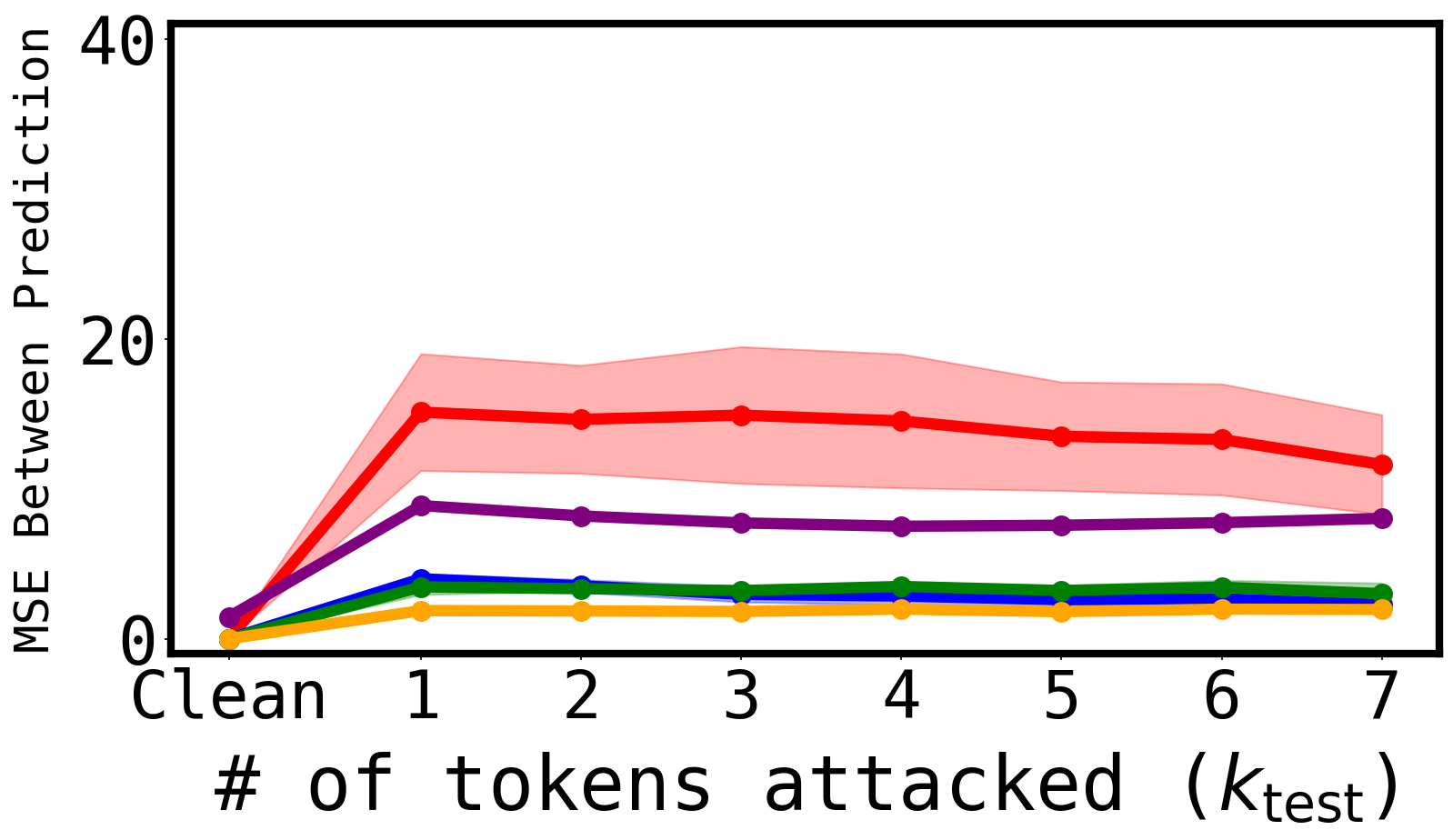}
       \caption{\xattackkk: OLS $\rightarrow$ Transformers.}
    \end{subfigure}
    \hspace{5em}
    \begin{subfigure}[b]{\figurewidth}
       \includegraphics[width=\textwidth]{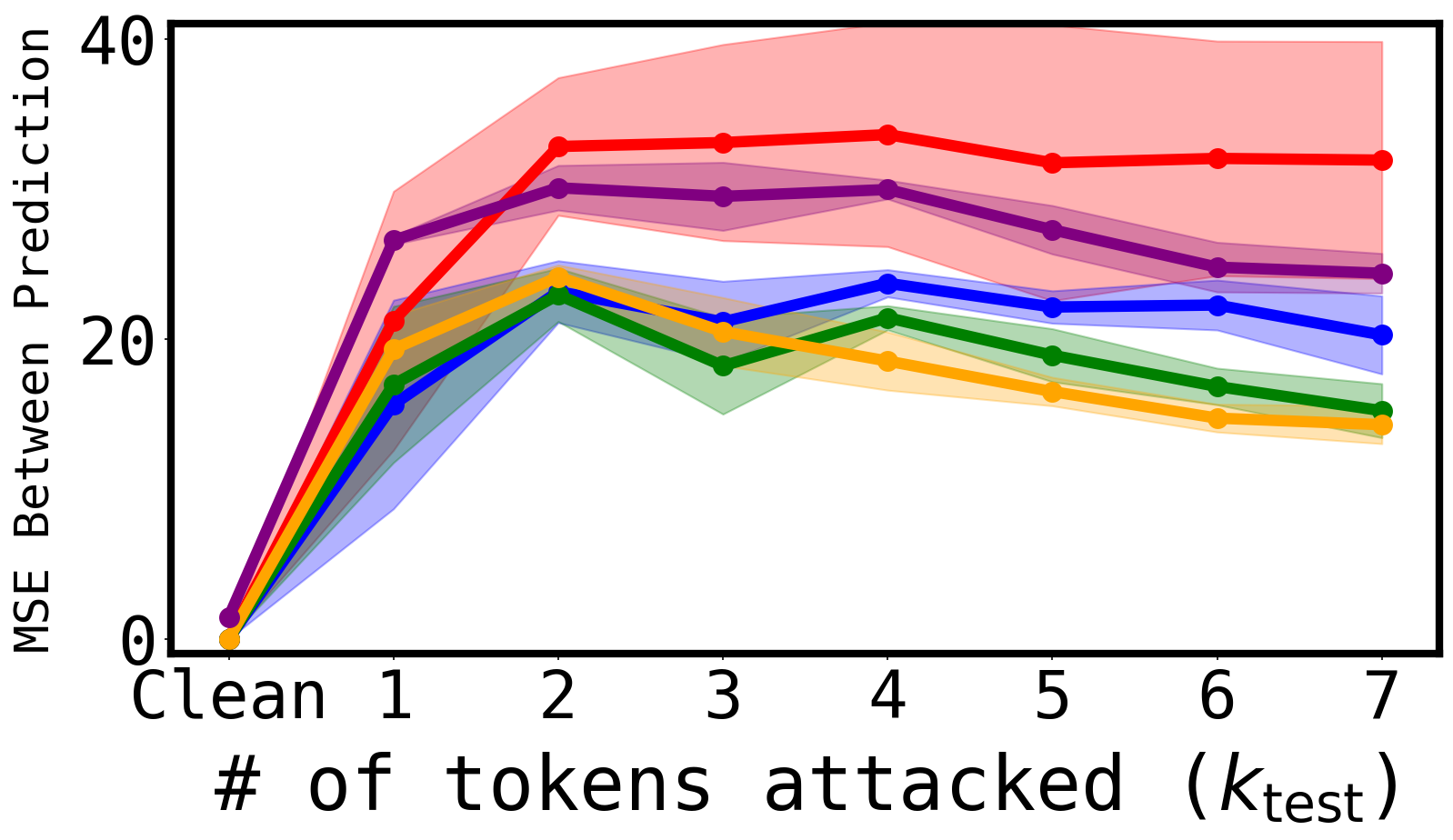}
       \caption{\xattackkk: Transformers $\rightarrow$ OLS.}
    \end{subfigure}
    \vspace{2em}
    \begin{subfigure}[b]{\figurewidth}
       \includegraphics[width=\textwidth]{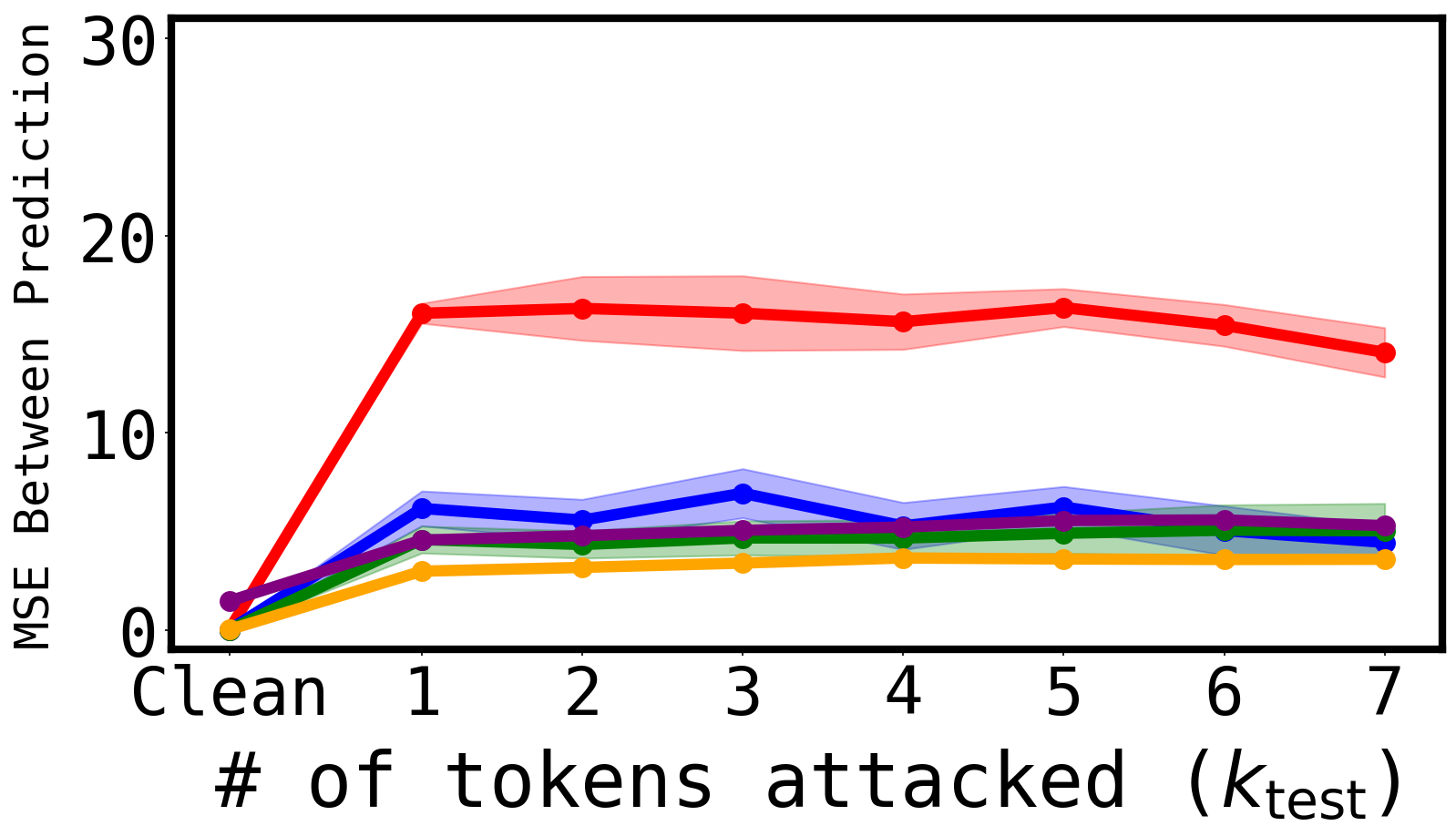}
       \caption{\yattackkk: OLS $\rightarrow$ Transformers.}
    \end{subfigure}
    \hspace{5em}
    \begin{subfigure}[b]{\figurewidth}
       \includegraphics[width=\textwidth]{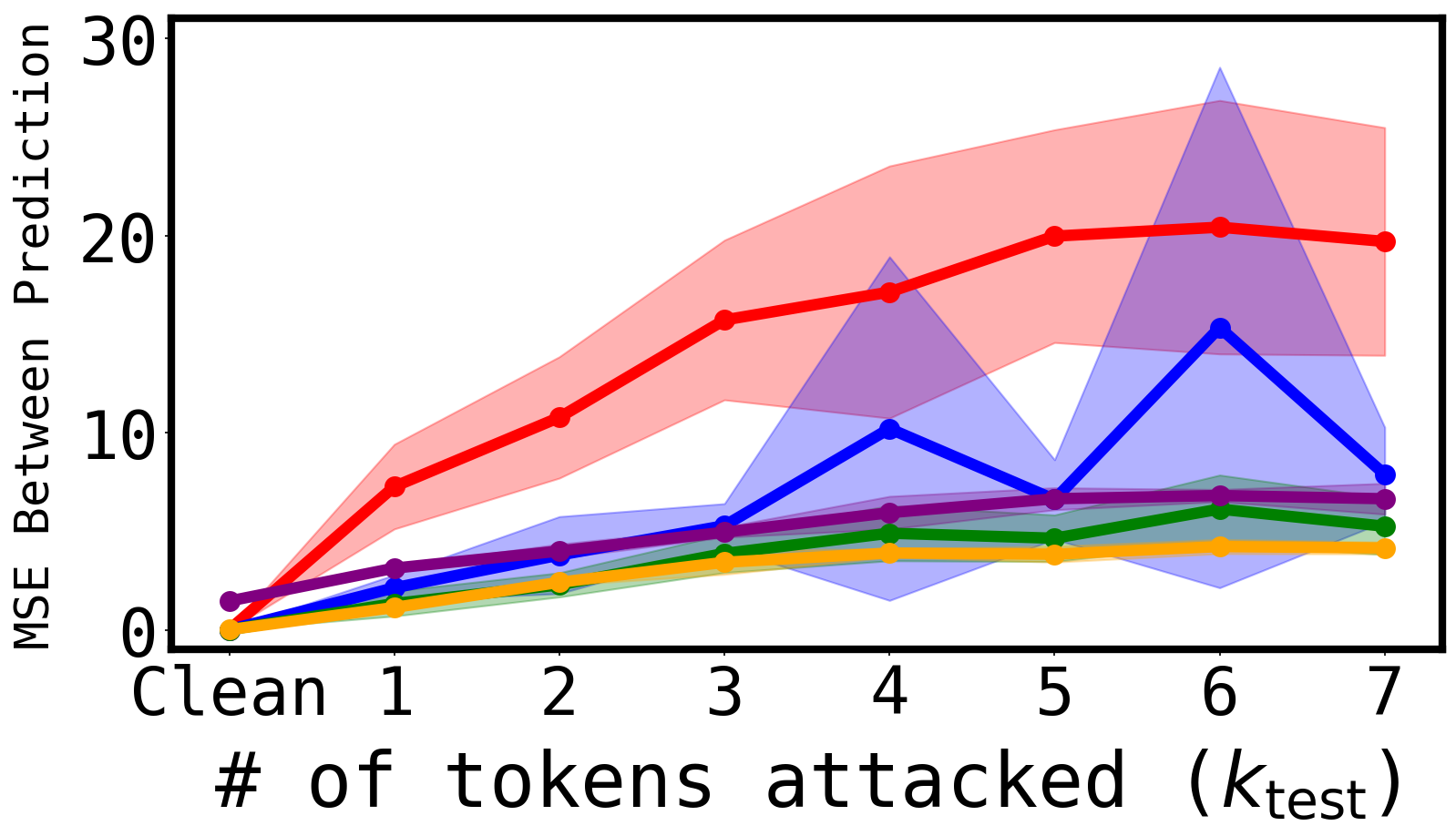}
       \caption{\yattackkk: Transformers $\rightarrow$ OLS.}
    \end{subfigure}
     \begin{subfigure}[b]{0.8\textwidth}
       \includegraphics[width=\textwidth]{results/OLS_TF_Transfer/ols_attack_x_idxs_3_legend.png}
    \end{subfigure}   
    \caption{The mean squared error between the predictions being made by the transformer and OLS on adversarial samples tends to be higher when the adversarial samples are sourced by attacking transformers. In the above plot, we use $\alpha=1.0$ for sampling $y_\bad$.}
    \label{appx.fig.ols.tf.transfer.idxs.plot}
\end{figure}

\FloatBarrier

\section{Training Details and Hyperparameters}
\label{appx.training.details}
\subsection{Linear Transformer}
To match the setup considered in Theorem~\ref{thm:jailbreaking.yadv}, we implement linear transformer as a single-layer attention-only linear transformer as described in equation~\ref{eq:lsa}.
We train the linear transformer for $2M$ steps with a batchsize of $1024$ and learning rate of $10^{-6}$.  

\subsection{Standard Transformer}
Our training setup closely mirrors that of \citet{garg2022can}. Similar to their setup, we use a curriculum where we begin training on data with dimensionality of $5$ and gradually increase the dimensionality to $20$ by incrementing it by $1$ every $5$k steps.
Details of our architecture are given in Table~\ref{appx.table:gpt.architecture}.
We guve the number of parameters present in various transformer models with different number of layers in Table~\ref{appx.table:number.of.parameters}.
Important training hyperparameters are mentioned in Table~\ref{appx.table:training.hps}.

\begin{table}[htbp]
    \centering
    \begin{tabular}{ll}
    \toprule
    \textbf{Parameter} & \textbf{Value} \\
    \midrule
    Embedding Size & 256 \\
    Number of heads & 8 \\
    Positional Embedding & Learned \\
    Number of Layers & 8 (unless mentioned otherwise) \\
    Causal Masking & Yes \\
    \bottomrule
    \end{tabular}
    \caption{Architecture for the transformer model.}
    \label{appx.table:gpt.architecture}
\end{table}

\begin{table}[htbp]
    \centering
    \begin{tabular}{ll}
    \toprule
    \textbf{Number of Layers} & \textbf{Parameter Count} \\
    \midrule
    $2$ & $1,673,601$ \\
    $3$ & $2,463,553$ \\
    $4$ & $3,253,505$ \\
    $6$ & $4,833,409$ \\
    $8$ & $6,413,313$ \\
    $12$ & $9,573,121$ \\
    $16$ & $12,732,929$ \\
    \bottomrule
    \end{tabular}
    \caption{Hyperparameters used for training transformer models with GPT-2 architecture.}
    \label{appx.table:number.of.parameters}
\end{table}

\begin{table}[htbp]
    \centering
    \begin{tabular}{ll}
    \toprule
    \textbf{Hyperparameter} & \textbf{Value} \\
    \midrule
    Learning Rate & $5 \times 10^{-4}$ \\
    Warmup Steps & 20,000 \\
    Total Training Steps & 500,000 \\
    Batch Size & 64 \\
    Optimizer & Adam \\
    \bottomrule
    \end{tabular}
    \caption{Hyperparameters used for training transformer models with GPT-2 architecture.}
    \label{appx.table:training.hps}
\end{table}

\FloatBarrier

\subsection{Adversarial Attack and Adversarial Training Details}
\label{appx.adversarial.attack.details}
We implement our adversarial attacks as simple gradient descent on the (selected) inputs with the target attack error as the optimization objective. We briefly experimented with variations of gradient descent, e.g., gradient descent with momentum but found those to perform at par with simple gradient descent.

When performing \xattack{}, we used a learning rate of $1$ and when performing \yattack{}, we used a learning rate of $100$.
When performing \zattack{}, we used a learning rate of $1$ when perturbing x-tokens and a learning rate of $100$ when perturbing y-tokens.
We chose the learning rates based on best performance within $100$ gradient steps. Using lower values of learning rates resulted in proportionally slower convergence, and hence were avoided.

In all our plots, we show results across three different models and use $1000$ samples for each model.

\textbf{Differences Between Adversarial Attacks and Adversarial Training}: The two major differences in our adversarial traning setup, compared with adversarial attacks setup are:
\begin{itemize}
    \item During adversarial attacks (done on trained models at test time), we sample $y_\bad$ according to the expression \ref{eq:y_bad.alpha}, but during adversarial training we sample $y_\bad$ by sampling a weight vector $w\sim N(0,I_d)$ independent of the task parameters $w_\tau$ and setting $y_\bad=w^\top x_\query$.
    \item During adversarial attacks, we perform 100 steps of gradient descent, but in adversarial training, we only perform 5 steps of gradient descent.
\end{itemize}
Both the above changes were done to help improve the efficiency of adversarial training.